\documentclass[a4paper]{article}

\usepackage{amsmath,amssymb,amsfonts,mathrsfs,bbm,amsthm,verbatim,cite,geometry,hyperref,nicefrac}

\usepackage[shortlabels,inline]{enumitem}
\usepackage[english]{babel}
\usepackage[utf8]{inputenc}
\usepackage[T1]{fontenc}

\newcommand{\bi}{\begin{itemize} }
\newcommand{\ei}{\end{itemize} }
\newcommand{\A}{\mathscr{A} }
\newcommand{\B}{\mathscr{B} }
\newcommand{\compactset}{\mathscr{K}}

\DeclareFontEncoding{LS1}{}{}
\DeclareFontSubstitution{LS1}{stix}{m}{n}
\DeclareMathAlphabet{\mathscr}{LS1}{stixscr}{m}{n}

\usepackage[sort,capitalise]{cleveref} %

\crefname{enumi}{item}{items}
\crefname{subsec}{Section}{Sections}
\crefname{listing}{Source Code}{Source Codes}
\crefname{equation}{}{}

\crefname{subsection}{Subsection}{Subsections}
\crefname{figure}{Figure}{Figures}

\usepackage{caption}
\usepackage{newfloat}
\DeclareFloatingEnvironment[fileext=frm,placement={!ht},name=Source Code]{listing}
\usepackage{tikz}
\usetikzlibrary{matrix,chains,positioning,decorations.pathreplacing,arrows,shapes}
\tikzset{ font={\fontsize{9pt}{12} \selectfont} }
\usepackage{adjustbox}

\hypersetup{colorlinks=true}

\theoremstyle{plain}
\newtheorem{theorem}{Theorem}[section]
\newtheorem{setting}[theorem]{Setting}
\newtheorem{lemma}[theorem]{Lemma}
\newtheorem{prop}[theorem]{Proposition}
\newtheorem{cor}[theorem]{Corollary}
\newtheorem{proposition} [theorem]{Proposition}
\theoremstyle{definition}

\numberwithin{equation}{section}

\usepackage{xparse}
\usepackage{etoolbox}

\usepackage{listings}
\usepackage[section]{placeins}
\usepackage{csvsimple}
\usepackage{siunitx}
\usepackage{longtable}

\ExplSyntaxOn

\NewDocumentCommand{\enum}{ O{;} m o }
 {
  \my_enum:nnn { #1 } { #2 } { #3 }
 }

\seq_new:N \l__my_enum_seq
\tl_new:N \l__my_enum_item_tl

\cs_new_protected:Nn \my_enum:nnn
 {
  \seq_set_split:Nnn \l__my_enum_seq { #1 } { #2 }
  \seq_remove_all:Nn \l__my_enum_seq {}
  \int_set_eq:NN \l__number_of_args { \seq_count:N \l__my_enum_seq }
  \seq_use:Nnnn \l__my_enum_seq { ~and~ } { ,~ } { ,~and~ }
  \IfNoValueTF{#3}{}{
    \space
    \int_compare:nNnTF{ \l__number_of_args } < {2}{ \prop_item:Nn \l__verbs {#3} }{ #3 }
  }
 }

\ExplSyntaxOff

 \ExplSyntaxOn

 \seq_new:N \g_cflist_loaded
 \seq_new:N \g_cflist_pending
 
 \NewDocumentCommand{\cfadd}{ m }
 {
   \seq_if_in:NnF \g_cflist_loaded { #1 } {
     \seq_if_in:NnF \g_cflist_pending { #1 } {
       \seq_gput_right:Nn \g_cflist_pending { #1 }
     }
   }
 }

 \NewDocumentCommand{\cfconsiderloaded}{ m }{
   \seq_gput_right:Nn \g_cflist_loaded {#1}
 }
 
 \NewDocumentCommand{\cfremove}{ m }
 {
   \seq_gremove_all:Nn \g_cflist_pending { #1 }
 }

 \NewDocumentCommand{\cfload}{ o }
 {
   \seq_if_empty:NTF \g_cflist_pending {\unskip} {
     (cf.\ \cref{\seq_use:Nn \g_cflist_pending {,} } )\IfValueTF{#1}{#1~}{\unskip}
     \seq_gconcat:NNN \g_cflist_loaded \g_cflist_loaded \g_cflist_pending
     \seq_gclear:N \g_cflist_pending
   }
 }
 
 \NewDocumentCommand{\cfclear} {} {
   \seq_gclear:N \g_cflist_loaded
   \seq_gclear:N \g_cflist_pending
 }
 
 \NewDocumentCommand{\cfout}{ o }
 {
   \seq_if_empty:NTF \g_cflist_pending {\unskip} {
     (cf.\ \cref{\seq_use:Nn \g_cflist_pending {,} } )\IfValueTF{#1}{#1~}{\unskip}
     \seq_gclear:N \g_cflist_pending
   }
 }
 
 \NewDocumentCommand{\ifnocf} { m } {
   \seq_if_empty:NT \g_cflist_pending { #1 }
 }
 
 \ExplSyntaxOff

\NewDocumentEnvironment {athm} {m m} {%
\begin{#1} \label{#2} \global\def\loc{#2}%
}{%
\end{#1}%
}

\NewDocumentEnvironment{aproof} {} {%
\begin{proof}[Proof~of~\cref{\loc}]%
}{%
\end{proof}%
}

\ExplSyntaxOff

\ExplSyntaxOn

\NewDocumentEnvironment{flexmath}{ m o }{
  \str_if_eq:noTF {a} {#1} {
    \equation
    \IfValueT{#2}{\label{eq:\loc.#2} }
    \aligned
  } {
    \catcode`&=9
    \renewcommand{\\}{}
    \str_if_eq:noTF {d} {#1} {
      \equation
      \IfValueT{#2}{\label{eq:\loc.#2} }
    } {
      \math
    }
  }
}{
  \str_if_eq:noTF {i} {#1} {
    \endmath
    \catcode`&=4
  } {
    \str_if_eq:noTF {d} {#1} {
      \endequation
    } {
      \endaligned
      \endequation
    }
  }
}

\ExplSyntaxOff

\ExplSyntaxOn

\bool_new:N \g_noteobserve

\NewDocumentCommand{\setnote}{}{
  \bool_gset_true:N \g_noteobserve
}

\NewDocumentCommand{\setobserve}{}{
  \bool_gset_false:N \g_noteobserve
}

\NewDocumentCommand{\nobs}{ o }{
  \IfValueT{#1}{
    \str_if_eq:noTF {note} {#1} {
      \bool_gset_true:N \g_noteobserve
    } {
      \str_if_eq:noTF {Note} {#1} {
        \bool_gset_true:N \g_noteobserve
      } {
        \bool_gset_false:N \g_noteobserve
      }
    }
  }
  \bool_if:nTF { \g_noteobserve } {
    \bool_gset_false:N \g_noteobserve
    note
  } {
    \bool_gset_true:N \g_noteobserve
    observe
  }
  \IfValueF{#1}{~}
}

\NewDocumentCommand{\Nobs}{ o }{
  \IfValueT{#1}{
    \str_if_eq:noTF {note} {#1} {
      \bool_gset_true:N \g_noteobserve
    } {
      \str_if_eq:noTF {Note} {#1} {
        \bool_gset_true:N \g_noteobserve
      } {
        \bool_gset_false:N \g_noteobserve
      }
    }
  }
  \bool_if:nTF { \g_noteobserve } {
    \bool_gset_false:N \g_noteobserve
    Note
  } {
    \bool_gset_true:N \g_noteobserve
    Observe
  }
  \IfValueF{#1}{~}
}

\int_new:N \g_furthermore

\NewDocumentCommand{\Moreover}{ o o }{
  \IfValueT{#1}{
    \str_case:nn {#1} {
      {Furthermore} {\int_set:Nn {\g_furthermore} {0} }
      {Moreover} {\int_set:Nn {\g_furthermore} {1} }
      {In~addition} {\int_set:Nn {\g_furthermore} {2} }
      {note} {\bool_gset_true:N \g_noteobserve}
      {observe} {\bool_gset_false:N \g_noteobserve}
    }
    \IfValueT{#2}{
      \str_case:nn {#2} {
        {Furthermore} {\int_set:Nn {\g_furthermore} {0} }
        {Moreover} {\int_set:Nn {\g_furthermore} {1} }
        {In~addition} {\int_set:Nn {\g_furthermore} {2} }
        {note} {\bool_gset_true:N \g_noteobserve}
        {observe} {\bool_gset_false:N \g_noteobserve}
      }
    }
  }
  \int_case:nn { \int_mod:nn {\g_furthermore} {3} } {
    { 0 } { Furthermore,~\nobs that}
    { 1 } { Moreover,~\nobs that}
    { 2 } { In~addition,~\nobs that}
  }
  \int_incr:N \g_furthermore
  \IfValueF{#1}{~}
}

\bool_new:N \g_hencetherefore

\NewDocumentCommand{\hence}{}{
  \bool_if:nTF { \g_hencetherefore } {
    \bool_gset_false:N \g_hencetherefore
    hence~
  } {
    \bool_gset_true:N \g_hencetherefore
    therefore~
  }
}

\NewDocumentCommand{\Hence}{}{
  \bool_if:nTF { \g_hencetherefore } {
    \bool_gset_false:N \g_hencetherefore
    Hence,~we~obtain~
  } {
    \bool_gset_true:N \g_hencetherefore
    Therefore,~we~obtain~
  }
}
\ExplSyntaxOff

\newcommand{\diml}{\mathbf{d}}

\ExplSyntaxOn

\NewDocumentCommand{\mEE}{ o m }{
  \IfValueTF{#1}{
    \str_case:on {#1} {
      {0}{\mathbb E\br{#2} }
      {1}{\mathbb E\mkern-1mu\br[\big]{#2} }
      {2}{\mathbb E\mkern-1.1mu\br[\Big]{#2} }
      {3}{\mathbb E\mkern-1.3mu\br[\bigg]{#2} }
      {4}{\mathbb E\mkern-1.5mu\br[\Bigg]{#2} }
    }
  } {
    \mathbb E\br{#2}
  }
}

\ExplSyntaxOff

\usepackage{mleftright}

\usepackage{mathtools}
\DeclarePairedDelimiter{\pr}{ (}{)}
\DeclarePairedDelimiter{\br}{[}{]}
\DeclarePairedDelimiter{\cu}{\{}{\} }
\DeclarePairedDelimiter{\abs}{\lvert}{\rvert}
\DeclarePairedDelimiter{\norm}{\lVert}{\rVert}

\makeatletter
\newcommand{\mylabel}[2]{#2\def\@currentlabel{#2} \label{#1} }
\makeatother

\newcommand{\E}{\mathbb{E} }

\newcommand{\F}{\mathcal{F} }

\renewcommand{\P}{\mathbb{P} }

\usepackage{mleftright}

\usepackage{subcaption}

\usepackage{marginfix}

\usepackage{mathrsfs}
\newcommand{\fw}{\mathfrak{w} }
\newcommand{\fd}{\mathfrak{d} }
\newcommand{\fb}{\mathfrak{b} }
\newcommand{\fM}{\mathfrak{M} }

\newcommand{\fL}{\mathfrak{L} }
\newcommand{\fC}{\mathfrak{C} }
\newcommand{\fG}{\mathfrak{G} }
\def\mN{\mathcal N}
\def\mQ{\mathcal Q}
\def\mL{\mathcal L}
\def\mA{\mathcal A}
\def\mG{\mathcal G}
\def\mX{\mathcal X}
\def\mM{\mathcal{M} }

\def\R{\mathscr{R} }
\def\Reals{\mathbb{R} }
\newcommand{\Rr}{\Reals }
\def\enne{\mathbb{N} }
\newcommand{\N}{\enne }
\def\ip #1#2{\left<#1,#2\right>}
\renewcommand{\d}{{\mathrm d} }

\newcommand{\qqandqq}{\qquad\text{and} \qquad}

\newcommand{\indicator}[1]{\mathbbm{1}_{\smash{#1}}}

\title{
\vspace{-1cm}
Convergence proof for stochastic gradient descent\\
in the training of deep neural networks with\\
ReLU activation for constant target functions}

\author{
Martin Hutzenthaler$^1$, 
Arnulf Jentzen$^{2,3} $,
Katharina Pohl$^4$, \\
Adrian Riekert$^5$,
and 
Luca Scarpa$^6$
\bigskip
\\
\small{$^1$ Faculty of Mathematics, University of Duisburg-Essen,}
\vspace{-0.1cm}\\
\small{Essen, Germany, e-mail: \texttt{martin.hutzenthaler}\textcircled{\texttt{a}}\texttt{uni-due.de}}
\smallskip
\\
\small{$^2$ School of Data Science and Shenzhen Research Institute of Big Data,}
\vspace{-0.1cm}\\
\small{The Chinese University of Hong Kong, Shenzhen, China, e-mail: \texttt{ajentzen}\textcircled{\texttt{a}}\texttt{cuhk.edu.cn}}
\smallskip
\\
\small{$^3$ Applied Mathematics: Institute for Analysis and Numerics,}
\vspace{-0.1cm}\\
\small{Faculty of Mathematics and Computer Science, University of M{\"u}nster,}
\vspace{-0.1cm}\\
\small{M{\"u}nster, Germany, e-mail: \texttt{ajentzen}\textcircled{\texttt{a}}\texttt{uni-muenster.de}}
\smallskip
\\
\small{$^4$ Faculty of Mathematics, University of Duisburg-Essen,}
\vspace{-0.1cm}\\
\small{Essen, Germany, e-mail: \texttt{katharina.pohl}\textcircled{\texttt{a}}\texttt{uni-due.de}}
\smallskip
\\
\small{$^5$ Applied Mathematics: Institute for Analysis and Numerics,}
\vspace{-0.1cm}\\
\small{Faculty of Mathematics and Computer Science, University of M{\"u}nster,}
\vspace{-0.1cm}\\
\small{M{\"u}nster, Germany, e-mail: \texttt{ariekert}\textcircled{\texttt{a}}\texttt{uni-muenster.de}}
\smallskip
\\
\small{$^6$ Department of Mathematics, Politecnico di Milano,}
\vspace{-0.1cm} \\
\small{Milano, Italy, e-mail: \texttt{luca.scarpa}\textcircled{\texttt{a}}\texttt{polimi.it}}
}

\begin{document}

\maketitle

\begin{abstract}
In many numerical simulations stochastic gradient descent (SGD) type optimization methods 
perform very effectively in the training of deep neural networks (DNNs) but till this day 
it remains an open problem of research to provide a mathematical convergence analysis 
which rigorously explains the success of SGD type optimization methods in the training 
of DNNs. In this work we study SGD type optimization methods in the training of 
fully-connected feedforward DNNs with rectified linear unit (ReLU) activation. 
We first establish general regularity properties for the risk functions and 
their generalized gradient functions appearing in the training of such DNNs 
and, thereafter, we investigate the plain vanilla SGD optimization method 
in the training of such DNNs under the assumption that the target function 
under consideration is a constant function. 
Specifically, we prove under the assumption that the learning rates (the step sizes of the SGD optimization method) 
are sufficiently small but not $L^1$-summable and under the assumption that 
the target function is a constant function that the expectation of the risk
of the considered SGD process converges in the training of such DNNs to zero 
as the number of SGD steps increases to infinity.
\end{abstract}

\pagebreak
\begin{samepage}
\tableofcontents
\end{samepage}

\section{Introduction}
\label{sec:intro}
In many numerical simulations stochastic gradient descent (SGD) optimization methods perform very effectively in the training of deep artificial neural networks (ANNs) but till this day it remains an open problem of research to provide a mathematical convergence analysis which rigorously explains the success of SGD optimization methods in the training of deep ANNs. 
Even though the mathematical analysis of gradient descent (GD) type optimization methods in the training of deep ANNs is still in its infancy, there are several auspicious analytical results for GD type optimization methods in the literature. In the following we mention several of such analytical results for GD type optimization methods in a very brief way and we point to the below mentioned references for further details.

We refer, for instance, to \cite{NIPS2013_7fe1f8ab, NIPS2011_40008b9a,MR701288,
MR2142598,10.5555/3042573.3042774,
JentzenKuckuckNeufeldVonWurstemberger2021,
SchmidtleRoux2013}
and the references mentioned therein for abstract convergence results for GD type optimization methods 
\emph{under convexity assumptions} on the objective function, 
by which we mean the function of which one intends to minimize in the optimization problem. 
The objective functions arising in the training of deep ANNs via GD type optimization methods 
are in nearly all situations not convex 
(see \cref{cor:conv_characterization} in \cref{ssec:risk_non-convex} 
in this article below for a characterization of convexity in the situation 
deep ANNs with the rectified linear unit (ReLU) activation) 
and in such scenarios there often exist infinitely many saddle 
and non-global local minimum points of the considered objective functions 
(cf., e.g., \cite{MR4450126_CJR2022,  MR3960812, NEURIPS2019_3fb04953, pmlr-v49-lee16, 
MR3754926, pmlr-v40-Ge15} and the references mentioned therein).

We also point, for example, to 
\cite{AbsilMahonyAndrews2005, akyildiz2021nonasymptotic, attouch:hal-00803898, MR1741189,
MR3941932, MR4138120, pmlr-v99-karimi19a, DBLP:journals/corr/KarimiNS16, MR4162089,lovas2021taming, 
patel2021stopping, wojtowytsch2021stochastic,
DBLP:journals/corr/abs-1904-06963} for \emph{abstract convergence results} 
for GD type optimization methods \emph{without convexity assumptions} on the objective function. 
Moreover, we point, for instance, to \cite{wojtowytsch2021stochastic, DBLP:journals/corr/abs-2102-09385, DBLP:journals/corr/KarimiNS16, doi:10.1137/120887795, MR4162089, attouch:hal-00803898} for convergence results 
for GD type optimization methods under \emph{Kurdyka--\L ojasiewicz type assumptions}.

We also mention convergence results for GD type optimization methods in the training of ANNs in the so-called \emph{overparametrized regime} 
in which, roughly speaking, the architecture of the considered ANNs is chosen large enough so that there are much more ANN parameters 
than input-output data pairs of the considered supervised learning problem. We refer, for example, to \cite{AroraDuHuLiWang2019,
DuLeeLiWangZhai2019, 
ZhengdaoRotskoff2020,
ZhangMartensGrosse2019,
MR4119555, jentzen2021convergence, NEURIPS2018_54fe976b 
} for convergence results in this overparametrized regime in the context of shallow ANNs and to \cite{NEURIPS2019_62dad6e2, 
pmlr-v97-allen-zhu19a, pmlr-v97-du19c, DBLP:journals/corr/abs-1904-06963,
MR4075425} for convergence results in this overparametrized regime in the context of deep ANNs.

In addition to the above mentioned results in the overparametrized regime, 
in the literature there are also several convergence results for GD type optimization methods in the training of shallow ANNs 
under \emph{specific assumptions on the target function} under consideration, 
by which we mean the function describing the relationship between the input and the output data (the factorization of the conditional expectation of the output datum given that the corresponding input datum is known). 
Specifically, we point to \cite{MR4577699_JentzenRiekert2022_JMLR} for the situation of piecewise affine linear target functions, we mention 
\cite{MR4564782_Eberleetal2023}
for the situation of piecewise polynomial target functions, and we refer to \cite{MR4438169, MR4468133} for the situation of constant target functions. 
In particular, a central contribution of this work is to extend the findings of \cite{MR4468133} from shallow ANNs with just one hidden layer to deep ANNs with an arbitrary large number of hidden layers.

Furthermore, we also point, e.g., to \cite{MR4055054} for \emph{lower bounds} for approximations errors for GD type optimizations methods and we refer, for instance, to \cite{MR4232648, MR4188518, pmlr-v99-shamir19a} for certain \emph{non-convergence results} 
for GD type optimization methods in the training of deep ANNs. 
Further references and overviews on GD type optimizations methods can, for example, also be found in
\cite{BercuFort2013}, \cite{MR3797719},
\cite{EMaWojtowytschWu2020},
\cite[Section 1.1]{MR4138120},
\cite[Section 1]{JentzenKuckuckNeufeldVonWurstemberger2021}, and \cite{DBLP:journals/corr/Ruder16}.

\def\layersep{3cm}

\begin{figure}
\centering
\begin{adjustbox}{width=\textwidth}
\begin{tikzpicture}[shorten >=1pt,->,draw=black!50, node distance=\layersep]
    \tikzstyle{neuron} =[circle, draw=black!80, 
    minimum size=17pt,line width=0.5mm, inner sep=0pt]
    \tikzstyle{input neuron} =[neuron, fill=red!50];
    \tikzstyle{output neuron} =[neuron,
    fill=violet!50];
    \tikzstyle{hidden neuron} =[neuron,
    fill=blue!50];
    \tikzstyle{annot} = [text width=9em, text centered]
    \tikzstyle{annot2} = [text width=4em, text centered]

    \foreach \name / \y in {1,...,5}
        \node[input neuron] (I-\name) at (0,-\y) {};
    \foreach \name / \y in {1,...,5}    
            \node[annot,left of=I-\name, node distance=0.7cm, align=center] () {$x_{\y} $};
    \foreach \name / \y in {1,...,8}
        \path[yshift=1.5cm]
            node[hidden neuron] (H-\name) at ( \layersep,-\y cm) {};      
    \foreach \name / \y in {1,...,7}
        \path[yshift=1cm]
            node[hidden neuron] (H2-\name) at (2*\layersep,-\y cm) {};      
    \foreach \name / \y in {1,...,9}
        \path[yshift=2cm]
            node[hidden neuron] (H3-\name) at (3*\layersep,-\y cm) {};      
    \foreach \name / \y in {1,2,3}
        \path[yshift=-1cm]
            node[output neuron](O-\name) at (4*\layersep,-\y cm) {};   
    \foreach \name / \y in {1,2,3}    
            \node[annot,right of=O-\name, node distance=1.1cm, align=center] () {$ \mN^{4, \theta }_{ \infty, \y}(x) $};
            
    \foreach \source in {1,...,5}
        \foreach \dest in {1,...,8}
            \path[>=stealth] (I-\source) edge (H-\dest);     
    \foreach \source in {1,...,8}
        \foreach \dest in {1,...,7}
            \path[>=stealth] (H-\source) edge (H2-\dest);     
    \foreach \source in {1,...,7}
        \foreach \dest in {1,...,9}
            \path[>=stealth] (H2-\source) edge (H3-\dest);
    \foreach \source in {1,...,9}
        \foreach \dest in {1,...,3}
            \path[>=stealth] (H3-\source) edge (O-\dest);

    \node[annot,above of=H-1, node distance=1.5cm, align=center] (hl) {1st hidden layer\\(2nd layer)};
    \node[annot,above of=H2-1, node distance=2cm, align=center] (hl2) {2nd hidden layer\\(3rd layer)};
    \node[annot,above of=H3-1, node distance=1cm, align=center] (hl3) {3rd hidden layer\\(4th layer)};
    \node[annot,left of=hl, align=center] {Input layer\\ (1st layer)};
    \node[annot,right of=hl3, align=center] {Output layer\\(5th layer)};
    
    \node[annot2,below of=H-8, node distance=1.5cm, align=center] (sl) {$ \ell_1=8$};
    \node[annot2,below of=H2-7, node distance=2cm, align=center] (sl2) {$ \ell_2=7$};
    \node[annot2,below of=H3-9, node distance=1cm, align=center] (sl3) {$ \ell_3=9$};
    \node[annot2,left of=sl, align=center] {$ \ell_0=5$};
    \node[annot2,right of=sl3, align=center] {$ \ell_4=3$};
\end{tikzpicture}
\end{adjustbox}
\caption{Graphical illustration of the used deep ANN architecture in the case of a simple example deep ANN with 3 hidden layers (corresponding to $L = 4$ affine linear transformations), with 5 neurons on the input layer (corresponding to $ \ell_0 = 5$), 8 neurons on the 1st hidden layer (corresponding to $ \ell_1 = 8$), 7 neurons on the 2nd hidden layer (corresponding to $ \ell_2 = 7$), 9 neurons on the 3rd hidden layer (corresponding to $ \ell_3 = 9$), and 3 neurons on the output layer (corresponding to $ \ell_4 = 3$). In this situation we have for every ANN parameter vector $ \theta \in \Reals^{ \fd } = \Reals^{213 } $ that the realization function 
$ \Reals^5 \ni x \mapsto \mN^{ 4, \theta }_{ \infty }(x) \in \Reals^3$ 
of the considered deep ANN maps the 5-dimensional input vector $x = ( x_1, x_2, x_3, x_4, x_5 ) \in [a,b]^5$ to the 3-dimensional 
output vector $ \mN^{ 4, \theta }_{ \infty }( x ) = 
( \mN^{4, \theta }_{ \infty, 1}(x), \mN^{4, \theta }_{ \infty, 2}(x), \mN^{4, \theta }_{ \infty, 3}(x)) $
.}
\label{figure_1}
\end{figure}
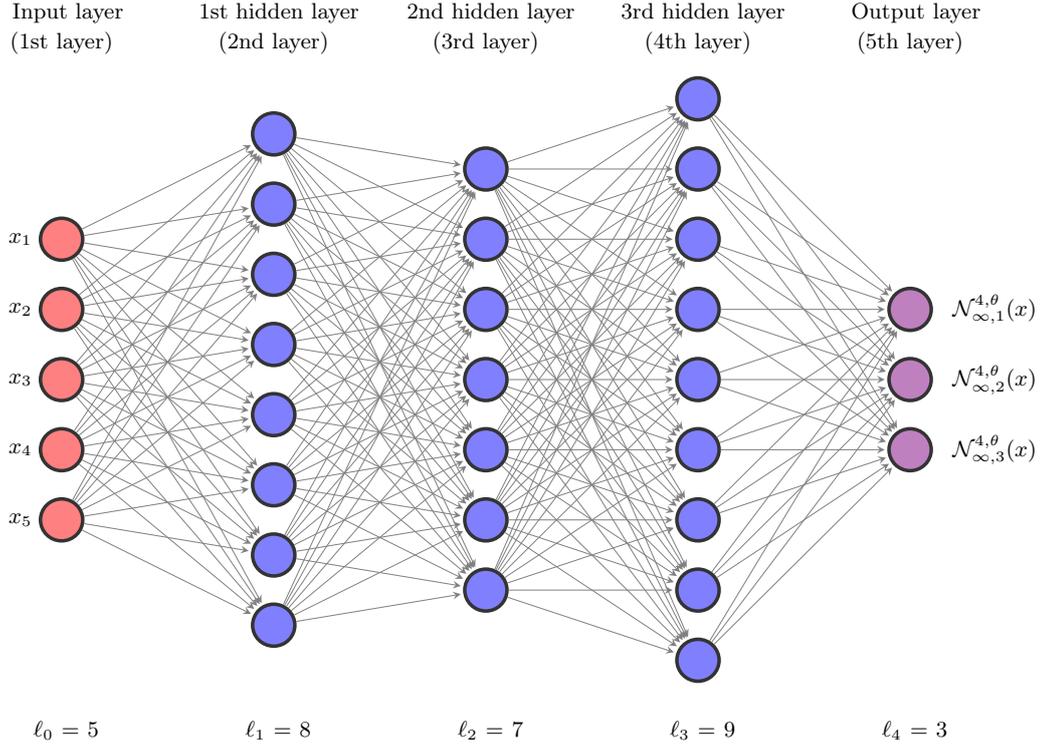

In this work we study, among other things, the plain vanilla SGD optimization method 
in the training of deep ANNs under the assumption that 
\emph{the target function under consideration is a constant function}. Specifically, 
one of the main results of this work, \cref{thm:sgd_conv} 
in \cref{ssec:SGD_convergence} below, proves that 
the expectation of the risk of the considered SGD process converges 
in the training of \emph{fully-connected feedforward deep ANNs with 
the rectified linear unit (ReLU) activation} to zero as the number of 
SGD steps increases to infinity. To illustrate the findings of this work in more detail, 
we present in \cref{thm:main_thm} below 
a special case of \cref{thm:sgd_conv} and 
we now add some further explanatory comments 
regarding the mathematical objects appearing in \cref{thm:main_thm}.

The natural number $ L \in \N= \{ 1, 2,3, \dots \} $ in \cref{thm:main_thm} 
is related to the number of hidden layers of 
the deep ANNs in \cref{thm:main_thm} and the sequence 
$ \ell_k \in \N $, $k \in \N_0 = \N \cup \{ 0 \} = \{ 0, 1, 2, \dots \} $, 
of natural numbers in \cref{thm:main_thm} describes 
the dimensions (the numbers of neurons) of the layers of the deep ANNs in \cref{thm:main_thm}. More formally, in \cref{thm:main_thm} we consider fully-connected feedforward deep ANNs with $L - 1$ hidden layers and, including the input and the output layer, $L + 1$ layers overall 
\begin{itemize}
\item 
with $ \ell_0 $ neurons on the input layer (with an $ \ell_0 $-dimensional input layer), 
\item 
with $ \ell_1$ neurons on the 1st hidden layer (with an $ \ell_1$-dimensional 1st hidden layer), 
\item 
with $ \ell_2$ neurons on the 2nd hidden layer (with an $ \ell_2$-dimensional 2nd hidden layer), 
\item 
$ \dots $, 
\item 
with $ \ell_{ L - 1 } $ neurons 
on the ($ L - 1 $)-th (and last) hidden layer (with an $ \ell_{L - 1} $-dimensional ($L-1$)-th hidden layer), and 
\item 
with $ \ell_L$ neurons on the output layer (with an $ \ell_L$-dimensional output layer). 
\end{itemize}
In between each of the layers there are real affine linear transformations 
described through \emph{weight matrices} and \emph{bias vectors}. 
The entries of the weight matrices are referred 
as weight parameters and the components of the bias vectors are referred to as 
bias parameters. 
Taking this into account
\begin{itemize}
\item 
there are $ \ell_1 \ell_0 $ weight parameters and $ \ell_1$ bias parameters to describe the affine linear transformation 
between the 
input layer and the 
1st hidden layer, 
\item 
there are $ \ell_2 \ell_1$ weight parameters and $ \ell_2$ bias parameters to 
describe the affine linear transformation between the 
1st hidden layer and the 
2nd hidden layer, 
\item 
$ \dots $
\item 
there are $ \ell_{ L - 1 } \ell_{ L - 2 } $ weight parameters and $ \ell_{ L - 1 } $ 
bias parameters to describe 
the affine linear transformation between the 
($L - 2$)-th hidden layer and the 
($L - 1$)-th hidden layer, and 
\item 
there are $ \ell_L\ell_{L-1} $ weight parameters and $ \ell_L$ bias parameters 
to describe the affine linear transformation between the 
($L - 1$)-th hidden layer and the 
output layer. 
\end{itemize}
Overall in \cref{thm:main_thm} we thus use 
\begin{equation} 
  \fd 
=
  \sum_{ k = 1 }^L ( \ell_k \ell_{ k - 1 } + \ell_k ) 
= 
  \sum_{ k = 1 }^L \ell_k ( \ell_{ k - 1 } + 1 )
\end{equation}
real numbers to describe the employed deep ANNs. We also refer to \cref{figure_1}
for a graphical illustration of the used deep ANN architecture in the case of a simple example deep ANN 
with 3 hidden layers (corresponding to $ 5 $ layers and $ L = 4 $ affine linear transformations)
\begin{enumerate}
\item[0.)]
with $ 5 $ neurons on the input layer (corresponding to $ \ell_0 = 5$), 
\item[1.)]
with $ 8 $ neurons on the 1st hidden layer (corresponding to $ \ell_1 = 8$), 
\item[2.)]
with $ 7 $ neurons on the 2nd hidden layer (corresponding to $ \ell_2 = 7$), 
\item[3.)]
with $ 9 $ neurons on the 3rd hidden layer (corresponding to $ \ell_3 = 9$), and 
\item[4.)]
with $ 3 $ neurons on the output layer (corresponding to $ \ell_4 = 3$). 
\end{enumerate}
The matrices $ \fw^{ \theta,k} $, $k \in \enne $, $ \theta \in \Reals^{ \fd } $, and the vectors $ \fb^{ \theta,k} $, $ k \in \enne $, $ \theta \in \Reals^{ \fd } $, in \cref{wb_thm1} in \cref{thm:main_thm} 
represent the weight matrices and the bias vectors 
and are thus used to describe the affine linear transformations in the considered deep ANN. More formally, 
for every ANN parameter vector $ \theta \in \Reals^{ \fd } $ 
and every $k \in \{ 1, \ldots, L \} $ we have that the 
$ ( \ell_k \times \ell_{ k - 1 } ) $-matrix 
\begin{equation}
  \fw^{ \theta,k} = 
  ( \fw^{ \theta,k}_{ i,j} )_{ 
    (i,j) \in \{ 1, \ldots, \ell_k \} \times \{ 1, \ldots, \ell_{k-1} \} 
  } 
\end{equation}
describes the linear part in the affine linear transformation 
between the $k$-th and the $(k+1) $-th layer (the $(k-1) $-th 
and the $k$-th hidden layer, respectively) 
and we have that the $l_k$-dimensional vector 
\begin{equation}
  \fb^{ \theta,k} = ( \fb^{ \theta,k}_1, 
  \fb^{ \theta,k}_2, \ldots, \fb^{ \theta,k}_{\ell_k} ) 
\end{equation}
describes the inhomogeneous part in the affine linear transformation between the $k$-th and the $(k+1) $-th layer (the $(k-1) $-th and the $k$-th hidden layer, respectively). 
The functions 
$ \mA^{ \theta }_k $, $ k \in \enne $, 
$ \theta \in \Reals^{ \fd } $, 
in \cref{thm:sgd_conv} 
are used to describe the affine linear transformations 
in the considered deep ANN in the sense 
that for every ANN parameter vector $ \theta \in \Reals^{ \fd } $ 
and every $ k \in \enne $ we have that 
the affine linear transformation 
between the $k$-th and the $(k+1) $-th layer (the $(k-1) $-th and the $k$-th hidden layer, respectively) is given through 
the function 
\begin{equation}
  \Reals^{ \ell_{k-1} } \ni x \mapsto A^{ \theta }_k(x) =\fb^{ k, \theta } +  \fw^{ k, \theta } x  \in \Reals^{ \ell_k} 
  .
\end{equation}
The real numbers $a \in \Reals $, $b \in (a, \infty) $ in \cref{thm:main_thm} are used to specify 
the \emph{region where the input data of the considered supervised learning problem take values in}. 
Specifically, we assume in \cref{thm:main_thm} that the input data takes values in the $ \ell_0 $-dimensional cube 
\begin{equation}
  [a,b]^{ \ell_0 } \subseteq \Reals^{ \ell_0 } 
  .
\end{equation}
The target function, which we intend to learn approximatively in \cref{thm:main_thm} and which describes the relationship between the input data and the output data, is thus a function from the $ \ell_0 $-dimensional set $[a,b]^{ \ell_0 } $ to the $ \ell_L $-dimensional set $ \Reals^{ \ell_L } $. 
We assume that the target function is a constant function which is represented by the $ \ell_L$-dimensional vector $ \xi \in \Reals^{ \ell_L } $ in \cref{thm:main_thm}. The target function in \cref{thm:main_thm} is thus given through the constant function $[a,b]^{ \ell_0 } \ni x \mapsto \xi \in \Reals^{ \ell_L } $.
In \cref{thm:main_thm} we study the training of deep ANNs with 
\emph{ReLU activation}. 
The ReLU activation function 
\begin{equation}
  \Reals \ni x \mapsto \max\{ x, 0 \} \in \Reals 
\end{equation}
fails to be differentiable and this lack of differentiability transfers 
from the activation function to the risk function. To specify the SGD optimization method, 
one thus need to specify a suitably generalized gradient function. 
We accomplish this by means of an \emph{appropriate approximation procedure} 
in which the ReLU activation function $ \Reals \ni x \mapsto \max\{ x, 0 \} \in \Reals $ 
is approximated via appropriate continuously differentiable functions 
which converge pointwise to the ReLU activation function and 
whose derivative converge pointwise to the left derivative 
of the ReLU activation function; see \cref{lim_R_thm1} below for details. 
Our approximation procedure in \cref{lim_R_thm1} 
is partially inspired by the shallow ANN case in \cite{MR4468133} 
(see also \cite{MR4438169}) 
but in the case of deep ANNs such an approximation procedure 
is a more much delicate issue since the non-differentiable 
ReLU activation function appears in the argument 
of the non-differentiable ReLU activation function 
in the case of risk functions associated to deep ANNs; 
see \cref{lem:properties2} below for further details. 
The mathematical description of the generalized gradient function 
in \cref{lim_R_thm1,def_G_thm1} in \cref{thm:main_thm} corresponds 
precisely to the standard procedure how such appropriately generalized gradients are computed in a {\sc Python} implementation in {\sc TensorFlow} (cf., e.g., \cite[Section 1 and Subsection 3.7]{MR4468133}). 
The functions 
\begin{equation}
  \fM_r \colon 
  ( \cup_{ n \in \enne } \Reals^n ) \to ( \cup_{ n \in \enne } \Reals^n ) 
\end{equation}
for $ r \in [ 1, \infty ] $ 
in \cref{thm:main_thm} specify appropriate 
\emph{multidimensional versions of the ReLU activation function} 
and its continuously differentiable approximations. 
The function $ \norm{\cdot} \colon ( \cup_{ n \in \enne } \Reals^n ) \to \Reals $ 
in \cref{thm:main_thm} is used to express the 
standard norm for tuples of real numbers. 
In particular, \nobs that for all $ n \in \enne= \{ 1, 2, 3, \dots \} $, 
$ x = ( x_1, \dots, x_n ) \in \Reals^n $ 
we have that the real number $ \norm{x} $ is nothing else but 
the standard norm of the $ n $-dimensional vector $ x $.  
The probability space $ ( \Omega, \F, \P) $ in \cref{thm:main_thm} 
is the probability space on which the input data is defined and 
the i.i.d.\ random variables $ X^{ n, m } : \Omega \to [a,b]^{ \ell_0 } $, $ n, m \in \enne_0 $, 
in \cref{thm:main_thm} describe the \emph{input data of the considered supervised learning problem}. 
The sequence of natural numbers $ M_n \in \N $, $n \in \N_0 $, 
in \cref{thm:main_thm} specifies the size of the \emph{mini-batches} 
in the SGD optimization method and the stochastic process 
$ \Theta = ( \Theta_n )_{ n \in \enne_0 } \colon \enne_0 \times \Omega \to \Reals^{ \fd } $ 
in \cref{thm:main_thm} specifies the SGD process. 
In particular, we \nobs for every SGD step $ n \in \N_0 $ 
that the random variable $ \Theta_n $ (the SGD process at time $ n $) 
and the input data random variables 
$ X^{ n, 1 }, X^{ n, 2 }, \dots, X^{ n, M_n } $ 
are used to compute the random variable $ \Theta_{ n + 1 } $ 
(the SGD process at time $ n + 1 $). 
The sequence of non-negative real numbers 
$ \gamma_n \in [0, \infty) $, $ n \in \enne_0 $, 
in \cref{thm:main_thm} describes the 
\emph{learning rates} 
(the \emph{step sizes}) for the SGD optimization method. 
In \cref{thm:main_thm} we assume -- in dependence of 
the size of the norm of the initial value of 
the SGD process, of the architecture of the 
considered deep ANNs, and of the target function -- that 
the learning rates are sufficiently small but also that 
the learning rates are not summable.

Under the above outlined hypotheses 
we prove in \cref{it:main_thm_3} 
in \cref{thm:main_thm} that 
the expected risk of the SGD process converges 
to zero as the number of SGD steps increases to infinity. 
We now present the precise statement of 
\cref{thm:main_thm}.

\begin{samepage}
\begin{theorem}
\label{thm:main_thm}
Let 
$ L, \fd \in \N $,
$ 
  ( \ell_k )_{ k \in \N_0 } \subseteq \N
$,
$ a \in \Reals $, 
$ b \in (a, \infty) $, 
$ \xi \in \Reals^{ \ell_L} $
satisfy $ \fd = \sum_{ k = 1 }^L \ell_k ( \ell_{k-1} + 1 ) $,
for every $ \theta = ( \theta_1, \ldots, \theta_{ \fd } ) \in \Rr^{ \fd } $
let 
$ 
  \fw^{ k, \theta } 
  = 
  ( \fw^{ k, \theta }_{ i,j} )_{ 
    (i,j) \in \{ 1, \ldots, \ell_k \} \times \{ 1, \ldots, \ell_{k-1} \} 
  }
  \in \Reals^{ \ell_k \times \ell_{ k - 1 } } 
$, $ k \in \N $, 
and 
$ 
  \fb^{ k, \theta } = ( \fb^{ k, \theta }_i)_{ i \in \{ 1, \ldots, \ell_k \} } 
  \in \Reals^{ \ell_k } 
$,
$ k \in \N $, 
satisfy for all  
$ k \in \{ 1, \ldots, L\} $, $ i \in \{ 1, \ldots, \ell_k \} $,
$ j \in \{ 1, \ldots, \ell_{k-1} \} $ that
\begin{equation}
\label{wb_thm1}
  \fw^{ k, \theta }_{ i,j} =
  \theta_{ (i-1)\ell_{k-1} + j+\sum_{h= 1 }^{ k-1} \ell_h( \ell_{h-1} + 1)}
\qqandqq
  \fb^{ k, \theta }_{ i} =
  \theta_{\ell_k\ell_{k-1} + i+\sum_{h= 1 }^{ k-1} \ell_h( \ell_{h-1} + 1)} 
  ,
\end{equation}
for every $ k \in \N $, $ \theta \in \Reals^{ \fd } $
let 
$ \mA_k^{ \theta } \colon \Reals^{ \ell_{k-1} } \to \Reals^{ \ell_k} $
satisfy 
for all $ x \in \Reals^{ \ell_{k-1} } $
that $ \mA_k^{ \theta }( x ) = \fb^{ k, \theta } + \fw^{ k, \theta } x $,
let 
$ \R_r \colon \Rr \to \Rr $, 
$ r \in [1, \infty] $,
satisfy for all 
$ r \in [1, \infty) $, 
$ x \in ( - \infty, 2^{ - 1 } r^{ - 1 } ] $, 
$ 
  y \in \Rr
$, 
$
  z \in [ r^{ - 1 }, \infty ) 
$
that 
\begin{equation}
\label{lim_R_thm1}
  \R_r \in C^1( \Reals, \Reals ) ,
\qquad
  \R_r(x) = 0,
\qquad 
  0 \leq \R_r(y) \leq \R_{ \infty }( y ) 
  = 
  \max\{ y, 0 \}
  ,
\qqandqq
  \R_r(z) = z
  ,
\end{equation}
assume 
$
  \sup_{ r \in [1, \infty) }
  \sup_{ x \in \Reals } 
  | ( \R_r )'( x ) | < \infty 
$,
let $ \norm{\cdot} \colon( \cup_{ n \in \enne } \Reals^n ) \to\Rr $ 
and
$ \mathfrak{M}_r \colon( \cup_{ n \in \Reals } \Reals^n ) \to ( \cup_{ n \in \enne } \Reals^n ) $, 
$ r \in [1, \infty] $, 
satisfy for all $ r \in [1, \infty] $, 
$ n \in \N $,
$ x = (x_1, \ldots, x_n ) \in \Rr^n $
that 
$ \norm{x} = ( \sum_{ i= 1 }^n | x_i |^2 )^{ 1 / 2 } $ 
and
$ \mathfrak{M}_r( x ) = ( \R_r(x_1), \ldots, \R_r( x_n ) ) $,
let $ ( \Omega, \mathcal{F}, \P) $ be a probability space, 
let $ X^{n,m} \colon \Omega \to [a,b]^{ \ell_0 } $, $ n, m \in \N_0 $, 
be i.i.d.\ random variables,
let $ ( M_n )_{ n \in \N_0 } \subseteq \N $,
for every $ r \in [1, \infty] $ 
let $ \mN^{ k, \theta }_r \colon \Reals^{ \ell_0 } \to \Reals^{ \ell_L} $, 
$ \theta \in \Reals^{ \fd } $,
$ k \in \enne $,
and $ \mathfrak{L}^n_r \colon  \Reals^{ \fd} \times \Omega  \to \Reals $, 
$ n \in \N_0 $, 
satisfy 
for all $ n \in \N_0 $,
$ \theta \in \Reals^{ \fd} $, $ k \in \N $, 
$ \omega \in \Omega $
that 
\begin{equation}
\label{def_N_thm1}
  \mN^{ 1, \theta }_r( x ) = \mA^{ \theta }_1( x ) 
  ,
\qquad
  \mN^{ k+1, \theta }_r( x )
  =
  \mA_{ k + 1 }^{ \theta }( 
    \mathfrak{M}_{ r^{ 1 / k } }( 
      \mN^{ k, \theta }_r( x ) 
    )
  ) ,
\end{equation}
and
$ 
  \mathfrak{L}^n_r( \theta , \omega ) = 
  \tfrac{ 1 }{ M_n } \textstyle\sum_{m= 1 }^{M_n} 
  \norm{ 
    \mN^{ L , \theta }_r( X^{ n, m }( \omega ) ) - \xi 
  }^2 
$,
let $ \mL \colon \Reals^{ \fd } \to \Reals $
satisfy for all
$ \theta \in \Reals^{ \fd } $
that 
$ 
  \mL( \theta ) 
  = 
  \E[ 
    \norm{ \mN^{ L ,  \theta }_{ \infty }( X^{ 0, 0 } ) - \xi }^2
  ] 
$,
for every $ n \in \N_0 $
let 
$ 
  \mathfrak{G}^n \colon \Reals^{ \fd } \times \Omega \to \Reals^{ \fd } 
$  
satisfy for all
$ \theta \in \Reals^{ \fd } $, 
$ 
  \omega \in \{ w \in \Omega \colon ( ( \nabla_\theta \mathfrak{L}^n_r ) ( \theta , w ) )_{ r \in [1, \infty) } 
  \text{ is } \allowbreak\text{convergent} \} 
$ 
that 
\begin{equation}
\label{def_G_thm1}
 \textstyle
 \mathfrak{G}^n ( \theta , \omega ) = \lim_{r \to \infty } ( \nabla_\theta \mathfrak{L}^n_r ) ( \theta, \omega )
 ,
\end{equation}
let $ \Theta = ( \Theta_n )_{n \in \N_0 } \colon \N_0 \times \Omega \to \Reals^{ \fd } $
be a stochastic process, 
let $ ( \gamma_n )_{n \in \N_0 } \subseteq [0, \infty) $,
assume that $ \Theta_0 $ and $ ( X^{n,m} )_{ (n,m) \in ( \N_0 ) ^2} $ are independent,
and
assume for all $ n \in \N_0 $ that
$ \Theta_{ n + 1 } = \Theta_n - \gamma_n \mathfrak{G}^n( \Theta_n ) $,
$(4 L \fd \max\{\abs{a}, \abs{b}, \norm{\xi}, 1 \} )^{2 L } \gamma_n \leq ( \norm{\Theta_0 } + 1)^{-2L} $, 
and $ \sum_{k=0}^\infty \gamma_k = \infty$.
Then
\begin{enumerate} [label=(\roman*)]
	\item\label{it:main_thm_1} 
	there exists $ \mathfrak{C} \in \Reals $ such that $ \P ( \sup_{n \in \N_0 }  \norm{ \Theta_n } \leq \mathfrak{C} ) = 1$,
	\item\label{it:main_thm_2} 
	it holds that $ \P( \limsup_{n \to \infty } \mL  ( \Theta_n ) = 0 ) = 1 $, and
	\item\label{it:main_thm_3} 
	it holds that 
	$ 
	  \limsup_{n \to \infty } \E [ \mL( \Theta_n ) ] = 0
	$.
\end{enumerate}
\end{theorem}
\end{samepage}
\cref{thm:main_thm} above is a direct consequence
of \cref{cor:sgd_conv} in \cref{ssec:SGD_convergence}. \cref{cor:sgd_conv}, in turn, 
follows from \cref{thm:sgd_conv} (one of the main results of this work).
A central argument in 
the proofs of 
\cref{thm:main_thm}
and 
\cref{thm:sgd_conv}, respectively, 
is to exploit the fact 
in the setup of \cref{thm:main_thm} 
we have that the function 
\begin{equation}
\label{eq:Lyapunov_intro}
\textstyle
  \Reals^{ \fd } 
  \ni \theta 
  \mapsto 
  \bigl[
    \sum_{ k = 1 }^L
    (
      k \| \fb^{ k, \theta } \|^2 
      +
      \sum_{ i = 1 }^{ \ell_k } 
      \sum_{ j = 1 }^{ \ell_{ k - 1 } }
      |
        \fw^{ k, \theta }_{ i, j } |^2 
      )
  \bigr]
  - 2 L 
  \langle 
    f(0), \fb^{ L, \theta } 
  \rangle 
  \in
  \Reals 
\end{equation}
(see \cref{V_c} in \cref{main_setting} in \cref{ssec:framework} below)
where 
$
  \langle 
    \cdot , \cdot 
  \rangle 
$
denotes the standard scalar of two 
vectors in $ \Reals^{ \ell_L } $ 
serves as a \emph{Lyapunov function} for 
the dynamics of the SGD processes (and also for 
the gradient flow (GF) trajectories 
and the deterministic GD processes); 
see 
\cref{cor:liap} in \cref{ssec:GF_lyapunov} 
and 
\cref{cor:GF_det_ito} in \cref{ssec:GF_weak_chain_lyapunov}
for GF trajectories, 
see 
\cref{cor:V_theta_diff_const}
in 
\cref{ssec:GD_upper}
for GD processes, 
and see 
\cref{cor:Theta_diff_sgd2}
and 
\cref{lem:Theta_diff_sgd}
in
\cref{ssec:SGD_upper}
for SGD processes. The proposal of the Lyapunov function 
in \cref{eq:Lyapunov_intro} 
is inspired by the ideas in our previous articles 
\cite{MR4438169,MR4468133} 
which introduce and analyze, 
up to an additive constant, 
the function in 
\cref{eq:Lyapunov_intro}
in the special case of shallow ANNs 
with $ 1 $-dimensional output layers 
($ L = 2 $, $ \ell_L = 1 $)
where 
\cref{eq:Lyapunov_intro}
reduces to the function 
\begin{equation}
\label{eq:Lyapunov_intro_shallow}
\textstyle
  \Reals^{ \fd } 
  \ni \theta 
  \mapsto
  \| \theta \|^2
  + 
  | \fb^{ 2, \theta } |^2
  - 4 f(0) \fb^{ 2, \theta } 
  \in \Reals 
\end{equation}
(see \cite[(4) in Section~1]{MR4438169}
and \cite[Setting~2.1]{MR4468133}). 
In this work we propose the Lyapunov function in 
\cref{eq:Lyapunov_intro}, 
which is somehow the deep counterpart of \cref{eq:Lyapunov_intro_shallow}, 
we verify that the function in 
\cref{eq:Lyapunov_intro} 
is indeed a Lyapunov function for the dynamics 
of GF trajectories 
(see \cref{cor:liap} and \cref{cor:GF_det_ito}), 
deterministic GD processes (see \cref{cor:V_theta_diff_const}), 
and SGD processes (see \cref{cor:Theta_diff_sgd2} and \cref{lem:Theta_diff_sgd}), 
and, finally, we exploit this 
to establish \cref{thm:main_thm}.

\cref{thm:main_thm} demonstrates 
in the training of deep ANNs with ReLU activation
under the key assumption 
that the target function 
of the considered supervised learning problem 
coincides with the constant function 
\begin{equation}
\label{eq:key_assumption_constant_function}
  [a,b]^{ \ell_0 } \ni x \mapsto \xi \in \Reals^{ \ell_L } ,
\end{equation}
where $ \xi \in \Reals^{ \ell_L } $ is an arbitrary vector in $ \Reals^{ \ell_L } $, 
that 
we have that 
the risk of the SGD process 
$
  \mL( \Theta_n )
$
converges 
in the $ \P $-almost sure (see \cref{it:main_thm_2} in \cref{thm:main_thm}) 
and the $ L^1( \P; \Reals ) $-sense 
(see \cref{it:main_thm_3} in \cref{thm:main_thm})
to zero 
as the number of SGD steps $ n \in \N $ 
increases to infinity. 
A natural question that arises from this convergence result is whether this key assumption,  
that the target function agrees 
with the constant function in \cref{eq:key_assumption_constant_function},  
can be omitted. 
In general, this is not possible because 
without the assumption in \cref{eq:key_assumption_constant_function} 
the conclusion of \cref{thm:main_thm} is, in general, not valid anymore. 
Indeed, 
in many situations when the target function 
is not constant, there exist non-global local minimum points in 
the risk landscape whose domain of attraction has a strictly positive 
probability and this, in turn, disproves the conclusion of \cref{thm:main_thm} 
(cf., e.g., \cite{MR4232648}).

However, one may wonder whether it is possible to extend \cref{thm:main_thm} 
to more general functions with a modified conclusion in which convergence 
of the risk of the SGD method to zero is asserted 
as both the number of SGD training steps $ n $ 
and the width/depth of the considered ANNs converge to infinity. 
Such a statement could under suitable assumptions 
very well be valid for more general target functions. 
The arguments employed in this work to prove 
\cref{thm:main_thm} seem, however, to be only of limited use 
for establishing such a modified conclusion for more general target functions 
and, in general, it remains a fundamental open problem of research 
in the area of mathematics of machine learning to prove (or disprove) 
such a modified conclusion for more general target functions. 
For this open problem of research the arguments in this work, 
in particular, the Lyapunov function in \cref{eq:Lyapunov_intro}, 
appear to be only of limited use as the function 
in \cref{eq:Lyapunov_intro}, in general, seems not to be 
a Lyapunov function in the case where the target 
function is not constant anymore (cf.\ \cite[Proposition~4.1]{MR4473797}). 
However, in the regime where the risk of the ANN parameters 
is strictly larger than the best constant approximation 
of the target function 
$
  \inf_{ c \in \Reals }
  \int_{ [a,b]^{ \ell_0 } }
  (
    f(x)
    -
    c
  )^2
  \mu( \d x )
$, 
we expect that the function in \cref{eq:Lyapunov_intro} 
still serves as a Lyapunov function even in the case 
of general non-constant target functions; 
cf.\ \cite[Proposition~4.1]{MR4473797} and \cref{prop:liap2}. 
This could most likely be used to establish a priori bounds for 
GD optimization methods 
in the regime 
\begin{equation}
\textstyle 
  \{ 
    \theta \in \Reals^{ \fd } 
    \colon
    \mL( \theta )
    >
    \inf_{ c \in \Reals }
    \int_{ [a,b]^{ \ell_0 } }
    (
      f(x)
      -
      c
    )^2
    \mu( \d x )
  \} 
  .
\end{equation}
similar as in \cite[Corollary~4.3]{MR4473797} in the case of shallow ANNs. 
Although such a priori bounds appear to be verifiable 
for general non-constant target functions using the function in \cref{eq:Lyapunov_intro}, 
such a priori bounds do, of course, not 
establish the above outlined modified conclusion 
of \cref{thm:main_thm} in the case of non-constant target functions, 
which remains an open problem of future research.

The remainder of this article is structured in the following way.
In \cref{sec:risk} we establish several fundamental regularity properties 
and representation results for the risk function and its generalized gradient function. 
While \cref{thm:main_thm} is restricted to the situation where the target function 
is a constant function, the findings in \cref{sec:risk} 
are proved for \emph{general measurable target functions} 
and are -- beyond their employment in the proof of \cref{thm:main_thm} -- 
of general use in the mathematical analysis 
of the training of deep ReLU ANNs. 
In \cref{sec:grad_fl} we use 
the main findings from \cref{sec:risk} to prove in the case 
where the target function is a constant function that 
the risks of GF processes converge with convergence rate $ 1 $ to zero. 
In \cref{sec:grad_des} we employ some of the results from \cref{sec:risk,sec:grad_fl} 
to establish that the risks of GD processes converge to zero provided that 
the target function is a constant function and that the learning rates 
(the step sizes) of the considered GD processes are not $ L^1 $-summable 
but sufficiently small. 
Finally, in \cref{sec:stoch_grad_des} we extend 
the findings from \cref{sec:grad_des} and prove 
that the expectations of the risks of SGD processes converge to zero 
provided that the target function is a constant function 
and that the learning rates are not $ L^1 $-summable but sufficiently small. 
\cref{thm:main_thm} above is a direct consequence of \cref{cor:sgd_conv} 
in \cref{sec:stoch_grad_des}.

\section{Properties of the risk function associated to deep artificial neural networks (ANNs)}
\label{sec:risk}

In this section we establish several fundamental regularity properties 
and representation results for the risk function and its generalized gradient function. 
While \cref{thm:main_thm} in the introductory section above is restricted 
to the situation where the target function 
is a constant function, the findings in this section 
are proved for \emph{general measurable target functions} 
and are -- beyond their deployment in the proof of \cref{thm:main_thm} above -- 
of general use in the mathematical analysis 
of the training of deep ReLU ANNs.

In \cref{main_setting} in \cref{ssec:framework} we present 
our mathematical framework to introduce, among other things, 
the number $ L \in \N $ of affine linear transformations 
in the considered deep ANNs, 
the dimensions $ \ell_0, \ell_1, \dots, \ell_L \in \N $ of the layers 
of the considered deep ANNs (the number of neurons on the layers of the considered 
deep ANNs), 
the number 
\begin{equation}
\textstyle 
  \fd = \sum_{ k = 1 }^L \ell_k ( \ell_{ k - 1 } + 1 ) \in \N 
\end{equation}
of real parameters of the considered deep ANNs, 
the realization functions 
$
  [a,b]^{ \ell_0 } \ni x \mapsto 
  \mN^{ L, \theta }_{ \infty }( x ) \in \Rr^{ \ell_L }
$, 
$
  \theta = ( \theta_1, \dots, \theta_{ \fd } ) \in \Rr^{ \fd } 
$, 
of the considered deep ReLU ANNs 
(see \cref{N}), 
the risk function 
$
  \mL_{ \infty } \colon \Rr^{ \fd } \to \Rr 
$ 
(see \cref{L_r}), 
the generalized gradient function 
$
  \mG = ( \mG_1, \dots, \mG_{ \fd } ) \colon \Rr^{ \fd } \to \Rr^{ \fd }
$
associated to the risk function,  
and the Lyapunov-type function 
$ V \colon \Reals^{ \fd } \to \Reals $ 
for the mathematical analysis of GD type optimization methods 
in the training of deep ReLU ANNs (see \cref{V_c}). 
In \cref{main_setting} we thus consider ReLU ANNs 
with $ L - 1 $ hidden layers
and, including the input and the output layer, 
with 
$ L + 1 $ layers overall.

In \cref{prop:G} 
in \cref{ssec:gen_grad_risk_functions} below 
(one of the main results of this section) 
we establish 
in \cref{G_w,G_b} an \emph{explicit representation for the generalized 
gradient function} 
$
  \mG \colon \Rr^{ \fd } \to \Rr^{ \fd }
$
of the risk function 
$
  \mL_{ \infty } \colon \Rr^{ \fd } \to \Rr 
$.
Our proof 
of \cref{prop:G} 
employs 
the approximation result for realization functions 
of deep ReLU ANNs 
in \cref{lem:properties2}
in \cref{ssec:properties}, 
the representation result for gradients 
of suitable approximations of the realization functions 
of deep ReLU ANNs 
in \cref{lem:N_der} in \cref{ssec:properties}, 
the a priori bound result in \cref{lem:der:unif:bounded}
in \cref{ssec:gen_grad_risk_functions}, 
and 
the elementary integrability result 
for the target function 
$
  f \colon [a,b]^{ \ell_0 } \to \Rr 
$
in \cref{lem:f:integrable} 
in \cref{ssec:gen_grad_risk_functions}.

Our proof of the approximation result in \cref{lem:properties2}, 
in turn, 
uses 
the elementary auxiliary result in \cref{lem:properties1} in \cref{ssec:properties}
and 
the elementary convergence rate result for 
suitable approximations of the ReLU activation function 
$ \Rr \ni x \mapsto \max\{ x, 0 \} \in \Rr $
in \cref{lem:approx_rectifier}
in \cref{ssec:properties}. 
In the elementary result 
in \cref{lem:rr:explicit}
in \cref{ssec:properties}
we also explicitly construct examples 
for such approximations of the ReLU activation function 
$ \Rr \ni x \mapsto \max\{ x, 0 \} \in \Rr $
(cf.\ \cref{lim_R0} in \cref{main_setting}).

In \cref{prop:G_upper_estimate} 
in \cref{ssec:upper_estimates} below 
(another main result of this section) 
we establish in \cref{G_upper_estimate} 
an \emph{explicit polynomial growth estimate 
for the generalized gradient function} 
$
  \mG = ( \mG_1, \dots, \mG_{ \fd } ) \colon \Rr^{ \fd } \to \Rr^{ \fd }
$
of the risk function 
$
  \mL_{ \infty } \colon \Rr^{ \fd } \to \Rr 
$. 
As a consequence of \cref{G_upper_estimate}, 
we show in \cref{cor:G_up} 
in \cref{ssec:upper_estimates} 
that the generalized gradient function 
$
  \mG \colon \Rr^{ \fd } \to \Rr^{ \fd }
$
is locally bounded. 
Our proof of \cref{cor:G_up} employs, 
beside \cref{G_upper_estimate}, 
also the well-known local Lipschitz continuity result 
in \cref{lem:loc_lip} in \cref{ssec:local_lipschitz}. 
Our proof of 
\cref{prop:G_upper_estimate}, in turn, 
makes use of the explicit representation 
of the generalized gradient function 
$
  \mG \colon \Rr^{ \fd } \to \Rr^{ \fd }
$
in \cref{G_w,G_b} in \cref{prop:G}.

In \cref{cor:conv_characterization} 
in \cref{ssec:risk_non-convex} 
we establish 
that the risk function 
$
  \mL_{ \infty } \colon \Rr^{ \fd } \to \Rr 
$
in the training of deep ReLU ANNs 
(see \cref{L_r}) 
is convex \emph{if and only if} 
the product of the total mass 
$ \mu( [a,b]^{ \ell_0 } ) $
and the number $ L - 1 $ of hidden layers 
vanishes 
\begin{equation}
  ( L - 1 ) 
  \mu( 
    [a,b]^{ \ell_0 } 
  ) 
  = 0
  .
\end{equation}
\cref{cor:conv_characterization} 
hence proves that, 
except of the degenerate situation 
where the underlying measure $ \mu \colon \mathcal{B}( [a,b]^{ \ell_0 } ) \to [0,\infty] $ 
is the zero measure 
or where the considered ANNs  
just describe affine linear transformations 
without any hidden layer, 
it holds for every arbitrary measurable target function 
$ f \colon [a,b]^{ \ell_0 } \to \Rr^{ \ell_L } $
that the risk function 
$
  \mL_{ \infty } \colon \Rr^{ \fd } \to \Rr 
$
is not convex. 
Our proof of \cref{cor:conv_characterization} 
employs 
the essentially well-known convexity result in \cref{cor:convexity_0}
in \cref{ssec:risk_convex}
and 
the basic non-convexity result 
in \cref{cor:risk_not_convex} in \cref{ssec:risk_non-convex}. 
Our proof of \cref{cor:convexity_0} uses 
the elementary characterization result for affine linear functions 
in \cref{lem:aff_lin} 
in \cref{ssec:risk_convex}
and the well-known convexity result 
in \cref{prop:outer_convexity}
in \cref{ssec:risk_convex}. 
In \cref{lem:convexity} in \cref{ssec:risk_convex} 
we recall the fact that convex functions from 
$ \Rr^{ \fd } $ to $ \Rr $ 
do not admit any no-global local minimum point 
(see \cref{no_non-global_local_minima}).

Many of the results in this section 
extend the results in \cite[Section~2]{MR4468133} 
from shallow ReLU ANNs with just one hidden layer 
to deep ReLU ANNs 
with an arbitrarily large number of hidden layers. 
In particular, 
\cref{lem:properties2} in \cref{ssec:properties} 
extends \cite[Proposition 2.2]{MR4468133},
\cref{prop:G} in \cref{ssec:gen_grad_risk_functions} 
extends \cite[Proposition 2.3]{MR4468133},
\cref{lem:loc_lip} in \cref{ssec:local_lipschitz} 
extends \cite[Lemma 2.4]{MR4468133},
\cref{prop:G_upper_estimate} in \cref{ssec:upper_estimates} 
extends \cite[Lemma 2.5]{MR4468133},
and \cref{cor:G_up} in \cref{ssec:upper_estimates} extends \cite[Corollary 2.6]{MR4468133}.

\subsection{Mathematical framework for deep ANNs with ReLU activation}
\label{ssec:framework}

\begin{setting}
\label{main_setting}
Let 
$ L, \fd \in \N $,
$ ( \ell_k)_{k\in \N_0 } \subseteq \N $,
$ a, \mathbf{a} \in \Reals $, $ b \in (a, \infty) $, 
$ \A \in (0,\infty) $, 
$ \B \in ( \A, \infty) $ 
satisfy $ \fd = \sum_{ k = 1 }^L \ell_k( \ell_{k-1} + 1) $
and $ \mathbf{a} =\max\{\abs{a}, \abs{b},1\} $,
for every $ \theta = ( \theta_1, \ldots, \theta_{ \fd }) \in \Rr^{ \fd } $ let $ \fw^{ k, \theta } = ( \fw^{ k, \theta }_{ i,j} )_{ (i,j) \in \{ 1, \ldots, \ell_k \} \times \{ 1, \ldots, \ell_{k-1} \} }
$ $ \in \Reals^{ \ell_k \times \ell_{k-1} } $, $ k \in \N $, 
and $ \fb^{ k, \theta } = ( \fb^{ k, \theta }_1, \ldots, \fb^{ k, \theta }_{\ell_k} )
\in \Reals^{ \ell_k} $,
$ k \in \N $, 
satisfy for all
$ k \in \{ 1, \ldots, L\} $, $ i \in \{ 1, \ldots, \ell_k \} $,
$ j \in \{ 1, \ldots, \ell_{k-1} \} $ that
\begin{equation}
\label{wb}
  \fw^{ k, \theta }_{ i, j } 
  =
  \theta_{ (i-1)\ell_{k-1} + j+\sum_{h= 1 }^{ k-1} \ell_h( \ell_{h-1} + 1)}
\qqandqq
  \fb^{ k, \theta }_i =
  \theta_{\ell_k\ell_{k-1} + i+\sum_{h= 1 }^{ k-1} \ell_h( \ell_{h-1} + 1)} 
  \, ,
\end{equation}
for every $ k \in \N $, $ \theta \in \Reals^{ \fd } $
let $ \mA_k^\theta = ( \mA_{k,1}^\theta, \ldots, \mA_{k, \ell_k}^\theta)
\colon \Reals^{ \ell_{k-1} } \to \Reals^{ \ell_k} $
satisfy 
for all $ x \in \Rr^{ \ell_{ k - 1 } } $ that 
$ 
  \mA_k^{ \theta }( x ) 
  =
  \fb^{ k, \theta } + \fw^{ k, \theta } x 
$,
let 
$ \R_r \colon \Rr \to \Rr $, 
$ r \in [1, \infty] $,
satisfy for all 
$ r \in [1, \infty) $, 
$ x \in ( - \infty, \A r^{ - 1 } ] $, 
$ 
  y \in \Rr
$, 
$
  z \in [ \B r^{ - 1 }, \infty ) 
$
that 
\begin{equation}
\label{lim_R0}
  \R_r \in C^1( \Reals, \Reals ) ,
\qquad
  \R_r(x) = 0,
\qquad 
  0 \leq \R_r(y) \leq \R_{ \infty }( y ) 
  = 
  \max\{ y, 0 \}
  ,
\qqandqq
  \R_r(z) = z
  ,
\end{equation}
assume 
$
  \sup_{ r \in [1, \infty) }
  \sup_{ x \in \Reals } 
  | ( \R_r )'( x ) | < \infty 
$,
let 
$
  \norm{\cdot} \colon ( \cup_{ n \in \enne } \Reals^n ) \to \Rr 
$,
$
  \ip{\cdot}{\cdot} \colon( \cup_{ n \in \enne }( \Reals^n\times \Reals^n )) \to\Reals
$, 
and
$
  \mathfrak{M}_r 
  \colon( \cup_{ n \in \enne } \Reals^n ) \to ( \cup_{ n \in \enne } \Reals^n )
$, 
$ r \in [1, \infty]
$, 
satisfy for all $ r \in [1, \infty] $, $ n \in \N $,
$ x = (x_1, \ldots, x_n ) $, $y= (y_1, \ldots,y_n ) \in \Rr^n $
that 
\begin{equation}\label{norm_scalar_M}
  \textstyle{
  \norm{x} = ( \sum_{ i= 1 }^n|x_{ i}|^2)^{ 1/2},
  \qquad\ip{x}{y} =\sum_{ i= 1 }^nx_iy_i, 
  \qquad\text{and}
  \qquad\mathfrak{M}_r(x)= ( \R_r(x_1), \ldots, \R_r(x_n ))},
\end{equation}
for every  $ \theta \in \Rr^{ \fd } $
let $ \mN^{ k, \theta }_r = ( \mN^{ k, \theta }_{r,1}, \ldots, \mN^{ k, \theta }_{r, \ell_k} )
\colon \Rr^{ \ell_0 } \to \Rr^{ \ell_k} $, 
$ r \in [1, \infty]$, $ k \in \N $,
and 
$ \mX^{ k, \theta }_{ i} \subseteq\Rr^{ \ell_0 } $, $ k, i \in \N $,
satisfy 
for all $ r \in [1, \infty] $, $ k \in \N $,
$ i \in \{ 1, \ldots, \ell_k \} $, $ x \in \Reals^{ \ell_0 } $
that
\begin{equation}
\label{N}
\begin{split}
&
  \mN^{ 1, \theta }_r(x) = \mA^\theta_1(x) ,
\qquad
  \mN^{ k + 1, \theta }_r( x ) 
  =
  \mA_{ k + 1}^{ \theta }(
    \mathfrak{M}_{ r^{ 1 / k } }(
      \mN^{ k, \theta }_r( x )
    )
  ),
\end{split}
\end{equation}
and 
$
  \mX^{ k, \theta }_i = 
  \{ 
    y \in [a,b]^{ \ell_0 } \colon
    \mN^{ k, \theta }_{ \infty, i }( y ) > 0 
  \} 
$,
let $ \mu \colon \mathcal{B}( [a,b]^{ \ell_0 } ) \to [0, \infty] $ be a measure,
let $ \mathfrak{m} \in \Reals $ satisfy
$ \mathfrak{m} =\mu([a,b]^{ \ell_0 } ) $,
let $ f = (f_1, \ldots, f_{\ell_L}) \colon [a,b]^{ \ell_0 } \to \Reals^{ \ell_L} $ be measurable,
for every $ r \in [1, \infty] $ 
let $ \mL_r\colon \Reals^{ \fd }\to \Reals $ satisfy 
for all $ \theta \in \Reals^{ \fd } $
that 
\begin{equation}
\label{L_r}
  \textstyle{
  \mL_r( \theta)=\int_{[a,b]^{ \ell_0 } }
  \|\mN_r^{L, \theta }(x)-f(x)\|^2\, \mu( \d x),}
\end{equation}
let $ \mG = ( \mG_1, \dots, \mG_{ \fd } ) \colon \Reals^{ \fd } \to \Reals^{ \fd } $ 
satisfy
for all $ \theta \in \{\vartheta \in \Reals^{ \fd }\colon(( \nabla\mL_r)( \vartheta))_{ r \in [1,\infty) }
\text{ is convergent} \} $ that
$ \mG( \theta)=\lim_{r\to \infty }( \nabla\mL_r)( \theta) $, and
let $V \colon \Reals^{ \fd }\to\Reals $ satisfy for all $ \theta \in \Reals^{ \fd } $  that
\begin{equation}
\label{V_c}
\textstyle
  V( \theta ) = 
  \bigl[
    \sum_{ k = 1 }^L
    (
      k \| \fb^{ k, \theta } \|^2 
      +
      \sum_{ i = 1 }^{ \ell_k } 
      \sum_{ j = 1 }^{ \ell_{ k - 1 } }
      |
        \fw^{ k, \theta }_{ i, j } |^2 
      )
  \bigr]
  - 2 L 
  \langle 
    f(0), \fb^{ L, \theta } 
  \rangle 
  .
\end{equation}
\end{setting}

\subsection{Approximations of the realization functions of the considered deep ANNs}
\label{ssec:properties}

\begin{lemma}
\label{lem:properties1}
Assume \cref{main_setting} and let 
$ \theta \in \Rr^{ \fd } $, $ r \in [1, \infty] $.
Then
\begin{enumerate}[(i)]
\item 
it holds for all $ k \in \enne $, $ i \in \{ 1, \ldots, \ell_k \} $,  
$ x = (x_1, \ldots, x_{\ell_{k-1} } ) \in \Rr^{ \ell_{k-1} } $
that
\begin{equation}
\label{A}
  \mA^\theta_{k,i}(x)=\fb^{ k, \theta }_{ i} + 
  \sum_{ j = 1 }^{ \ell_{k-1} } \fw^{ k, \theta }_{ i,j} x_{j},
\end{equation}
\item
\label{it:N_long1} 
it holds for all $ i \in \{ 1, \ldots, \ell_1\} $,
$ x = (x_1, \ldots, x_{\ell_0 } ) \in \Rr^{ \ell_0 } $ that
\begin{equation}
\label{N_long1}
  \mN^{ 1, \theta }_{r,i}(x)=\fb^{ 1, \theta }_{ i} + 
  \sum_{ j = 1 }^{ \ell_0 } \fw^{ 1, \theta }_{ i,j} x_{j},
\end{equation}
and
\item
\label{it:N_long2}
it holds for all $ k \in \N $, $ i \in \{ 1, \ldots, \ell_{k+1} \} $,
$ x \in \Rr^{ \ell_0 } $ that
\begin{equation}
\label{N_long2}
  \mN^{ k+1, \theta }_{r,i}(x)=\fb^{ k+1, \theta }_{ i} + 
  \sum_{ j = 1 }^{ \ell_k} \fw^{ k+1, \theta }_{ i,j} 
  \R_{r^{ 1/k} }( \mN^{ k, \theta }_{r,j}(x))
  .
\end{equation}
\end{enumerate}
\end{lemma}
\begin{proof}[Proof of \cref{lem:properties1}]
  \Nobs that \cref{N} and
  the assumption that 
  for all $ k \in \enne $, $ x \in \Reals^{ \ell_{k-1} } $
  it holds that
  $ \mA_k^\theta(x)=\fb^{ k, \theta } + \fw^{ k, \theta }x$
  establish
  \cref{A}, \cref{N_long1}, and \cref{N_long2}.
  The proof of \cref{lem:properties1} is thus complete.
\end{proof}

\begin{lemma}[Approximations of the rectifier function]
\label{lem:approx_rectifier}
Assume \cref{main_setting}. 
Then 
\begin{enumerate}[(i)]
\item 
\label{item:approx_rectifier:item_i}
it holds for all $ x \in \Rr $ that
\begin{equation}
\label{lim_R}
  \limsup\nolimits_{ r \to \infty 
  }( 
    | \R_r(x) - \R_{ \infty }(x) |
    +
    | ( \R_r )'(x) - \indicator{ (0, \infty) }( x ) | 
  ) = 0 
\end{equation}
and 
\item 
\label{item:approx_rectifier:item_ii}
it holds for all 
$ r \in [1, \infty) $, $ x \in \Rr $
that 
$
  | \R_r( x ) - \R_{ \infty }( x ) | 
  \leq \B r^{ - 1 } 
$. 
\end{enumerate}
\end{lemma}

\begin{proof}[Proof of \cref{lem:approx_rectifier}] 
\Nobs that \cref{lim_R0} ensures that for all 
$ r \in [1, \infty) $, $ x \in (-\infty,0] $ it holds that 
$
  \R_r( x ) 
  = 0 
  =
  \max\{ x, 0 \} 
  = \R_{ \infty }( x )
$. 
\Hence for all $ r \in [1, \infty) $, $ x \in (-\infty,0] $ that
\begin{equation}
\label{eq:approx_rectifier:negative_x}
  | \R_r(x) - \R_{ \infty }(x) |
  +
  | ( \R_r )'(x) - \indicator{ (0, \infty) }( x ) | 
  = 0 .
\end{equation} 
\Moreover \cref{lim_R0} proves that 
for all $ x \in (0, \infty) $ 
there exists $ R \in [1, \infty) $ 
such that for all $ r \in [R, \infty) $, 
$ y \in ( \nicefrac{ x }{ 2 } , \infty ) $
it holds that
$
  \R_r( y ) = \R_{ \infty }( y )
$. 
\Hence for all $ x \in (0, \infty) $ that
\begin{equation}
\textstyle
  \limsup_{ r \to \infty }
  \bigl(
    | \R_r(x) - \R_{ \infty }(x) |
    +
    | ( \R_r )'(x) - ( \R_{ \infty } )'(x) | 
  \bigr)
  = 0 .
\end{equation} 
Combining this with \cref{eq:approx_rectifier:negative_x} 
establishes \cref{item:approx_rectifier:item_i}. 
\Nobs that \cref{lim_R0} shows that for all 
$ r \in [1, \infty) $, 
$ y \in ( 0, \B r^{ - 1 } ) $ 
it holds that 
\begin{equation}
\label{eq:approx_rectifier:middle_y}
  | \R_r( y ) - \R_{ \infty }( y ) |
  =
  \R_{ \infty }( y ) - \R_r( y )
  \leq 
  \R_{ \infty }( y ) 
  \leq 
  y 
  \leq 
  \B r^{ - 1 }
  .
\end{equation}
\Moreover \cref{lim_R0} proves that 
for all $ r \in [1, \infty) $, $ x \in (-\infty,0] \cup [ \B r^{ - 1 }, \infty) $
it holds that 
$
  \R_r( x ) = \R_{ \infty }( x )
$. 
Combining this with \cref{eq:approx_rectifier:middle_y} 
establishes \cref{item:approx_rectifier:item_ii}. 
The proof of \cref{lem:approx_rectifier} is thus complete. 
\end{proof}

\begin{lemma}[Examples of approximations of the rectifier function] 
\label{lem:rr:explicit}
	Let $\A \in (0, \infty )$,
	$\B \in (\A , \infty )$,
	$\eta \in C^1 ( \Rr , \Rr )$ satisfy for all $x \in (- \infty , 0 ]$,
	$y \in \Rr $,
	$z \in [1 , \infty )$
	that 
	$ \eta( x ) = 0 \le \eta( y ) \le 1 = \eta( z ) $
	and for every $r \in [1, \infty )$ 
	let $\R_r \colon \Rr \to \Rr$
	satisfy for all
	$x \in \Rr$
	that
	$ \R_r( x ) = \max\cu{ x, 0 } \eta( \frac{ r x - \A }{ \B - \A } ) $.
	Then
	\begin{enumerate} [(i)]
		\item \label{lem:rr:explicit:item1} it holds for all $r \in [1, \infty )$,
		$x \in (- \infty , \A r^{-1} ]$ that $\R_r ( x ) = 0$,
		\item \label{lem:rr:explicit:item1b}
		it holds for all $r \in [1, \infty )$,
		$x \in [\B r^{-1}, \infty )$
		that $\R_r ( x ) = x$, 
		\item \label{lem:rr:explicit:item2} it holds for all $r \in [1, \infty )$,
		$x \in \Rr$ that $0 \le \R_r ( x ) \le \max \cu{x , 0 }$,
		\item \label{lem:rr:explicit:item3} it holds for all $r \in [1, \infty )$
		that $\R_r \in C^1 ( \Rr , \Rr)$,
		and
		\item \label{lem:rr:explicit:item4} it holds that $\sup_{r \in [1, \infty )}  \sup_{x \in \Rr } \abs{ ( \R_r)' ( x ) } < \infty $.
	\end{enumerate}
\end{lemma}
\begin{proof}[Proof of \cref{lem:rr:explicit}]
\Nobs that the assumption that for all $ x \in ( - \infty, 0] $ 
it holds that $ \eta( x ) = 0 $ 
establishes \cref{lem:rr:explicit:item1}. 
\Nobs that the assumption that for all $ x \in [1,\infty) $
it holds that $ \eta( x ) = 1 $
proves \cref{lem:rr:explicit:item1b}. 
\Nobs the assumption that for all $ x \in \Rr $ it holds that 
$ 0 \leq \eta( x ) \leq 1 $ 
establishes \cref{lem:rr:explicit:item2}.
\Nobs that \cref{lem:rr:explicit:item1},
the fact that for all $ r \in [1, \infty) $,
$ x \in (0, \infty) $ 
it holds that
$
  \R_r( x ) = x \eta( \frac{ r x - \A }{ \B - \A } ) 
$,
and the assumption that $ \eta \in C^1( \Rr, \Rr) $
establish \cref{lem:rr:explicit:item3}. 
\Nobs that the chain rule implies for all $ r \in [1, \infty) $,
$ x \in ( 0, \B r^{ - 1 }] $ that
\begin{equation}
  \abs{ ( \R_r )'( x ) }
= 
  \bigl|
    \eta( \tfrac{ r x - \A }{ \B - \A } )
    + 
    x 
    \bigl[
      \tfrac{ r }{ \B - \A } 
    \bigr]
    \eta'(
      \tfrac{ r x - \A }{ \B - \A } 
    )
  \bigr|
\le 
  1 + 
  \bigl[
    \tfrac{ \B }{ \B - \A } 
  \bigr]
  \bigl[ 
    \sup\nolimits_{ y \in \Rr } 
    \abs{ \eta'( y ) } 
  \bigr] 
  .
\end{equation}
Combining this with \cref{lem:rr:explicit:item1,lem:rr:explicit:item1b}
proves that for all $ r \in [1, \infty ) $, $ x \in \Rr $
it holds that
\begin{equation} 
  \abs{
    (\R_r)'( x ) 
  }
\le 
  \max\cu[\big]{ 
    0, 1, 
    1 
    + 
    \bigl[
      \tfrac{ \B }{ \B - \A } 
    \bigr] 
    \bigl[ 
      \sup\nolimits_{ y \in \Rr }\abs{ \eta'( y ) }
    \bigr] 
  } 
  \leq 
    1 
    + 
    \bigl[
      \tfrac{ \B }{ \B - \A } 
    \bigr] 
    \bigl[ 
      \sup\nolimits_{ y \in \Rr }\abs{ \eta'( y ) }
    \bigr] 
  < \infty 
  .
\end{equation} 
This establishes \cref{lem:rr:explicit:item4}.
The proof of \cref{lem:rr:explicit} is thus complete.
\end{proof}

\begin{prop}[Approximations of realization functions 
of deep ReLU ANNs]
\label{lem:properties2}
Assume \cref{main_setting}, 
for every $ r \in [1, \infty] $, $ k \in \N $ 
let 
$
  \mM_{ r, k } \colon \Rr^{ \ell_k } \to \Rr^{ \ell_k } 
$ 
satisfy for all
$ x \in \Rr^{ \ell_k } $ 
that 
$
  \mM_{ r, k }( x ) = \mathfrak{M}_r( x ) 
$,
and let $ k \in \cu{1, \ldots, L } $, $ \theta \in \Rr^{ \fd } $.
Then
  \begin{enumerate}[(i)]
  \item\label{it:N_short} 
  it holds for all $ r \in [1, \infty] $ that 
\begin{equation}
\label{N_short}
\begin{split}
  \mathcal N_r^{ k, \theta } =
  \begin{cases}
  \mA_k^\theta &\colon k=1 \\[1ex]
  \mA_k^\theta \circ
  \mM_{r^{ 1/(k-1)},k-1} \circ\mA_{k-1}^\theta \circ
  \ldots
  \circ 
  \mM_{ r^{ 1 / 2 }, 1 }
  \circ \mA_2^{ \theta } \circ \mM_{r,1} \circ \mA_1^{ \theta }
  & \colon k > 1 ,
  \end{cases}
\end{split}
\end{equation}
\item
\label{it:lim_aux0} 
it holds 
that 
\begin{equation}
\label{lim_aux0}
\begin{split}
&
\textstyle
  \sup_{ r \in [1,\infty) }
  \sup_{ x \in \Rr^{ \ell_0 } }
  \max_{
    i \in \{ 1, \dots, \ell_k \} 
  }
  \bigl(
    r^{ 1 / k }
    \abs{ 
      \R_{ r^{ 1 / k } }(
        \mN^{ k, \theta }_{ r, i }( x ) 
      )
      -
      \R_{ \infty }( \mN^{ k, \theta }_{ \infty, i }( x ) )
    }
  \bigr)
\\
&
\textstyle
  \leq 
  \B 
  \bigl[
    \sum_{ j = 0 }^{ k - 1 }
    \bigl(
      \sup_{
        r \in [1,\infty) 
      }
      \sup_{
        x \in \Rr 
      }
      |
        ( \R_r )'( x )
      |
    \bigr)^j
    \bigl(
      \sum_{ v = 1 }^{ \fd }
      | \theta_{ \fd } |
    \bigr)^j
  \bigr]
  < \infty
  ,
\end{split}
\end{equation}
\item
\label{it:lim_aux0_B} 
it holds that 
\begin{equation}
\label{lim_aux0_B}
\begin{split}
&
\textstyle
  \sup_{ r \in [1,\infty) }
  \sup_{ x \in \Rr^{ \ell_0 } }
  \max_{
    i \in \{ 1, \dots, \ell_k \} 
  }
  \bigl(
    r^{ 1 / ( \max\{ k - 1, 1 \} ) }
    \abs{ 
      \mN^{ k, \theta }_{ r, i }( x ) 
      -
      \mN^{ k, \theta }_{ \infty, i }( x ) 
    }
  \bigr)
\\
&
\textstyle
  \leq 
  \B
  \bigl[
    \sum_{ j = 1 }^{ k }
    \bigl(
      \sup_{
        r \in [1,\infty) 
      }
      \sup_{
        x \in \Rr 
      }
      |
        ( \R_r )'( x )
      |
    \bigr)^{ j - 1 }
    \bigl(
      \sum_{ v = 1 }^{ \fd }
      | \theta_{ \fd } |
    \bigr)^j
  \bigr]
  \indicator{
    (1,\infty)
  }( k )
  < \infty
  ,
\end{split}
\end{equation}
\item
\label{it:lim_aux1} 
it holds for all 
$ i \in \{ 1, \ldots, \ell_k \} $, $ x \in \Reals^{ \ell_0 } $ that 
\begin{equation}
\label{lim_aux1}
  \limsup\nolimits_{ r \to \infty }
  \bigl(
    |
      \mN^{ k, \theta }_{ r, i }( x ) 
      -
      \mN^{ k, \theta }_{ \infty, i }( x )
    | 
    +
    |
      \R_{ r^{ 1 / k } }( \mN^{ k, \theta }_{ r, i }( x ) )
      -
      \R_{ \infty }( \mN^{ k, \theta }_{ \infty, i }( x ) )
    | 
  \bigr) = 0
  ,
\end{equation}
and
\item
\label{it:lim_aux2} 
it holds for all 
$ i \in \{ 1, \ldots, \ell_k \} $, $ x \in \Reals^{ \ell_0 } $ that 
\begin{equation}
\label{lim_aux2}
  \limsup\nolimits_{ r \to \infty }
  |
    ( \R_{ r^{ 1 / k } } )'( \mN^{ k, \theta }_{ r, i }( x ) )
    -
    \indicator{ \mX^{ k, \theta }_i }( x ) 
  | = 0
  .
\end{equation}
\end{enumerate}
\end{prop}
\begin{proof}[Proof of \cref{lem:properties2}]
Throughout this proof let 
$ \mathbf{L} \in \Rr $
satisfy
\begin{equation}
\label{eq:def_L_Lipschitz_constant_ReLU}
\textstyle
  \mathbf{L} = 
  \sup_{
    r \in [1,\infty)
  }
  \sup_{
    x \in \Rr
  }
  | ( \R_r )'( x ) | ,
\end{equation}
let $ c \in \Rr $ satisfy 
$
  c = \mathbf{L} [ \sum_{ j = 1 }^{ \fd } | \theta_j | ]
$, 
and for every $ K \in \cu{1, \ldots, L } $, $ r \in [1,\infty] $ 
let 
$
  e_{ K, r } \in \Rr
$ 
satisfy 
\begin{equation}
\label{eq:eK_def}
\textstyle
  e_{ K, r } 
  =
  \sup_{ x \in \Rr^{ \ell_0 } }
  \max_{
    i \in \{ 1, \dots, \ell_K \}
  }
    |
      \R_{ r^{ 1 / K } }( 
        \mN^{ K, \theta }_{ r, i }( x ) 
      )
      -
      \R_{ \infty }( 
        \mN^{ K, \theta }_{ \infty, i }( x ) 
      )
    |
  .
\end{equation}
\Nobs that \cref{eq:def_L_Lipschitz_constant_ReLU} 
and the fundamental theorem of calculus ensure that 
for all $ r \in [1,\infty) $, $ x, y \in \Rr $ 
it holds that
\begin{equation} 
\label{eq:Lipschitz_ReLU_approx}
  | 
    \R_r( x ) - \R_r( y )
  |
  \leq 
  \mathbf{L} | x - y | .
\end{equation}
\Moreover \cref{N} establishes \cref{it:N_short}. 
\Nobs that 
\cref{lem:approx_rectifier} 
assures that 
for all $ r \in [ 1, \infty ) $ it holds that
\begin{equation}
\label{eq:rec_approximation_quantitative}
\textstyle
  \sup_{ y \in \Rr }
  | \R_r( y ) - \R_{ \infty }( y ) | \leq \B r^{ - 1 } 
  .
\end{equation}
\Cref{it:N_long1} of \cref{lem:properties1} 
\hence ensures that 
for all $ r \in [1, \infty) $,
$ i \in \{ 1, \ldots, \ell_1\} $, 
$ x \in \Reals^{ \ell_0 } $
it holds that
\begin{equation}
  \abs{
    \R_r( \mN^{ 1, \theta }_{ r, i }( x ) )
    -
    \R_{ \infty }( \mN^{ 1, \theta }_{ \infty, i }( x ) )
  } 
  =
  \abs{
    \R_r( \mN^{ 1, \theta }_{ \infty, i }( x ) )
    -
    \R_{ \infty }( \mN^{ 1, \theta }_{ \infty, i }( x ) )
  } 
  \leq
  \B r^{ - 1 }
  .
\end{equation}
Combining this with \cref{eq:eK_def} proves for all $ r \in [1,\infty) $ that
\begin{equation}
\label{eq:e1_estimate}
  e_{ 1, r }
  \leq 
  \B r^{ - 1 }
  .
\end{equation}
\Moreover \cref{it:N_long2} of \cref{lem:properties1} 
and \cref{eq:Lipschitz_ReLU_approx}
ensure that for all
$ r \in [1, \infty) $, 
$ K \in  \N \cap [1 , L ) $, 
$ i \in \{ 1, \ldots, \ell_{ K + 1 } \} $,
$ x \in \Reals^{ \ell_0 } $ 
it holds that
\begin{equation}
\label{R_mvt}
\begin{split}
  &
  \bigl|
    \R_{ r^{ 1 / (K+1) } 
    }( 
      \mN^{ K + 1, \theta }_{ r, i }( x ) 
    ) 
    -
    \R_{ r^{ 1 / (K+1) } }( 
      \mN^{ K + 1, \theta }_{ \infty, i }( x )
    )
  \bigr|
\leq 
  \mathbf{L}
  | 
    \mN^{ K + 1, \theta }_{ r, i }( x ) -
    \mN^{ K + 1, \theta }_{ \infty, i }( x ) 
  |
\\
&= 
\textstyle
 \mathbf{L}
  \bigl|
    \sum_{ j = 1 }^{ \ell_K } 
    \fw^{ K + 1, \theta }_{ i, j }
    \bigl( 
      \R_{ r^{ 1 / K } }( 
        \mN^{K, \theta }_{ r, j }( x ) 
      )
      -
      \R_{ \infty }( \mN^{ K, \theta }_{ \infty, j }( x ) )
    \bigr)
  \bigr|
\\
& \leq 
\textstyle
  \mathbf{L}
  \bigl[
    \sum_{ j = 1 }^{ \ell_K } 
    |
      \fw^{ K + 1, \theta }_{ i, j }
    |
    |
      \R_{ r^{ 1 / K } }( 
        \mN^{K, \theta }_{ r, j }( x ) 
      )
      -
      \R_{ \infty }( \mN^{ K, \theta }_{ \infty, j }( x ) )
    |
  \bigr]
\\
& \leq 
\textstyle
  \mathbf{L}
  \bigl[
    \sum_{ j = 1 }^{ \ell_K } 
    |
      \fw^{ K + 1, \theta }_{ i, j }
    |
  \bigr]
  \bigl[
    \max_{ 
      j \in \{ 1, 2, \dots, \ell_K \}
    }
    |
      \R_{ r^{ 1 / K } }( 
        \mN^{ K, \theta }_{ r, j }( x ) 
      )
      -
      \R_{ \infty }( \mN^{ K, \theta }_{ \infty, j }( x ) )
    |
  \bigr]
\\
& \leq 
\textstyle
  \mathbf{L}
  \bigl[
    \sum_{ j = 1 }^{ \fd } 
    | \theta_j |
  \bigr]
  \bigl[
    \max_{ 
      j \in \{ 1, 2, \dots, \ell_K \}
    }
    |
      \R_{ r^{ 1 / K } }( 
        \mN^{ K, \theta }_{ r, j }( x ) 
      )
      -
      \R_{ \infty }( \mN^{ K, \theta }_{ \infty, j }( x ) )
    |
  \bigr]
=
  \mathbf{L}
  \bigl[
    \sum_{ j = 1 }^{ \fd } 
    | \theta_j |
  \bigr]
  e_{ K, r }
  .
\end{split}
\end{equation}
Combining this with \cref{eq:rec_approximation_quantitative}  
ensures that for all
$ r \in [1,\infty) $, 
$ K \in \N \cap [1 , L ) $, 
$ i \in \{ 1, \ldots, \ell_{K+1} \} $,
$ x \in \Reals^{ \ell_0 } $
it holds that 
\begin{equation}
\label{induction}
\begin{split}
&
  \bigl|
    \R_{ r^{ 1 / (K+1) } }( 
      \mN^{ K + 1, \theta }_{ r, i }( x ) 
    )
    -
    \R_{ \infty }( 
      \mN^{ K + 1, \theta }_{ \infty, i }( x ) 
    )
  \bigr|
\\
&
  \leq 
    \big|\R_{r^{ 1/(K+1)} }( \mN^{K+1, \theta }_{r,i}(x))-\R_{r^{ 1/(K+1)} }( \mN^{K+1, \theta }_{ \infty, i}(x))\big|
\\
&
\quad
  +
    \bigl|
      \R_{ r^{ 1 / (K+1) } }( 
        \mN^{ K + 1, \theta }_{ \infty, i }( x ) 
      )
      -
      \R_{ \infty }( 
        \mN^{ K + 1, \theta }_{ \infty, i }( x ) 
      )
    \bigr| 
\\
&
\textstyle
  \leq
  \mathbf{L}
  \bigl[
    \sum_{ j = 1 }^{ \fd } 
    | \theta_j |
  \bigr]
  e_{ K, r }
  + 
  \B r^{ - 1 / (K + 1) } 
  =
  e_{ K, r } c
  + 
  \B r^{ - 1 / (K + 1) } 
  .
\end{split}
\end{equation}
\Hence for all $ r \in [1,\infty) $, $ K \in \N \cap [1 , L ) $ that 
$
  e_{ K + 1, r }
  \leq
  e_{ K, r } c 
  + 
  \B r^{ - 1 / (K + 1) } 
$. 
This shows for all $ r \in [1,\infty) $, $ K \in \N \cap ( 1 , L ] $ 
that 
\begin{equation}
\textstyle
\begin{split}
  e_{ K, r }
& 
\textstyle
  \leq
  c 
  \bigl[ e_{ K - 1, r } \bigr]
  + 
  \B r^{ - 1 / K }
\leq 
  c
  \bigl[
    e_{ K - 2, r } c 
    +
    \B r^{ - 1 / ( K - 1 ) }
  \bigr]
  + 
  \B r^{ - 1 / K }
\\ 
& 
\textstyle
=
  c^2
  e_{ K - 2, r }
  +
  \sum_{ j = 0 }^1
  \B
  c^j
  r^{ - 1 / ( K - j ) }
\leq
  \ldots
\leq
  c^{ K - 1 }
  e_{ K - ( K - 1 ), r }
  +
  \sum_{ j = 0 }^{ K - 2 }
  \B
  c^j
  r^{ - 1 / ( K - j ) }
\\
&
\textstyle 
  =
  c^{ K - 1 }
  e_{ 1, r }
  +
  \sum_{ j = 0 }^{ K - 2 }
  \B
  c^j
  r^{ - 1 / ( K - j ) }
  .
\end{split}
\end{equation}
Combining this 
with \cref{eq:e1_estimate}
demonstrates 
for all $ r \in [1,\infty) $, $ K \in \N \cap [ 1 , L ] $ that 
\begin{equation}
\begin{split}
\textstyle
  e_{ K, r }
& 
\textstyle
  \leq 
  c^{ K - 1 }
  e_{ 1, r }
  +
  \sum_{ j = 0 }^{ K - 2 }
  \B
  c^j
  r^{ - 1 / ( K - j ) }
\leq 
  c^{ K - 1 }
  \B r^{ - 1 }
  +
  \sum_{ j = 0 }^{ K - 2 }
  \B
  c^j
  r^{ - 1 / ( K - j ) }
\\ &
\textstyle
=
  \sum_{ j = 0 }^{ K - 1 }
  \B
  c^j
  r^{ - 1 / ( K - j ) }
\leq 
  \bigl[
    \sum_{ j = 0 }^{ K - 1 }
    \B
    c^j
  \bigr]
  r^{ 
    \max\{ 
      - 1 / K
      ,
      - 1 / ( K - 1 )
      ,
      \dots 
      ,
      - 1
    \}
  }
=
  \bigl[
    \sum_{ j = 0 }^{ K - 1 }
    c^j
  \bigr]
  \B
  r^{ 
    - 1 / K
  }
  .
\end{split}
\end{equation}
This establishes \cref{it:lim_aux0}. 
\Nobs that \cref{wb,N,eq:eK_def} ensure that 
for all 
$ r \in [1,\infty) $, 
$ K \in \N \cap ( 1 , L ] $, 
$ i \in \{ 1, 2, \dots, \ell_K \} $, 
$ x \in \Rr^{ \ell_0 } $
it holds that 
\begin{equation}
\begin{split}
&
\textstyle
  |
    \mN^{ K, \theta }_{ r, i }( x ) 
    -
    \mN^{ K, \theta }_{ \infty, i }( x ) 
  |
=
  \bigl|
    \sum_{ j = 1 }^{ \ell_{ K - 1 } }
    \fw^{ K, \theta }_{ i, j }
    \bigl(  
      \R_{ r^{ 1 / (K-1) } }(
        \mN^{ K - 1, \theta }_{ r, j }( x ) 
      )
      -
      \R_{ \infty }(
        \mN^{ K - 1, \theta }_{ \infty, j }( x ) 
      )
    \bigr)
  \bigr|
\\ & 
\textstyle
\leq
  \sum_{ j = 1 }^{ \ell_{ K - 1 } }
  \bigl(  
    \bigl|
      \fw^{ K, \theta }_{ i, j }
    \bigr|
    \bigl|
      \R_{ r^{ 1 / (K-1) } }(
        \mN^{ K - 1, \theta }_{ r, j }( x ) 
      )
      -
      \R_{ \infty }(
        \mN^{ K - 1, \theta }_{ \infty, j }( x ) 
      )
    \bigr|
  \bigr)
\\ & 
\textstyle
\leq
  \bigl(  
    \sum_{ j = 1 }^{ \ell_{ K - 1 } }
    \bigl|
      \fw^{ K, \theta }_{ i, j }
    \bigr|
  \bigr)
  \bigl(
    \max_{
      j \in \{ 1, 2, \dots, \ell_{ K - 1 } \} 
    }
    \bigl|
      \R_{ r^{ 1 / (K-1) } }(
        \mN^{ K - 1, \theta }_{ r, j }( x ) 
      )
      -
      \R_{ \infty }(
        \mN^{ K - 1, \theta }_{ \infty, j }( x ) 
      )
    \bigr|
  \bigr)
\\ & 
\textstyle
\leq
  \bigl(  
    \sum_{ j = 1 }^{ \fd }
    | \theta_j |
  \bigr)
  \,
  e_{ K - 1, r }
  .
\end{split}
\end{equation}
\Hence for all 
$ K \in \N \cap ( 1 , L ] $ 
that 
\begin{equation}
\begin{split}
&
\textstyle
  \sup_{ r \in [1,\infty) }
  \sup_{ x \in \Rr^{ \ell_0 } }
  \max_{ i \in \{ 1, 2, \dots, \ell_K \} }
  \bigl(
    r^{ 
      1 / ( \max\{ K - 1, 1 \} )
    }
    |
      \mN^{ K, \theta }_{ r, i }( x ) 
      -
      \mN^{ K, \theta }_{ \infty, i }( x ) 
    |
  \bigr)
\\ &
\textstyle
\leq
  \bigl(  
    \sum_{ j = 1 }^{ \fd }
    | \theta_j |
  \bigr)
  \bigl[
    \sup_{ r \in [1,\infty) }
    \bigl(
      r^{ 
        1 / ( \max\{ K - 1, 1 \} )
      }
      e_{ K - 1, r }
    \bigr)
  \bigr]
=
  \bigl(  
    \sum_{ j = 1 }^{ \fd }
    | \theta_j |
  \bigr)
  \bigl[
    \sup_{ r \in [1,\infty) }
    \bigl(
      r^{ 
        1 / ( K - 1 )
      }
      e_{ K - 1, r }
    \bigr)
  \bigr]
  .
\end{split}
\end{equation}
Combining this with \cref{N}, 
\cref{eq:eK_def}, 
and \cref{it:lim_aux0} 
establishes \cref{it:lim_aux0_B}. 
\Nobs that \cref{it:lim_aux0,it:lim_aux0_B} 
prove \cref{it:lim_aux1}. 
It thus remains to prove \cref{it:lim_aux2}. 
For this \nobs that 
\cref{it:lim_aux1} 
assures 
for all 
$ i \in \{ 1, \dots, \ell_k \} $, 
$ x \in \Rr^{ \ell_0 } $ 
with 
$
  \mN^{ k, \theta }_{ \infty, i }( x ) < 0
$
that there exists 
$ R \in [1,\infty) $ 
such that for all 
$ r \in [R,\infty) $
it holds that
$
  \mN^{ k, \theta }_{ r, i }( x ) < 0
$. 
Combining this with \cref{lim_R0} demonstrates 
for all 
$ i \in \{ 1, \dots, \ell_k \} $, 
$ x \in \Rr^{ \ell_0 } $ 
with 
$
  \mN^{ k, \theta }_{ \infty, i }( x ) < 0
$
that there exists 
$ R \in [1,\infty) $
such that for all 
$ r \in [R,\infty) $
it holds that 
\begin{equation}
\label{conv_case1}
  ( \R_{ r^{ 1 / k } } )'(
    \mN^{ k, \theta }_{ r, i }( x )
  ) 
  = 0 
  =
  \indicator{ (0, \infty) } \bigl( 
    \mN^{ k, \theta }_{ \infty, i }( x )
  \bigr)
  =
  \indicator{ \mX^{ k, \theta }_i }( x ) 
  .
\end{equation}
\Moreover 
\cref{it:lim_aux1} shows 
for all 
$ i \in \{ 1, \dots, \ell_k \} $, 
$ x \in \Rr^{ \ell_0 } $ 
with 
$
  \mN^{ k, \theta }_{ \infty, i }( x ) > 0
$
that there exists 
$ R \in [1,\infty) $ 
such that for all 
$ r \in [R,\infty) $
it holds that
$
  \mN^{ k, \theta }_{ r, i }( x ) > \B r^{ - 1 / k }
$. 
Combining this with \cref{lim_R0} demonstrates 
for all 
$ i \in \{ 1, \dots, \ell_k \} $, 
$ x \in \Rr^{ \ell_0 } $ 
with 
$
  \mN^{ k, \theta }_{ \infty, i }( x ) > 0
$
that there exists 
$ R \in [1,\infty) $
such that for all 
$ r \in [R,\infty) $
it holds that 
\begin{equation}
\label{conv_case2}
  ( \R_{ r^{ 1 / k } } )'(
    \mN^{ k, \theta }_{ r, i }( x )
  ) 
  = 1 
  =
  \indicator{ (0, \infty) } \bigl( 
    \mN^{ k, \theta }_{ \infty, i }( x )
  \bigr)
  =
  \indicator{ \mX^{ k, \theta }_i }( x ) 
  .
\end{equation}
\Moreover 
\cref{it:lim_aux0_B} 
assures that 
for all 
$ i \in \{ 1, 2, \dots, \ell_k \} $,
$ x \in \Rr^{ \ell_0 } $
it holds that 
\begin{equation}
\begin{split}
&
\textstyle
  \sup_{ r \in [1,\infty) }
  \bigl(
    r^{ 
      ( 1 + \indicator{ \{ 1 \} }( k ) ) / ( k - \indicator{ (1,\infty) }( k ) ) 
    }
    |
      \mN^{ k, \theta }_{ r, i }( x )
      -
      \mN^{ k, \theta }_{ \infty, i }( x )
    |
  \bigr)
\\ & 
=
\textstyle
  \sup_{ r \in [1,\infty) }
  \bigl(
    r^{ 
      ( 1 + \indicator{ \{ 1 \} }( k ) ) / ( \max\{ k - 1, 1 \} ) 
    }
    |
      \mN^{ k, \theta }_{ r, i }( x )
      -
      \mN^{ k, \theta }_{ \infty, i }( x )
    |
  \bigr)
  < \infty 
  .
\end{split}
\end{equation}
This and the fact that 
$
  -
  ( 1 + \indicator{ \{ 1 \} }( k ) ) / ( k - \indicator{ (1,\infty) }( k ) )
  <
  -
  1 / k
$
assure 
for all $ i \in \{ 1, 2, \dots, \ell_k \} $, 
$ x \in \Rr^{ \ell_0 } $
with 
$ \mN^{ k, \theta }_{ \infty, i }( x ) = 0 $ 
that there exist 
$ \mathfrak{C}, R \in [1,\infty) $ 
such that for all 
$ r \in [R,\infty) $
it holds that
\begin{equation}
\label{eq:conv_case3_abs_estimate}
\begin{split}
  \abs{ \mN^{ k, \theta }_{ r, i }( x ) }
  =
  \abs{ 
    \mN^{ k, \theta }_{ r, i }( x )
    -
    \mN^{ k, \theta }_{ \infty, i }( x )
  }
  \leq 
  \mathfrak{C} 
  \bigl[
    r^{ 
      - ( 1 + \indicator{ \{ 1 \} }( k ) ) 
      / ( k - \indicator{ (1,\infty) }( k ) ) 
    }
  \bigr]
  < \A r^{ - 1 / k }
  .
\end{split}
\end{equation}
\Moreover \cref{lim_R0} assures 
that for all 
$ r \in [1,\infty) $, 
$ y \in \Rr $
with 
$
  | y | < \A r^{ - 1 / k } 
$
it holds that 
$
  ( \R_{ r^{ - 1 / k } } )'( y )
  = 0
$. 
Combining this with \cref{eq:conv_case3_abs_estimate} 
proves for all 
$ i \in \{ 1, 2, \dots, \ell_k \} $, 
$ x \in \Rr^{ \ell_0 } $ 
with 
$ \mN^{ k, \theta }_{ \infty, i }( x ) = 0 $ 
that there exists 
$ R \in [1,\infty) $ 
such that for all $ r \in [ R, \infty) $ 
it holds that
\begin{equation}
\label{conv_case3}
\textstyle
  ( \R_{ r^{ - 1 / k } } )'(
    \mN^{ k, \theta }_{ r, i }( x )
  ) 
  = 
  0 
  =
  \indicator{ (0, \infty) }\bigl( 
    \mN^{ k, \theta }_{ \infty, i }( x )
  \bigr)
  =
  \indicator{ \mX^{ k, \theta }_i }( x ) 
  .
\end{equation}
Combining this, \cref{conv_case1}, and \cref{conv_case2} establishes \cref{it:lim_aux2}. 
The proof of \cref{lem:properties2} is thus complete.
\end{proof}

\begin{lemma}
\label{lem:N_der}
  Assume \cref{main_setting},
  for every $k \in \N_0$ let $\diml_k \in \N_0$ satisfy $\diml_k = \sum_{n=1}^k \ell_n ( \ell_{n-1} + 1 )$,
   and let
  $ x = (x_1, \ldots, x_{\ell_0 } ) \in [a,b]^{ \ell_0 } $, $ r\in [1 , \infty ) $.   
  Then
  \begin{enumerate}[label=(\roman*)]
  \item 
  \label{it:N_C1}
  it holds for all 
  $ k \in \{ 1, \ldots, L\} $,
  $ i \in \{ 1, \ldots, \ell_k \} $
  that
  $ \Reals^{ \fd }\ni\theta \mapsto \mN^{ k, \theta }_{r,i}(x) \in \Reals^{ \fd } $
  is differentiable,
  \item
  \label{it:N_der1}
  it holds for all 
  $K \in \cu{1, \ldots, L }$,
  $ k \in \{ 1, \ldots, K\} $,
  $ i \in \{ 1, \ldots, \ell_k \} $,
  $ j \in \{ 1, \ldots, \ell_{k-1} \} $,
  $ h \in \{ 1, \ldots, \ell_K \} $,
   $ \theta = ( \theta_1, \ldots, \theta_{ \fd }) \in \Reals^{ \fd } $
  that
  \begin{equation}
  \label{N_der1}
  \begin{multlined}
    \frac{ \partial }{ \partial\theta_{ (i-1) \ell_{k-1} + j + \diml_{k-1}} 
    }
    \bigl( 
      \mN_{r, h }^{K, \theta }(x)
    \bigr)
\\
    = 
    \sum_{
      \substack{v_k,v_{k+1}, \ldots,v_K \in \enne, \\ \forall w\in \enne\cap[k, K ]\colon v_w\leq\ell_w} 
    }
    \Bigl[
      \R_{r^{1/ ( \max \cu{k-1 , 1 } ) } } ( \mN^{ \max\{k-1,1\}, \theta }_{r,j}(x))\indicator{ (1, K ]}(k)+    
      x_j\indicator{\{1\} }(k)
    \Bigr]
\\
    \cdot
    \Bigl[ \indicator{ \{ i \} }( v_k ) \Bigr]
    \Bigl[\indicator{ \{ h \} }( v_K ) \Bigr]
    \Bigl[
      \textstyle{\prod}_{n=k+1}^K
      \big( \fw^{n, \theta }_{v_n,v_{n-1} } 
      \big[( \R_{r^{1/(n-1)}})'( \mN^{n-1, \theta }_{r,v_{n -1} }(x))\big]
      \big)
    \Bigr]
    ,
  \end{multlined}  
  \end{equation}
  and
  \item
  \label{it:N_der2}
  it holds for all 
  $K \in \cu{1, \ldots, L }$,
  $ k \in \{ 1, \ldots, K\} $,
  $ i \in \{ 1, \ldots, \ell_k \} $,
  $ h \in \{ 1, \ldots, \ell_K \} $,
   $ \theta = ( \theta_1, \ldots, \theta_{ \fd }) \in \Reals^{ \fd } $
  that
    \begin{equation}
  \label{N_der2}
  \begin{split}
  &
    \frac{ \partial }{ \partial \theta_{ \ell_k \ell_{ k - 1 } + i + \diml_{ k - 1 } } 
    }
    \bigl( 
      \mN_{ r, h }^{ K, \theta }( x ) 
    \bigr)
\\
    &= \sum_{\substack{v_k,v_{k+1}, \ldots,v_K \in \enne, \\ \forall w\in \enne\cap[k, K ]\colon v_w\leq\ell_w} }
    \Big[\indicator{\{ i \} }(v_k)\Big]
    \Big[\indicator{\{ h \} }(v_K)\Big]
    \Big[\textstyle{\prod}_{n=k+1}^K
    \big( \fw^{n, \theta }_{v_n,v_{n-1} } 
    \big[( \R_{r^{1/(n-1)}})'( \mN^{n-1, \theta }_{r,v_{n-1} }(x))\big]
    \big)\Big].
  \end{split}  
  \end{equation}
  \end{enumerate}
\end{lemma}
\begin{proof}[Proof of \cref{lem:N_der}]
\Nobs that \cref{it:N_long1,it:N_long2} of \cref{lem:properties1}
and the assumption that $ \R_r \in C^1( \Reals, \Reals) $
establish \cref{it:N_C1}. 
We now prove \cref{N_der1} and \cref{N_der2}
by induction on $ K \in \{ 1, \dots, L \} $. 
\Nobs that \cref{it:N_long1}
of \cref{lem:properties1}
implies that
for all 
$ i \in \{ 1, \ldots, \ell_1 \} $, 
$ j \in \{ 1, \ldots, \ell_0 \} $,
$ h \in \{ 1, \ldots, \ell_1 \} $,
$ \theta = ( \theta_1, \ldots, \theta_{ \fd }) \in \Reals^{ \fd } $
it holds that
\begin{equation}
    \frac{ \partial }{ \partial \theta_{ (i-1) \ell_0 + j } } 
    \bigl( 
      \mN_{ r, h }^{ 1, \theta }( x ) 
    \bigr)
    = x_j \indicator{ \{ h \} }( i )
  \qqandqq
    \frac{ \partial }{ \partial \theta_{ \ell_0 \ell_1 + i } }
    \bigl( 
      \mN_{r, h }^{ 1, \theta }( x )
    \bigr)
    = \indicator{ \{ h \} }( i )
    .
\end{equation}
This establishes \cref{N_der1} and \cref{N_der2} 
in the base case $ K = 1 $.
For the induction step
let $ K \in \N \cap [ 1, L ) $ 
satisfy for all 
$ k \in \{ 1, \ldots, K \} $, 
$ i \in \{ 1, \ldots, \ell_k \} $, 
$ j \in \{ 1, \ldots, \ell_{k-1} \} $,
$ h \in \{ 1, \ldots, \ell_K\} $,
$ \theta = ( \theta_1, \ldots, \theta_{ \fd } ) \in \Reals^{ \fd } $
that
\begin{equation}
\label{N_der1_ind}
\begin{split}
&
  \frac{\partial}{\partial\theta_{ (i-1)\ell_{k-1} + j+ \diml_{k-1} } }
  \bigl( \mN_{ r, h }^{ K, \theta }( x ) \bigr)
\\
    &= \sum_{\substack{v_k,v_{k+1}, \ldots,v_K\in \enne, \\ \forall w\in \enne\cap[k,K]\colon v_w\leq\ell_w} }
    \Big[\R_{r ^{1/(\max \cu{k-1 , 1 } ) } } ( \mN^{ \max\{k-1,1\}, \theta }_{r,j}(x)) \indicator{ (1,K]}(k)
    +x_j \indicator{\{1\} }(k)
    \Big] \\
    &\quad \cdot
    \Big[\indicator{\{ i \} }(v_k)\Big]
    \Big[\indicator{\{ h \} }(v_K)\Big]
    \Big[\textstyle{\prod}_{n=k+1}^K
    \big( \fw^{n, \theta }_{v_n,v_{n-1} } 
    \big[( \R_{r^{1/(n-1)}})'( \mN^{n-1, \theta }_{r,v_{n-1} }(x))\big]
    \big)\Big]
\end{split} 
\end{equation}
and
\begin{equation}
\label{N_der2_ind}
\begin{split}
    &
    \frac{ \partial }{ \partial\theta_{ \ell_k \ell_{ k - 1 } + i + \diml_{ k - 1 } } }
    \bigl( 
      \mN_{ r, h }^{ K, \theta }( x ) 
    \bigr)
  \\
    &= \sum_{\substack{v_k,v_{k+1}, \ldots,v_K\in \enne, \\ \forall w\in \enne\cap[k,K]\colon v_w\leq\ell_w)} }
    \Big[\indicator{\{ i \} }(v_k)\Big]
    \Big[\indicator{\{ h \} }(v_K)\Big]
    \Big[\textstyle{\prod}_{n=k+1}^K
    \big( \fw^{n, \theta }_{v_n,v_{n-1} } 
    \big[( \R_{r^{1/(n-1)}})'( \mN^{n-1, \theta }_{r,v_{n-1} }(x))\big]
    \big)\Big]
  .
\end{split}  
\end{equation}
\Nobs that \cref{it:N_long2} of \cref{lem:properties1}
and \cref{N_der1_ind}
demonstrate that 
for all
$ k \in \{ 1, \ldots, K \} $,
$ i \in \{ 1, \ldots, \ell_k \} $, $ j \in \{ 1, \ldots, \ell_{k-1} \} $,
$ h \in \{ 1, \ldots, \ell_{K+1} \} $,
$ \theta = ( \theta_1, \ldots, \theta_{ \fd }) \in \Reals^{ \fd } $
it holds that
\begin{equation}
\begin{split}
    &\frac{\partial}{\partial\theta_{ (i-1)\ell_{k-1} + j+ \diml_{k-1} } }
    \bigl( \mN_{r, h }^{K+1, \theta }(x) \bigr)
  \\
    &
    =
    \frac{ 
      \partial 
    }{ 
      \partial\theta_{ (i-1) \ell_{ k - 1 } + j + \diml_{ k - 1 } } 
    }
    \biggl( 
      \fb^{ K + 1, \theta }_h 
      + 
      \sum_{ \mathscr{i} = 1 }^{ \ell_K } 
      \fw^{ K+1, \theta }_{ h, \mathscr{i} } 
      \R_{ r^{ 1 / K } }( 
        \mN^{K, \theta }_{ r, \mathscr{i} }( x )
      )
    \biggr)
\\
&
  =
  \sum_{ \mathscr{i} = 1 }^{ \ell_{K} }
  \Biggl[
    \fw^{ K + 1, \theta }_{ h, \mathscr{i} }
    \bigl[
      ( 
        \R_{ r^{ 1 / K } } )'( \mN^{ K, \theta }_{ r, \mathscr{i} }( x ) 
      )
    \bigr]
    \biggl( 
      \frac{ \partial }{ \partial\theta_{ (i-1) \ell_{ k - 1 } + j + \diml_{k-1} } 
      }
      \bigl( 
        \mN_{ r, \mathscr{i} }^{ K, \theta }( x ) 
      \bigr)
    \biggr)
  \Biggr]
\\  
&
  = 
  \sum_{ \mathscr{i} = 1 }^{ \ell_K }
  \Biggl[
    \fw^{ K+1, \theta }_{ h, \mathscr{i} }
    \bigl[
      ( \R_{ r^{ 1 / K } } )'( \mN^{ K, \theta }_{ r, \mathscr{i} }( x ) )
    \bigr]
\\
&
  \quad\cdot
  \sum_{
    \substack{ v_k, v_{k+1}, \ldots, v_K \in \enne, 
    \\ 
    \forall w \in \enne\cap[k,K]\colon v_w\leq\ell_w} }
    \Big[\R_{r^{1/(\max \cu{k-1 , 1 } ) } }( \mN^{ \max\{k-1,1\}, \theta }_{r, j }(x))\indicator{ (1,K]}(k)
    +x_{j} \indicator{\{1\} }(k)
    \Big] \\
    &\quad\cdot
    \Bigl[ \indicator{ \{ i \} }( v_k ) \Bigr]
    \Bigl[ \indicator{ \{ \mathscr{i} \} }( v_K ) \Bigr]
    \Big[\textstyle{\prod}_{n=k+1}^K
    \big( \fw^{n, \theta }_{v_n,v_{n-1} } 
    \big[( \R_{r^{1/ ( n - 1 ) }} )'( \mN^{n-1, \theta }_{r,v_{n-1} }(x))\big]
    \big)\Big] \Bigg] \\
    &=\sum_{\substack{v_k,v_{k+1}, \ldots,v_{K+1} \in \enne, \\\forall w\in \enne\cap[k,K+1]\colon v_w\leq\ell_w} }
    \Big[\R_{r^{1/(\max \cu{k-1 , 1 } ) } }( \mN^{ \max\{k-1,1\}, \theta }_{r,j}(x)) \indicator{ (1,K]}(k)
    +x_j \indicator{\{1\} }(k)
    \Big]    
    \\
    &\quad\cdot
    \Big[\indicator{\{ i \} }(v_k )\Big]
    \Big[\indicator{\{ h  \} }(v_{K+1} )\Big]    
    \Big[\textstyle{\prod}_{n = k + 1}^{K+1}
    \big( \fw^{n, \theta }_{v_ n,v_{n-1} } \big[( \R_{r^{1/ ( n - 1 ) } } ) ' ( \mN^{n-1, \theta }_{r,v_{n-1} }(x))\big]
    \big)\Big]\bigg]
  .
\end{split}
\end{equation}
\Moreover \cref{it:N_long2} of \cref{lem:properties1} 
ensures that for all 
$ i \in \{ 1, \dots, \ell_{K+1} \} $, 
$ j \in \{ 1, \dots, \ell_K \} $,
$ h  \in \{ 1, \dots, \ell_{K+1} \} $, 
$ \theta = ( \theta_1, \dots, \theta_{ \fd } ) \in \Reals^{ \fd } $
it holds that
\begin{equation}
\begin{split}
    \frac{\partial}{\partial\theta_{ (i-1)\ell_{K} + j + \diml_K } }( \mN_{r,  h  }^{K+1, \theta }(x))
    &= \frac{\partial}{\partial\theta_{ (i-1)\ell_{K} + j + \diml_K } }
    \biggl( 
      \fb^{ K + 1, \theta }_h 
      + 
      \sum_{ \mathscr{i} = 1 }^{ \ell_K } 
      \fw^{ K+1, \theta }_{ h, \mathscr{i} } 
      \R_{ r^{ 1 / K } }( \mN^{ K, \theta }_{ r, \mathscr{i} }( x ) )
    \biggr)
\\
&
  =
  \R_{ r^{ 1 / K } }( \mN^{ K, \theta }_{ r, j }( x ) ) 
  \indicator{ \{ h \} }( i ) .
\end{split}
\end{equation}
\Moreover \cref{it:N_long2} of \cref{lem:properties1} and \cref{N_der2_ind}
demonstrate that 
for all
$ k \in \{ 1, \ldots,K\} $, 
$ i \in \{ 1, \ldots, \ell_k \} $, 
$  h  \in \{ 1, \ldots, \ell_{K+1} \} $,
$ \theta = ( \theta_1, \ldots, \theta_{ \fd } ) \in \Reals^{ \fd } $
it holds that
\begin{equation}
\begin{split}
&
  \frac{ \partial }{ \partial\theta_{ \ell_k \ell_{ k - 1 } + i + \diml_{ k - 1 } } 
  }\bigl( 
    \mN_{ r , h }^{ K + 1, \theta }( x ) 
  \bigr)
  = 
  \frac{ \partial }{ \partial\theta_{ \ell_k \ell_{ k - 1 } + i + \diml_{ k - 1 } } }
  \biggl( 
    \fb^{ K + 1, \theta }_h 
    + 
    \sum_{ 
      \mathscr{i} = 1 
    }^{ 
      \ell_K 
    } 
    \fw^{ K + 1, \theta }_{ h, \mathscr{i} } 
    \R_{ r^{ 1 / K } }( \mN^{ K, \theta }_{ r, \mathscr{i} }( x ) )
  \biggr)
\\
    &=\sum_{ \mathscr{i} = 1 }^{ \ell_{K} }
    \fw^{K+1, \theta }_{ h, \mathscr{i} }
    \big[( \R_{r^{1/K}} )'( \mN^{K, \theta }_{ r, \mathscr{i} }(x))\big]
    \bigg( \frac{\partial}{\partial\theta_{\ell_k\ell_{k-1} + i +\diml_{k-1} } }( \mN_{ r, \mathscr{i} }^{K, \theta }(x))\bigg)\\  
    &= \sum_{ i_{K} = 1 }^{ \ell_{K} }
    \Biggl[ 
      \fw^{K+1, \theta }_{ h, \mathscr{i} }
      \bigl[
        ( \R_{ r^{ 1 / K } } )'( \mN^{ K, \theta }_{ r, \mathscr{i} }( x ) ) 
      \bigr]
    \sum_{
      \substack{ v_k, v_{k+1}, \dots, v_K \in \enne, 
      \\
      \forall w \in \enne \cap [k,K] \colon v_w \leq \ell_w } 
    }
    \Bigl[
      \indicator{ \{ i \} }( v_k ) 
    \Bigr]
    \Bigl[
      \indicator{ \{ \mathscr{i} \} }( v_K ) 
    \Bigr]
\\
    &\quad\cdot \Big[\textstyle{\prod}_{n=k+1}^K
    \big( \fw^{n, \theta }_{v_n,v_{n-1} } \big[( \R_{r^{1/(n - 1 ) } } )'( \mN^{n-1, \theta }_{r,v_{n-1} }(x))\big]
    \big)\Big]
    \Bigg] \\
    &=
    \sum_{\substack{v_k,v_{k+1}, \ldots,v_{K+1} \in \enne, \\\forall w\in \enne\cap[k,K+1]\colon v_w\leq\ell_w} }
    \Big[\indicator{\{ i \} }(v_k )\Big]
    \Big[\indicator{\{ h  \} }(v_{K+1} )\Big]
    \Big[\textstyle{\prod}_{n=k+1}^{K+1}
    \fw^{n, \theta }_{v_n,v_{n-1} } 
    \big[( \R_{r^{1 / ( n - 1 ) } } )'( \mN^{n-1, \theta }_{r,v_{n-1} }(x))\big]
    \Big] .
\end{split}
\end{equation}
\Moreover \cref{it:N_long2} of \cref{lem:properties1} shows for all
$ i \in \{ 1, \dots, \ell_{K+1} \} $, 
$ h \in \{ 1, \dots, \ell_{K+1} \} $,
$ \theta = ( \theta_1, \ldots, \theta_{ \fd } ) \in \Reals^{ \fd } $
that
\begin{equation}
\begin{split}
  \frac{ \partial }{ 
    \partial \theta_{ \ell_{ K + 1 } \ell_K + i + \diml_K } 
  }( 
    \mN_{ r, h }^{ K + 1, \theta }( x ) 
  )
&
  = \frac{\partial}{\partial\theta_{\ell_{K+1} \ell_{K} + i+\diml_K } }
    \bigg( \fb^{K+1, \theta }_ h  + \sum_{ \mathscr{i} = 1 }^{ \ell_K} \fw^{K+1, \theta }_{ h, \mathscr{i} } \R_{r^{1 / K } }( \mN^{K, \theta }_{ r, \mathscr{i} }(x))\bigg)\\
&
  = \indicator{ \{ h \} }( i ) .
\end{split}
\end{equation}  
Induction thus establishes \cref{N_der1,N_der2}. 
The proof of \cref{lem:N_der} is thus complete.
\end{proof}

\subsection{Explicit representations for the generalized gradients of the risk function}
\label{ssec:gen_grad_risk_functions}

\begin{lemma} \label{lem:der:unif:bounded}
Assume \cref{main_setting}
and let $ K \subseteq \Rr^{ \fd } $ be compact.
Then
\begin{enumerate} [(i)]
\item 
\label{lem:unif:bound:item1} 
it holds for all $ x \in [0, \infty) $ that
\begin{equation} 
  \sup\nolimits_{ r \in [ 1, \infty ) }
  \sup\nolimits_{ y \in [ -x, x ] }
  \bigl(  
    \abs{ \R_r ( y ) } + \abs{ ( \R_r ) ' ( y ) } 
  \bigr) < \infty ,
\end{equation}
\item 
\label{lem:unif:bound:item2} 
it holds for all $ k \in \cu{ 1, \dots, L } $ that
\begin{equation} 
\label{lem:unif:bound:item2:eq}
  \sup\nolimits_{ \theta \in K }
  \sup\nolimits_{ r \in [ 1, \infty) }
  \sup\nolimits_{ i \in \{ 1, \dots, \ell_k \} }
  \sup\nolimits_{ x \in [a,b]^{ \ell_0 } }
  \abs{ \mN^{ k, \theta }_{ r, i }( x ) }
  < \infty ,
\end{equation} 
\item 
\label{lem:unif:bound:item3} 
it holds for all $ k \in \cu{ 1, \dots, L } $ that
\begin{equation}
  \sup\nolimits_{ \theta \in K }
  \sup\nolimits_{ r, s, t \in [ 1, \infty ) }
  \sup\nolimits_{ i \in \{ 1, \dots, \ell_k \} }  
  \sup\nolimits_{ x \in [a,b]^{ \ell_0 } }
  \pr[\big]{ 
      \abs{ \R_s( \mN^{ k, \theta }_{ r, i }( x ) ) 
    }
    +
    \abs{
      ( \R_t )'( \mN^{ k, \theta }_{ r, i }( x ) ) 
    } 
  } < \infty ,
\end{equation}
and
\item 
\label{lem:unif:bound:item4} 
it holds that
\begin{equation}
  \sup\nolimits_{ \theta \in K }
  \sup\nolimits_{ r \in [1 , \infty) }
  \sup\nolimits_{ i \in \{ 1, \ldots, \ell_L \} }
  \sup\nolimits_{j \in \cu{1, \ldots, \fd } }
  \sup\nolimits_{ x \in [a,b]^{ \ell_0 } }
  \bigl| 
    \tfrac{ \partial }{ \partial \theta_j } 
    \pr[\big]{ \mN^{ L, \theta }_{ r, i }( x ) } 
  \bigr| 
  < \infty .
\end{equation} 
\end{enumerate}
\end{lemma}
\begin{proof}[Proof of \cref{lem:der:unif:bounded}]
\Nobs that the fundamental theorem of 
calculus and the assumption that for all $ r \in [ 1, \infty ) $ it holds that
$ \R_r \in C^1( \Reals, \Reals ) $ ensure that for all
$ r \in [ 1, \infty ) $, $ x \in \Reals $
we have that
$ 
  \R_r(x) 
  = \R_r(0) + \int_0^x ( \R_r )'( y ) \, \d y 
  = \int_0^x( \R_r)'(y)\, \d y 
$.
Combining this with the assumption that
$ 
  \sup_{ r \in [ 1, \infty ) } \sup_{ y \in \Rr } 
  \abs{ ( \R_r )'( y ) } < \infty 
$ 
shows that for all $ x \in [0, \infty) $
it holds that 
$ 
  \sup_{ r \in [ 1, \infty) } 
  \sup_{ y \in [-x,x] } 
  \abs{ \R_r( y ) } < \infty 
$.
This establishes \cref{lem:unif:bound:item1}.
We now prove \cref{lem:unif:bound:item2:eq} 
by induction on $ k \in \cu{ 1, \ldots, L } $.
\Nobs that the assumption that $ K $ is compact
ensures that there exists $ \fC \in \Rr $
which satisfies that 
\begin{equation}
\label{eq:sup_smaller_than_fC}
\textstyle 
  \sup_{ \theta \in K } \norm{ \theta } < \fC 
  .
\end{equation}
\Nobs that \cref{eq:sup_smaller_than_fC} 
and \cref{it:N_long1} of \cref{lem:properties1}
demonstrate for all $ \theta \in K $, $ r\in [1 , \infty ) $,
$ i \in \{ 1, \dots, \ell_1 \} $, $ x \in [a,b]^{ \ell_0 } $
that
\begin{equation}
\textstyle
  \abs{ \mN^{ 1, \theta }_{ r, i }( x ) }
  \leq 
  \abs{ \fb^{ 1, \theta }_i } 
  + 
  \sum_{ j = 1 }^{ \ell_0 } 
  \abs{ \fw^{ 1, \theta }_{ i, j } }
  \abs{ x_j }
  \leq 
  \mathbf{a} ( \ell_0 + 1 ) \fC 
  . 
\end{equation}
This shows \cref{lem:unif:bound:item2:eq} in the base case $ k = 1 $.
For the induction step let $ k \in \N \cap [1, L) $ satisfy that
\begin{equation}
\label{lem:unif:bound:item2:eq:ind}
  \sup\nolimits_{ \theta \in K }
  \sup\nolimits_{ r \in [ 1, \infty) }
  \sup\nolimits_{ i \in \{ 1, \dots, \ell_k \} }
  \sup\nolimits_{ x \in [a,b]^{ \ell_0 } }
  \abs{ \mN^{ k, \theta }_{ r, i }( x ) }
  < \infty .
\end{equation}
\Nobs that 
\cref{eq:sup_smaller_than_fC}
and 
\cref{it:N_long2} of \cref{lem:properties1}
imply for all $ \theta \in K$,
$ r \in [ 1, \infty) $,
$ i \in \{ 1, \ldots, \ell_{ k + 1 } \} $,
$ x \in [a,b]^{ \ell_0 } $
that
\begin{equation}
\textstyle
  \abs{ \mN^{ k+1, \theta }_{ r, i }( x ) }
  \leq 
  \abs{ \fb^{ k+1, \theta }_i }
  +
  \sum_{ j = 1 }^{ \ell_k } 
  \abs{ \fw^{ k+1, \theta }_{ i, j } } 
  \abs{ \R_r( \mN^{ k, \theta }_{ r, j }( x ) ) }
  \leq  
  \fC 
  \bigl(
    1 
    + 
    \ell_k 
    \max_{ j \in \cu{ 1, \dots, \ell_k } } 
    \abs{ 
      \R_r( \mN^{ k, \theta }_{ r, j }( x ) ) 
    } 
  \bigr)
  .
\end{equation}
This, \cref{lem:unif:bound:item1}, and \cref{lem:unif:bound:item2:eq:ind}
demonstrate that 
\begin{equation}
  \sup\nolimits_{ \theta \in K }
  \sup\nolimits_{ r \in [ 1, \infty) }
  \sup\nolimits_{ i \in \{ 1, \ldots, \ell_{k+1} \} }
  \sup\nolimits_{x\in [a,b]^{ \ell_0 } }
  \abs{ \mN^{ k+1, \theta }_{ r, i }( x ) }
  < \infty .
\end{equation}
Induction thus establishes \cref{lem:unif:bound:item2:eq}. 
This completes the proof of \cref{lem:unif:bound:item2}.
\Nobs that \cref{lem:unif:bound:item1,lem:unif:bound:item2}
prove \cref{lem:unif:bound:item3}.
\Nobs that \cref{lem:unif:bound:item3}
and \cref{lem:N_der} establish \cref{lem:unif:bound:item4}.
The proof of \cref{lem:der:unif:bounded} is thus complete.
\end{proof}

\begin{lemma}[Integrability properties of the target function]
\label{lem:f:integrable}
Assume \cref{main_setting}. Then
\begin{equation}
\label{lem:f:integrable:eqclaim}
  \int_{ [a,b]^{ \ell_0 } } \norm{ f(x) }^2 \, \mu( \d x ) 
  + 
  \int_{ [a,b]^{ \ell_0 } } \norm{ f(x) } \, \mu( \d x ) < \infty 
  .
\end{equation}
\end{lemma}
\begin{proof}[Proof of \cref{lem:f:integrable}]
\Nobs that \cref{L_r}, 
the fact that $ \mL_{ \infty }( 0 ) \in \Rr $,
and the fact that for all $ x \in \Rr^{ \ell_0 } $
it holds that
$ \mN^{ L, 0 }_{ \infty }( x ) = 0 $
assure that
\begin{equation}
\label{lem:f:integrable:eq1}
\textstyle 
  \int_{ [a,b]^{ \ell_0 } } 
  \norm{ f(x) }^2 \, \mu( \d x) 
  = \mL_{ \infty }( 0 ) < \infty .
\end{equation}
H\"{o}lder's inequality 
and the fact that $ \mathfrak{m} = \mu( [ a, b ]^{ \ell_0 } ) \in \Rr $
hence show that
\begin{equation}
\begin{split}
\textstyle 
  \int_{ [a,b]^{ \ell_0 } } \norm{ f(x) } \, \mu( \d x )
& 
\textstyle 
  \leq 
  \left[ 
    \int_{ [a,b]^{ \ell_0 } } 1 \, \mu( \d x )
  \right]^{ \nicefrac{ 1 }{ 2 } }
  \left[ 
    \int_{ [a,b]^{ \ell_0 } } \norm{ f(x) }^2 \, \mu( \d x )
  \right]^{ \nicefrac{ 1 }{ 2 } }
\\ & 
\textstyle 
  =
  \sqrt{ \mathfrak{m} }
  \left[ 
    \int_{ [a,b]^{ \ell_0 } } \norm{ f(x) }^2 \, \mu( \d x )
  \right]^{ \nicefrac{ 1 }{ 2 } }
  < \infty 
  .
\end{split}
\end{equation}
	Combining this with \cref{lem:f:integrable:eq1} establishes \cref{lem:f:integrable:eqclaim}.
	The proof of \cref{lem:f:integrable} is thus complete.
\end{proof}

\begin{theorem}
\label{prop:G}
Assume \cref{main_setting}, 
for every $ k \in \enne_0 $ 
let $ \diml_k \in \enne_0 $ satisfy 
$ \diml_k = \sum_{ n = 1 }^k \ell_n ( \ell_{ n - 1 } + 1 ) $,
and let $ \theta = ( \theta_1, \ldots, \theta_{ \fd } ) \in \Rr^{ \fd } $. 
Then
\begin{enumerate}[label=(\roman*)]
\item\label{G1}
it holds for all $ r \in [ 1, \infty ) $ 
that $ \mL_r \in C^1( \Reals^{ \fd }, \Reals ) $
\item\label{G2'}
it holds for all $ r \in [ 1, \infty ) $, $ k \in \{ 1, \ldots, L \} $, 
$ i \in \{ 1, \ldots, \ell_k \} $,
$ j \in \{ 1, \ldots, \ell_{k-1} \} $ that
\begin{equation}\label{der_w}
\begin{split}
& 
  \pr*{  \frac{\partial\mL_r}
   {\partial\theta_{ (i-1)\ell_{k-1} + j +\diml_{k-1} } }
    } ( \theta)
\\
&
  =
  \sum_{\substack{v_k,v_{k+1},  
  \ldots,v_L\in \enne, \\\forall w\in \enne
  \cap[k, L]\colon v_w\leq\ell_w} }
  \int_{[a,b]^{ \ell_0 } }2\,
  \Big[
  \R_{r^{1/ ( \max \cu{k - 1 , 1 } ) } } ( \mN^{ \max\{k-1,1\}, \theta }_{r,j}(x))
  \indicator{ (1, L]}(k)
  +x_j \indicator{\{1\} }(k)
  \Big]\\
  &\quad\cdot
  \Big[ \indicator{\{ i \} }(v_k)\Big]
  \Big[ \mN_{r,v_L}^{L, \theta }(x)-f_{v_L}(x)\Big] 
  \Big[
  \textstyle{\prod}_{n={k+1} }^{L} \big(
  \fw^{n, \theta }_{v_n, v_{n-1} }
  \big[( \R_{r^{1/(n - 1 ) } } ) ' ( \mN^{n-1, \theta }_{r,v_{n-1} }(x))\big]
  \big)
  \Big]\, \mu( \d x),
\end{split}
\end{equation}
\item\label{G2''}
   it holds for all $ r\in [1 , \infty)  $, $ k \in \{ 1, \ldots, L\} $, 
  $ i \in \{ 1, \ldots, \ell_k \} $ that
  \begin{equation}\label{der_b}
  \begin{split}
  & \pr*{  \frac{\partial\mL_r}
  {\partial\theta_{\ell_k\ell_{k-1} + i+ \diml_{k-1} } }
   } ( \theta)
  =
  \sum_{\substack{v_k,v_{k+1},  
  \ldots,v_L\in \enne, \\\forall w\in \enne
  \cap[k, L]\colon v_w\leq\ell_w} }
  \int_{[a,b]^{ \ell_0 } }2\,
  \Big[ \indicator{\{ i \} }(v_k)\Big]
  \\
  &\quad\cdot
  \Big[ \mN_{r,v_L}^{L, \theta }(x)-f_{v_L}(x)\Big]
  \Big[
  \textstyle{\prod}_{n={k+1} }^{L} \big(
  \fw^{n, \theta }_{v_n, v_{n-1} }
  \big[( \R_{r^{1 / ( n - 1 ) } } )'( \mN^{n-1, \theta }_{r,v_{n-1} }(x))\big]
  \big)
  \Big]\, \mu( \d x),
\end{split}
\end{equation}
\item \label{G3}
it holds that $ \limsup\nolimits_{r\to \infty }  \pr*{  \abs{ \mL_r( \theta)-\mL_{ \infty }( \theta) } +
 \|( \nabla\mL_r)( \theta)-\mG( \theta)\|  }  =0$,
 \item  \label{G4'}
 it holds for all $ k \in \{ 1, \ldots, L\} $,
  $ i \in \{ 1, \ldots, \ell_k \} $, 
  $ j \in \{ 1, \ldots, \ell_{k-1} \} $ that
  \begin{equation}\label{G_w}
  \begin{split}
  &\mG_{ (i-1)\ell_{k-1} + j + \diml_{k - 1 } }( \theta)\\
  &=
  \sum_{\substack{v_k,v_{k+1},  
  \ldots,v_L\in \enne, \\\forall w\in \enne
  \cap[k, L]\colon v_w\leq\ell_w} }
  \int_{[a,b]^{ \ell_0 } }2\,
  \Big[
  \R_{ \infty }( \mN^{ \max\{k-1,1\}, \theta }_{ \infty,j}(x))
  \indicator{ (1, L]}(k)
  +x_j \indicator{\{1\} }(k)
  \Big]\\
  &\quad\cdot
  \Big[ \indicator{\{ i \} }(v_k)\Big]
  \Big[ \mN_{ \infty,v_L}^{L, \theta }(x)-f_{v_L}(x)\Big] 
  \Big[
  \textstyle{\prod}_{n={k+1} }^{L} \big(
  \fw^{n, \theta }_{v_n, v_{n-1} }
  \indicator{\mX^{n-1, \theta }_{v_{n-1} }}(x)
  \big)
  \Big]\, \mu( \d x),
\end{split}
\end{equation}
and
 \item  \label{G4''}
 it holds for all $ k \in \{ 1, \ldots, L\} $,
  $ i \in \{ 1, \ldots, \ell_k \} $ that
  \begin{equation}\label{G_b}
  \begin{split}
    &\mG_{\ell_k\ell_{k-1} + i+ \diml_{k-1 } }( \theta)
    =\sum_{\substack{v_k,v_{k+1},  
  \ldots,v_L\in \enne, \\\forall w\in \enne
  \cap[k, L]\colon v_w\leq\ell_w} }
  \int_{[a,b]^{ \ell_0 } }2\,
  \Big[ \indicator{\{ i \} }(v_k)\Big]
   \\
  &\quad\cdot
  \Big[ \mN_{ \infty,v_L}^{L, \theta }(x)-f_{v_L}(x)\Big]
  \Big[
  \textstyle{\prod}_{n={k+1} }^{L} \big(
  \fw^{n, \theta }_{v_n, v_{n-1} }
  \indicator{\mX^{n-1, \theta }_{v_{n-1} }}(x)
  \big)
  \Big]\, \mu( \d x).
\end{split}
\end{equation}
\end{enumerate}
\end{theorem}
\begin{proof}[Proof of \cref{prop:G}]
\Nobs that \cref{lem:N_der}
and the chain rule show that for all 
$ r \in [1, \infty) $, $ i \in \cu{ 1, \dots, \fd } $,
$ x \in [a,b]^{ \ell_0 } $
it holds that 
$
  \Rr^{ \fd } \ni \vartheta 
  \mapsto \norm{ \mN_r^{ L, \vartheta }( x ) - f (x ) }^2 \in \Rr 
$
is differentiable at $ \theta $ and that
\begin{equation} 
\textstyle 
  \frac{ \partial }{ \partial \theta_i } 
  \| 
    \mN_r^{ L, \theta }( x ) - f( x ) 
  \|^2 
  = 
  2 
  \sum_{ j = 1 }^{ \ell_L } 
  \bigl(
    ( \mN_{ r, j }^{ L, \theta }( x ) - f_j( x ) ) 
    \tfrac{ \partial }{ \partial \theta_i } 
    \bigl[
      \mN_{ r, j }^{ L, \theta }( x ) 
    \bigr] 
  \bigr)
  .
\end{equation}
Combining this with \cref{lem:f:integrable},
\cref{lem:der:unif:bounded},
\cref{lem:N_der},
and the dominated convergence theorem establishes \cref{G1,G2',G2''}. 
\Nobs that the dominated convergence theorem 
and the fact that 
for all $ x \in [a,b]^{ \ell_0 } $ it holds that
$ 
  \limsup_{ r \to \infty }
  |
    ( \mN^{ L, \theta }_r(x) - f(x) )
    - 
    ( \mN^{ L, \theta }_{ \infty }( x ) - f(x) )
  |
  = 0
$
ensure that
\begin{equation}
\label{L_lim}
  \limsup\nolimits_{ r \to \infty } 
  |
    \mL_r( \theta )
    - \mL_{ \infty }( \theta ) 
  | = 0 .
\end{equation}
\Moreover \cref{it:lim_aux1,it:lim_aux2} of \cref{lem:properties2}
demonstrate that for all
$ k \in \{ 1, \ldots, L \} $,
$ i \in \{ 1, \ldots, \ell_k \} $, $ j \in \{ 1, \ldots, \ell_{k-1} \} $,
$ x \in [ a, b ]^{ \ell_0 } $
it holds that
\begin{equation}
\label{limit1}
\begin{split}
  & \lim_{r\to \infty } \Biggl( \sum_{\substack{v_k,v_{k+1},  
  \ldots,v_L\in \enne, \\\forall w\in \enne
  \cap[k, L]\colon v_w\leq\ell_w} }
  \Big[
  \R_{r^{1 / ( \max \cu{k - 1 , 1 } ) } } ( \mN^{ \max\{k-1,1\}, \theta }_{r,j}(x))
  \indicator{ (1, L]}(k)
  +x_j \indicator{\{1\} }(k)
  \Big]\\
  &\quad\cdot
  \Big[ \indicator{\{ i \} }(v_k)\Big]
  \Big[ \mN_{r,v_L}^{L, \theta }(x)-f_{v_L}(x)\Big] 
  \Big[
  \textstyle{\prod}_{n={k+1} }^{L} \big(
  \fw^{n, \theta }_{v_n, v_{n-1} }
  \big[( \R_{r^{1 / ( n - 1 ) } } ) ' ( \mN^{n-1, \theta }_{r,v_{n-1} }(x))\big]
  \big)
  \Big]\Biggr)\\
  &=\sum_{\substack{v_k,v_{k+1},  
  \ldots,v_L\in \enne, \\\forall w\in \enne
  \cap[k, L]\colon v_w\leq\ell_w} }
  \Big[
  \R_{ \infty }( \mN^{ \max\{k-1,1\}, \theta }_{ \infty,j}(x))
  \indicator{ (1, L]}(k)
  +x_j \indicator{\{1\} }(k)
  \Big]\\
  &\quad\cdot
  \Big[ \indicator{\{ i \} }(v_k)\Big]
  \Big[ \mN_{ \infty,v_L}^{L, \theta }(x)-f_{v_L}(x)\Big] 
  \Big[
  \textstyle{\prod}_{n={k+1} }^{L} \big(
  \fw^{n, \theta }_{v_n, v_{n-1} }
  \indicator{\mX^{n-1, \theta }_{n-1} }(x)
  \big)
  \Big]
\end{split}
\end{equation}
and
\begin{equation}
\label{limit2}
\begin{split}
  & \lim_{ r \to \infty } 
  \Biggl( \sum_{\substack{v_k,v_{k+1},  
  \ldots,v_L\in \enne, \\\forall w\in \enne
  \cap[k, L]\colon v_w\leq\ell_w} }
  \Big[ \indicator{\{ i \} }(v_k)\Big]
  \Big[ \mN_{r,v_L}^{L, \theta }(x)-f_{v_L}(x)\Big] \\
  &\quad\cdot
  \Big[
  \textstyle{\prod}_{n={k+1} }^{L} \big(
  \fw^{n, \theta }_{v_n, v_{n-1} }
  \big[( \R_{ r ^{1 / ( n - 1 ) } } )'
  ( \mN^{n-1, \theta }_{r,v_{n-1} }(x))\big]
  \big)
  \Big]\Biggr)\\
  &=\sum_{\substack{v_k,v_{k+1},  
  \ldots,v_L\in \enne, \\\forall w\in \enne
  \cap[k, L]\colon v_w\leq\ell_w} }
  \Big[ \indicator{ \{ i \} }(v_k)\Big]
  \Big[ \mN_{ \infty,v_L}^{L, \theta }(x)-f_{v_L}(x)\Big] 
  \Big[
  \textstyle{\prod}_{n={k+1} }^{L} \big(
  \fw^{n, \theta }_{v_n, v_{n-1} }
  \indicator{ \mX^{n-1, \theta }_{n-1} }(x)
  \big)
  \Big]
  \end{split}
  \end{equation}
Combining 
\cref{lem:der:unif:bounded}, 
\cref{lem:f:integrable},
\cref{L_lim}, 
\cref{limit1}, \cref{limit2},
and the dominated convergence theorem
establishes \cref{G3,G4',G4''}.
The proof of \cref{prop:G} is thus complete.
\end{proof}

\subsection{Local Lipschitz continuity properties of the risk function}
\label{ssec:local_lipschitz}

\begin{lemma}
\label{lem:loc_lip}
Assume \cref{main_setting} and let $ \compactset \subseteq \Reals^{ \fd } $ be compact.
Then there exists $ \mathscr{L} \in \Reals $ which satisfies 
for all $ \theta, \vartheta \in \compactset $ that
\begin{equation}
\label{loc_lip}
  \abs{ 
    \mL_{ \infty }( \theta ) - \mL_{ \infty }( \vartheta ) 
  }
  +
  \sup\nolimits_{
    x \in [a,b]^{ \ell_0 } 
  } 
  \norm{ 
    \mN_{ \infty }^{ L, \theta }( x )
    -
    \mN_{ \infty }^{ L, \vartheta }( x ) 
  }
\leq
  \mathscr{L} 
  \norm{ \theta - \vartheta }
  .
\end{equation}
\end{lemma}
\begin{proof}[Proof of \cref{lem:loc_lip}]
\Nobs that, e.g., Beck et al.~\cite[Theorem 2.1]{MR4440215_BJK2022}
(applied with 
$ d \curvearrowleft \fd $, 
$ 
  l = (l_0, l_1, \ldots, l_L) \allowbreak \curvearrowleft ( \ell_0, \ell_1, \ldots, \ell_L) 
$
in the notation of \cite[Theorem 2.1]{MR4440215_BJK2022}) 
implies that for all $ \theta, \vartheta \in \compactset $
it holds that
\begin{equation}
\label{N_loc_lip1}
  \sup\nolimits_{ x \in [a,b]^{ \ell_0 } } 
  \norm{ 
    \mN_{ \infty }^{ L, \theta }( x ) - 
    \mN_{ \infty }^{ L, \vartheta }( x ) 
  }
\leq 
  L \mathbf{a} 
  \bigl[
    \textstyle\prod\nolimits_{ p = 0 }^{ L - 1 } ( \ell_p + 1 ) 
  \bigr]
  \sqrt{ \ell_L } 
  \bigl( 
    \max\{ 1, \norm{\vartheta}, \norm{\theta } \} 
  \bigr) 
  \norm{ \theta - \vartheta } .
\end{equation} 
\Moreover the fact that $ \compactset $ is compact
ensures that there exists $ \kappa \in [1, \infty) $
which satisfies for all $ \theta \in \compactset $ that
\begin{equation}
\label{phi_bd_on_cpt}
  \norm{ \theta } 
  \leq \kappa 
  < \infty 
  .  
\end{equation}
\Nobs that \cref{N_loc_lip1,phi_bd_on_cpt} 
demonstrate that 
there exists $ \mathscr{L} \in \Reals $ which satisfies
for all $ \theta, \vartheta \in \compactset $ that
\begin{equation}
\label{N_loc_lip2}
  \sup\nolimits_{ x \in [a,b]^{ \ell_0 } } 
  \norm{
    \mN_{ \infty }^{ L, \theta }( x ) 
    - 
    \mN_{ \infty }^{ L, \vartheta }( x ) 
  }
  \leq
  \mathscr{L} \norm{ \theta - \vartheta } .
\end{equation}
Combining this with the Cauchy-Schwarz inequality 
shows for all $ \theta, \vartheta \in \compactset $ that
\begin{equation}
\label{L_loc_lip1}
\begin{split}
&
  \abs{ \mL_{ \infty }( \theta ) - \mL_{ \infty }( \vartheta ) } 
\\
&
  = 
  \abs*{ 
    \br[\bigg]{ 
      \int_{ [a,b]^{ \ell_0 } 
      } 
      \norm{
        \mN_{ \infty }^{ L, \theta }( x ) - f(x) 
      }^2
      \,
      \mu( \d x)  } 
      - 
      \br[\bigg]{ 
        \int_{ [a,b]^{ \ell_0 } } 
        \norm{\mN_{ \infty }^{L, \vartheta}(x)-f(x)}^2 \, \mu( \d x ) } } 
\\
&
  \leq \int_{[a,b]^{ \ell_0 } } \abs[\big]{ \norm{\mN_{ \infty }^{L, \theta }(x)-f(x)}^2 
    -\norm{\mN_{ \infty }^{L, \vartheta}(x)-f(x)}^2 } \, \mu( \d x) 
\\
&
  = \int_{[a,b]^{ \ell_0 } } \abs[\big]{ \ip{\mN_{ \infty }^{L, \theta }(x)-\mN_{ \infty }^{L, \vartheta}(x) }{ \mN_{ \infty }^{L, \theta }(x)+\mN_{ \infty }^{L, \vartheta}(x)-2f(x)} } 
  \, \mu( \d x ) 
\\
& 
  \le \int_{[a,b]^{ \ell_0 } } \norm{\mN_{ \infty }^{L, \theta }(x)-\mN_{ \infty }^{L, \vartheta}(x)} 
  \norm{\mN_{ \infty }^{L, \theta }(x)+\mN_{ \infty }^{L, \vartheta}(x)-2f(x)} 
  \, \mu( \d x ) 
\\
&
  \leq\mathscr{L} \norm{\theta-\vartheta}
    \br[\bigg]{ \int_{[a,b]^{ \ell_0 } } 
    \norm{\mN_{ \infty }^{L, \theta }(x)+\mN_{ \infty }^{L, \vartheta}(x)-2f(x)} 
  \, \mu( \d x ) } .
\end{split}
\end{equation}
  This, \cref{phi_bd_on_cpt}, 
  \cref{N_loc_lip2},
  and the fact that for all $ x \in [a,b]^{ \ell_0 } $
  it holds that $ \mN_{ \infty }^{L,0}(x)=0$ 
  imply that for all $ \theta, \vartheta \in \compactset $
  we have that
  \begin{equation}
  \label{L_loc_lip2}
  \begin{split}
    &\abs{\mL_{ \infty }( \theta)-\mL_{ \infty }( \vartheta)}
    \leq\mathscr{L} \norm{\theta-\vartheta}
    \br[\bigg]{ \int_{[a,b]^{ \ell_0 } } 
    \norm{\mN_{ \infty }^{L, \theta }(x)+\mN_{ \infty }^{L, \vartheta}(x)} + \norm{2f(x)} \, \mu( \d x) } \\
    &\leq
    \mathscr{L} \norm{\theta-\vartheta}
    \br[\bigg]{ \mathfrak{m}
    \pr*{ \sup\nolimits_{y \in [a,b]^{ \ell_0 } }
    [\norm{\mN_{ \infty }^{L, \theta }(y)}
    +\norm{\mN_{ \infty }^{L, \vartheta}(y)}] }
    +2\int_{[a,b]^{ \ell_0 } } \norm{f(x)} \, \mu( \d x) }
    \\
    &=
    \mathscr{L} \norm{\theta-\vartheta}
    \bigg[ \mathfrak{m}
    \pr*{ \sup\nolimits_{y \in [a,b]^{ \ell_0 } }
    [\norm{\mN_{ \infty }^{L, \theta }(y)-\mN^{L,0}(y)}
    +\norm{\mN_{ \infty }^{L, \vartheta}(y)-\mN^{L,0}(y)}] }\\
    &\quad
    +2\int_{[a,b]^{ \ell_0 } } \norm{f(x)} \, \mu( \d x)\bigg] \\
    &\leq
    \mathscr{L} \norm{\theta-\vartheta}
    \bigg[\mathfrak{m}
    \big(
    \mathscr{L} \norm{\theta }
    +\mathscr{L} \norm{\vartheta} \big)
    +2 \int_{[a,b]^{ \ell_0 } } \norm{f(x)} \, \mu( \d x) \bigg]\\
    &\leq
    2 \mathscr{L}
    \bigg[\kappa \mathscr{L} \mathfrak{m}
    +\int_{[a,b]^{ \ell_0 } } \norm{f(x)} \, \mu( \d x) \bigg]
    \norm{\theta-\vartheta}.
\end{split}
\end{equation}
Combining this, \cref{lem:f:integrable}, and \cref{N_loc_lip2} 
establishes \cref{loc_lip}.
The proof of \cref{lem:loc_lip} is thus complete.
\end{proof}

\subsection{Upper estimates for the norm of the 
generalized gradients of the risk function}
\label{ssec:upper_estimates}

\begin{theorem}
\label{prop:G_upper_estimate}
Assume \cref{main_setting}, 
let $ \theta \in \Reals^{ \fd } $,
for every $ k \in \N $,
$ i \in \cu{ 1, \dots, \ell_k } $
let $ Q_{ k, i } \in \Rr $ satisfy 
$
  Q_{ k, i } = 
  \abs{ \fb^{ k, \theta }_i }^2 
  + 
  \sum_{ j = 1 }^{ \ell_{ k - 1 } } 
  \abs{ \fw^{ k, \theta }_{ i, j } }^2 
$,
and let $ \mQ_k \in \Rr $,
$ k \in \N_0 $,
satisfy for all
$ k \in \N $
that
$ \mQ_0 = 1 $
and 
$ \mQ_k = 1 + \sum_{ i = 1 }^{ \ell_k } Q_{ k, i }  
$. 
Then
\begin{enumerate}[label=(\roman*)]
\item
\label{it:R_up}
it holds for all 
$ k \in \cu{ 1, \dots, L } $, 
$ i \in \{ 1, \dots, \ell_k \} $, 
$ x \in [a,b]^{ \ell_0 } $
that
\begin{equation}
\label{R_up}
  \abs{
    \R_{ \infty }( \mN^{ k, \theta }_{ \infty, i }( x ) ) 
  }^2
\leq
  \abs{ \mN^{ k, \theta }_{ \infty, i }( x ) }^2
\leq 
  \mathbf{a}^2 
  Q_{ k, i } 
  \bigl[ 
    \textstyle\prod_{ p = 0 }^{ k - 1 } 
    \pr[\big]{ 
      ( \ell_p + 1 ) \mQ_p 
    }
  \bigr] 
  ,
\end{equation}
\item 
  \label{it:Q_bound}
  it holds for all $K \in \cu{1, \ldots, L }$ that $\prod_{p=0}^K \mQ_p \le ( \norm{\theta}^2 + 1 ) ^K$,
  \item
  \label{it:sum_w_up}
  it holds for all 
  $K\in \{ 1, \ldots, L\} $, $ k \in \{ 1, \ldots,K\} $,
  $ i \in \{ 1, \ldots, \ell_k \} $
  that 
    \begin{equation}
    \label{sum_w_up}
     \sum_{\substack{v_k,v_{k+1}, \ldots,v_K\in \enne, \\\forall w\in \enne\cap[k,K]\colon v_w\leq\ell_w} }
    \Big[\indicator{ \{ i \} }(v_k)\Big]
    \Big[\textstyle{\prod}_{n={k+1} }^{K}
    \abs{\fw^{n, \theta }_{v_n, v_{n-1} }}^2
    \Big]
    \leq \norm{\theta }^{2(K-k)},   
    \end{equation} 
  and
  \item
  \label{it:G_upper_estimate}
  it holds that
  \begin{equation}\label{G_upper_estimate}
   \norm{ \mG ( \theta ) } ^2
   \le 4 L  \mathfrak{m} \mathbf{a}^2  \br[\big]{ \textstyle\prod_{p=0}^L ( \ell_p + 1 ) } 
   ( \norm{\theta}^2 +1)^{L - 1 } 
   \mL_{ \infty } ( \theta ) .
  \end{equation}
  \end{enumerate}
\end{theorem}
\begin{proof}[Proof of \cref{prop:G_upper_estimate}]
We first prove \cref{it:R_up} by induction on 
$ k \in \cu{ 1, \ldots, L } $. 
\Nobs that \cref{N}, 
the fact that for all 
$ x \in \Reals $ it holds that $ \abs{ \R_{ \infty } ( x ) } = | \max\{ x, 0 \} | \leq \abs{x} $,
and the Cauchy-Schwarz inequality
ensure that
for all
$ i \in \{ 1, \ldots, \ell_1 \} $,
$ x = (x_1, \dots, x_{\ell_0 } ) \in [a,b]^{ \ell_0 } $ 
it holds that
\begin{equation}
\label{R_1_up}
\begin{split}
&
  \abs{
    \R_{ \infty }( \mN^{ 1, \theta }_{ \infty, i }( x ) ) 
  }^2
  \le 
  \abs{ \mN^{ 1, \theta }_{ \infty, i }( x ) }^2
  = 
  \bigl| 
    \fb^{ 1, \theta }_i + 
    \textstyle\sum_{ j = 1 }^{ \ell_0 } \fw^{ 1, \theta }_{ i, j } x_i 
  \bigr|^2 
\\
& 
  \le 
  ( \ell_0 + 1 ) 
  \bigl( 
    \abs{ \fb^{ 1, \theta }_i }^2 
    + 
    \textstyle\sum_{ j = 1 }^{ \ell_0 } 
    \abs{ \fw^{ 1, \theta }_{ i, j } }^2 
    \abs{ x_j }^2 
  \bigr)
  \le 
  \mathbf{a}^2 
  ( \ell_0 + 1 ) 
  \bigl( 
    \abs{ \fb^{ 1, \theta }_i }^2 
    + 
    \textstyle\sum_{ j = 1 }^{ \ell_0 } 
    \abs{ \fw^{ 1, \theta }_{ i, j } }^2 
  \bigr)
  =
  \mathbf{a}^2 
  ( \ell_0 + 1 ) Q_{ 1, i } 
  .
\end{split}
\end{equation}
This establishes \cref{it:R_up} in the base case $ k = 1 $.
For the induction step let $ k \in \N \cap [ 1, L) $ 
satisfy for all 
$ i \in \cu{ 1, \ldots, \ell_k } $, $ x \in [a,b]^{ \ell_0 } $
that
\begin{equation}
\label{eq:r:up:inductstep}
  \abs{
    \R_{ \infty }( 
      \mN^{ k, \theta }_{ \infty, i }( x ) 
    ) 
  }^2
\le 
  \abs{ 
    \mN^{ k, \theta }_{ \infty, i }( x ) 
  }^2
\le 
  \mathbf{a}^2 Q_{ k, i } 
  \bigl[ 
    \textstyle\prod_{ p = 0 }^{ k - 1 }  
    \pr[\big]{
      ( \ell_p + 1 ) \mQ_p 
    }
  \bigr]
  .
\end{equation}
\Nobs that \cref{eq:r:up:inductstep}, 
the fact that for all 
$ x \in \Reals $ it holds that $ \abs{ \R_{ \infty } ( x ) } = | \max\{ x, 0 \} | \leq \abs{x} $, 
the fact that for all $ p \in \N_0 $ it holds that $ \mQ_p \geq 1 $, 
and the Cauchy-Schwarz inequality demonstrate that 
for all $ i \in \cu{ 1, \ldots, \ell_{ k + 1 } } $, 
$ x = ( x_1, \ldots, x_{ \ell_0 } ) \in [a,b]^{ \ell_0 } $ 
it holds that
\begin{equation}
\begin{split}
&
  \abs{
    \R_{ \infty }( 
      \mN^{ k + 1, \theta }_{ \infty, i }( x ) 
    ) 
  }^2
\\ &
  \le 
  \abs{ 
    \mN^{ k + 1, \theta }_{ \infty, i }( x ) 
  }^2
= 
  \bigl|
    \fb^{ k + 1, \theta }_i
    + 
    \textstyle\sum_{ j = 1 }^{ \ell_k } 
    \fw^{ k + 1, \theta }_{ i, j } 
    \R_{ \infty }( \mN^{ k, \theta }_{ \infty, j }( x ) ) 
  \bigr|^2 
\\
& 
  \le 
  ( \ell_k + 1 ) 
  \pr*{ 
    \abs{ \fb^{ k + 1, \theta }_i }^2 
    + 
    \textstyle\sum_{ j = 1 }^{ \ell_k } 
    \abs{ \fw^{ k + 1, \theta }_{ i, j } }^2 
    \abs{ 
      \R_{ \infty }( \mN^{ k, \theta }_{ \infty, j }( x ) ) 
    }^2 
  } 
\\
& 
  \le 
  ( \ell_k + 1 ) 
  \pr*{ 
    \abs{ \fb^{ k + 1, \theta }_i }^2 
    +
    \textstyle\sum_{ j = 1 }^{ \ell_k } 
    \bigl[
      \abs{ \fw^{ k + 1 , \theta }_{ i, j } }^2 
      \mathbf{a}^2 
      Q_{ k, j } 
      \bigl[
        \textstyle\prod_{ p = 0 }^{ k - 1 } 
        \pr[\big]{
          ( \ell_p + 1 ) \mQ_p 
        } 
      \bigr]
    \bigr] 
  } 
\\
& 
  \le 
  ( \ell_k + 1 ) 
  \mathbf{a}^2 
  \bigl[ 
    \textstyle\prod_{ p = 0 }^{ k - 1 } 
    \pr[\big]{
      ( \ell_p + 1 ) \mQ_p 
    }
  \bigr]
  \bigl(
    \abs{ 
      \fb^{ k + 1, \theta }_i 
    }^2 
    + 
    \textstyle\sum_{ j = 1 }^{ \ell_k } 
    \abs{ 
      \fw^{ k + 1, \theta }_{ i, j } 
    }^2
  \bigr)
  \bigl(
    1 + \textstyle\sum_{ j = 1 }^{ \ell_k } Q_{ k, j } 
  \bigr)
\\
&  
  =
  \mathbf{a}^2 
  \bigl[ 
    \textstyle\prod_{ p = 0 }^k 
    \pr[\big]{
      ( \ell_p + 1 ) \mQ_p 
    }
  \bigr]
  \bigl(
    \abs{ 
      \fb^{ k + 1, \theta }_i 
    }^2 
    + 
    \textstyle\sum_{ j = 1 }^{ \ell_k } 
    \abs{ 
      \fw^{ k + 1, \theta }_{ i, j } 
    }^2
  \bigr)
  = 
  \mathbf{a}^2 
  Q_{ k + 1, i }
  \bigl[
    \textstyle\prod_{ p = 0 }^k
    \pr[\big]{
      ( \ell_p + 1 ) \mQ_p 
    } 
  \bigr] 
  .
\end{split}
\end{equation} 
Induction thus establishes \cref{it:R_up}.
\Nobs that the fact that for all $ k \in \cu{ 1, \ldots, L } $ 
it holds that $ \mQ_0 = 1 $ 
and
\begin{equation}
\textstyle
  \mQ_k 
  = 
  1 + 
  \sum_{ i = 1 }^{ \ell_k }
  Q_{ k, i }
  =
  1 
  + 
  \textstyle\sum_{ i = 1 }^{ \ell_k } 
  \bigl(
    \abs{ \fb^{ k, \theta }_i }^2 
    + 
    \textstyle\sum_{ j = 1 }^{ \ell_{ k - 1 } } 
    \abs{ \fw^{ k, \theta } _{ i, j } }^2 
  \bigr)
  \le 
  1 + \norm{ \theta }^2
\end{equation}
establishes \cref{it:Q_bound}.
We now prove \cref{it:sum_w_up} by induction on 
$ K \in \cu{ 1, \dots, L } $. 
\Nobs that the fact that for all $ i \in \{ 1, \dots, \ell_K \} $
it holds that
$
  \sum_{ v = 1 }^{ \ell_K } 
  \indicator{ \{ i \} }( v )
  = 1 
$
establishes \cref{sum_w_up} in the base case $ K = 1 $.
For the induction step let 
$  
  K \in \enne \cap [ 1 , L ) 
$ 
satisfy for all $ k \in \{ 1, \dots, K \} $, $ i \in \{ 1, \dots, \ell_k \} $ 
that
\begin{equation}
\label{sum_w_up_K}
  \sum_{
    \substack{ 
      v_k, v_{ k + 1 }, \dots, v_K \in \enne, \\
      \forall w \in \enne \cap [k, K] \colon v_w \leq \ell_w 
    } 
  }
  \Bigl[
    \indicator{  \{ i \} }( v_k ) 
  \Bigr]
  \Bigl[
    \textstyle{ \prod }_{ n = k + 1 }^K
    \abs{
      \fw^{ n, \theta }_{ v_n, v_{ n - 1 } } 
    }^2
  \Bigr]
  \leq 
  \norm{ \theta }^{ 2 (K - k) } 
  .   
\end{equation}
\Nobs that \cref{sum_w_up_K} implies that for all
$ k \in \{ 1, \dots, K + 1 \} $, 
$ i \in \{ 1, \dots, \ell_k \} $
it holds that 
\begin{equation}
\begin{split}
&
  \sum_{ 
    \substack{ 
      v_k, v_{ k + 1 }, \dots, v_{ K + 1 } \in \enne, \\ 
      \forall w \in \enne \cap [k,K+1] \colon v_w \leq \ell_w 
    } 
  }
    \Bigl[
      \indicator{  \{ i \} }( v_k )
    \Bigr]
    \Bigl[
      \textstyle{ \prod }_{ n = k + 1 }^{ K + 1 }
      \abs{ 
        \fw^{ n, \theta }_{ v_n, v_{ n - 1 } } 
      }^2
    \Bigr]
\\
&
  =
  \br*{
    \sum_{
      \substack{
        v_k, v_{ k + 1 }, \dots, v_{ K } 
        \in \enne, 
      \\
        \forall w \in \enne \cap [k,K] \colon v_w \leq \ell_w 
      } 
    }
    \Bigl[
      \indicator{ \{ i \} }( v_k ) 
    \Bigr]
    \Bigl[
      \textstyle{ \prod }_{ n = k + 1 }^K
      \abs{
        \fw^{ n, \theta }_{ v_n, v_{ n - 1 } } 
      }^2
    \Bigr]
}
  \br*{
    \sum\limits_{ v_{ K + 1 } = 1 }^{ \ell_{ K + 1 } }
    \abs{
      \fw^{ K + 1, \theta }_{ v_{ K + 1 }, v_K } 
    }^2
}
\\
&
  \leq
  \br*{
    \sum_{
      \substack{ 
        v_k, v_{ k + 1 }, \dots, v_{ K } 
        \in \enne, 
      \\
        \forall w \in \enne \cap [k,K] \colon v_w \leq \ell_w 
      } 
    }
    \Bigl[
      \indicator{  \{ i \} }( v_k ) 
    \Bigr] 
    \Bigl[
      \textstyle{ \prod }_{ n = k + 1 }^K
      \abs{
        \fw^{ n, \theta }_{ v_n, v_{ n - 1 } } 
      }^2
    \Bigr] }
  \norm{ \theta }^2 
  \leq \norm{ \theta }^{ 2 ( K + 1 - k ) } 
  .
\end{split}
\end{equation}
Induction thus establishes \cref{it:sum_w_up}.
It thus remains to prove \cref{it:G_upper_estimate}. 
For this 
assume without loss of generality that $ \mathfrak{m} > 0 $,
for every $ k \in \N_0 $ let $ \diml_k \in \N_0 $
satisfy $ \diml_k = \sum_{ n = 1 }^k \ell_n ( \ell_{ n - 1 } + 1 ) $, 
and 
for every $ k \in \cu{ 1, \dots, L } $, $ i \in \cu{ 1, \dots, \ell_k } $ 
let 
$
  W_{ k, i } \in \Rr 
$
satisfy
\begin{equation}
  W_{ k, i } = 
  \sum_{
    \substack{ 
      v_k, v_{ k + 1 }, \dots, v_L \in \enne, 
    \\
      \forall w \in \enne \cap [k, L] \colon v_w \leq \ell_w 
    } 
  }
  \Bigl[
    \indicator{  \{ i \} }( v_k ) 
  \Bigr]
  \Bigl[
    \textstyle{ \prod }_{ n = k + 1 }^L
    \abs{
      \fw^{ n, \theta }_{ v_n, v_{ n - 1 } } 
    }
  \Bigr]
  .
\end{equation}
\Nobs that \cref{G4''} of \cref{prop:G} proves 
that for all $ k \in \{ 1, \dots, L \} $, 
$ i \in \{ 1, \dots, \ell_k \} $ it holds that 
\begin{equation}
\label{G_up1a}
\begin{split}
&
  \abs{
    \mG_{ \ell_k \ell_{ k - 1 } + i + \diml_{ k - 1 } }( \theta ) 
  }^2
\\
&
  =
  \Biggl[ 
    \sum_{ 
      \substack{ 
        v_k, v_{ k + 1 }, \dots, v_L \in \enne, 
      \\
        \forall \, w \in \enne \cap [k, L] \colon 
        v_w \leq \ell_w 
      } 
    }
    \int_{ [ a, b ]^{ \ell_0 } }
    2 \,
    \Bigl[ 
      \indicator{  \{ i \} }( v_k )
    \Bigr]
    \Bigl[ 
      \mN_{ \infty, v_L }^{ L, \theta }( x ) - f_{ v_L }( x )
    \Bigr]
\\
&
\quad\cdot
    \Bigl[
      \textstyle{ \prod }_{ n = k + 1 }^L 
      \bigl(
        \fw^{ n, \theta }_{ v_n, v_{ n - 1 } }
        \indicator{  
          \mX^{ n - 1, \theta }_{ v_{n-1} } 
        }( x )
      \bigr)
    \Bigr] \, \mu( \d x) 
  \Biggr]^2
\\
&
  \leq 
  4 
  \Biggl[
    \int_{ [a,b]^{ \ell_0 } }
    \norm{
      \mN^{ L, \theta }_{ \infty }( x ) - f( x ) 
    }
    \Biggl[
      \sum_{
        \substack{ 
          v_k, v_{k+1}, \dots, v_L \in \enne, 
        \\
          \forall \, w \in \enne \cap [k, L] \colon 
          v_w \leq \ell_w 
        } 
      }
      \Bigl[
        \indicator{  \{ i \} }( v_k )
      \Bigr]
      \Bigl[ 
        \textstyle{ \prod }_{ n = k + 1 }^L
        \abs{ \fw^{ n, \theta }_{ v_n, v_{n-1} } }
      \Bigr]
    \Biggr] \, 
    \mu( \d x ) 
  \Biggr]^2
\\
&
  =
  4 
  \biggl[
    W_{ k, i }
    \displaystyle 
    \int_{
      [a,b]^{ \ell_0 } 
    }
    \norm{
      \mN^{ L, \theta }_{ \infty }( x ) - f( x ) 
    }
    \, \mu( \d x ) 
  \biggr]^2
  .
\end{split}
\end{equation}
Jensen's inequality \hence shows that 
for all $ k \in \{ 1, \dots, L \} $, 
$ i \in \{ 1, \dots, \ell_k \} $ 
it holds that
\begin{equation}
\label{G_up1b}
\begin{split}
&
  \abs{
    \mG_{ \ell_k \ell_{k-1} + i + \diml_{ k - 1 } }( \theta ) 
  }^2
\leq
  4 \mathfrak{m}^2 ( W_{ k, i } )^2
  \Biggl[
    \displaystyle\frac{ 1 }{ \mathfrak{m} }
    \int_{ [a,b]^{ \ell_0 } }
    \norm{ \mN^{ L, \theta }_{ \infty }( x ) - f( x ) }
    \, \mu( \d x )
  \Biggr]^2
\\
&
  \leq
  4 \mathfrak{m}^2 
  ( W_{ k, i } )^2
  \biggl[ 
    \displaystyle\frac{ 1 }{ \mathfrak{m} }
    \int_{ [ a, b ]^{ \ell_0 } }
    \norm{ 
      \mN^{ L, \theta }_{ \infty }( x ) - f( x ) 
    }^2
    \, \mu( \d x )
  \biggr]
    =
    4\mathfrak{m}
     ( W_{k,i} )^2
  \mL_{ \infty }( \theta ) 
  .
\end{split}
\end{equation}
\Moreover \cref{G4'} of \cref{prop:G} and
Jensen's inequality imply 
for all $ k \in \{ 1, \dots, L \} $,
$ i \in \{ 1, \dots, \ell_k \} $, 
$ j \in \{ 1, \dots, \ell_{k-1} \} $ that
\begin{equation}  
\label{G_up2a}
\begin{split}
&
  \abs{
    \mG_{ 
      (i - 1) \ell_{ k - 1 } + j + \diml_{ k - 1 } 
    }( \theta )  
  }^2
\\
&
  =
  \Biggl[ 
    \sum_{
      \substack{ 
        v_k, v_{k+1}, \dots, v_L \in \enne, 
      \\ 
        \forall w \in \enne \cap [k, L] \colon 
        v_w \leq \ell_w 
      } 
    }
    \int_{
      [a,b]^{ \ell_0 } 
    }
    2 \,
    \Bigl[
      \R_{ \infty }( 
        \mN^{ \max\{ k - 1, 1 \}, \theta }_{ \infty, j }( x ) 
      )
      \indicator{  (1, L] }( k )
      + x_j \indicator{  \{ 1 \} }( k )
    \Bigr]
\\
&  
  \quad\cdot
    \Bigl[ 
      \indicator{  \{ i \} }( v_k )
    \Bigr]
    \Bigl[ 
      \mN_{ \infty, v_L }^{ L, \theta }( x ) - f_{ v_L }( x )
    \Bigr] 
    \Bigl[
      \textstyle{ \prod }_{ n = k + 1 }^L 
      \bigl(
        \fw^{ n, \theta }_{ v_n, v_{n-1} }
        \indicator{  \mX^{ n - 1, \theta }_{ v_{n-1} } }( x )
      \bigr)
    \Bigr] \, \mu( \d x)
  \Biggr]^2
\\
&
  \leq
  4 
  \Biggl[ 
    \int_{ [a,b]^{ \ell_0 } }
      \Bigl[
        \R_{ \infty }( 
          \mN^{ \max\{ k - 1, 1 \}, \theta }_{ \infty, j }( x ) 
        )
        \indicator{  (1, L] }( k )
        + x_j \indicator{  \{ 1 \} }( k )
      \Bigr]
      \norm{ \mN_{ \infty }^{ L, \theta }( x ) - f(x) } 
      W_{ k, i } 
    \, \mu( \d x ) 
  \Biggr]^2
\\
&
  =
  4 \mathfrak{m}^2
  \Biggl[  
    W_{ k, i }
    \biggl[ 
      \frac{ 1 }{ \mathfrak{m} }
      \int_{ [a,b]^{ \ell_0 } } 
      \,
        \bigl[
          \R_{ \infty }( 
            \mN^{ \max\{ k - 1, 1 \}, \theta }_{ \infty, j }( x ) 
          )
          \indicator{  (1, L] }( k )
          +
          x_j \indicator{  \{ 1 \} }( k )
        \bigr]
        \norm{
          \mN_{ \infty }^{ L, \theta }( x ) - f( x ) 
        }
      \, \mu( \d x ) 
    \biggr]
  \Biggr]^2
\\
&
  \leq
  4 \mathfrak{m}^2
  ( W_{ k, i } )^2
  \biggl[
    \frac{ 1 }{ \mathfrak{m} }
    \int_{ [a,b]^{ \ell_0 } }
    \bigl[
      \R_{ \infty }( 
        \mN^{ \max\{ k - 1, 1 \}, \theta }_{ \infty, j }( x ) 
      )
      \indicator{  (1, L] }( k )
      +
      x_j \indicator{  \{ 1 \} }( k )
    \bigr]^2
    \norm{ 
      \mN_{ \infty }^{ L, \theta }( x ) - f( x ) 
    }^2
     \, \mu( \d x) 
  \biggr]
  .
\end{split}
\end{equation}
This ensures for all 
$ i \in \cu{ 1, \dots, \ell_1 } $, 
$ j \in \cu{ 1, \dots, \ell_0 } $
that
\begin{equation}
\label{G_w1up}
  \abs{
    \mG_{ ( i - 1 ) \ell_0 + j }( \theta ) 
  }^2 
  \le 
  4 \mathfrak{m} \mathbf{a}^2 
  ( W_{ 1, i } )^2 
  \mL_{ \infty }( \theta ) 
  .
\end{equation}
\Moreover \cref{G_up2a,it:R_up} demonstrate for all
$ k \in \N \cap ( 1 , L ] $,
$ i \in \{ 1, \ldots, \ell_k \} $, 
$ j \in \{ 1, \ldots, \ell_{k-1} \} $
that
\begin{equation}
\label{G_up2b}
\begin{split}
&
  \abs{
    \mG_{ (i - 1) \ell_{ k - 1 } + j + \diml_{ k - 1 } 
    }( \theta ) 
  }^2
\\ &
\textstyle 
  \leq
  4 \mathfrak{m}
  ( W_{ k, i } )^2
  \Bigl[ 
    \int_{ 
      [a,b]^{ \ell_0 } 
    }
    \mathbf{a}^2
    Q_{ k - 1, j } 
    \bigl[
      \textstyle{ \prod }_{ p = 0 }^{ k - 2 } 
      ( \ell_p + 1 ) \mQ_p
    \bigr]
    \norm{ 
      \mN_{ \infty }^{ L, \theta }( x ) - f( x ) 
    }^2
    \mu( \d x )
  \Bigr] 
\\
&
  \leq
  4 \mathfrak{m}
  ( W_{ k, i } )^2
  \mathbf{a}^2
  Q_{ k - 1, j } 
  \bigl[
    \textstyle{ \prod }_{ p = 0 }^{ k - 2 }
    ( \ell_p + 1 ) \mQ_p
  \bigr]
  \mL_{ \infty }( \theta ) 
  .
\end{split}
\end{equation}
\Moreover the Cauchy-Schwarz inequality and 
\cref{it:sum_w_up} imply 
for all $ k \in \{ 1, \dots, L\} $, $ i \in \{ 1, \dots, \ell_k \} $ that 
\begin{equation}
\label{sum_w_sq}
  ( W_{ k, i } )^2
\leq 
  \bigl[
    \textstyle{ \prod }_{ p = k + 1 }^L 
    \ell_p
  \bigr]
  \Biggl[
    \sum\nolimits_{ 
      \substack{ 
        v_k, v_{k+1}, \dots, v_L \in 
        \enne, 
      \\
        \forall w \in \enne \cap [k, L] \colon 
        v_w \leq \ell_w 
      } 
    }
    \Bigl[
      \indicator{  \{ i \} }( v_k ) 
    \Bigr]
    \Bigl[
      \textstyle{ \prod }_{ n = k + 1 }^L
      \abs{
        \fw^{ n, \theta }_{ v_n, v_{n-1} } 
      }^2
    \Bigr]
  \Biggr]
\le 
  \bigl[
    \textstyle{ \prod }_{ p = k + 1 }^L 
    \ell_p
  \bigr] 
  \norm{ \theta }^{ 2 ( L - k ) } 
  .
\end{equation}
Combining this with \cref{G_up1b} and
\cref{G_w1up} assures that
\begin{equation} 
\label{eq:G_up_k=1}
\begin{split}
&
  \textstyle\sum\limits_{ i = 1 }^{ \ell_1 } 
  \br*{
    \abs{
      \mG_{ \ell_1 \ell_0 + i }( \theta ) 
    }^2 
    + 
    \sum\limits_{ j = 1 }^{ \ell_0 } 
    \abs{ \mG_{ (i - 1) \ell_0 + j }( \theta ) }^2 
  } 
  \le  
  \sum\limits_{ i = 1 }^{ \ell_1 } 
  \Biggl[
    \biggl[
      4 \mathfrak{m} ( W_{ 1, i } )^2 
      + 
      \sum\limits_{ j = 1 }^{ \ell_0 } 
      4 \mathfrak{m} \mathbf{a}^2 ( W_{ 1, i } )^2 
    \biggr]
    \mL_{ \infty }( \theta ) 
  \Biggr]
\\
&
  \le 4 \mathfrak{m} 
  \bigl[ 
    \textstyle\prod_{ p = 2 }^L \ell_p 
  \bigr] 
  \norm{ \theta }^{ 2 ( L - 1 ) } 
  ( \ell_1 + \mathbf{a}^2 \ell_1 \ell_0 ) 
  \mL_{ \infty }( \theta )
\le 
  4 \mathfrak{m} \mathbf{a}^2 
  \bigl[ 
    \textstyle\prod_{ p = 0 }^L ( \ell_p + 1 ) 
  \bigr]
  \norm{ \theta }^{ 2 ( L - 1 ) } 
  \mL_{ \infty }( \theta ).
\end{split} 
\end{equation}
\Moreover \cref{G_up1b}, \cref{G_up2b}, and \cref{sum_w_sq} prove that
\begin{equation}
\label{G_ue_sum1}
\begin{split}
& 
  \textstyle\sum\limits_{ k = 2 }^L 
  \sum\limits_{ i = 1 }^{ \ell_k }
  \br*{ 
    \abs{ 
      \mG_{ 
        \ell_k \ell_{ k - 1 } + i + \diml_{ k - 1 } 
      }( \theta ) 
    }^2
    +
    \sum\limits_{ j = 1 }^{ \ell_{ k - 1 } }
    \abs{
      \mG_{ (i - 1) \ell_{ k - 1 } + j + \diml_{ k - 1 } }( \theta ) 
    }^2 
  } 
\\
&
  \leq
  4 
  \mathfrak{m} 
  \Biggl[
    \textstyle\sum\limits_{ k = 2 }^L 
    \br*{ 
      \sum\limits_{ i = 1 }^{ \ell_k } 
      \pr*{  
        ( W_{ k, i } )^2 
        +
        \mathbf{a}^2 \sum\limits_{ j = 1 }^{ \ell_{ k - 1 } } 
        ( W_{ k, i } )^2
        Q_{ k - 1, j } 
        \bigl[
          \textstyle{ \prod }_{ p = 0 }^{ k - 2 }
          ( \ell_p + 1 ) \mQ_p
        \bigr] 
      }  
    } 
  \Biggr]
  \mL_{ \infty }( \theta )
\\ 
&
  = 
  4 
  \mathfrak{m} 
  \Biggl[
    \textstyle\sum\limits_{ k = 2 }^L 
    \biggl[ 
      \sum\limits_{ i = 1 }^{ \ell_k } 
      \Bigl(
        ( W_{ k, i } )^2 
        \Bigl(
          1
          +
          \mathbf{a}^2 
          (
            \mQ_{ k - 1 } - 1
          )
          \bigl[
            \textstyle{ \prod }_{ p = 0 }^{ k - 2 }
            ( \ell_p + 1 ) \mQ_p
          \bigr] 
        \Bigr)
      \Bigr)
    \biggr]
  \Biggr]
  \mL_{ \infty }( \theta )
\\
& 
  \leq 
  4 \mathfrak{m} 
  \biggl[ 
    \textstyle\sum\limits_{ k = 2 }^L 
    \br*{ 
      \bigl[ 
        \textstyle\prod_{ p = k + 1 }^L \ell_p 
      \bigr]
      \norm{ \theta }^{ 2 ( L - k ) } 
      \ell_k 
      \bigl( 
        1 
        + 
        \mathbf{a}^2 ( \mQ_{ k - 1 } - 1 ) 
        \bigl[ 
          \prod_{ p = 0 }^{ k - 2 } 
          \bigl( 
            ( \ell_p + 1 ) \mQ_p 
          \bigr)
        \bigr]
      \bigr)
    } 
    \mL_{ \infty }( \theta ) 
  \biggr]
\\
& 
  \leq 
  4 \mathbf{a}^2 \mathfrak{m} 
  \biggl[ 
    \textstyle\sum\limits_{ k = 2 }^L 
    \br*{ 
      \bigl[ 
        \textstyle\prod_{ p = k }^L \ell_p 
      \bigr]
      \norm{ \theta }^{ 2 ( L - k ) } 
      \bigl( 
        1 
        + 
        ( \mQ_{ k - 1 } - 1 ) 
        \bigl[ 
          \prod_{ p = 0 }^{ k - 2 } 
          ( \ell_p + 1 ) 
        \bigr]
        \bigl[ 
          \prod_{ p = 0 }^{ k - 2 } 
          \mQ_p 
        \bigr]
      \bigr)
    } 
  \biggr]
  \mL_{ \infty }( \theta ) 
  .
\end{split}
\end{equation}
This, \cref{it:Q_bound}, 
and the fact that for all 
$ k \in \N $ it holds that 
$
  1 \leq \prod_{ p = 0 }^{ k - 2 } \mQ_p 
$
show that 
\begin{equation}
\begin{split}
&
  \textstyle\sum\limits_{ k = 2 }^L 
  \sum\limits_{ i = 1 }^{ \ell_k }
  \br*{ 
    \abs{ 
      \mG_{ 
        \ell_k \ell_{ k - 1 } + i + \diml_{ k - 1 } 
      }( \theta ) 
    }^2
    +
    \sum\limits_{ j = 1 }^{ \ell_{ k - 1 } }
    \abs{
      \mG_{ (i - 1) \ell_{ k - 1 } + j + \diml_{ k - 1 } }( \theta ) 
    }^2 
  } 
\\
& 
  \leq 
  4 \mathbf{a}^2 \mathfrak{m} 
  \biggl[ 
    \textstyle\sum\limits_{ k = 2 }^L 
    \br*{ 
      \bigl[ 
        \textstyle\prod_{ p = k }^L \ell_p 
      \bigr]
      \norm{ \theta }^{ 2 ( L - k ) } 
      \bigl[ 
        \prod_{ p = 0 }^{ k - 2 } 
        ( \ell_p + 1 ) 
      \bigr]
      \Bigl( 
        1 
        + 
        ( \mQ_{ k - 1 } - 1 ) 
        \bigl[ 
          \prod_{ p = 0 }^{ k - 2 } 
          \mQ_p 
        \bigr]
      \Bigr)
    } 
  \biggr]
  \mL_{ \infty }( \theta ) 
\\
& 
  \leq 
  4 \mathbf{a}^2 \mathfrak{m} 
  \bigl[ 
    \textstyle\prod_{ p = 0 }^L 
    ( \ell_p + 1 ) 
  \bigr] 
  \biggl[
    \textstyle\sum\limits_{ k = 2 }^L 
    \bigl(
      \norm{ \theta }^{ 2 ( L - k ) } 
      \bigl[
        \prod_{ p = 0 }^{ k - 1 } \mQ_p 
      \bigr]
    \bigr)
  \biggr]
  \mL_{ \infty }( \theta) 
\\
& 
  \le 
  4 \mathbf{a}^2 \mathfrak{m} 
  \bigl[
    \textstyle\prod_{ p = 0 }^L ( \ell_p + 1 ) 
  \bigr]
  \biggl[ 
    \textstyle\sum\limits_{ k = 2 }^L 
    \bigl(
      \norm{ \theta }^{ 2 ( L - k ) } 
      ( \norm{ \theta }^2 + 1 )^{ k - 1 } 
    \bigr)
  \biggr] 
  \mL_{ \infty }( \theta )
\\
& 
  \leq 
  4 \mathbf{a}^2 
  \mathfrak{m} ( L - 1 ) 
  \bigl[ 
    \textstyle\prod_{ p = 0 }^L 
    ( \ell_p + 1 ) 
  \bigr] 
  ( \norm{\theta}^2 + 1 )^{ L - 1 } 
  \mL_{ \infty }( \theta ) 
  .
\end{split}
\end{equation}
Combining this with \cref{eq:G_up_k=1} ensures that 
\begin{equation}
\label{G_ue_sum2}
\begin{split}
  \norm{ \mG( \theta ) }^2
& 
  = 
  \textstyle\sum\limits_{ k = 1 }^L 
  \sum\limits_{ i = 1 }^{ \ell_k }
  \biggl[ 
    \abs{ 
      \mG_{ \ell_k \ell_{ k - 1 } + i + \diml_{ k - 1 } }( \theta ) 
    }^2
    +
    \sum\limits_{ j = 1 }^{ \ell_{ k - 1 } }
    \abs{
      \mG_{ ( i - 1 ) \ell_{ k - 1 } + j + \diml_{ k - 1 } }( \theta ) 
    }^2 
  \biggr] 
\\
&
  \leq 
  4 \mathfrak{m} \mathbf{a}^2 
  \bigl[ 
    \textstyle\prod_{ p = 0 }^L 
    ( \ell_p + 1 ) 
  \bigr] 
  \norm{ \theta }^{ 2 ( L - 1 ) } 
  \mL_{ \infty }( \theta ) 
\\
& 
  \quad 
  + 4 \mathbf{a}^2 \mathfrak{m} ( L - 1 ) 
  \bigl[ 
    \textstyle\prod_{ p = 0 }^L ( \ell_p + 1 ) 
  \bigr] 
  ( \norm{ \theta }^2 + 1 )^{ L - 1 } 
  \mL_{ \infty }( \theta ) 
\\
& 
  \le 4 \mathbf{a}^2 \mathfrak{m} L 
  \bigl[ 
    \textstyle\prod_{ p = 0 }^L ( \ell_p + 1 ) 
  \bigr] 
  ( \norm{ \theta }^2 + 1 )^{ L - 1 } 
  \mL_{ \infty }( \theta ) .
\end{split}
\end{equation}
  This establishes \cref{it:G_upper_estimate}.
  The proof of \cref{prop:G_upper_estimate} is thus complete.
\end{proof}

\begin{cor}
\label{cor:G_up}
  Assume \cref{main_setting} and let
  $K\subseteq\Reals^{ \fd } $ be compact.
  Then
  \begin{equation}\label{sup_G_on_K}
  \sup\nolimits_{\theta \in K} \norm{\mG( \theta)}<\infty.
  \end{equation}
\end{cor}
\begin{proof}[Proof of \cref{cor:G_up}]
  \Nobs that \cref{lem:loc_lip} and 
  the assumption that $K$ is compact demonstrate
  that $ \sup_{\theta \in K} \mL_{ \infty }( \theta)<\infty$.
  \Cref{it:G_upper_estimate} of \cref{prop:G_upper_estimate} \hence 
  establishes
  \cref{sup_G_on_K}.
  The proof of \cref{cor:G_up} 
  is thus complete.
\end{proof}

\subsection{Convexity properties of the risk function}
\label{ssec:risk_convex}

\begin{prop}[Non-existence of non-global local minima for convex objective functions]
\label{lem:convexity}
  Let $ \fd\in \enne $ and let
  $ \mL\colon \Reals^{ \fd }\to\Reals $ satisfy 
  for all $ \theta, \vartheta \in \Reals^{ \fd } $,
  $ \lambda\in [0,1]$ that
  $ \mL( \lambda\theta+(1-\lambda)\vartheta)
    \leq \lambda\mL( \theta)+(1-\lambda)\mL( \vartheta) $.
  Then 
  \begin{equation}\label{no_non-global_local_minima}
  \big\{\theta \in \Reals^{ \fd }\colon 
  \big( \exists\,
   \varepsilon\in (0, \infty)\colon
  \inf\nolimits_{\vartheta \in \{\varphi\in \Reals^{ \fd }\colon
  \norm{\varphi-\theta }
  \leq\varepsilon\} }
  \mL( \vartheta) = \mL( \theta)
  >\inf\nolimits_{\vartheta \in \Reals^{ \fd } } \mL( \vartheta)\big) \big\}
  =\emptyset.
  \end{equation}
\end{prop}
\begin{proof}[Proof of \cref{lem:convexity}]
  We prove \cref{no_non-global_local_minima}
  by contradiction.
  For this let $ \theta, \vartheta \in \Reals^{ \fd } $,
  $ \varepsilon\in (0, \infty) $ satisfy
  \begin{equation}\label{theta_local_minimum}
  \inf\nolimits_{\varphi\in \{\phi\in \Reals^{ \fd }\colon
  \norm{\phi-\theta }
  \leq\varepsilon\} }
  \mL( \varphi) = \mL( \theta)
  >\mL( \vartheta).
  \end{equation}
  \Nobs that \cref{theta_local_minimum} and the 
  fact that for all 
  $ \lambda\in [0,1]$ it holds that
  $ \mL( \lambda\vartheta+(1-\lambda)\theta)
    \leq \lambda\mL( \vartheta)+(1-\lambda)\mL( \theta) $
  imply that
  for all $ \lambda\in (0,1) $
  it holds that
  \begin{equation}\label{not_local_minima}
    \mL( \lambda\vartheta+(1-\lambda)\theta)
    \leq \lambda\mL( \vartheta)+(1-\lambda)\mL( \theta)
    <\lambda\mL( \theta)+(1-\lambda)\mL( \theta)
    =\mL( \theta).
  \end{equation}
  Combing 
  this and \cref{theta_local_minimum}
  with the fact that 
  $ \inf\nolimits_{\varphi\in \{\phi\in \Reals^{ \fd }\colon
  \norm{\phi-\theta }
  \leq\varepsilon\} }
  \mL( \varphi) = 
  \inf\nolimits_{\varphi\in \{\phi\in \Reals^{ \fd }\colon
  \norm{\phi}
  \leq\varepsilon\} }
  \mL( \theta+\varphi) $
  shows that for all
  $ \lambda\in (0,1) $
  it holds that
  \begin{equation}
    \mL( \theta+\lambda( \vartheta-\theta))
    < \inf\nolimits_{\varphi\in \{\phi\in \Reals^{ \fd }\colon
  \norm{\phi}
  \leq\varepsilon\} }
  \mL( \theta+\varphi).
  \end{equation}
  This contradiction establishes 
  \cref{no_non-global_local_minima}.
  The proof of \cref{lem:convexity}
  is thus complete.
\end{proof}

\begin{lemma}[Characterization of affine linearity]
\label{lem:aff_lin}
   Let $V$ and $W$ be $ \Reals $-vector spaces and let 
   $ \varphi\colon V\to W$
   be a function. 
   Then the following three statements are equivalent:
   \begin{enumerate}[label=(\roman*)]
     \item\label{it:conv_lin1} 
     It holds for all 
     $ \lambda\in [0,1]$, $v,w\in V$ that 
       $ \varphi( \lambda v+(1-\lambda)w)
       =\lambda\varphi(v)+(1-\lambda)\varphi(w) $.
     \item\label{it:conv_lin2} 
     It holds for all 
     $ \lambda\in (0,1) $, $v,w\in V$ that
       $ \varphi( \lambda v+(1-\lambda)w)
       =\lambda\varphi(v)+(1-\lambda)\varphi(w) $.
     \item\label{it:aff_lin} 
     It holds for all 
     $ \lambda\in \Reals $, $v,w\in V$ that
     $ \varphi( \lambda v+w)-\varphi(0)
     =\lambda( \varphi(v)-\varphi(0))
     +( \varphi(w)-\varphi(0)) $.
   \end{enumerate}
\end{lemma}
\begin{proof}[Proof of \cref{lem:aff_lin}]
  \Nobs that the fact that for
  all $ \lambda\in \{0,1\} $,
  $v,w\in V$ 
  it  holds that 
  $ \varphi( \lambda v+(1-\lambda)w)
       =\lambda\varphi(v)+(1-\lambda)\varphi(w) $
  establishes that
  (\cref{it:conv_lin1}
  $ \leftrightarrow$ \cref{it:conv_lin2}).
  We now prove that (\cref{it:aff_lin}
  $ \rightarrow$ \cref{it:conv_lin1}). 
  \Nobs that \cref{it:aff_lin} ensures that
  for all $ \lambda\in [0,1]$, $v,w\in V$
  it holds that
  \begin{equation}
  \begin{split}
    &\varphi( \lambda v+(1-\lambda)w)
    =\lambda( \varphi(v)-\varphi(0))+\varphi((1-\lambda)w)\\
    &=\lambda( \varphi(v)-\varphi(0))+(1-\lambda)( \varphi(w)-\varphi(0))+\varphi(0)
     =\lambda\varphi(v)+(1-\lambda)\varphi(w).
  \end{split}
  \end{equation}
  This establishes that (\cref{it:aff_lin} $ \rightarrow$ \cref{it:conv_lin1}).
  We now prove that (\cref{it:conv_lin1}
  $ \rightarrow$ \cref{it:aff_lin}).
  \Nobs that \cref{it:conv_lin1} 
  implies that
  for all  
  $ \lambda\in [0,1]$,
  $v\in V$ it holds that
  \begin{equation}\label{phi_factor1}
    \varphi( \lambda v)
    =\varphi( \lambda v+(1-\lambda)0)
    =\lambda( \varphi(v))+(1-\lambda)\varphi(0)
    =\lambda( \varphi(v)-\varphi(0))+\varphi(0).
  \end{equation}
  \Moreover 
  \cref{it:conv_lin1} 
  shows that for all
  $ \lambda\in (1, \infty) $, $v\in V$
  it holds that 
  \begin{equation}\label{phi_factor2a}
    \textstyle{
    \lambda\varphi(v)
    =\lambda\varphi( \frac{1}{\lambda} \lambda v
    +(1-\frac{1}{\lambda} )0)
    =\lambda\frac{1}{\lambda} \varphi( \lambda v)
    +\lambda(1-\frac{1}{\lambda} )\varphi(0)
    =\varphi( \lambda v)-(1-\lambda)\varphi(0).}
  \end{equation}
  This and \cref{phi_factor1} 
  prove that for all $ \lambda\in [0, \infty) $, $v\in V$
  it holds that
  \begin{equation}\label{phi_factor2}
    \varphi( \lambda v)
    =\lambda( \varphi(v)-\varphi(0))+\varphi(0).
  \end{equation}
  Combining this with \cref{it:conv_lin1}
  ensures that for all 
  $v\in V$
  it holds that
  \begin{equation}\label{phi_0 }
  \begin{split}    
    &0=\varphi(0)-\varphi(0)
    =\textstyle{\varphi( \frac{1}{2}2v-\frac{1}{2}2v)-\varphi(0)
    =\frac{1}{2} \varphi(2v)+\frac{1}{2} \varphi(-2v)-\varphi(0)} \\
    &=\textstyle{
    \frac{1}{2}2( \varphi(v)-\varphi(0))+\frac{1}{2} \varphi(0)
    +\frac{1}{2}2( \varphi(-v)-\varphi(0))+\frac{1}{2} \varphi(0)
    -\varphi(0)
    =\varphi(v)+\varphi(-v)-2\varphi(0)}.
  \end{split}     
  \end{equation}
  This
  and \cref{phi_factor2} show that for all
  $ \lambda\in (-\infty,0) $, $v\in V$
  it holds that
  \begin{equation}\label{phi_factor3}
  \begin{split}
    &\varphi( \lambda v)-\varphi(0)
    =-\varphi(-\lambda v)+\varphi(0)
    =-( \varphi(-\lambda v)-\varphi(0))\\
    &=-(-\lambda( \varphi(v)-\varphi(0))
    +\varphi(0)-\varphi(0))
    =\lambda( \varphi(v)-\varphi(0)).
  \end{split}
  \end{equation} 
  Combining this with \cref{phi_factor2} 
  proves that 
  for all $ \lambda\in \Reals $, $v\in V$
  it holds that
  \begin{equation}\label{phi_factor4}
    \varphi( \lambda v)=\lambda( \varphi(v)-\varphi(0))+\varphi(0).
  \end{equation}
  \Hence 
  for all $v,w\in V$
  that
  \begin{equation}
  \begin{split}
    &\varphi(v+w)
    =\textstyle{\varphi( \frac{1}{2}2v+\frac{1}{2}2w)
    =\frac{1}{2} \varphi(2v)+\frac{1}{2} \varphi(2w)} \\
    &=\textstyle{\frac{1}{2}2( \varphi(v)-\varphi(0))+\frac{1}{2} \varphi(0)+\frac{1}{2}2( \varphi(w)-\varphi(0))+\frac{1}{2} \varphi(0)
    =\varphi(v)+\varphi(w)-\varphi(0).}
  \end{split}
  \end{equation}  
  This and \cref{phi_factor4} 
  demonstrate that for all 
  $ \lambda\in \Reals $, $v,w\in V$
  it holds that
  \begin{equation}
  \begin{split}
    &\varphi( \lambda v +w)-\varphi(0)
    =\varphi( \lambda v)+\varphi(w)-2\varphi(0)
    =\lambda( \varphi(v)-\varphi(0))+\varphi(0)+\varphi(w)-2\varphi(0)\\
    &=\lambda( \varphi(v)-\varphi(0))+( \varphi(w)-\varphi(0)).
  \end{split}  
  \end{equation}
  This establishes that (\cref{it:conv_lin1}
  $ \rightarrow$ \cref{it:aff_lin}).
  The proof of \cref{lem:aff_lin}
  is thus complete.
\end{proof}

\begin{proposition}[Convexity with respect to the last affine linear transformation]
\label{prop:outer_convexity}
Assume \cref{main_setting},
let $ \delta \in \N_0 $ satisfy 
$
  \delta = 
  \sum_{ k = 1 }^{ L - 1 } 
  \ell_k ( \ell_{ k - 1 } + 1 ) 
$,
let 
$ 
  \theta = 
  \allowbreak 
  ( \theta_j )_{ 
    j \in \enne 
    \cap [ 0, \delta ] 
  }
    \colon \enne \cap [ 0, \delta ]
  \to \Reals 
$
be a function,
and let 
$ \mathbb{L} \colon \Reals^{ \ell_L( \ell_{L-1} + 1)} \to \Reals $
  satisfy for all
  $v= (v_1, \ldots,v_{\ell_L( \ell_{L-1} + 1)} ) \in \Reals^{ \ell_L( \ell_{L-1} + 1)} $
  that
  \begin{equation}\label{bbL}
  \mathbb{L}(v)=
    \mL_{ \infty }\big( 
      \theta_1, \theta_2, \ldots, \theta_{ \delta },
      v_1, v_2, \ldots, v_{\ell_L( \ell_{L-1} + 1)} 
    \big) .
  \end{equation}
  Then 
  it holds for all 
  $v, w \in \Reals^{ \ell_L( \ell_{L-1} + 1)} $,
  $ \lambda\in [0,1]$
that
\begin{equation}
\label{eq:convexity_to_proof}
  \mathbb{L}( \lambda v+(1-\lambda)w)
  \leq \lambda\mathbb{L}(v)+(1-\lambda)\mathbb{L}(w)
  .
\end{equation}
\end{proposition}
\begin{proof}[Proof of \cref{prop:outer_convexity}]
  Throughout this proof
  let $ \psi= ( \psi_1, \ldots, \psi_{\fd} )\colon \Reals^{ \ell_L( \ell_L+1)} \to \Reals^{ \fd } $
  satisfy for all $v= (v_1, \ldots,$
  $v_{\ell_L( \ell_{L-1} + 1)} )
  \in \Reals^{ \ell_L( \ell_{L-1} + 1)} $
  that
  \begin{equation}\label{psi}
  \psi(v) = 
    ( 
      \theta_1, \theta_2, \ldots, \theta_{ \delta },
      v_1, v_2, \ldots, v_{ \ell_L ( \ell_{ L - 1 } + 1 ) } 
    )
  \end{equation}
  and
  for every $v\in \Reals^{ \ell_L( \ell_{L-1} + 1)} $
  let $ \mathscr{N}^{v} \colon \Reals^{ \ell_0 } \to \Reals^{ \ell_k} $
  satisfy for all
  $ x \in \Reals^{ \ell_0 } $
  that
  \begin{equation}\label{scrN}
  \mathscr{N}^{v}(x)
  =\mN_{ \infty }^{L, \psi(v)}(x).
  \end{equation}
  \Nobs that 
  \cref{lem:aff_lin},
  \cref{wb}, and
  \cref{psi}
  ensure that for all
  $v,w \in \Reals^{ \ell_L( \ell_{L-1} + 1)} $, 
  $ \lambda\in [0,1]$,
  $ x \in \Reals^{ \ell_{L-1} } $
  it holds that
  \begin{equation}\label{wb_lincomb}
  \begin{split}
    &\fb^{L, \psi( \lambda v +(1-\lambda)w)} + \fw^{L, \psi( \lambda v +(1-\lambda)w)}x
    =\fb^{L, \lambda\psi(v) +(1-\lambda)\psi(w)} + \fw^{L, \lambda\psi(v) +(1-\lambda)\psi(w)}x\\
    &=\lambda\fb^{L, \psi(v)} + 
    (1-\lambda)\fb^{L, \psi(w)}
    +\lambda\fw^{L, \psi(v)}x+
    (1-\lambda)\fw^{L, \psi(w)}x\\
    &=\lambda\big[\fb^{L, \psi(v)}
    +\fw^{L, \psi(v)}x\big]
    +(1-\lambda)\big[\fb^{L, \psi(w)}
    +\fw^{L, \psi(w)}x\big].
  \end{split}
  \end{equation}  
  Next \nobs that \cref{N}
  shows that for all   
  $v \in \Reals^{ \ell_L( \ell_{L-1} + 1)} $, 
  $ x \in \Reals^{ \ell_0 } $
  it holds that
  \begin{equation}\label{N_lincomb1}
    \mN_{ \infty }^{L, \psi(v)}(x)
    = \begin{cases}
      \fb^{L, \psi(v)} + \fw^{L, \psi(v)}x &\colon L=1 \\[1 ex]
      \fb^{L, \psi(v)} + \fw^{L, \psi(v)} \big( \mathfrak{M}_{ \infty }( \mN^{L-1, \psi(v)}_{ \infty }(x))\big) &\colon L>1.
    \end{cases}
  \end{equation}  
  \Hence that for all   
  $v,w \in \Reals^{ \ell_L( \ell_{L-1} + 1)} $, 
  $ \lambda\in [0,1]$,
  $ x \in \Reals^{ \ell_0 } $
  it holds that
  \begin{equation}\label{N_lincomb2a}
  \begin{split}
    &\mN_{ \infty }^{L, \psi( \lambda v +(1-\lambda)w)}(x)\\
    &= \begin{cases}
      \fb^{L, \psi( \lambda v +(1-\lambda)w)} + \fw^{L, \psi( \lambda v +(1-\lambda)w)}x &\colon L=1 \\[1 ex]
      \fb^{L, \psi( \lambda v +(1-\lambda)w)} + \fw^{L, \psi( \lambda v +(1-\lambda)w)} \big( \mathfrak{M}_{ \infty }( \mN^{L-1, \psi( \lambda v +(1-\lambda)w)}_{ \infty }(x))\big) &\colon L>1.
    \end{cases} \\
  \end{split}
  \end{equation}
  Combining this with \cref{wb_lincomb}
  implies that for all 
  $v,w \in \Reals^{ \ell_L( \ell_{L-1} + 1)} $, 
  $ \lambda\in [0,1]$,
  $ x \in \Reals^{ \ell_0 } $
  it holds that
   \begin{equation}\label{N_lincomb2b}
  \begin{split}
    &\mN_{ \infty }^{L, \psi( \lambda v +(1-\lambda)w)}(x)\\
    &= \begin{cases}
    \lambda\big[\fb^{L, \psi(v)}
    +\fw^{L, \psi(v)}x\big]
    +(1-\lambda)\big[\fb^{L, \psi(w)}
    +\fw^{L, \psi(w)}x\big] 
    & \colon L=1 \\[1 ex]
      \lambda\big[\fb^{L, \psi(v)}
    +\fw^{L, \psi(v)} \big( \mathfrak{M}_{ \infty }( \mN^{L-1, \psi( \lambda v +(1-\lambda)w)}_{ \infty }(x))\big)\big]\\
    +(1-\lambda)\big[\fb^{L, \psi(w)}
    +\fw^{L, \psi(w)} \big( \mathfrak{M}_{ \infty }( \mN^{L-1, \psi( \lambda v +(1-\lambda)w)}_{ \infty }(x))\big)\big] 
    & \colon L>1.
    \end{cases}
  \end{split}
  \end{equation}
  \Cref{it:N_short} of 
  \cref{lem:properties2},
  \cref{wb}, and
  the fact that for all 
  $ i \in \enne\cap [1, \sum_{k= 1 }^{L-1} \ell_k( \ell_{k-1} + 1)]$,
  $v,w\in \Reals^{ \ell_L( \ell_{L-1} + 1)} $ 
  it holds that
  $ \psi_i(v)=\psi_i(w) $
  \hence imply 
  that for all   
  $v,w \in \Reals^{ \ell_L( \ell_{L-1} + 1)} $, 
  $ \lambda\in [0,1]$,
  $ x \in \Reals^{ \ell_0 } $
  it holds that
  \begin{equation}\label{N_lincomb2c}
  \begin{split}  
    \mN_{ \infty }^{L, \psi( \lambda v +(1-\lambda)w)}(x)
    = \begin{cases}
    \lambda\big[\fb^{L, \psi(v)}
    +\fw^{L, \psi(v)}x\big]
    +(1-\lambda)\big[\fb^{L, \psi(w)}
    +\fw^{L, \psi(w)}x\big] 
     & \colon L=1 \\[1 ex]
      \lambda\big[\fb^{L, \psi(v)}
    +\fw^{L, \psi(v)} \big( \mathfrak{M}_{ \infty }( \mN^{L-1, \psi(v)}_{ \infty }(x))\big)\big]\\
    +(1-\lambda)\big[\fb^{L, \psi(w)}
    +\fw^{L, \psi(w)} \big( \mathfrak{M}_{ \infty }( \mN^{L-1, \psi(w)}_{ \infty }(x))\big)\big] 
    & \colon L>1.
    \end{cases}
  \end{split}
  \end{equation}
  Moreover, \nobs that
  \cref{N_lincomb1}
  ensures that for all   
  $v,w \in \Reals^{ \ell_L( \ell_{L-1} + 1)} $, 
  $ \lambda\in [0,1]$,
  $ x \in \Reals^{ \ell_0 } $
  it holds that
  \begin{equation}\label{N_lincomb3}
  \begin{split}  
    &\lambda\, \mN_{ \infty }^{L, \psi(v)}(x)
    +(1-\lambda)\mN_{ \infty }^{L, \psi(w)}(x)\\
    &= \begin{cases}
    \lambda\big[\fb^{L, \psi(v)}
    +\fw^{L, \psi(v)}x\big]
    +(1-\lambda)\big[\fb^{L, \psi(w)}
    +\fw^{L, \psi(w)}x\big] 
     & \colon L=1 \\[1 ex]
      \lambda\big[\fb^{L, \psi(v)}
    +\fw^{L, \psi(v)} \big( \mathfrak{M}_{ \infty }( \mN^{L-1, \psi(v)}_{ \infty }(x))\big)\big]\\
    +(1-\lambda)\big[\fb^{L, \psi(w)}
    +\fw^{L, \psi(w)} \big( \mathfrak{M}_{ \infty }( \mN^{L-1, \psi(w)}_{ \infty }(x))\big)\big] 
    & \colon L>1.
    \end{cases}
  \end{split}
  \end{equation}  
  Combining this with
  \cref{N_lincomb2c}
  demonstrates that for all
  $v,w \in \Reals^{ \ell_L( \ell_{L-1} + 1)} $, 
  $ \lambda\in [0,1]$,
  $ x \in \Reals^{ \ell_0 } $
  it holds that
  \begin{equation}
    \mN_{ \infty }^{L, \psi( \lambda v +(1-\lambda)w)}(x)
    = \lambda\, \mN_{ \infty }^{L, \psi(v)}(x)
    +(1-\lambda)\mN_{ \infty }^{L, \psi(w)}(x).
  \end{equation}
  This and \cref{scrN} 
  show
  that for all
  $v,w \in \Reals^{ \ell_L( \ell_{L-1} + 1)} $, 
  $ \lambda\in [0,1]$,
  $ x \in \Reals^{ \ell_0 } $ it holds that
  \begin{equation}
  \begin{split} \label{N2}
    \mathscr{N}^{ \lambda v +(1-\lambda)w}(x)
    &=\mN_{ \infty }^{L, \psi( \lambda v +(1-\lambda)w)}(x)
     =\lambda\, \mN_{ \infty }^{L, \psi(v)}(x)
    +(1-\lambda)\mN_{ \infty }^{L, \psi(w)}(x)\\
    &=\lambda\, \mathscr{N}^{v}(x)
    +(1-\lambda)\mathscr{N}^{w}(x).
  \end{split}
  \end{equation}
  Next \nobs that \cref{L_r},
  \cref{bbL}, \cref{psi},
  and \cref{scrN} ensure that 
  for all $v\in \Reals^{ \ell_L( \ell_{L-1} + 1)} $
  it holds that
  \begin{equation}
    \mathbb{L}(v)
    =\mL_{ \infty }( \psi(v))
    =\int_{[a,b]^{ \ell_0 } }
  \|\mN_{ \infty }^{L, \psi(v)}(x)-f(x)\|^2\, \mu( \d x)
    =\int_{[a,b]^{ \ell_0 } } \norm{\mathscr{N}^{v}(x)-f(x)}^2\, \mu( \d x).
  \end{equation}  
  Combining this, \cref{N2}, and
  the fact that for all $ \lambda\in [0,1]$, 
  $x,y \in \Reals $
  it holds that $ ( \lambda x+(1-\lambda)y)^2
  \leq \lambda x^2+(1-\lambda)y^2$
  shows that for all
  $v,w\in \Reals^{ \ell_L( \ell_{L-1} + 1)} $, 
  $ \lambda\in [0,1]$ it holds that
  \begin{equation}
  \begin{split}
    &\mathbb{L}( \lambda v+(1-\lambda) w)
    =\int_{[a,b]^{ \ell_0 } } \norm{\mathscr{N}^{ \lambda v +(1-\lambda)w}(x)-f(x)}^2\, \mu( \d x)\\
    &=\int_{[a,b]^{ \ell_0 } } \norm{
    \lambda( \mathscr{N}^{v}(x)-f(x))
    +(1-\lambda)( \mathscr{N}^{w}(x)-f(x))}^2\, \mu( \d x)\\
    &\leq\int_{[a,b]^{ \ell_0 } } \Big[\lambda\norm{
    \mathscr{N}^{v}(x)-f(x)}
    +(1-\lambda)\norm{\mathscr{N}^{w}(x)-f(x)} \Big]^2\, \mu( \d x)\\
    &\leq\int_{[a,b]^{ \ell_0 } } \Big[\lambda\norm{
    \mathscr{N}^{v}(x)-f(x)}^2
    +(1-\lambda)\norm{\mathscr{N}^{w}(x)-f(x)}^2\Big]\, \mu( \d x)\\
  &= \lambda\bigg[\int_{[a,b]^{ \ell_0 } } \norm{
    \mathscr{N}^{v}(x)-f(x)}^2\, \mu( \d x)\bigg]
  +(1-\lambda) \bigg[\int_{[a,b]^{ \ell_0 } }
  \norm{\mathscr{N}^{w}(x)
  -f(x)}^2\, \mu( \d x)\bigg]\\
  &= \lambda \mathbb{L}(v)+(1-\lambda)\mathbb{L}(w).
  \end{split}
\end{equation}  
This establishes \cref{eq:convexity_to_proof}.
The proof of \cref{prop:outer_convexity} is thus complete.
\end{proof}

\begin{cor}[Convexity of the risk function]
\label{cor:convexity_0}
  Assume \cref{main_setting} and assume
  $(L-1)\mathfrak{m} =0$.
  Then 
  it holds for all 
  $ \theta, \vartheta \in \Reals^{ \fd } $,
  $ \lambda\in [0,1]$ that
  \begin{equation}\label{L_convex}
  \mL_{ \infty }( \lambda\theta+(1-\lambda)\vartheta)
  \leq \lambda\mL_{ \infty }( \theta)+(1-\lambda)\mL_{ \infty }( \vartheta).
  \end{equation}
\end{cor}
\begin{proof}[Proof of \cref{cor:convexity_0}]
  Throughout this proof we distinguish between 
  the case $ \mathfrak{m} =0$
  and the case $ \mathfrak{m} \neq 0$.
  We first prove \cref{L_convex} 
  in the case 
  \begin{equation}\label{m=0}
    \mathfrak{m} =0.
  \end{equation}
  \Nobs that  \cref{L_r} and \cref{m=0} ensure
  that for all $ \theta \in \Reals^{ \fd } $
  it holds that
  \begin{equation}
    \mL_{ \infty }( \theta)
    = \int_{[a,b]^{ \ell_0 } } \norm{\mN^{L, \theta }_{ \infty }(x)-\xi}^2\, \mu( \d x)
    = 0.
  \end{equation} 
  \Hence that for all 
  $ \theta, \vartheta \in \Reals^{ \fd } $,
  $ \lambda\in [0,1]$
  it holds that
  \begin{equation}
    \mL_{ \infty }( \lambda\theta +(1-\lambda)\vartheta)
    =0
    = \lambda\mL_{ \infty }( \theta)
    +(1-\lambda)\mL_{ \infty }( \vartheta).
  \end{equation}
  This establishes \cref{L_convex} in the case 
  $ \mathfrak{m} =0$.
  Next we prove \cref{L_convex} in the case 
  \begin{equation}\label{m!=0}
    \mathfrak{m} \neq 0.
  \end{equation} 
  \Nobs that \cref{m!=0} and 
  the assumption that $(L-1)\mathfrak{m} =0$
  imply that $L=1$.
  \cref{prop:outer_convexity}
  \hence demonstrates that for all 
  $ \theta, \vartheta \in \Reals^{ \fd } $,
  $ \lambda\in [0,1]$
  it holds that
  \begin{equation}
    \mL_{ \infty }( \lambda\theta+(1-\lambda)\vartheta)
    \leq \lambda \mL_{ \infty }( \theta)+(1-\lambda)\mL_{ \infty }( \vartheta).
  \end{equation}
  This establishes \cref{L_convex} in the case 
  $ \mathfrak{m} \neq 0$.
  The proof of \cref{cor:convexity_0}
  is thus complete.
\end{proof}

\subsection{Non-convexity properties of the risk function}
\label{ssec:risk_non-convex}

\begin{prop}[Arithmetical averages in the argument of the risk function]
\label{lem:L_equality}
Assume \cref{main_setting}, 
assume $ L > 1 $, 
let $ \xi= ( \xi_1, \ldots, \xi_{\ell_L} ) \in \Reals^{ \ell_L} $, 
let $ ( \alpha_{ i,j,k} )_{ (i,j,k) \in ( \enne_0)^3} \subseteq \Reals $
satisfy for all 
$ i, j \in \enne_0 $, $ k \in \N $ that
$ 
  \alpha_{ i, j, k } = 
  i \ell_{ k - 1 } + j 
  + \sum_{ h = 1 }^{ k - 1 } \ell_h ( \ell_{ h - 1 } + 1 ) 
$,
  and 
  let $ \theta = ( \theta_1, \ldots, \theta_{ \fd }) $,
  $ \vartheta= ( \vartheta_1, \ldots, \vartheta_{ \fd })
  \in \Reals^{ \fd } $ satisfy  
  for all 
  $ i \in \{ 1, \ldots, \ell_L\} $,
  $ j \in \{ 1, \ldots, \fd\} \backslash 
  ( \cup_{k= 1 }^{ \ell_L} \{\alpha_{\ell_{L-1},1, L-1},
  \alpha_{0,1, L}, 
  \alpha_{\ell_L,k, L} \} ) $  
  that
  \begin{equation}\label{alpha1}
  \begin{gathered}
    \theta_{\alpha_{\ell_{L-1},1, L-1} }
    =\vartheta_{\alpha_{0,1, L} }
    =1,
    \qquad
   \theta_{\alpha_{\ell_L,i, L} } 
   = \vartheta_{\alpha_{\ell_L,i, L} } 
   =\xi_{ i}, \\
    \text{and } \qquad
    \theta_{\alpha_{0,1, L} }
     =\vartheta_{\alpha_{\ell_{L-1},1, L-1} }
     =\theta_j=\vartheta_j=0.
     \end{gathered}
   \end{equation} 
Then  
$
  \mL_{ \infty }( \theta ) = \mL_{ \infty }( \vartheta )
$
and
\begin{equation}\label{L_equality}
    \mL_{ \infty }\bigg( \frac{\theta+\vartheta}{2} \bigg)
  = \bigg[\frac{\mL_{ \infty }( \theta)+\mL_{ \infty }( \vartheta)}{2} \bigg]
    +\frac{\mathfrak{m} }{16}
    +\frac{1}{2} \bigg[ \xi_1 \mathfrak{m}
    -\int_{[a,b]^{ \ell_0 } }f_1(x)\, \mu( \d x)\bigg].
  \end{equation}
\end{prop}
\begin{proof}[Proof of \cref{lem:L_equality}]
\Nobs that \cref{it:N_long1,it:N_long2} of \cref{lem:properties1}, 
and \cref{alpha1} ensure that 
for all $ x \in [a,b]^{ \ell_0 } $
it holds that
$ 
  \mN^{L, \theta }_{ \infty }(x)
  =\mN^{L, \vartheta}_{ \infty }(x)
  =\xi
$.
This and \cref{L_r} imply that 
\begin{equation}
\label{L_theta}
  \mL_{ \infty }( \theta ) = \mL_{ \infty }( \vartheta )
  =
  \int\nolimits_{ [a,b]^{ \ell_0 } } 
  \norm{ \xi - f(x) }^2 
  \, \mu( \d x) .
\end{equation}
\Moreover \cref{it:N_long1,it:N_long2} of \cref{lem:properties1} 
and \cref{alpha1} demonstrate that
for all $ x \in [a,b]^{ \ell_0 } $
it holds that 
\begin{equation}
\label{L_theta_vartheta1}
  \mN^{ L, ( \theta + \vartheta ) / 2 }_{ \infty, 1 }( x )
  =
  \textstyle{ \frac{ 1 }{ 2 } } 
  \max\big\{ \frac{ 1 }{ 2 }, 0 \big\} + \xi_1
  = \xi_1 + \frac{ 1 }{ 4 }
  .
\end{equation}
\Moreover \cref{it:N_long1,it:N_long2} of \cref{lem:properties1} 
and \cref{alpha1} ensure that
for all $ i \in \enne \cap (1, \ell_L ] $, 
$ x \in [a,b]^{ \ell_0 } $
it holds that 
$ 
  \mN^{ L, (\theta + \vartheta) / 2 }_{ \infty, i }( x ) = \xi_i 
$.
Combining this with \cref{L_theta} and \cref{L_theta_vartheta1} 
shows that  
\begin{equation}
\begin{split}
  \mL_{ \infty }\bigg( 
    \frac{ \theta +\vartheta }{ 2 } 
  \bigg)
&
  =
  \int_{ [a,b]^{ \ell_0 } } 
  \biggl[
    \bigl( \xi_1 + \tfrac{1}{4} - f_1(x) \bigr)^2 
    +
    \textstyle
    \sum\limits_{ i = 2 }^{ \ell_L } 
    \bigl( 
      \xi_i - f_i(x) 
    \bigr)^2 
  \biggr]
  \, \mu( \d x )
\\
&
  =
  \biggl[
    \int_{ [a,b]^{ \ell_0 } } 
    \norm{ \xi - f(x) }^2 \, 
    \mu( \d x )
  \biggr]
  +
  \biggl[
    \int_{ [a,b]^{ \ell_0 } } 
    \tfrac{ 1 }{ 16 } \, 
    \mu( \d x)
  \biggr]
  +
  \biggl[
    \int_{ [a,b]^{ \ell_0 } } 
    \tfrac{ ( \xi_1 - f_1(x) ) }{ 2 } 
    \, \mu( \d x )
  \biggr]
\\
&
  = 
  \bigg[
    \frac{
      \mL_{ \infty }( \theta ) + 
      \mL_{ \infty }( \vartheta ) 
    }{ 2 } 
  \bigg]
  + \frac{ \mathfrak{m} }{ 16 }
  + 
  \frac{ 1 }{ 2 } 
  \bigg[ 
    \xi_1 \mathfrak{m}
    - \int_{ [a,b]^{ \ell_0 } } f_1(x) \, \mu( \d x ) 
  \bigg]
  .
\end{split}
\end{equation}  
This establishes \cref{L_equality}. 
The proof of \cref{lem:L_equality} is thus complete.
\end{proof}

\begin{cor}[Non-convexity of the risk function]
\label{cor:risk_not_convex}
Assume \cref{main_setting}, 
assume $ (L-1) \mathfrak{m} \neq 0 $,
let 
$ 
  \xi= ( \xi_1, \ldots, \xi_{\ell_L} ) \in \Reals^{ \ell_L} 
$ 
satisfy 
$ 
  \xi_{1} =\mathfrak{m}^{-1} \int_{[a,b]^{ \ell_0 } }f_1(x)\, \mu( \d x) 
$, 
  let $ ( \alpha_{ i,j,k} )_{ (i,j,k) \in ( \enne_0)^3} \subseteq\Reals $
  satisfy for all 
  $i,j\in \enne_0 $, $ k \in \enne $ that
  $ \alpha_{ i,j,k} =i\ell_{k-1} + j+\sum_{h= 1 }^{ k-1} \ell_h( \ell_{h-1} + 1) $,
  and 
  let $ \theta = ( \theta_1, \ldots, \theta_{ \fd }) $,
  $ \vartheta= ( \vartheta_1, \ldots, \vartheta_{ \fd })
  \in \Reals^{ \fd } $ satisfy 
  for all 
  $ i \in \{ 1, \ldots, \ell_L\} $,
  $ j \in \{ 1, \ldots, \fd\} \backslash 
  ( \cup_{k= 1 }^{ \ell_L} \{\alpha_{\ell_{L-1},1, L-1},
  \alpha_{0,1, L}, 
  \alpha_{\ell_L,k, L} \} ) $  
  that
  \begin{equation}\label{alpha2}
  \begin{gathered}
    \theta_{\alpha_{\ell_{L-1},1, L-1} }
    =\vartheta_{\alpha_{0,1, L} }
    =1,
    \qquad
   \theta_{\alpha_{\ell_L,i, L} } 
   = \vartheta_{\alpha_{\ell_L,i, L} } 
   =\xi_{ i}, \\
    \text{and } \qquad
    \theta_{\alpha_{0,1, L} }
     =\vartheta_{\alpha_{\ell_{L-1},1, L-1} }
     =\theta_j=\vartheta_j=0.
     \end{gathered}
   \end{equation} 
  Then 
  \begin{equation}\label{risk_not_convex1}
    \mL_{ \infty }\bigg( \frac{\theta+\vartheta}{2} \bigg)
    =\bigg[\frac{\mL_{ \infty }( \theta)+\mL_{ \infty }( \vartheta)}{2} \bigg]+\frac{\mathfrak{m} }{16}
  > \frac{\mL_{ \infty }( \theta)+\mL_{ \infty }( \vartheta)}{2}.
  \end{equation}
\end{cor}
\begin{proof}[Proof of \cref{cor:risk_not_convex}]
  \Nobs that \cref{lem:L_equality}
  and the assumption that 
  $ \xi_{1}
  =\mathfrak{m}^{-1} \int_{[a,b]^{ \ell_0 } }f_1(x)\, \mu( \d x) $
  demonstrate that
  \begin{equation}
  \begin{split}
    \mL_{ \infty }\bigg( \frac{\theta+\vartheta}{2} \bigg)
    &=\bigg[\frac{\mL_{ \infty }( \theta)+\mL_{ \infty }( \vartheta)}{2} \bigg]
    +\frac{\mathfrak{m} }{16}
    +\frac{1}{2} \bigg[ \xi_1 \mathfrak{m}
    -\int_{[a,b]^{ \ell_0 } }f_1(x)\mu( \d x)\bigg]\\
    &=\bigg[\frac{\mL_{ \infty }( \theta)+\mL_{ \infty }( \vartheta)}{2} \bigg]+\frac{\mathfrak{m} }{16}.
  \end{split}
  \end{equation}  
  The fact that $ \mathfrak{m} \neq0$ \hence
  implies that 
  \begin{equation}
  \mL_{ \infty }\bigg( \frac{\theta+\vartheta}{2} \bigg)
  > \frac{\mL_{ \infty }( \theta)+\mL_{ \infty }( \vartheta)}{2}.
  \end{equation}
  The proof of \cref{cor:risk_not_convex}
  is thus complete.
\end{proof}

\begin{cor}[Characterization of convexity of the risk function]
\label{cor:conv_characterization}
Assume~\cref{main_setting}.
Then it holds that 
$ (L-1) \mathfrak{m} = 0 $
if and only if
it holds for all $ \theta, \vartheta \in \Reals^{ \fd } $, $ \lambda\in [0,1]$ that
\begin{equation}
\label{L_convex2}
  \mL_{ \infty }( \lambda \theta + ( 1 - \lambda) \vartheta )
  \leq \lambda \mL_{ \infty }( \theta )
  + (1 - \lambda) \mL_{ \infty }( \vartheta )
  .
\end{equation}
\end{cor}
\begin{proof}[Proof of \cref{cor:conv_characterization}]
\Nobs that \cref{cor:convexity_0} and \cref{cor:risk_not_convex}
establish \cref{L_convex2}. 
The proof of \cref{cor:conv_characterization} is thus complete.
\end{proof}

\section{Gradient flow (GF) processes in the training of deep ANNs}
\label{sec:grad_fl}

In this section we establish 
in \cref{thm:GF_conv} in \cref{ssec:GF_convergence} below 
in the training of deep ReLU ANNs,  
under the assumption that the target function 
$ f \colon [a,b]^{ \ell_0 } \to \Rr $ in \cref{L_r} 
is a constant function, 
that the risk of 
every solution 
$ \Theta = ( \Theta_t )_{ t \in [0,\infty) } \colon [0,\infty) \to \Rr^{ \fd } $
of the associated GF differential equation 
converges with rate $ 1 $ to zero. 
Our proof of \cref{thm:GF_conv} 
uses the elementary regularity result for the 
Lyapunov function 
$ V \colon \Rr^{ \fd } \to \Rr $ 
in \cref{prop:liap1} 
in \cref{ssec:GF_lyapunov} below
(cf.\ \cref{V_c} in \cref{main_setting} above), 
the weak chain rule for compositions of the Lyapunov function 
$ V \colon \Rr^{ \fd } \to \Rr $ 
and GF solutions in the constant target function case 
in \cref{cor:GF_det_ito} in \cref{ssec:GF_weak_chain_lyapunov} below, 
and 
the weak chain rule for compositions 
of the risk function $ \mL_{ \infty } \colon \Rr^{ \fd } \to \Rr $ 
and GF solutions 
in the general target function case 
in \cref{lem:GF_det_ito2} 
in \cref{ssec:GF_weak_chain_risk} below 
(cf.\ \cref{L_r} in \cref{main_setting}).

Our proof of the weak chain rule for compositions of $ V \colon \Rr^{ \fd } \to \Rr $ 
and GF solutions in the constant target function case in \cref{cor:GF_det_ito}, in turn, 
employs 
the weak chain rule for compositions of $ V \colon \Rr^{ \fd } \to \Rr $ 
and GF solutions in the general target function case in \cref{prop:GF_det_ito} 
in \cref{ssec:GF_weak_chain_lyapunov}. 
Our proof of the weak chain rule 
for compositions of $ V \colon \Rr^{ \fd } \to \Rr $ 
and GF solutions in the general target function case in \cref{prop:GF_det_ito} 
make use of 
the local boundedness result for the generalized gradient 
function $ \mG \colon \Rr^{ \fd } \to \Rr $
in \cref{cor:G_up} in \cref{ssec:upper_estimates} above
as well as of the identity for the scalar product of 
$
  \Rr^{ \fd } \ni \theta \mapsto ( \nabla V )( \theta ) \in \Rr^{ \fd }
$
and 
$
  \mG \colon \Rr^{ \fd } \to \Rr^{ \fd }
$
in 
\cref{prop:liap2}
in \cref{ssec:GF_lyapunov}. 
In \cref{cor:liap} in \cref{ssec:GF_lyapunov} 
we also specialize \cref{prop:liap2} 
to the constant target function case. 
In particular, \cref{cor:liap} implies that 
the scalar product of 
$
  \Rr^{ \fd } \ni \theta \mapsto ( \nabla V )( \theta ) \in \Rr^{ \fd }
$
and 
$ \mG $ 
is non-negative and, thereby, proves that $ V $ serves 
as a Lyapunov function for the considered GF processes.

Our proof of the weak chain rule for compositions 
of $ \mL_{ \infty } \colon \Rr^{ \fd } \to \Rr $ 
and GF solutions 
in the general target function case 
in \cref{lem:GF_det_ito2} 
employs \cref{prop:G} in \cref{ssec:gen_grad_risk_functions} above, 
\cref{cor:G_up} in \cref{ssec:upper_estimates}, 
and the uniform local boundedness result for the gradients 
$
  \Rr^{ \fd } \ni \theta \mapsto ( \nabla \mL_r )( \theta ) \in \Rr^{ \fd }
$, 
$
  r \in [1,\infty)
$,
of the approximated risk functions 
$
  \mL_r \colon \Rr^{ \fd } \to \Rr 
$,
$ r \in [1,\infty) $, 
in \cref{lem:grad_L_bd} in \cref{ssec:GF_weak_chain_risk} 
(cf.\ \cref{L_r} in \cref{main_setting}).

The results in this section extend the findings 
in \cite[Subsection 2.6]{MR4468133} 
and \cite[Section 3]{MR4438169} 
from shallow ReLU ANNs with just one hidden layer 
to deep ReLU ANNs with an arbitrarily large number of hidden layers. 
In particular, 
\cref{prop:liap1} in \cref{ssec:GF_lyapunov} extends \cite[Proposition 2.8]{MR4468133},
\cref{prop:liap2} in \cref{ssec:GF_lyapunov} extends \cite[Proposition 2.9]{MR4468133}, 
\cref{cor:liap} in \cref{ssec:GF_lyapunov} extends \cite[Corollary 2.10]{MR4468133}, 
\cref{prop:GF_det_ito} in \cref{ssec:GF_weak_chain_lyapunov} 
and 
\cref{cor:GF_det_ito} in \cref{ssec:GF_weak_chain_lyapunov} 
extend \cite[Lemma 3.2]{MR4438169} 
(cf.\ also \cite[Proposition 4.1]{MR4473797}), 
\cref{lem:grad_L_bd} in \cref{ssec:GF_weak_chain_risk} 
extends 
\cite[Lemma 3.4]{MR4438169},
\cref{lem:GF_det_ito2} in \cref{ssec:GF_weak_chain_risk} 
extends \cite[Lemma 3.5]{MR4438169},
and 
\cref{thm:GF_conv} in \cref{ssec:GF_convergence}
extends \cite[Theorem 3.7]{MR4438169}.

\subsection{Lyapunov type estimates for the dynamics of GF processes} 
\label{ssec:GF_lyapunov}

\begin{proposition}
  \label{prop:liap1}
  Assume \cref{main_setting} and let $ \theta \in \Reals^{ \fd } $. Then
  \begin{enumerate}[label=(\roman*)]
  \item\label{it:V_expl} it holds that 
\begin{equation}
\textstyle
\label{V_expl}
  V( \theta ) 
  =
  \sum_{ k = 1 }^L 
  \bigl(
    k 
    \bigl[ 
      \sum_{ i = 1 }^{ \ell_k }
      | \fb^{ k, \theta }_i |^2 
    \bigr] 
    + 
    \bigl[ 
      \sum_{ i = 1 }^{ \ell_k } 
      \sum_{ j = 1 }^{ \ell_{ k - 1 } }
      |
        \fw^{ k, \theta }_{ i, j } 
      |^2 
    \bigr]     
    - 2 L 
    \bigl[ 
      \sum_{ i = 1 }^{ \ell_L } 
      f_i(0) 
      \fb^{ L, \theta }_i 
    \bigr] 
  \bigr) ,
\end{equation}
  \item\label{it:ineq} it holds that
  $ \frac{1}{2} \|\theta \|^2-2L^2\|f(0)\|^2 \leq V( \theta) \leq 2L\|\theta \|^2 + L\|f(0)\|^2 $, and
  \item\label{it:grad_V} it holds that
  \begin{equation}
  \label{grad_V}
  ( \nabla V)( \theta)=
  2 \pr*{  \fw^{ 1, \theta }, \fb^{ 1, \theta },
  \fw^{2, \theta },2\fb^{2, \theta },
  \ldots,
  \fw^{L-1, \theta },(L-1)\fb^{L-1, \theta },
   \fw^{L, \theta }, L( \fb^{L, \theta }-f(0))  } \,.
  \end{equation}
  \end{enumerate}
\end{proposition} 
\begin{proof}[Proof of \cref{prop:liap1}]
  \Nobs that \cref{V_c}
  establishes \cref{it:V_expl}.
  Next \nobs that \cref{V_c}, the Cauchy-Schwarz
  inequality, and the Young inequality 
  demonstrate that
  \begin{equation}
  \begin{split} \label{V_ineq1}
    &V( \theta)
    =\bigg[\sum_{k= 1 }^L(
  k\|\fb^{ k, \theta } \|^2 +
  \sum_{ i= 1 }^{ \ell_k} \sum_{ j = 1 }^{ \ell_{k-1} }
  |\fw^{ k, \theta }_{ i, j}|^2 )\bigg]
  - 2L\langle f(0), \fb^{L, \theta } \rangle\\
    &=\norm{\theta }^2+\sum_{k= 1 }^L (k-1)\norm{\fb^{ k, \theta } }^2-2L\langle f(0), \fb^{L, \theta } \rangle
    \geq \norm{\theta }^2-2L\langle f(0), \fb^{L, \theta } \rangle\\
    &\geq \norm{\theta }^2-2L\norm{f(0)} \norm{\fb^{L, \theta } } 
    \geq \norm{\theta }^2-2L^2\norm{f(0)}^2-\frac{1}{2} \norm{\fb^{L, \theta } }^2
    \geq \frac{1}{2} \norm{\theta }^2-2L^2\norm{f(0)}^2.\\
  \end{split}
\end{equation}
\Moreover \cref{V_c} shows that
\begin{equation}
  \begin{split} \label{V_ineq2a}
    V( \theta)
    & =\bigg[\sum_{k= 1 }^L(
  k\|\fb^{ k, \theta } \|^2 +
  \sum_{ i= 1 }^{ \ell_k} \sum_{ j = 1 }^{ \ell_{k-1} }
  |\fw^{ k, \theta }_{ i, j}|^2 )\bigg]
  - 2L\langle f(0), \fb^{L, \theta } \rangle\\
    &=\norm{\theta }^2+\sum_{k= 1 }^L (k-1)\norm{\fb^{ k, \theta } }^2-2L\langle f(0), \fb^{L, \theta } \rangle 
    \leq \norm{\theta }^2+\sum_{k= 1 }^L (k-1)\norm{\fb^{ k, \theta } }^2+\abs{2L\langle f(0), \fb^{L, \theta } \rangle} .
  \end{split}
  \end{equation}
  The Cauchy-Schwarz
  inequality and the Young inequality
  \hence imply that
  \begin{equation}
  \begin{split} \label{V_ineq2b}
    V( \theta)
     &\leq \norm{\theta }^2+\sum_{k= 1 }^L (k-1)\norm{\fb^{ k, \theta } }^2+2L\norm{f(0)} \norm{\fb^{L, \theta } } \\
    &\leq \norm{\theta }^2+\sum_{k= 1 }^L (k-1)\norm{\fb^{ k, \theta } }^2+L\norm{f(0)}^2+L\norm{\fb^{L, \theta } }^2
    \leq 2L\norm{\theta }^2+L\norm{f(0)}^2.
  \end{split}
  \end{equation}
Combining this and \cref{V_ineq1} proves \cref{it:ineq}.
\Nobs that \cref{it:V_expl} establishes \cref{it:grad_V}.
The proof of \cref{prop:liap1} is thus complete.
\end{proof}

\begin{proposition}
\label{prop:liap2}
Assume \cref{main_setting} and let $ \theta \in \Reals^{ \fd } $.
Then
\begin{equation}
\label{eq_V}
  \ip{ ( \nabla V)( \theta)}{\mG( \theta)} 
  =
  4 L 
  \br*{ 
    \int_{ [a,b]^{ \ell_0 } } 
    \ip{
      \mN^{ L, \theta }_{ \infty }( x ) - f( x ) 
    }{
      \mN^{ L, \theta }_{ \infty }( x ) - f(0) 
    } 
    \, \mu( \d x) 
  }
  .
\end{equation}
\end{proposition}
\begin{proof}[Proof of \cref{prop:liap2}]
Throughout this proof let $(\diml_k)_{k \in \N_0} \subseteq \N_0$
	satisfy for all $k \in \N_0$ that $\diml_k = \sum_{n=1}^k \ell_n ( \ell_{n-1} + 1 ) $.
  \Nobs that
  \cref{it:grad_V}
  of \cref{prop:liap1}
  demonstrates that
  \begin{equation}
  \label{VG}
  \begin{split}
&
    \langle( \nabla V( \theta), \mG( \theta)\rangle
    = \sum_{k= 1 }^L\sum_{ i_k= 1 }^{ \ell_k}
    2k\fb^{ k, \theta }_{ i_k} \mG_{\ell_k \ell_{k-1} + i_k + \diml_{k-1} }( \theta)\\
    &
    +\sum_{k= 1 }^L
    \sum_{ i_k= 1 }^{ \ell_k}
    \sum_{ i_{k-1} = 1 }^{ \ell_{k-1} }
    2\fw^{ k, \theta }_{ i_k,i_{k-1} } \mG_{ (i_k-1)\ell_{k-1} + i_{k-1} +  \diml_{k - 1 } }( \theta)
    -\sum_{ i_L= 1 }^{ \ell_L}2Lf_{ i_L}(0)\mG_{\ell_{L} \ell_{L-1} + i_{L} + \diml_{L - 1 } }( \theta).
  \end{split}
  \end{equation}
  Next we claim that for all 
  $ k \in \{ 1, \ldots, L\} $
  it holds that
  \begin{equation}   
  \label{sum_ind}
  \begin{split}
  &\sum_{m=k}^L
  \sum_{ i_{m} = 1 }^{ \ell_{m} }
  \fb^{m, \theta }_{ i_{m} }
  \mG_{\ell_{m} \ell_{m-1} + i_{m} + \diml_{m - 1 } }( \theta)
  +\sum_{ i_k= 1 }^{ \ell_k}
  \sum_{ i_{k-1} = 1 }^{ \ell_{k-1} }
  \fw^{ k, \theta }_{ i_k, i_{k-1} }
  \mG_{ (i_k-1)\ell_{k-1} + i_{k-1} + \diml_{k-1 } }( \theta)\\
  &=2\int_{[a,b]^{ \ell_0 } }
  \ip{\mN^{L, \theta }_{ \infty }(x)-f(x)}{\mN^{L, \theta }_{ \infty }(x)} \, \mu( \d x).
  \end{split}
  \end{equation}
  We prove \cref{sum_ind} by induction on
  $ k \in \{ 1, \ldots, L\} $.
  \Nobs that \cref{G4',G4''}
  of \cref{prop:G} and \cref{N} ensure that
  \begin{equation}
  \begin{split}
    &\sum_{ i= 1 }^{ \ell_{L} }
    \fb^{L, \theta }_{ i}
    \mG_{\ell_{L} \ell_{L-1} + i+ \diml_{L - 1 } }( \theta)
    +\sum_{ i= 1 }^{ \ell_L}
    \sum_{ j = 1 }^{ \ell_{L-1} }
    \fw^{L, \theta }_{ i, j}
    \mG_{ (i-1)\ell_{L-1} + j+ \diml_{L - 1 } }( \theta)\\
    &=\sum_{ i= 1 }^{ \ell_{L} }
    2\fb^{L, \theta }_{ i}
    \int_{[a,b]^{ \ell_0 } }( \mN^{L, \theta }_{ \infty, i}(x)-f_{ i}(x))\mu( \d x)\\
    &\quad +\sum_{ i= 1 }^{ \ell_L}
    \sum_{ j = 1 }^{ \ell_{L-1} }
    2\fw^{L, \theta }_{ i, j}
    \int_{[a,b]^{ \ell_0 } } 
    \Big[
  \R_{ \infty }( \mN^{ \max\{L-1,1\}, \theta }_{ \infty,j}(x))
  \indicator{  (1, \infty)}(L)
  +x_j \indicator{ \{1\} }(L)
  \Big]    
    ( \mN^{L, \theta }_{ \infty, i}(x)-f_{ i}(x))\mu( \d x)\\
    &=2\int_{[a,b]^{ \ell_0 } }
    \ip{\mN^{L, \theta }_{ \infty }(x)-f(x)}{\mN^{L, \theta }_{ \infty }(x)} \, \mu( \d x).
  \end{split}
  \end{equation}
  This establishes \cref{sum_ind} 
  in the base case $k=L$.
  For the induction step let $ k \in \enne\cap [2, L]$
  satisfy 
  \begin{equation}   
  \label{sum_indhyp}
  \begin{split}
    &\sum_{m=k}^L
    \sum_{ i_{m} = 1 }^{ \ell_{m} }
    \fb^{m, \theta }_{ i_{m} }
    \mG_{\ell_{m} \ell_{m-1} + i_{m} + \diml_{m - 1 } }( \theta)
    +\sum_{ i_k= 1 }^{ \ell_k}
    \sum_{ i_{k-1} = 1 }^{ \ell_{k-1} }
    \fw^{ k, \theta }_{ i_k, i_{k-1} }
    \mG_{ (i_k-1)\ell_{k-1} + i_{k-1} + \diml_{k - 1 } }( \theta)\\
    &=2\int_{[a,b]^{ \ell_0 } }
    \ip{\mN^{L, \theta }_{ \infty }(x)-f(x)}{\mN^{L, \theta }_{ \infty }(x)} \, \mu( \d x).
  \end{split}
  \end{equation}  
  \Nobs that \cref{G4',G4''}
  of \cref{prop:G} 
  show that
  \begin{equation}   
  \label{sum_k_1a}
  \begin{split}
    &\sum_{m=k-1}^L
    \sum_{ i_{m} = 1 }^{ \ell_{m} }
    \fb^{m, \theta }_{ i_{m} }
    \mG_{\ell_{m} \ell_{m-1} + i_{m} + \diml_{m - 1 } }( \theta)
    +\sum_{ i_{k-1} = 1 }^{ \ell_{k-1} }
    \sum_{ i_{k-2} = 1 }^{ \ell_{k-2} }
    \fw^{ k-1, \theta }_{ i_{k-1}, i_{k-2} }
    \mG_{ (i_{k-1}-1)\ell_{k-2} + i_{k-2} + \diml_{ k - 2 }}( \theta)\\ 
    &=
    \sum_{m=k}^L
    \sum_{ i_{m} = 1 }^{ \ell_{m} }
    \fb^{m, \theta }_{ i_{m} }
    \mG_{\ell_{m} \ell_{m-1} + i_{m} + \diml_{m - 1 } }( \theta)
     +\sum_{ i_{k-1} = 1 }^{ \ell_{k-1} }
    \fb^{ k-1, \theta }_{ i_{k-1} }
    \mG_{\ell_{k-1} \ell_{k-2} + i_{k-1} +  \diml_{ k - 2 } }( \theta)\\
    &\quad
    +\sum_{ i_{k-1} = 1 }^{ \ell_{k-1} }
    \sum_{ i_{k-2} = 1 }^{ \ell_{k-2} }
    \fw^{ k-1, \theta }_{ i_{k-1}, i_{k-2} }
    \mG_{ (i_{k-1}-1)\ell_{k-2} + i_{k-2} + \diml_{ k - 2 } }( \theta)\\
    &=
    \sum_{m=k}^L
    \sum_{ i_{m} = 1 }^{ \ell_{m} }
    \fb^{m, \theta }_{ i_{m} }
    \mG_{\ell_{m} \ell_{m-1} + i_{m} + \diml_{m - 1 } }( \theta)\\
    &\quad +2\sum_{ i_{k-1} = 1 }^{ \ell_{k-1} }
    \Bigg[\int_{[a,b]^{ \ell_0 } }
    \Bigg( \fb^{ k-1, \theta }_{ i_{k-1} }
    +\sum_{ i_{k-2} = 1 }^{ \ell_{k-2} }
    \fw^{ k-1, \theta }_{ i_{k-1}, i_{k-2} } \\
    &\quad\cdot
    \Big[\R_{ \infty }( \mN^{ \max\{k-2,1\}, \theta }_{ \infty, i_{k-2} }(x))
  \indicator{ (1, L]}(k-1)
  +x_{ i_{k-2} } \indicator{\{1\} }(k-1)\Big]\Bigg)\\
    &\quad\cdot
    \sum\nolimits_{\substack{v_{k-1},v_k, \ldots,v_L\in \enne, \\
    ( \forall m\in \enne\cap[k - 1 , L]\colon v_m\leq\ell_m)} }
    \bigg[
     \pr*{  \mN^{L, \theta }_{ \infty,v_L}(x)-f_{v_L}(x) } 
    \\
    &\quad\cdot
    \indicator{\{i_{k-1} \} }(v_{k-1} )
    \Big[\textstyle{\prod}_{q={k} }^{L} \big(
    \fw^{q, \theta }_{v_{q}, v_{q-1} }
    \big[\indicator{\mX^{q-1, \theta }_{v_{q-1} }}(x)\big]
    \big)\Big]\bigg]\Bigg]\, \mu( \d x)
    \Bigg].
    \end{split}
    \end{equation}
  \Hence that
    \begin{equation} 
    \begin{split} \label{sum_k_1b}
    &\sum_{m=k-1}^L
    \sum_{ i_{m} = 1 }^{ \ell_{m} }
    \fb^{m, \theta }_{ i_{m} }
    \mG_{\ell_{m} \ell_{m-1} + i_{m} + \diml_{m - 1 } }( \theta)
    +\sum_{ i_{k-1} = 1 }^{ \ell_{k-1} }
    \sum_{ i_{k-2} = 1 }^{ \ell_{k-2} }
    \fw^{ k-1, \theta }_{ i_{k-1}, i_{k-2} }
    \mG_{ (i_{k-1}-1)\ell_{k-2} + i_{k-2} + \diml_{k - 2 } }( \theta)\\
    &=
    \sum_{m=k}^L
    \sum_{ i_{m} = 1 }^{ \ell_{m} }
    \fb^{m, \theta }_{ i_{m} }
    \mG_{\ell_{m} \ell_{m-1} + i_{m} + \diml_{m - 1 } }( \theta)\\
    &\quad +2\sum_{ i_{k-1} = 1 }^{ \ell_{k-1} }
    \Bigg[\int_{[a,b]^{ \ell_0 } }
    \sum_{ i_k = 1 }^{ \ell_k}
    \fw^{ k, \theta }_{ i_k, i_{k-1} }
    \big[\indicator{\mX^{ k-1, \theta }_{ i_{k-1} }}(x)\big]
    \Bigg( \fb^{ k-1, \theta }_{ i_{k-1} }
    +\sum_{ i_{k-2} = 1 }^{ \ell_{k-2} }
    \fw^{ k-1, \theta }_{ i_{k-1}, i_{k-2} } \\
    &\quad    
    \Big[\R_{ \infty }( \mN^{ \max\{k-2,1\}, \theta }_{ \infty, i_{k-2} }(x))
  \indicator{ (1, L]}(k-1)
  +x_{ i_{k-2} } \indicator{ \{1\} }(k-1)\Big]
    \Bigg)\\
    &\quad\cdot
    \sum_{\substack{v_k,v_{k+1}, \ldots,v_L\in \enne, \\
    \forall w\in \enne\cap[k, L]\colon v_w\leq\ell_w} }
    \Big[\mN^{L, \theta }_{ \infty,v_L}(x)-f_{v_L}(x)\Big]
    \Big[\indicator{\{i_{k-1} \} }(v_k )\Big]\\
    &\quad\cdot
    \Big[\textstyle{\prod}_{n={k+1} }^{L} \big(
    \fw^{n, \theta }_{v_n, v_{n-1} }
    \big[\indicator{\mX^{n-1, \theta }_{v_{n-1} }}(x)\big]
    \big)\Big]\, \mu( \d x)
    \Bigg].
    \end{split}
    \end{equation}
  \Cref{G4'} 
  of \cref{prop:G}, 
  \cref{N}, and \cref{sum_indhyp} 
  \hence imply that
    \begin{equation}\label{sum_k_2}
    \begin{split}
    &\sum_{m=k-1}^L
    \sum_{ i_{m} = 1 }^{ \ell_{m} }
    \fb^{m, \theta }_{ i_{m} }
    \mG_{\ell_{m} \ell_{m-1} + i_{m} + \diml_{m - 1 } }( \theta)
    +\sum_{ i_{k-1} = 1 }^{ \ell_{k-1} }
    \sum_{ i_{k-2} = 1 }^{ \ell_{k-2} }
    \fw^{ k-1, \theta }_{ i_{k-1}, i_{k-2} }
    \mG_{ (i_{k-1}-1)\ell_{k-2} + i_{k-2} + \diml_{k - 2 } }( \theta)\\ 
    &=
    \sum_{m=k}^L
    \sum_{ i_{m} = 1 }^{ \ell_{m} }
    \fb^{m, \theta }_{ i_{m} }
    \mG_{\ell_{m} \ell_{m-1} + i_{m} + \diml_{m - 1 } }( \theta)\\
    &\quad 
    +2\sum_{ i_{k-1} = 1 }^{ \ell_{k-1} }
    \Bigg[\int_{[a,b]^{ \ell_0 } }
    \sum_{ i_k = 1 }^{ \ell_k}
    \fw^{ k, \theta }_{ i_k, i_{k-1} }
    \big[\R_{ \infty }( \mN^{ k-1, \theta }_{ \infty, i_{k-1} }(x))\big]
    \\
    &\quad\cdot
    \sum_{\substack{v_k,v_{k+1}, \ldots,v_L\in \enne, \\
    \forall w\in \enne\cap[k, L]\colon v_w\leq\ell_w} }
    \Big[\mN^{L, \theta }_{ \infty,v_L}(x)-f_{v_L}(x)\Big]
    \Big[\indicator{ \{i_{k-1} \} }(v_k)\Big]\\
    &\quad\cdot
    \Big[\textstyle{\prod}_{n={k+1} }^{L} \big(
    \fw^{n, \theta }_{v_n, v_{n-1} }
    \indicator{ \mX^{n-1, \theta }_{v_{n-1} }}(x)
    \big)\Big]\, \mu( \d x)
    \Bigg]\\
    &=
    \sum_{m=k}^L
    \sum_{ i_{m} = 1 }^{ \ell_{m} }
    \fb^{m, \theta }_{ i_{m} }
    \mG_{\ell_{m} \ell_{m-1} + i_{m} + \diml_{m - 1 } }( \theta ) +
    \sum_{ i_k= 1 }^{ \ell_k}    
    \sum_{ i_{k-1} = 1 }^{ \ell_{k-1} }
    \fw^{ k, \theta }_{ i_k,i_{k-1} }
    \mG_{ (i_k-1)\ell_{k-1} + i_{k-1} + \diml_{k - 1 } } ( \theta) \\
    &=2\int_{[a,b]^{ \ell_0 } }
    \ip{\mN^{L, \theta }_{ \infty }(x)-f(x)}{\mN^{L, \theta }_{ \infty }(x)} \, \mu( \d x).
  \end{split}
  \end{equation}
  Induction thus establishes 
  \cref{sum_ind}.
  Next \nobs that \cref{sum_ind} ensures that
  \begin{equation}
  \label{sum_sum}
  \begin{split}
    &\sum_{k= 1 }^L
    k\sum_{ i_k = 1 }^{ \ell_k}
    \fb^{ k, \theta }_{ i_k}
    \mG_{\ell_k \ell_{k-1} + i_k + \diml_{k - 1 } }( \theta)
    +\sum_{k= 1 }^L
    \sum_{ i_k= 1 }^{ \ell_k}
    \sum_{ i_{k-1} = 1 }^{ \ell_{k-1} }
    \fw^{ k, \theta }_{ i_k, i_{k-1} }
    \mG_{ (i_k-1)\ell_{k-1} + i_{k-1} + \diml_{k - 1 } }( \theta)\\
    &=2L\int_{[a,b]^{ \ell_0 } }
    \ip{\mN^{L, \theta }_{ \infty }(x)-f(x)}{\mN^{L, \theta }_{ \infty }(x)} \, \mu( \d x).
  \end{split}
  \end{equation}
  Combining this 
  and \cref{VG}
  demonstrates that
  \begin{equation}
  \begin{split}
    \langle( \nabla V( \theta), \mG( \theta)\rangle
    &= \sum_{k= 1 }^L\sum_{ i_k= 1 }^{ \ell_k}
    2k\fb^{ k, \theta }_{ i_k} \mG_{\ell_k \ell_{k-1} + i_k + \diml_{k - 1 } }( \theta)\\
    &\quad +\sum_{k= 1 }^L
    \sum_{ i_k= 1 }^{ \ell_k}
    \sum_{ i_{k-1} = 1 }^{ \ell_{k-1} }
    2\fw^{ k, \theta }_{ i_k,i_{k-1} } \mG_{ (i_k-1)\ell_{k-1} + i_{k-1} + \diml_{k - 1 } }( \theta)\\
    &\quad -\sum_{ i_L= 1 }^{ \ell_L}2Lf_{ i_L}(0)\mG_{\ell_{L} \ell_{L-1} + i_{L} + \diml_{ L - 1 } }( \theta)\\
    &= 4L\int_{[a,b]^{ \ell_0 } }
    \ip{\mN^{L, \theta }_{ \infty }(x)-f(x)}{\mN^{L, \theta }_{ \infty }(x)} \, \mu( \d x)\\
    &\quad-4L\sum_{ i_L= 1 }^{ \ell_L}
    \Bigg[\int_{[a,b]^{ \ell_0 } }f_{ i_L}(0)( \mN^{L, \theta }_{ \infty, i_L}(x)-f_{ i_L}(x))\, \mu( \d x)\Bigg]\\
    &=4L\int_{[a,b]^{ \ell_0 } }
    \ip{\mN^{L, \theta }_{ \infty }(x)-f(x)}{\mN^{L, \theta }_{ \infty }(x)-f(0)} \, \mu( \d x).
  \end{split}
  \end{equation}
  The proof of \cref{prop:liap2}
  is thus complete.
\end{proof}

\begin{cor}
\label{cor:liap}
Assume \cref{main_setting}, assume for all $ x \in [a,b]^{ \ell_0 } $ that $f(x)=f(0) $,
and let $ \theta \in \Reals^{ \fd } $. Then 
\begin{equation}
\label{liap}
  \ip{ ( \nabla V)( \theta)}{\mG( \theta)} =4L\mL_{ \infty }( \theta)
  \, .
\end{equation}
\end{cor}
\begin{proof}[Proof of \cref{cor:liap}]
  \Nobs that \cref{prop:liap2} and 
  the fact that for all $ r\in [1 , \infty ] $
  it holds that
  $ \mL_r( \theta)
  =\int_{[a,b]^{ \ell_0 } } \norm{\mN_r^{L, \theta }(x)-f(0)}^2\mu( \d x) $
  establish \cref{liap}.
  The proof of \cref{cor:liap} is thus complete.
\end{proof}

\subsection{Weak chain rule for compositions of Lyapunov functions and GF processes}
\label{ssec:GF_weak_chain_lyapunov}

\begin{prop}
\label{prop:GF_det_ito}
  Assume \cref{main_setting} and
  let $T\in (0, \infty) $, 
  $ \Theta\in C([0,T], \Reals^{ \fd }) $ satisfy
  for all $t\in [0,T]$ that
  $ \Theta_t=\Theta_0-\int_0^t\mG( \Theta_s)\, \d s$.
  Then it holds for all $t\in [0,T]$ that
  \begin{equation}
  \label{GF_det_ito}
  V ( \Theta_t)
  =V ( \Theta_0)
  -4L \int_0^t \int_{[a,b]^{ \ell_0 } } \ip{\mN^{L, \theta }_{ \infty }(x)-f(x)}
    {\mN^{L, \theta }_{ \infty }(x)-f(0)} \, \mu( \d x)\, \d s.
  \end{equation}
\end{prop}
\begin{proof}[Proof of \cref{prop:GF_det_ito}]
\Nobs that \cref{cor:G_up} and the assumption that 
$ \Theta \in C( [0,T], \Reals^{ \fd } ) $ ensure that
$ 
  [0,T] \ni t \mapsto \mG( \Theta_t ) \in \Reals^{ \fd } 
$
is bounded.
\cref{prop:liap2} and, e.g., Cheridito et al. 
\cite[Lemma 3.1]{MR4438169}
(applied with 
$
  T \curvearrowleft T 
$, 
$
  n \curvearrowleft \fd 
$,
$
  F \curvearrowleft ( \Reals^{ \fd } \ni \theta \mapsto V( \theta) \in \Reals ) 
$,
$ 
  \vartheta \curvearrowleft 
  ( [0,T] \ni t \mapsto \mG( \Theta_t ) \in \Reals^{ \fd } ) 
$ 
in the notation of \cite[Lemma 3.1]{MR4438169})
\hence prove that for all $ t \in [0,T] $ it holds that
\begin{equation}
\begin{split}
  V( \Theta_t ) - V( \Theta_0 )
&
  = - \int_0^t \langle ( \nabla V )( \Theta_s ), \mG( \Theta_s ) \rangle \, \d s 
\\
&
  = 
  - 4 L \int_0^t \int_{ [a,b]^{ \ell_0 } } 
  \ip{
    \mN^{ L, \theta }_{ \infty }( x ) - f( x ) 
  }{ 
    \mN^{ L, \theta }_{ \infty }( x ) - f(0) 
  } \, \mu( \d x) \, \d s .
\end{split}
\end{equation}
The proof of \cref{prop:GF_det_ito} is thus complete.
\end{proof}

\begin{cor}
\label{cor:GF_det_ito}
Assume \cref{main_setting}, 
assume for all $ x \in [a,b]^{ \ell_0 } $ that
$ f(x) = f(0) $, 
and
let $ T \in (0, \infty) $,
$ \Theta \in C( [0,T], \Reals^{ \fd } ) $ 
satisfy for all $ t \in [0,T] $ that
$ 
  \Theta_t = \Theta_0 - \int_0^t \mG( \Theta_s ) \, \d s 
$.
Then it holds for all $ t \in [0,T] $ that
\begin{equation}
\label{GF_det_ito_const}
  V( \Theta_t )
  =
  V( \Theta_0 ) 
  - 4 L \int_0^t \mL_{ \infty }( \Theta_s ) \, \d s .
\end{equation}
\end{cor}
\begin{proof}[Proof of \cref{cor:GF_det_ito}]
\Nobs[Observe] \cref{prop:GF_det_ito} 
and the fact that 
$
  \mL_{ \infty }( \theta )
  = \int_{ [a,b]^{ \ell_0 } } 
  \norm{ 
    \mN_{ \infty }^{ L, \theta }( x ) - f( x ) 
  }^2
  \mu( \d x ) 
  = \int_{ [a,b]^{ \ell_0 } } 
  \langle 
    \mN_{ \infty }^{ L, \theta }( x ) - f( x ) 
    ,
    \mN_{ \infty }^{ L, \theta }( x ) - f( 0 ) 
  \rangle 
  \, \mu( \d x ) 
$
establish \cref{GF_det_ito_const}.
The proof of \cref{cor:GF_det_ito} is thus complete.
\end{proof}

\subsection{Weak chain rule for the risk of GF processes}
\label{ssec:GF_weak_chain_risk}

\begin{lemma}
\label{lem:grad_L_bd}
  Assume \cref{main_setting}
  and let $K\subseteq\Reals^{ \fd } $ be compact.
  Then
  \begin{equation}\label{grad_L_bd} 
  \sup\nolimits_{\theta \in K} \sup\nolimits_{r\in [1 , \infty )  } \norm{ ( \nabla\mL_r)( \theta)}
  <\infty.
  \end{equation}
\end{lemma}
\begin{proof}[Proof of \cref{lem:grad_L_bd}]
	Throughout this proof assume without loss of generality that $\mathfrak{m} > 0$
	and for every $k \in \N_0$
	let $\diml_k \in \N_0$
	satisfy $\diml_k = \sum_{n=1}^k \ell_n ( \ell_{n-1} + 1 )$.
  \Nobs that
 \cref{lem:der:unif:bounded}
  ensures that there exists 
  $ \mathfrak{D} \in [1 , \infty) $ which satisfies
  for all $ k \in \{ 1, \ldots, L\} $ that
  \begin{equation}\label{D}
    \mathfrak{D}
    \geq \mathbf{a} + \sup_{\theta \in K} \sup_{r, s, t \in [1 ,\infty )  }
    \sup_{ i \in \{ 1, \ldots, \ell_k \} }
    \sup_{x\in [a,b]^{ \ell_0 } }
    \pr[\big]{ \abs{\mN^{ k, \theta }_{r,i}(x)}
    +\abs{\R_s( \mN^{ k, \theta }_{r,i}(x))}
    +\abs{( \R_t ) ' ( \mN^{ k, \theta }_{r,i}(x))} } .
  \end{equation}
  \Moreover \cref{G2''} of
  \cref{prop:G} 
  proves
  that for all $ \theta \in K$,
  $ r\in [1, \infty )  $, $ k \in \{ 1, \ldots, L\} $,
  $ i \in \{ 1, \ldots, \ell_k \} $
  it holds that
  \begin{equation}
  \label{grad_L_up1a}
  \begin{split}
    &\Big|\pr[\Big]{ \frac{\partial\mL_r}{\partial \theta_{\ell_k\ell_{k-1} + i+ \diml_{k-1} } } } ( \theta ) \Big|^2\\
    &=
    \Bigg( \sum_{\substack{v_k,v_{k+1},  
  \ldots,v_L\in \enne, \\\forall w\in \enne
  \cap[k, L]\colon v_w\leq\ell_w} }
  \int_{[a,b]^{ \ell_0 } }2\,
  \Big[ \indicator{ \{ i \} }(v_k)\Big]
  \Big[ \mN_{r,v_L}^{L, \theta }(x)-f_{v_L}(x)\Big]
  \\
  &\quad\cdot
  \Big[
  \textstyle{\prod}_{n={k+1} }^{L} \big(
  \fw^{n, \theta }_{v_n, v_{n-1} }
  \big[( \R_{r^{1 / ( n - 1 ) } } ) ' ( \mN^{n-1, \theta }_{r,v_{n-1} }(x))\big]
  \big)
  \Big]\, \mu( \d x)\Bigg)^2\\
    &\leq 
    4\mathfrak{D}^{2 L  }
    \Bigg( \int_{[a,b]^{ \ell_0 } }
    \norm{\mN^{L, \theta }_r(x)-f(x)}
    \Bigg[\sum_{\substack{v_k,v_{k+1}, \ldots,v_L\in \enne, \\
    \forall w\in \enne\cap[k, L]\colon v_w\leq\ell_w} }
    \Big[\indicator{ \{ i \} }(v_k)\Big]\Big[
    \textstyle{\prod}_{n={k+1} }^{L}
    \abs{\fw^{n, \theta }_{v_n, v_{n-1} }}
    \Big]\Bigg]\, \mu( \d x)\Bigg)^2\\
    &=
    4\mathfrak{m}^2 \mathfrak{D}^{ 2 L } \Bigg(
    \Bigg[\sum_{\substack{v_k,v_{k+1}, \ldots,v_L\in \enne, \\
    \forall w\in \enne\cap[k, L]\colon v_w\leq\ell_w} }
    \Big[\indicator{ \{ i \} }(v_k)\Big]\Big[
    \textstyle{\prod}_{n={k+1} }^{L}
    \abs{\fw^{n, \theta }_{v_n, v_{n-1} }}
    \Big]\Bigg]\\
    &\quad\cdot
    \frac{1}{\mathfrak{m} }
    \int_{[a,b]^{ \ell_0 } }
    \norm{\mN^{L, \theta }_r(x)-f(x)}
    \, \mu( \d x)\Bigg)^2.
    \end{split}
    \end{equation}
  Jensen's inequality 
  \hence shows for all 
  $ \theta \in K$,
  $ r\in [1 , \infty ) $, $ k \in \{ 1, \ldots, L\} $,
  $ i \in \{ 1, \ldots, \ell_k \} $ that
  \begin{equation}
  \label{grad_L_up1b}
  \begin{split}
    &\Big| \pr[\Big]{ \frac{\partial\mL_r}{\partial \theta_{\ell_k\ell_{k-1} + i+\diml_{k-1} } }  } ( \theta)\Big|^2\\
    &\leq
    4\mathfrak{m}^2 \mathfrak{D}^{2 L } 
    \Bigg[\sum_{\substack{v_k,v_{k+1}, \ldots,v_L\in \enne, \\
    \forall w\in \enne\cap[k, L]\colon v_w\leq\ell_w} }
    \Big[\indicator{ \{ i \} }(v_k)\Big]\Big[
    \textstyle{\prod}_{n={k+1} }^{L}
    \abs{\fw^{n, \theta }_{v_n, v_{n-1} }}
    \Big]\Bigg]^2
    \displaystyle\frac{1}{\mathfrak{m} } \int_{[a,b]^{ \ell_0 } }
    \norm{\mN^{L, \theta }_r(x)-f(x)}^2
    \, \mu( \d x)\\
    &=
    4\mathfrak{m} \mathfrak{D}^{2 L }
    \Bigg[\sum_{\substack{v_k,v_{k+1}, \ldots,v_L\in \enne, \\
    \forall w\in \enne\cap[k, L]\colon v_w\leq\ell_w} }
    \Big[\indicator{ \{ i \} }(v_k)\Big]\Big[
    \textstyle{\prod}_{n={k+1} }^{L}
    \abs{\fw^{n, \theta }_{v_n, v_{n-1} }}
    \Big]\Bigg]^2
    \mL_r( \theta).
    \end{split}
    \end{equation} 
  \Moreover \cref{G2'} of
  \cref{prop:G}
  demonstrates for all $ \theta \in K$,
  $ r\in [1 , \infty ) $,
  $ k \in \{ 1, \ldots, L\} $,
  $ i \in \{ 1, \ldots, \ell_k \} $, 
  $ j \in \{ 1, \ldots, \ell_{k-1} \} $ that
  \begin{equation}    
  \label{grad_L_up2a}
  \begin{split}
    &\Big|\pr[\Big]{ \frac{ \partial \mL_r}{\partial\theta_{ (i-1)\ell_{k-1} + j+ \diml_{k-1} } } } ( \theta)\Big|^2\\
    &=
    \Bigg( \sum_{\substack{v_k,v_{k+1},  
  \ldots,v_L\in \enne, \\\forall w\in \enne
  \cap[k, L]\colon v_w\leq\ell_w} }
  \int_{[a,b]^{ \ell_0 } }2\,
  \Big[
  \R_{r^{1 / ( \max \cu{k - 1  , 1 } ) } } ( \mN^{ \max\{k-1,1\}, \theta }_{r,j}(x))
  \indicator{ (1, L]}(k)
  +x_j \indicator{\{1\} }(k)
  \Big]\\
  &\quad\cdot
  \Big[ \indicator{\{ i \} }(v_k)\Big]
  \Big[ \mN_{r,v_L}^{L, \theta }(x)-f_{v_L}(x)\Big] 
  \Big[
  \textstyle{\prod}_{n={k+1} }^{L} \big(
  \fw^{n, \theta }_{v_n, v_{n-1} }
  \big[( \R_{r^{1 / ( n - 1 ) } } ) ' ( \mN^{n-1, \theta }_{r,v_{n-1} }(x))\big]
  \big)
  \Big]\, \mu( \d x)\Bigg)^2\\
    &\leq
    4\mathfrak{D}^{2 L } 
    \Bigg( \int_{[a,b]^{ \ell_0 } }
    \norm{\mN_r^{L, \theta }(x)-f(x)}
    \Bigg[\sum_{\substack{v_k,v_{k+1}, \ldots,v_L\in \enne, \\
    \forall w\in \enne\cap[k, L]\colon v_w\leq\ell_w} }
    \Big[\indicator{ \{ i \} }(v_k)\Big]\Big[
    \textstyle{\prod}_{n={k+1} }^{L}
    \abs{\fw^{n, \theta }_{v_n, v_{n-1} }}
    \Big]\Bigg] \, \mu( \d x)\Bigg)^2\\
    &=
    4\mathfrak{m}^2\mathfrak{D}^{2 L } 
    \Bigg( \Bigg[\sum_{\substack{v_k,v_{k+1}, \ldots,v_L\in \enne, \\
    \forall w\in \enne\cap[k, L]\colon v_w\leq\ell_w} }
    \Big[\indicator{ \{ i \} }(v_k)\Big]\Big[
    \textstyle{\prod}_{n={k+1} }^{L}
    \abs{\fw^{n, \theta }_{v_n, v_{n-1} }}
    \Big]\Bigg]\\
    &\quad\cdot 
    \frac{1}{\mathfrak{m} }
    \int_{[a,b]^{ \ell_0 } } \,
    \norm{\mN_{ \infty }^{L, \theta }(x)-f(x)}
     \, \mu( \d x)\Bigg)^2.
  \end{split}
  \end{equation}
  Jensen's inequality \hence proves that 
  for all $ \theta \in K$,
  $ r\in [1 , \infty )  $, $ k \in \{ 1, \ldots, L\} $,
  $ i \in \{ 1, \ldots, \ell_k \} $, 
  $ j \in \{ 1, \ldots, \ell_{k-1} \} $
  we have that
  \begin{equation}    
  \label{grad_L_up2b}
  \begin{split}
    &\Big| \pr[\Big]{ \frac{\partial \mL_r}{\partial\theta_{ (i-1)\ell_{k-1} + j+ \diml_{k-1} } }  } ( \theta)\Big|^2\\
    &\leq
    4\mathfrak{m}^2\mathfrak{D}^{2 L }
    \Bigg[\sum_{\substack{v_k,v_{k+1}, \ldots,v_L\in \enne, \\
    \forall w\in \enne\cap[k, L]\colon v_w\leq\ell_w} }
    \Big[\indicator{\{ i \} }(v_k)\Big]\Big[
    \textstyle{\prod}_{n={k+1} }^{L}
    \abs{\fw^{n, \theta }_{v_n, v_{n-1} }}
    \Big]\Bigg]^2
    \displaystyle \frac{1}{\mathfrak{m} } \int_{[a,b]^{ \ell_0 } }
    \norm{\mN_{ \infty }^{L, \theta }(x)-f(x)}^2
    \, \mu( \d x)\\
    &=
    4\mathfrak{m} \mathfrak{D}^{2 L }
    \Bigg[\sum_{\substack{v_k,v_{k+1}, \ldots,v_L\in \enne, \\
    ( \forall w\in \enne\cap[k, L]\colon v_w\leq\ell_w)} }
    \Big[\indicator{\{ i \} }(v_k)\Big]\Big[
    \textstyle{\prod}_{n={k+1} }^{L}
    \abs{\fw^{n, \theta }_{v_n, v_{n-1} }}
    \Big]\Bigg]^2 \mL_r( \theta).
  \end{split}
  \end{equation}
  This and
  \cref{grad_L_up1b}
  assure for all 
  $\theta \in K$,
  $ r\in [1 , \infty )  $
   that
  \begin{equation}
  \label{grad_L_upper}
  \begin{split}
    &\norm{ ( \nabla\mL_r)( \theta)}^2\\ 
    &= 
    \sum\limits_{k= 1 }^L\sum\limits_{ i= 1 }^{ \ell_k}
    \bigg[\Big| \pr[\Big]{ \frac{\partial\mL_r}{\partial\theta_{\ell_k\ell_{k-1} + i+ \diml_{k - 1 } } }  } ( \theta)\Big|^2
    +\sum\limits_{ j = 1 }^{ \ell_{k-1} }
    \Big| \pr[\Big]{ \frac{\partial\mL_r}{\partial\theta_{ (i-1)\ell_{k-1} + j+ \diml_{k - 1 } } } } ( \theta)\Big|^2\bigg]\\
    &\leq 
    4\mathfrak{m} \mathfrak{D}^{2L }  
    \Bigg[ \sum\limits_{k= 1 }^L\sum\limits_{ i= 1 }^{ \ell_k} 
    ( \ell_{k-1} + 1 ) 
    \sum_{\substack{v_k,v_{k+1}, \ldots,v_L\in \enne, \\
    \forall w\in \enne\cap[k, L]\colon v_w\leq\ell_w} }
    \Big[\indicator{ \{ i \} }(v_k)\Big]\Big[
    \textstyle{\prod}_{n={k+1} }^{L}
    \abs{\fw^{n, \theta }_{v_n, v_{n-1} }}
    \Big]\Bigg]^2
    \mL_r( \theta).
  \end{split}
  \end{equation}  
  Furthermore, \nobs that \cref{L_r}
  implies
  for all $ \theta \in K$, $ r\in [1 , \infty )  $
  that
  \begin{equation}
  \begin{split}
  \label{L_bd}
    \mL_r( \theta)
    &=\int_{[a,b]^{ \ell_0 } } \norm{\mN^{L, \theta }_r(x)
    -f(x)}^2\mu( \d x)
    \leq 2 \int_{[a,b]^{ \ell_0 } } \Big[\norm{\mN^{L, \theta }_r(x)}^2
    +\norm{f(x)}^2\Big]\mu( \d x) \\
    &\leq 2 \mathfrak{m} \br*{ \sup\nolimits_{y \in [a,b]^{ \ell_0 } }
    \norm{\mN^{L, \theta }_r(y)}^2 }
    +2\int_{[a,b]^{ \ell_0 } } \norm{f(x)}^2\mu( \d x).
  \end{split}  
  \end{equation}
  This,
  \cref{lem:f:integrable},
   and \cref{D}
  prove that
  $ \sup_{\theta \in K} \sup_{r\in [1 , \infty )  } \mL_r( \theta)
  <\infty$.
  Combining this with 
  \cref{grad_L_upper} and
  \cref{it:sum_w_up} of \cref{prop:G_upper_estimate}
  establishes \cref{grad_L_bd}.
  The proof of \cref{lem:grad_L_bd}
  is thus complete.
\end{proof}

\begin{prop}
\label{lem:GF_det_ito2}
  Assume \cref{main_setting}
  and
  let $T\in (0, \infty) $,
  $ \Theta\in C([0,T], \Reals^{ \fd }) $ satisfy 
  for all $t\in [0,T]$ that
  $ \Theta_t=\Theta_0-\int_0^t\mG( \Theta_s)\, \d s$.
  Then it holds for all $t\in [0,T]$ that
  \begin{equation}
  \label{GF_det_ito2}
    \mL_{ \infty }( \Theta_t)=\mL_{ \infty }( \Theta_0)
    -\int_0^t\norm{\mG( \Theta_s)}^2\, \d s.
  \end{equation}
\end{prop}
\begin{proof}[Proof of \cref{lem:GF_det_ito2}]
  \Nobs that, 
  e.g., Cheridito et al. 
  \cite[Lemma 3.1]{MR4438169}
  (applied with $T\curvearrowleft T$, $n\curvearrowleft\fd $,
  $F\curvearrowleft ( \Reals^{ \fd }\ni\theta \mapsto\mL_r( \theta) \in \Reals) $,
  $ \vartheta \curvearrowleft ([0,T]\ni t\mapsto\mG( \Theta_t) \in \Reals^{ \fd }) $ in the notation of \cite[Lemma 3.1]{MR4438169})
  implies that for all $ r\in [1 , \infty )  $, $t\in [0,T]$
  we have that
  \begin{equation}
  \label{L_theta_diff}
    \mL_r( \Theta_t)-\mL_r( \Theta_0)
    =-\int_0^t\langle( \nabla \mL_r( \Theta_s)), \mG( \Theta_s)\rangle\, \d s.
  \end{equation}
  Furthermore, \nobs that \cref{G3} of 
  \cref{prop:G} ensures that 
  for all $t\in [0,T]$ it holds that
  $ \lim_{r\to \infty } $ $ ( \mL_r( \Theta_t)-\mL_r( \Theta_0))
  =\mL_{ \infty }( \Theta_t)-\mL_{ \infty }( \Theta_0) $
  and
  $ \lim_{r\to \infty } \langle ( \nabla \mL_r)( \Theta_t), \mG( \Theta_t)\rangle
  =\langle ( \mG( \Theta_t), \mG( \Theta_t)\rangle$
  $=\norm{\mG( \Theta_t)}^2$.
  Moreover, \nobs that the assumption
  that $ \Theta\in C([0,T], \Reals^{ \fd }) $ assures that
  there exists a compact set $K\subseteq\Reals^{ \fd } $ such that
  for all $t\in [0,T]$ it holds that
  $ \Theta_t\in K$.
  This, the Cauchy-Schwarz inequality, 
  \cref{cor:G_up}, 
  \cref{G1} of \cref{prop:G}, and
  \cref{lem:grad_L_bd}
  \hence demonstrate that
  \begin{equation}
  \begin{split}
    &\sup\nolimits_{r\in [1 , \infty )  } \sup\nolimits_{t\in [0,T]}
    \abs[\big]{\langle ( \nabla\mL_r)( \Theta_t), \mG( \Theta_t)\rangle}
    \leq
    \sup\nolimits_{r\in [1 , \infty ) } \sup\nolimits_{\theta \in K}
    \abs[\big]{\langle ( \nabla\mL_r)( \theta), \mG( \theta)\rangle} \\
    &\leq
    \sup\nolimits_{r\in [1 , \infty ) } \sup\nolimits_{\theta \in K}
    \pr[\big]{ \norm{ ( \nabla\mL_r)( \theta)} \norm{\mG( \theta)} }
    <\infty.
  \end{split}
  \end{equation}
  The dominated convergence theorem
  and \cref{G3} of \cref{prop:G} \hence
  show that for all $t\in [0,T]$
  it holds that
  \begin{equation}
    \lim\limits_{r\to \infty }  \br*{ \int_0^t
    \langle( \nabla\mL_{ \infty })( \Theta_s), \mG( \Theta_s)\rangle\, \d s } 
    = \int_0^t \br*{ \lim\limits_{r\to \infty }
    \langle( \nabla\mL_{ \infty })( \Theta_s), \mG( \Theta_s)\rangle } \, \d s
    =\int_0^t \norm{\mG( \Theta_s)}^2\, \d s.  
  \end{equation}
  Combining this with
  \cref{L_theta_diff}
  establishes \cref{GF_det_ito2}.
  The proof of \cref{lem:GF_det_ito2}
  is thus complete.
\end{proof}

\subsection{Convergence analysis for GF processes}
\label{ssec:GF_convergence}

\begin{theorem}
  \label{thm:GF_conv}
  Assume \cref{main_setting}, 
  assume for all $ x \in [a,b]^{ \ell_0 } $ that
  $f(x)=f(0) $,
  and let
  $ \Theta\in C([0, \infty), \Reals^{ \fd } ) $ satisfy for all
  $t\in [0, \infty) $ that
  $ \Theta_t=\Theta_0-\int_0^t\mG( \Theta_s)\, \d s$.
  Then
  \begin{enumerate} [label=(\roman*)]
    \item\label{GF_con_it1}
    it holds that 
    $ \sup_{t\in [0, \infty)} \norm{\Theta_t}
    \leq[2V( \Theta_0)+4L^2\norm{f(0)}^2 ]^{ 1/2}
    <\infty$,
    \item\label{GF_con_it2} 
    it holds for all $t\in (0, \infty) $ that
    $ \mL_{ \infty }( \Theta_t)
    \leq\frac{1}{2t}[\norm{\Theta_0 }^2+2L\norm{f(0)}^2]
    <\infty$,
    and
    \item\label{GF_con_it3}
    it holds that
    $ \limsup_{t\to \infty } \mL_{ \infty }( \Theta_t)=0$.
  \end{enumerate}
\end{theorem}
\begin{proof}[Proof of \cref{thm:GF_conv}]
  \Nobs that \cref{it:ineq} of 
  \cref{prop:liap1} implies that
  for all $t\in [0, \infty) $ it holds that
  $ \norm{\Theta_t} \leq
  \big[2V( \Theta_t)+4L^2\norm{f(0)}^2\big]^{ 1/2} $.
  Moreover, \nobs that \cref{cor:GF_det_ito}
  and the fact that for all $ \theta \in \Reals^{ \fd } $
  it holds that $ \mL_{ \infty }( \theta)\geq 0$
  prove that for all $t\in [0, \infty) $ we have that
  $V( \Theta_t)\leq V( \Theta_0) $.
  This establishes \cref{GF_con_it1}.
  Next \nobs that
  \cref{lem:GF_det_ito2} implies that
  $[0, \infty)\ni t\mapsto\mL_{ \infty }( \Theta_t) \in [0, \infty) $
  is non-increasing.  
  Combining this with
  \cref{cor:GF_det_ito} and 
  \cref{it:ineq} of \cref{prop:liap1}
  demonstrates that for all
  $t\in [0, \infty) $ it holds that
  \begin{equation}
  \begin{split}
    &t\mL_{ \infty }( \Theta_t)
    =\int_0^t\mL_{ \infty }( \Theta_t)\, \d s
    \leq \int_0^t\mL_{ \infty }( \Theta_s)\, \d s
    =\frac{1}{4L}[V( \Theta_0)-V( \Theta_t)]\\
    &\leq \frac{1}{4L} \Big[2L\norm{\Theta_0 }^2+L\norm{f(0)}^2
    -\frac{1}{2} \norm{\Theta_t}^2+2L^2\norm{f(0)}^2\Big]\\
    &\leq \frac{1}{4L} \Big[2L\norm{\Theta_0 }^2
    +(L+2L^2)\norm{f(0)}^2\Big]
    \leq\frac{1}{2} \norm{\Theta_0 }^2+L\norm{f(0)}^2
    <\infty. 
  \end{split}
  \end{equation}
  \Hence for all 
  $t\in (0, \infty) $ that
  \begin{equation}
  \mL_{ \infty }( \Theta_t)
  \leq\frac{1}{2t} \Big[\norm{\Theta_0 }^2+2L\norm{f(0)}^2\Big].
  \end{equation}
  This establishes \cref{GF_con_it2,GF_con_it3}.
  The proof of \cref{thm:GF_conv}
  is thus complete.
\end{proof}

\section{Gradient descent (GD) processes in the training of deep ANNs}
\label{sec:grad_des}

In this section we use some of the results from \cref{sec:risk,sec:grad_fl} above 
to establish in \cref{thm:GD_conv}
in \cref{ssec:GD_convergence} below 
that in the training of deep ReLU ANNs 
we have that the sequence of risks 
$
  \mL_{ \infty }( \Theta_n )
$, 
$ n \in \N_0 $, 
of any time-discrete GD process 
$
  \Theta = ( \Theta_n )_{ n \in \N_0 } \colon \N_0 \to \Rr^{ \fd }
$
converges to zero provided that 
the target function 
$ f \colon [a,b]^{ \ell_0 } \to \Rr^{ \ell_L } $
is a constant function and 
provided that the learning rates 
(the step sizes) 
$ 
  \gamma_n \in [0,\infty)
$, 
$ n \in \N_0 $, 
in the GD optimization method are sufficiently small 
(see \cref{gamma_ass2} below for details) 
but fail to be $ L^1 $-summable so that 
\begin{equation}
\textstyle 
  \sum_{ n = 0 }^{ \infty } \gamma_n = \infty 
  .
\end{equation}
Our proof of \cref{thm:GD_conv} 
employs 
\cref{lem:loc_lip} in \cref{ssec:local_lipschitz} above,
\cref{cor:G_up} in \cref{ssec:upper_estimates} above, 
\cref{prop:liap1} in \cref{ssec:GF_lyapunov} above, 
as well as the recursive upper estimate for the composition 
$
  V( \Theta_n )
$,
$ n \in \N_0 $, 
of the Lyapunov function $ V \colon \Rr^{ \fd } \to \Rr $ 
and the time-discrete GD process 
$
  \Theta = ( \Theta_n )_{ n \in \N_0 } \colon \N_0 \to \Rr^{ \fd }
$
in \cref{V_theta_diff2} in \cref{lem:V_theta_diff2} in \cref{ssec:GD_upper} below.

Our proof of the time-discrete Lyapunov estimate   
in \cref{V_theta_diff2} in \cref{lem:V_theta_diff2} 
is based on induction and applications of
the time-discrete Lyapunov estimate 
in \cref{cor:V_theta_diff_const}
in \cref{ssec:GD_upper}. 
The time-discrete Lyapunov estimate 
in \cref{cor:V_theta_diff_const}, in turn, 
is an immediate consequence of 
the time-discrete Lyapunov estimate 
in \cref{cor:V_theta_diff}
in \cref{ssec:GD_upper}. 
While 
\cref{thm:GD_conv}, 
\cref{lem:V_theta_diff2}, 
and 
\cref{cor:V_theta_diff_const} 
are restricted to the situation where the target function 
$ f \colon [a,b]^{ \ell_0 } \to \Rr^{ \ell_L } $ 
is a constant function, 
the Lyapunov estimate in 
\cref{cor:V_theta_diff} 
is applicable in the general case of a measurable 
target function 
$ f \colon [a,b]^{ \ell_0 } \to \Rr^{ \ell_L } $.

Our proof of \cref{cor:V_theta_diff} 
is based on an application of the one-step Lyapunov estimate 
for the GD method 
in \cref{cor:liap3}
in \cref{ssec:GD_lyapunov}. 
Our proof of \cref{cor:liap3}, in turn, 
uses the one-step Lyapunov estimate 
for the GD method 
in \cref{lem:liap2} 
in \cref{ssec:GD_lyapunov}. 
In \cref{cor:liap3}
in \cref{ssec:GD_lyapunov} 
we also specialize 
the one-step Lyapunov estimate for the GD method in \cref{lem:liap2} 
to the situation where the target function 
$
  f \colon [a,b]^{ \ell_0 } \to \Rr^{ \ell_L }
$
is a constant function. 
\cref{cor:liap3} 
is employed in our convergence analysis of the SGD method 
in \cref{sec:stoch_grad_des} below.

The results in this section extend the findings in \cite[Section~2]{MR4468133} from 
shallow ReLU ANNs with just one hidden layer to deep ReLU ANNs with an arbitrarily large number of hidden layers. 
In particular, 
\cref{lem:liap2} in \cref{ssec:GD_lyapunov} and 
\cref{cor:liap3} in \cref{ssec:GD_lyapunov}
extend \cite[Lemma 2.12]{MR4468133},
\cref{cor:V_diff} in \cref{ssec:GD_lyapunov}
extends
\cite[Corollary 2.13]{MR4468133}, 
\cref{cor:V_theta_diff} 
in \cref{ssec:GD_upper}, 
\cref{cor:V_theta_diff_const}
in \cref{ssec:GD_upper}, 
and 
\cref{lem:V_theta_diff2} in \cref{ssec:GD_upper}
extend \cite[Corollary~2.14 and Lemma~2.15]{MR4468133}, 
and \cref{thm:GD_conv} in \cref{ssec:GD_convergence}
extends \cite[Theorem 2.16]{MR4468133}.

\subsection{Lyapunov type estimates for the dynamics of GD processes}
\label{ssec:GD_lyapunov}

\begin{lemma}
\label{lem:liap2} 
Assume \cref{main_setting}
and let $ \gamma \in \Reals $, $ \theta \in \Reals^{ \fd } $. Then
\begin{equation}
\label{liap2}
\begin{split}
  V( \theta - \gamma \mG( \theta) ) - V( \theta )  
&
\textstyle
  = 
  \gamma^2 \norm{ \mG( \theta ) }^2
  +
  \gamma^2 
  \bigl[ 
    \sum_{ k = 1 }^L
    \sum_{ i = 1 }^{ \ell_k }
    ( k - 1 ) 
    \abs{
      \mG_{ 
        \ell_k \ell_{k-1} + i + \sum_{ h = 1 }^{ k - 1 } 
        \ell_h ( \ell_{h-1} + 1 ) 
      }( \theta ) 
    }^2 
  \bigr] 
\\
&
\textstyle 
\quad 
  - 4 \gamma L 
  \bigl[ 
    \int_{ [a,b]^{ \ell_0 } } 
    \ip{
      \mN^{ L, \theta }_{ \infty }( x ) - f( x ) 
    }{
      \mN^{ L, \theta }_{ \infty }( x ) - f(0)
    } \, \mu( \d x ) 
  \bigr]
\\
&
\textstyle 
  \leq
  \gamma^2 L \norm{ \mG( \theta ) }^2
  - 4 \gamma L 
  \bigl[
    \int_{ [a,b]^{ \ell_0 } } 
    \ip{
      \mN^{ L, \theta }_{ \infty }( x ) - f( x )
    }{
      \mN^{ L, \theta }_{ \infty }( x ) - f(0) 
    } 
    \, \mu( \d x ) 
  \bigr] .    
  \end{split}
  \end{equation}
\end{lemma}
\begin{proof}[Proof of \cref{lem:liap2}]
  Throughout this proof let
  $ \mathbf{e}_1, \mathbf{e}_2, \dots, \mathbf{e}_\fd \in \Rr^\fd $ 
  satisfy $ \mathbf{e}_1= (1,0, \dots,0) $, 
  $ \mathbf{e}_2= (0,1,0, \dots,0) $, $ \dots$,
  $ \mathbf{e}_\fd = (0, \dots,0,1) $
  and let
  $g\colon \Rr\to\Rr$ satisfy 
  for all $t\in \Rr$ that
  \begin{equation}
    \label{eq_g}
    g(t)=V( \theta-t\mG( \theta)).
  \end{equation}
  \Nobs that \cref{eq_g} and the fundamental theorem of calculus
  demonstrate that
  \begin{equation}\label{V_diff1}
  \begin{split}
    V( \theta-\gamma\mG( \theta))&=g( \gamma)=g(0)+\int_0^{ \gamma}g'(t)\, \d t
    =g(0)+\int_0^{ \gamma} \ip{ ( \nabla V)( \theta-t\mG( \theta))}{ (-\mG( \theta))} \d t\\
    &=V( \theta)-\int_0^{ \gamma} \ip{ ( \nabla V)( \theta-t\mG( \theta))}{ \mG( \theta)} \d t.        
  \end{split}
  \end{equation}
  \cref{prop:liap2} \hence proves that
  \begin{equation}\label{V_diff2}
  \begin{split}
    &V( \theta-\gamma\mG( \theta))\\
    &=V( \theta)-\int_0^{ \gamma} \ip{ ( \nabla V)( \theta)}{\mG( \theta)} \d t
    +\int_0^{ \gamma} \ip{ ( \nabla V)( \theta)-( \nabla V)( \theta-t\mG( \theta))}{ \mG( \theta)} \d t\\  
    &=V( \theta)-4\gamma L \br*{ \int_{[a,b]^{ \ell_0 } } \ip{\mN^{L, \theta }_{ \infty }(x)-f(x)}
    {\mN^{L, \theta }_{ \infty }(x)-f(0)} \, \mu( \d x) } \\
    &\quad +\int_0^{ \gamma} \ip{ ( \nabla V)( \theta)-( \nabla V)( \theta-t\mG( \theta))}{\mG( \theta)} \d t.
  \end{split}
  \end{equation}
  Moreover, \nobs that \cref{it:grad_V} of
  \cref{prop:liap1} implies that for all $t\in \Reals $
  it holds that
  \begin{equation}
    ( \nabla V)( \theta)-( \nabla V)( \theta-t\mG ( \theta))
    = 2t\mG( \theta)+2 \br*{ \sum_{k= 1 }^L\sum_{ i= 1 }^{ \ell_k}(k-1)\fb^{ k,t\mG( \theta)}_{ i} \mathbf{e}_{\ell_k \ell_{k-1} + i+\sum_{h= 1 }^{ k-1} \ell_h ( \ell_{h-1} + 1)}  } .
  \end{equation}
  Combining this with \cref{V_diff2} 
  shows that
  \begin{equation}\label{V_diff3} 
  \begin{split}
    &V( \theta-\gamma\mG( \theta))\\
    &=V( \theta)
    -4\gamma L \br*{ \int_{[a,b]^{ \ell_0 } } \ip{\mN^{L, \theta }_{ \infty }(x)-f(x)}
    {\mN^{L, \theta }_{ \infty }(x)-f(0)} \, \mu( \d x) } \\
    &\quad+\int_0^\gamma \ip{ 2t\mG( \theta)+2 \br*{ \sum_{k= 1 }^L\sum_{ i= 1 }^{ \ell_k}(k-1)\fb^{ k,t\mG( \theta)}_{ i} \mathbf{e}_{\ell_k \ell_{k-1} + i+\sum_{h= 1 }^{ k-1} \ell_h ( \ell_{h-1} + 1)}  }  }{ \mG( \theta) } \, \d t\\ 
    &=V( \theta)
    -4\gamma L \br*{ \int_{[a,b]^{ \ell_0 } } \ip{\mN^{L, \theta }_{ \infty }(x)-f(x)}
    {\mN^{L, \theta }_{ \infty }(x)-f(0)} \, \mu( \d x) } 
    +2\norm{\mG( \theta)}^2 \br*{ \int_0^\gamma t \, \d t }  \\  
    &\quad +2 \br*{ \int_0^\gamma \ip{ \sum_{k= 1 }^L\sum_{ i= 1 }^{ \ell_k}(k-1)\fb^{ k,t\mG( \theta)}_{ i} \mathbf{e}_{\ell_k \ell_{k-1} + i+\sum_{h= 1 }^{ k-1} \ell_h ( \ell_{h-1} + 1)} } { \mG( \theta) }
     \, \d t } \\
     &=V( \theta)
     -4\gamma L \br*{ \int_{[a,b]^{ \ell_0 } } \ip{\mN^{L, \theta }_{ \infty }(x)-f(x)}
    {\mN^{L, \theta }_{ \infty }(x)-f(0)} \, \mu( \d x) } 
    +\gamma^2\norm{\mG( \theta)}^2\\
    &\quad+2 \br*{ \sum_{k= 1 }^L\sum_{ i= 1 }^{ \ell_k}(k-1)\big|\big\langle\mathbf{e}_{\ell_k \ell_{k-1} + i+\sum_{h= 1 }^{ k-1} \ell_h ( \ell_{h-1} + 1)}, \mG( \theta)\big\rangle\big|^2 }  \br*{ \int_0^\gamma t\, \d t } .
    \end{split}
    \end{equation}
\Hence that 
    \begin{equation}
    \begin{split}
    V( \theta-\gamma\mG( \theta)) 
    &=V( \theta)
    -4\gamma L \br*{ \int_{[a,b]^{ \ell_0 } } \ip{\mN^{L, \theta }_{ \infty }(x)-f(x)}
    {\mN^{L, \theta }_{ \infty }(x)-f(0)} \, \mu( \d x) } 
    +\gamma^2\norm{\mG( \theta)}^2\\
    &\quad+\gamma^2 \br*{ \sum_{k= 1 }^L\sum_{ i= 1 }^{ \ell_k}(k-1)\big|\big\langle\mathbf{e}_{\ell_k \ell_{k-1} + i+\sum_{h= 1 }^{ k-1} \ell_h ( \ell_{h-1} + 1)}, \mG( \theta)\big\rangle\big|^2 } \\
    &=V( \theta)
    -4\gamma L \br*{ \int_{[a,b]^{ \ell_0 } } \ip{\mN^{L, \theta }_{ \infty }(x)-f(x)}
    {\mN^{L, \theta }_{ \infty }(x)-f(0)} \, \mu( \d x) } 
    +\gamma^2\norm{\mG( \theta)}^2\\
    &\quad+\gamma^2 \br*{ \sum_{k= 1 }^L\sum_{ i= 1 }^{ \ell_k}(k-1)\abs{\mG_{\ell_k \ell_{k-1} + i+\sum_{h= 1 }^{ k-1} \ell_h ( \ell_{h-1} + 1)}( \theta)}^2 } . 
  \end{split}
  \end{equation}
  The proof of \cref{lem:liap2} is thus complete.
\end{proof}

\begin{cor}
\label{cor:liap3} 
Assume \cref{main_setting}, 
assume for all $ x \in [a,b]^{ \ell_0 } $ that $ f(x) = f(0) $, 
and let $ \gamma \in \Reals $, $ \theta \in \Reals^{ \fd } $. 
Then
\begin{equation}\label{liap3}
\begin{split}
&
  V( \theta - \gamma \mG( \theta ) ) 
  - 
  V( \theta )  
\\ 
&
\textstyle
  = \gamma^2 \norm{ \mG( \theta ) }^2
  + \gamma^2 
  \bigl[
    \sum_{ k = 1 }^L 
    \sum_{ i = 1 }^{ \ell_k } 
    ( k - 1 ) 
    \abs{
      \mG_{ 
        \ell_k \ell_{ k - 1 } + i + \sum_{ h = 1 }^{ k - 1 } \ell_h ( \ell_{ h - 1 } + 1 ) 
      }( \theta ) 
    }^2 
  \bigr] 
  - 4 \gamma L \mL_{ \infty }( \theta ) 
\\
&
  \leq
  \gamma^2 L \norm{ \mG( \theta ) }^2 
  - 4 \gamma L \mL_{ \infty }( \theta ) .    
\end{split}
\end{equation}
\end{cor}
\begin{proof}[Proof of \cref{cor:liap3}]
\Nobs that \cref{lem:liap2},
  \cref{L_r},
  and the assumption that for all
  $ x \in [a,b]^{ \ell_0 } $
  it holds that 
  $f(x)=f(0) $
  establish \cref{liap3}.
  The proof of \cref{cor:liap3} is thus complete.
\end{proof}

\begin{cor}
\label{cor:V_diff}
Assume \cref{main_setting}
and let $ \gamma\in [0, \infty) $, 
$ \theta \in \Reals^{ \fd } $. Then
\begin{equation}
\label{V_diff4}
\begin{split}
  V( \theta - \gamma \mG( \theta ) ) - V( \theta )
&
\leq 
  4 \gamma^2 \mathfrak{m} L^2
  \mathbf{a}^2 
  \bigl[
    \textstyle{ \prod }_{ p = 0 }^L
    ( \ell_p + 1 ) 
  \bigr]
  ( 2 V( \theta ) + 4 L^2 \norm{ f(0) }^2 + 1 )^{ (L-1) } 
  \mL_{ \infty }( \theta )
\\
&
\textstyle 
  \quad 
  - 4 \gamma L 
  \bigl[ 
    \int_{ [a,b]^{ \ell_0 } } 
    \ip{
      \mN^{ L, \theta }_{ \infty }( x ) - f( x ) 
    }{
      \mN^{ L, \theta }_{ \infty }( x ) - f(0) 
    } 
    \, \mu( \d x ) 
  \bigr] .
\end{split}  
\end{equation}
\end{cor}

\begin{proof}[Proof of \cref{cor:V_diff}]
  \Nobs that \cref{it:G_upper_estimate} of 
  \cref{prop:G_upper_estimate}
  and \cref{it:ineq} of 
  \cref{prop:liap1}
  demonstrate that
  \begin{equation} 
  \label{G_up5}
  \begin{split}
    \norm{\mG( \theta)}^2
    &\leq 4\mathfrak{m} L  \mathbf{a}^2 \br[\big]{ 
    \textstyle{\prod}_{p={0} }^{L}( \ell_p+1 ) }
    ( \norm{\theta }^2 + 1 )^{ ( L-1 ) } \mL_{ \infty }( \theta)\\
    &\leq 4\mathfrak{m} L \mathbf{a}^2
    	\br[\big]{ 
    \textstyle{\prod}_{p={0} }^{L}( \ell_p+1) }
    (2V( \theta)+4L^2\norm{f(0)}^2 + 1 )^{ (L-1)}
    \mL_{ \infty }( \theta).
  \end{split}  
  \end{equation}
  Combining this with
  \cref{lem:liap2}
  establishes \cref{V_diff4}.
  The proof of \cref{cor:V_diff} is thus complete.
\end{proof}

\subsection{Upper estimates for compositions of Lyapunov functions and GD processes}
\label{ssec:GD_upper}

\begin{cor}
\label{cor:V_theta_diff}
Assume \cref{main_setting},
let $ ( \gamma_n )_{ n \in \enne_0 } \subseteq [0, \infty) $, 
let $ ( \Theta_n )_{ n \in \enne_0 } \colon \enne_0 \to \Reals^{ \fd } $ 
satisfy for all $ n \in \enne_0 $ that 
$ \Theta_{ n + 1 } = \Theta_n - \gamma_n \mG( \Theta_n ) $,
and let $ n \in \enne_0 $.
Then
\begin{equation}
\label{V_theta_diff}
\begin{split}
  V( \Theta_{n+1} )-V( \Theta_n )
&
  \leq 4 ( \gamma_n )^2 \mathfrak{m} L^2
     \mathbf{a}^2 \br[\big]{ 
    \textstyle{\prod}_{p={0} }^{L}( \ell_p+1) }
    (2V( \Theta_n )+4L^2\norm{f(0)}^2 + 1  )^{ (L-1)}
    \mL_{ \infty }( \Theta_n )
\\
&
\textstyle
  \quad 
  - 4 \gamma_n L 
  \bigl[
    \int_{ [a,b]^{ \ell_0 } } 
    \ip{ 
      \mN^{ L, \Theta_n }_{ \infty }( x ) - f(x) 
    }{
      \mN^{ L, \Theta_n }_{ \infty }( x ) - f(0) 
    } 
    \, \mu( \d x ) 
  \bigr] .
  \end{split}  
  \end{equation}
\end{cor}
\begin{proof}[Proof of \cref{cor:V_theta_diff}]
\Nobs that \cref{cor:V_diff} establishes \cref{V_theta_diff}.
  The proof of \cref{cor:V_theta_diff} is thus complete.
\end{proof}

\begin{cor}
\label{cor:V_theta_diff_const}
Assume \cref{main_setting},
assume for all $ x \in [a,b]^d $ that $ f(x) = f(0) $,
let $ ( \gamma_n )_{ n \in \enne_0 } \subseteq [0, \infty) $, 
let $ ( \Theta_n )_{ n \in \enne_0 } \colon \enne_0 \to \Reals^{ \fd } $ 
satisfy for all $ n \in \enne_0 $ that 
$ \Theta_{ n + 1 } = \Theta_n - \gamma_n \mG( \Theta_n ) $,
and let $ n \in \enne_0 $.
Then
\begin{equation}
\label{V_theta_diff_const}
\begin{split}
&
  V( \Theta_{ n + 1 } ) - V( \Theta_n )
\\
&
  \leq 4 L
  \bigl( 
    ( \gamma_n )^2 \mathfrak{m} L
    \mathbf{a}^2 
    \bigl[
      \textstyle{ \prod }_{ p = 0 }^L
      ( \ell_p + 1 ) 
    \bigr]
    ( 2 V( \Theta_n ) + 4 L^2 \norm{ f(0) }^2 + 1 )^{ (L - 1) } 
    - \gamma_n 
  \bigr)
  \mL_{ \infty }( \Theta_n )
  .
\end{split}  
\end{equation}
\end{cor}
\begin{proof}[Proof of \cref{cor:V_theta_diff_const}]
\Nobs that \cref{cor:V_theta_diff} 
and the assumption that for all 
$ x \in [a,b]^{ \ell_0 } $
it holds that $ f(x) = f(0) $
establish \cref{V_theta_diff_const}.
The proof of \cref{cor:V_theta_diff_const} is thus complete.
\end{proof}

\begin{lemma}
\label{lem:V_theta_diff2}
Assume \cref{main_setting},
assume for all $ x \in [a,b]^d $ that $ f(x) = f(0) $,
let $ ( \gamma_n )_{ n \in \enne_0 } \subseteq [0, \infty) $, 
let $ ( \Theta_n )_{ n \in \enne_0 } \colon \enne_0 \to \Reals^{ \fd } $ 
satisfy for all $ n \in \enne_0 $ that 
$ \Theta_{ n + 1 } = \Theta_n - \gamma_n \mG( \Theta_n ) $, 
and assume 
\begin{equation}
\label{gamma_ass1}
  \sup\nolimits_{ n \in \enne_0 } 
  ( 
    \gamma_n  
    \mathfrak{m}
  )
  \le  
  \bigl( 
    L \mathbf{a}^2 \br[\big]{ 
    \textstyle{\prod}_{p={0} }^{L}( \ell_p+1) }
    (2V( \Theta_0)+4L^2\norm{f(0)}^2 + 1 )^{ (L-1)} 
  \bigr)^{ - 1 } 
  .
\end{equation}
Then it holds for all $ n \in \enne_0 $ that
\begin{align}
\label{V_theta_diff2}
&
  V( \Theta_{ n + 1 } ) - V( \Theta_n )
\\
\nonumber 
&
  \leq - 4 L \gamma_n
  \bigl(
    1 
    - 
    \br[\big]{ \sup\nolimits_{m\in \enne_0 } \gamma_m }
    \mathfrak{m}
    L \mathbf{a}^2 \br[\big]{ 
      \textstyle{ \prod }_{ p = 0 }^L
      ( \ell_p + 1 ) 
    }
    ( 2 V( \Theta_0 ) + 4 L^2 \norm{ f(0) }^2 + 1 )^{ (L - 1) } 
  \bigr)
  \mL_{ \infty }( \Theta_n ) 
  \leq 0 .
\end{align}
\end{lemma}
\begin{proof}[Proof of \cref{lem:V_theta_diff2}]
Throughout this proof let $ \mathfrak{g} \in \Reals $
satisfy $ \mathfrak{g} = \sup_{ n \in \enne_0 } ( \gamma_n \mathfrak{m} ) $.
We prove \cref{V_theta_diff2} by 
induction on $ n \in \enne_0 $.
\Nobs that \cref{cor:V_theta_diff_const},
\cref{gamma_ass1}, and
the fact that $ \gamma_0 \mathfrak{m} \leq\mathfrak{g} $ imply that
\begin{equation}
\label{V_theta_diff_ind1}
\begin{split}
&
  V( \Theta_1 ) - V( \Theta_0 )
\\
&
  \leq 4 L
  \bigl( 
    - \gamma_0 +
    \gamma_0^2 \mathfrak{m} L
    \mathbf{a}^2 
    \br[\big]{ 
      \textstyle{ \prod }_{ p = 0 }^L ( \ell_p + 1 ) 
    }
    ( 2 V( \Theta_0 ) + 4 L^2 \norm{ f(0) }^2 + 1 )^{ (L-1) } 
  \bigr)
  \mL_{ \infty }( \Theta_0 )
\\
&
  \leq 4 L
  \bigl(
    -\gamma_0+
    \gamma_0 \mathfrak{g} L
      \mathbf{a}^2 \br[\big]{ 
    \textstyle{\prod}_{p={0} }^{L}( \ell_p+1) }
    (2V( \Theta_0)+4L^2\norm{f(0)}^2 + 1 )^{ (L-1)} 
  \bigr)
  \mL_{ \infty }( \Theta_0 )
\\
&
  =  
  - 4 L \gamma_0
  \bigl( 
    1 -
    \mathfrak{g} 
    L \mathbf{a}^2 
    \br[\big]{ 
      \textstyle{\prod}_{p={0} }^{L}( \ell_p + 1 ) 
    }
    ( 2 V( \Theta_0 ) + 4 L^2 \norm{ f(0) }^2 + 1 )^{ (L - 1) } 
  \bigr)
  \mL_{ \infty }( \Theta_0 )
  \leq 0 .
\end{split}
\end{equation}
This establishes \cref{V_theta_diff2} 
in the base case $ n = 0 $. 
For the induction step let $ n \in \enne $ satisfy 
for all $ m \in \{ 0, 1, \dots, n - 1 \} $ that
\begin{equation}
\label{V_theta_diff_indhyp}
\begin{split}
&
  V( \Theta_{m+1} )-V( \Theta_m )
\\
&
  \leq 
  - 4 L \gamma_m
  \bigl( 
    1 - \mathfrak{g} 
    L  \mathbf{a}^2 \br[\big]{
    \textstyle{\prod}_{p={0} }^{L}( \ell_p+1) }
    (2V( \Theta_0)+4L^2\norm{f(0)}^2 + 1 )^{ (L-1)} 
  \bigr)
  \mL_{ \infty }( \Theta_m ) .
\end{split}
\end{equation}   
\Nobs that \cref{V_theta_diff_indhyp}
and the fact that for all 
$ \theta \in \Reals^{ \fd } $ 
it holds that $ \mL_{ \infty }( \theta)\geq 0 $ 
ensure that
$
  V( \Theta_n ) \leq V( \Theta_{ n - 1 } ) 
  \leq \dots \leq V( \Theta_0 ) 
$.
Combining this with \cref{cor:V_theta_diff_const}, \cref{gamma_ass1}, 
and the fact that 
$ \gamma_n \mathfrak{m} \leq\mathfrak{g} $
demonstrates that 
\begin{equation}
\label{V_theta_diff_ind2}
\begin{split}
&
  V( \Theta_{ n + 1 } ) - V( \Theta_n )
\\
&
  \leq 4 L
  \bigl( 
    - \gamma_n +
    ( \gamma_n )^2 \mathfrak{m}
    L \mathbf{a}^2 \br[\big]{ 
    \textstyle{\prod}_{p={0} }^{L}( \ell_p+1) }
    (2V( \Theta_0)+4L^2\norm{f(0)}^2 + 1 )^{ (L-1)} 
  \bigr)
  \mL_{ \infty }( \Theta_n )
\\
&
  \leq 4 L
  \bigl(
    - \gamma_n+
    \gamma_n\mathfrak{g} 
     L \mathbf{a}^2 \br[\big]{ 
    \textstyle{\prod}_{p={0} }^{L}( \ell_p+1) }
    (2V( \Theta_0)+4L^2\norm{f(0)}^2 + 1 )^{ (L-1)} 
  \bigr) 
   \mL_{ \infty }( \Theta_n )
\\
&
  =  -4L\gamma_n
    \bigl( 1-
    \mathfrak{g} 
    L \mathbf{a}^2 \br[\big]{
    \textstyle{\prod}_{p={0} }^{L}( \ell_p+1) }
    (2V( \Theta_0)+4L^2\norm{f(0)}^2 + 1 )^{ (L-1)}  
  \bigr)
  \mL_{ \infty }( \Theta_n )
  \leq 0 .
\end{split}
\end{equation}
Induction thus establishes \cref{V_theta_diff2}.
The proof of \cref{lem:V_theta_diff2} is thus complete.
\end{proof}

\subsection{Convergence analysis for GD processes}
\label{ssec:GD_convergence}

\begin{theorem}
\label{thm:GD_conv}
Assume \cref{main_setting},
assume for all $ x \in [a,b]^d $ that $ f(x) = f(0) $,
let $ ( \gamma_n )_{ n \in \enne_0 } \subseteq [0, \infty) $, 
let $ ( \Theta_n )_{ n \in \enne_0 } \colon \enne_0 \to \Reals^{ \fd } $ 
satisfy for all $ n \in \enne_0 $ that 
$ 
  \Theta_{ n + 1 } = \Theta_n - \gamma_n \mG( \Theta_n ) 
$,
and assume
\begin{equation}
\label{gamma_ass2}
  \sup\nolimits_{ n \in \enne_0 } 
  ( 
    \gamma_n \mathfrak{m} 
  )
  <
  \bigl(
    L \mathbf{a}^2 
    \bigl[
      \textstyle{\prod}_{p={0} }^{L}( \ell_p+1) 
    \bigr]
    ( 2 V( \Theta_0 ) + 4 L^2 \norm{ f(0) }^2 + 1 )^{ (L - 1) } 
  \bigr)^{ - 1 }
\end{equation} 
and
$ 
  \sum_{ n = 0 }^{ \infty } \gamma_n = \infty 
$.
Then 
\begin{enumerate}[label=(\roman*)]
\item
\label{GD_conv_it1}
it holds that 
$ 
  \sup_{ n \in \enne_0 } 
  \norm{ \Theta_n } 
  \leq 
  [ 2 V( \Theta_0 ) + 4 L^2 \norm{ f(0) }^2 ]^{ 1 / 2 } < \infty
$ 
and
\item 
\label{GD_conv_it2}
it holds that
$ 
  \limsup_{ n \to \infty } \mL_{ \infty }( \Theta_n ) = 0 
$.
\end{enumerate}  
\end{theorem}
\begin{proof}[Proof of \cref{thm:GD_conv}]
Throughout this proof let $ \eta\in (0, \infty) $ satisfy
\begin{equation}
\label{eta}
\begin{split} 
  \eta
&
  = 4 L
  \bigl(
    1 - 
    \br*{ 
      \sup\nolimits_{ m \in \enne_0 } \gamma_m 
    }
    \mathfrak{m}
    L \mathbf{a}^2 
    \bigl[
      \textstyle{ \prod }_{ p = 0 }^L 
      ( \ell_p + 1 ) 
    \bigr]
    ( 
      2 V( \Theta_0 ) + 4 L^2 \norm{ f(0) }^2 + 1 
    )^{ (L - 1) } 
  \bigr)
\end{split}  
\end{equation}   
and let $ \varepsilon\in \Reals $ satisfy
$ 
  \varepsilon =
  ( \nicefrac{ 1 }{ 3 } 
  )[ 
    \min\{ 
      1, 
      \limsup_{ n \to \infty } \mL_{ \infty }( \Theta_n )
    \}
  ]
$.
\Nobs that \cref{it:ineq} of \cref{prop:liap1} ensures that
for all $ n \in \enne_0 $ it holds that
\begin{equation}  
\label{Theta_up}
  \norm{ \Theta_n }
  \leq
  [ 2 V( \Theta_n ) + 4 L^2 \norm{ f(0) }^2 ]^{ 1 / 2 } 
  .
\end{equation}
\Moreover \cref{lem:V_theta_diff2} implies that 
  for all $ n \in \enne_0 $ it holds that
  $V ( \Theta_n )\leq V( \Theta_{n-1} )\leq\dots\leq V( \Theta_0) $.
  This and \cref{Theta_up} 
  establish \cref{GD_conv_it1}.
  Next \nobs that 
  \cref{it:ineq} of 
  \cref{prop:liap1}
  ensures that
  for all $ n \in \enne $
  we have that
  $V( \Theta_n )\geq 
  \frac{1}{2} \norm{\Theta_n}^2
  -2L^2 \norm{f(0)}^2
  \geq -2L^2 \norm{f(0)}^2$.  
  Combining this with 
  \cref{lem:V_theta_diff2} and 
  \cref{eta} 
  implies that
  for all $N\in \enne $ it holds that
  \begin{equation}
  \label{L_infty_up1}
    \eta \br*{ \sum_{n=0}^{N-1} \gamma_n\mL_{ \infty }( \Theta_n ) } 
    \leq\sum_{n=0}^{N-1}(V( \Theta_n )-V( \Theta_{n+1} ))
    =V( \Theta_0)-V( \Theta_N)
    \leq V( \Theta_0)
    +2L^2 \norm{f(0)}^2.
  \end{equation}
  This demonstrates that
  \begin{equation}
  \label{L_infty_up2}
    \textstyle\sum\limits_{n=0}^\infty[\gamma_n\mL_{ \infty }( \Theta_n )]
    \leq \eta ^{- 1 }
    (V( \Theta_0) + 2L^2 \norm{f(0)}^2)
    <\infty.
  \end{equation}
  Combining this with the assumption that
  $ \sum_{n=0}^\infty\gamma_n=\infty$ ensures
  that $ \liminf_{ n \to \infty } \mL_{ \infty }( \Theta_n )=0$.
  In the following we prove \cref{GD_conv_it2}
  by contradiction.   
  For this assume that
  \begin{equation}
  \label{L_infty_con}
    \limsup \nolimits_{ n \to \infty } \mL_{ \infty }( \Theta_n )
    >0.
  \end{equation}
  \Nobs that \cref{L_infty_con}
  implies that 
  \begin{equation}
  \label{L_infty_eps}
    0=\liminf\nolimits_{ n \to \infty } \mL_{ \infty }( \Theta_n )
    <\varepsilon
    <2\varepsilon
    <\limsup\nolimits_{ n \to \infty } \mL_{ \infty }( \Theta_n ).
  \end{equation}
  This ensures that there exist 
  $ (m_k, n_k) \in \enne^2 $, $ k \in \enne $, which satisfy
  for all $ k \in \enne $ that
  $ m_k < n_k < m_{k+1} $,
  $ \mL_{ \infty }( \Theta_{m_k} )>2\varepsilon $,  
  and
  $ \mL_{ \infty }( \Theta_{n_k} )<\varepsilon
  <\min_{j\in \enne\cap[m_k,n_k)} \mL_{ \infty }( \Theta_j) $.
  \Nobs that \cref{L_infty_up2}
  and the fact that
  for all $ k \in \enne $, $ j \in \enne\cap[m_k,n_k) $
  it holds that
  $1 \leq \frac{1}{\varepsilon} \mL_{ \infty }( \Theta_j) $
  demonstrate that
  \begin{equation}
  \label{sum_gamma}
    \sum\limits_{k= 1 }^\infty\sum_{ j =m_k}^{n_k-1} \gamma_j
    \leq\frac{1}{\varepsilon}  \br*{ \sum\limits_{k= 1 }^\infty\sum_{ j =m_k}^{n_k-1}( \gamma_j\mL_{ \infty }( \Theta_j)) } 
    \leq\frac{1}{\varepsilon}  \br*{ \sum\limits_{ j =0}^\infty( \gamma_j\mL_{ \infty }( \Theta_j)) } 
    <\infty.
  \end{equation}  
  \Moreover \cref{cor:G_up}
  and \cref{GD_conv_it1}
  imply that there exists $ \mathfrak{C} \in \Reals $
  which satisfies that
  \begin{equation}
  \label{G_bd_C}
    \sup\nolimits_{ n \in \enne_0 } \norm{\mG( \Theta_n )}
    \leq\mathfrak{C}.
  \end{equation}
  \Nobs that \cref{G_bd_C}, the triangle inequality, and \cref{sum_gamma} prove that 
  \begin{equation}
  \label{sum_diff_Theta}
  \begin{split}
    \sum\limits_{k= 1 }^\infty\norm{\Theta_{n_k}-\Theta_{m_k} }
    \leq \sum\limits_{k= 1 }^\infty\sum\limits_{ j =m_k}^{n_k-1} \norm{\Theta_{j+1}-\Theta_{j} }
    =\sum\limits_{k= 1 }^\infty\sum\limits_{ j =m_k}^{n_k-1}( \gamma_j\norm{\mG( \Theta_{j} )} )
    \leq
    \mathfrak{C}
     \br*{ \sum\limits_{k= 1 }^\infty\sum\limits_{ j =m_k}^{n_k-1} \gamma_j  } 
    <\infty.
  \end{split}
  \end{equation}
  Next \nobs that \cref{lem:loc_lip}
  and \cref{GD_conv_it1} ensure that there exists
  $ \mathscr{L} \in \Reals $ which satisfies for all
  $m,n\in \enne_0 $ that
  $ \abs{\mL( \Theta_m)-\mL( \Theta_n )}
  \leq\mathscr{L} \norm{\Theta_m-\Theta_n} $.
  Combining this with 
  \cref{sum_diff_Theta}
  shows that
  \begin{equation}
    \limsup\nolimits_{k\to \infty } \abs{\mL_{ \infty }( \Theta_{n_k} )
    -\mL_{ \infty }( \Theta_{m_k} )}
    \leq\limsup\nolimits_{k\to \infty }( \mathscr{L} \norm{\Theta_{n_k}-\Theta_{m_k} } )
    =0.
  \end{equation}  
  The fact that for all $ k \in \enne_0 $
  it holds that
  $ \mL_{ \infty }( \Theta_{n_k} )<\varepsilon
  <2\varepsilon<\mL_{ \infty }( \Theta_{m_k} ) $
  \hence implies that
  \begin{equation}
    0<\varepsilon
    \leq\inf\nolimits_{k\in \enne } \abs{\mL_{ \infty }( \Theta_{n_k} )
    -\mL_{ \infty }( \Theta_{m_k} )}
    \leq\limsup\nolimits_{k\to \infty } \abs{\mL_{ \infty }( \Theta_{n_k} )
    -\mL_{ \infty }( \Theta_{m_k} )}
    =0.
  \end{equation}  
  This contradiction establishes \cref{GD_conv_it2}.  
  The proof of \cref{thm:GD_conv} is thus complete.
\end{proof}

\section{Stochastic gradient descent (SGD) processes in the training of deep ANNs}
\label{sec:stoch_grad_des}

In this section we use some of the results from 
\cref{sec:risk,sec:grad_fl,sec:grad_des} 
above to establish in \cref{thm:sgd_conv}
in \cref{ssec:SGD_convergence} below 
that in the training of deep ReLU ANNs 
we have that the sequence of risks 
$
  \mL( \Theta_n )
$, 
$ n \in \N_0 $, 
of any time-discrete SGD process 
$
  \Theta = ( \Theta_n )_{ n \in \N_0 } \colon \N_0 \times \Omega \to \Rr^{ \fd }
$
converges to zero provided that 
the target function 
$ f \colon [a,b]^{ \ell_0 } \to \Rr^{ \ell_L } $
is a constant function and 
provided that the learning rates 
(the step sizes) 
$ 
  \gamma_n \in [0,\infty)
$, 
$ n \in \N_0 $, 
in the GD optimization method are sufficiently small 
(see \cref{ass_delta} in \cref{thm:sgd_conv} for details) 
but fail to be $ L^1 $-summable so that 
$
  \sum_{ n = 0 }^{ \infty } \gamma_n = \infty 
$.

In \cref{cor:sgd_conv} in \cref{ssec:SGD_convergence} 
we simplify the statement of \cref{thm:sgd_conv} 
by imposing a more restrictive smallness condition 
on the learning rates 
$ ( \gamma_n )_{ n \in \N_0 } \subseteq [0,\infty) $
(compare the smallness assumption on the learning rates in \cref{ass_Theta} in \cref{cor:sgd_conv} 
with the smallness assumption on the learning rates in \cref{ass_delta} in \cref{thm:sgd_conv} for details). 
\cref{thm:main_thm} in the introduction is a direct consequence of \cref{cor:sgd_conv}. 
In our proof of \cref{cor:sgd_conv} 
we employ the elementary upper estimate 
for the Lyapunov function $ V \colon \Rr^{ \fd } \to \Rr $ 
(see \cref{V_c_sgd} in \cref{SGD_setting} in \cref{ssec:SGD_framework} below)
in \cref{prop:liap1}
in \cref{ssec:GF_lyapunov} above 
to verify that the smallness assumption 
in \cref{ass_delta} in \cref{thm:sgd_conv} 
is satisfied so that \cref{thm:sgd_conv} 
can be applied.

Our proof of \cref{thm:sgd_conv} employs 
\cref{prop:liap1} in \cref{ssec:GF_lyapunov}, 
the well-known integrability result 
in \cref{cor:L_exp} in \cref{ssec:SGD_properties1} below, 
the uniform local boundedness result 
for the generalized gradients 
$ 
  \fG^n = ( \fG^n_1, \dots, \fG^n_{ \fd } ) 
  \colon \Reals^{ \fd } \times \Omega \to \Reals^{ \fd } 
$, 
$
  n \in \N_0
$,
(see \cref{eq:generalized_stochastic_gradients} in \cref{SGD_setting})
of the empirical risk functions  
$ 
  \fL^n_{ \infty } \colon \Reals^{ \fd } \times \Omega \to \Reals 
$, 
$ 
  n \in \N_0
$, 
in \cref{lem:G_cpt_bd} in \cref{ssec:SGD_properties2} below, 
the local Lipschitz continuity result 
for the risk function 
$
  \mL \colon \Rr^{ \fd } \to \Rr 
$
in \cref{lem:loc_lip} in \cref{ssec:local_lipschitz} above, 
and the probabilistic recursive upper estimate 
for the composition 
$ V( \Theta_n ) $, $ n \in \N_0 $, 
of the Lyapunov function 
$ V \colon \Rr^{ \fd } \to \Rr $
and the time-discrete SGD process 
$
  \Theta = ( \Theta_n )_{ n \in \N_0 } \colon \N_0 \times \Omega \to \Rr^{ \fd } 
$
in \cref{cor:Theta_diff_sgd2} 
in \cref{ssec:SGD_upper} below.

Our proof of \cref{lem:G_cpt_bd} also 
uses the local Lipschitz continuity result 
for the risk function in 
\cref{lem:loc_lip} in \cref{ssec:local_lipschitz}
as well as the explicit pathwise polynomial growth estimates 
for the generalized gradients
$ 
  \fG^n = ( \fG^n_1, \dots, \fG^n_{ \fd } ) 
  \colon \Reals^{ \fd } \times \Omega \to \Reals^{ \fd } 
$, 
$
  n \in \N_0
$, 
of the empirical risk functions 
in \cref{lem:G_upper} 
in \cref{ssec:SGD_properties2}. 
\cref{lem:G_upper}, in turn, is a direct consequence 
of the explicit polynomial growth estimate 
for the generalized gradient function 
$
  \mG \colon \Rr^{ \fd } \to \Rr^{ \fd } 
$
in \cref{prop:G_upper_estimate} 
in \cref{ssec:upper_estimates} above.

Our proof of \cref{cor:Theta_diff_sgd2} 
is based on induction as well as on the 
pathwise recursive upper estimate 
for the composition 
$ V( \Theta_n ) $, $ n \in \N_0 $, 
of the Lyapunov function 
$ V \colon \Rr^{ \fd } \to \Rr $
and the time-discrete SGD process 
$
  \Theta = ( \Theta_n )_{ n \in \N_0 } \colon \N_0 \times \Omega \to \Rr^{ \fd } 
$
in \cref{lem:Theta_diff_sgd} in \cref{ssec:SGD_upper} below.  
Our proof of \cref{lem:Theta_diff_sgd} employs 
\cref{prop:liap1} in \cref{ssec:GF_lyapunov}, 
the explicit pathwise polynomial growth estimates 
for the generalized gradients
$ 
  \fG^n = ( \fG^n_1, \dots, \fG^n_{ \fd } ) 
  \colon \Reals^{ \fd } \times \Omega \to \Reals^{ \fd } 
$, 
$
  n \in \N_0
$, 
of the empirical risk functions 
in \cref{lem:G_upper}, 
and 
the one-step Lyapunov estimate 
for the SGD method 
in \cref{lem:V_diff_sgd} 
in \cref{ssec:SGD_lyapunov} below. 
\cref{lem:V_diff_sgd} is a direct 
consequence of the one-step Lyapunov estimate 
for the GD method 
in \cref{cor:liap3}
in \cref{ssec:GD_lyapunov} above.

In \cref{cor:L_exp} 
we demonstrate that for every time point $ n \in \N_0 $ we have that 
the expectation 
$
  \E[ \fL^n_{ \infty }( \Theta_n ) ] 
$
of the empirical risk  
of the SGD process at time $ n $ 
coincides with the expectation 
$
  \E[ \mL( \Theta_n ) ] 
$
of the risk of the SGD process at time $ n $. 
Our proof of \cref{cor:L_exp} employs 
the well-known integrability result in 
\cref{prop:L_exp} in \cref{ssec:SGD_properties1} 
and the well-known measurability result in 
\cref{lem:G_meas} in \cref{ssec:SGD_properties1}. 
In \cref{prop:L_exp} we assert 
that the expectations  
$
  \E[ \fL^n_{ \infty }( \theta ) ] 
$, 
$ \theta \in \Rr^{ \fd } $, 
$ n \in \N $, 
of the empirical risk 
functions 
$
  \fL^n_{ \infty } \colon \Rr^{ \fd } \times \Omega \to \Rr
$, 
$ n \in \N $, 
coincide with the risk function 
$
  \Rr^{ \fd } \ni \theta \mapsto \mL( \theta ) \in \Rr
$.

In \cref{lem:G_meas} we collect well-known measurability and independence properties 
for the input data, the SGD process, and the generalized gradients of the empirical 
risk functions. 
Our proof of \cref{lem:G_meas} makes use of 
the local Lipschitz continuity result for the risk function in 
\cref{lem:loc_lip} in \cref{ssec:local_lipschitz} 
and of the explicit representation result for 
the generalized gradients of (approximations of) the empirical risk functions 
in \cref{prop:L_grad} 
in \cref{ssec:SGD_representation} below. 
\cref{prop:L_grad} is a direct consequence of the 
explicit representation 
for the generalized gradient function
of the risk function 
in \cref{prop:G} 
in \cref{ssec:gen_grad_risk_functions} above.

The findings in this section extend the findings in \cite[Section~4]{MR4468133} 
from shallow ReLU ANNs with just one hidden layer to deep ReLU ANNs 
with an arbitrarily large number of hidden layers.
In particular,
\cref{prop:L_grad} in \cref{ssec:SGD_representation}
extends \cite[Proposition 3.2]{MR4468133},
\cref{prop:L_exp} in \cref{ssec:SGD_properties1}
generalizes \cite[Proposition 3.3]{MR4468133},
\cref{lem:G_meas} in \cref{ssec:SGD_properties1} extends \cite[Lemma 3.4]{MR4468133},
\cref{cor:L_exp} in \cref{ssec:SGD_properties1} extends \cite[Corollary 3.5]{MR4468133},
\cref{lem:G_upper} in \cref{ssec:SGD_properties2}
extends \cite[Lemma 3.6]{MR4468133},
\cref{lem:G_cpt_bd} in \cref{ssec:SGD_properties2} 
extends \cite[Lemma 3.7]{MR4468133},
\cref{lem:V_diff_sgd} in \cref{ssec:SGD_lyapunov}
extends \cite[Lemma 3.9]{MR4468133},
\cref{lem:Theta_diff_sgd} 
in \cref{ssec:SGD_upper} 
extends \cite[Lemma 3.10]{MR4468133},
\cref{cor:Theta_diff_sgd2} 
in \cref{ssec:SGD_upper} 
extends \cite[Corollary 3.11]{MR4468133},
\cref{thm:sgd_conv} in \cref{ssec:SGD_convergence} extends \cite[Theorem 3.12]{MR4468133},
and \cref{cor:sgd_conv} in \cref{ssec:SGD_convergence}
extends \cite[Corollary 3.13]{MR4468133}.

\subsection{Mathematical framework for SGD processes and deep ANNs with ReLU activation}
\label{ssec:SGD_framework}

\begin{setting}
\label{SGD_setting}
Let 
$ L, \fd \in \N $, $ ( \ell_k)_{ k \in \N_0 } \subseteq \N $, 
$ \xi\in \Reals^{ \ell_L} $, 
$ a, \mathbf{a} \in \Reals $, $ b \in (a, \infty) $,
$ \A \in (0,\infty) $, 
$ \B \in ( \A, \infty) $ 
satisfy $ \fd = \sum_{ k = 1 }^L \ell_k( \ell_{k-1} + 1) $ 
and $ \mathbf{a} =\max\{\abs{a}, \abs{b}, 1 \} $,
for every $ \theta = ( \theta_1, \ldots, \theta_{ \fd }) \in \Reals^{ \fd } $
let $ \fw^{ k, \theta } = ( \fw^{ k, \theta }_{ i,j} )_{ (i,j) \in \{ 1, \ldots, \ell_k \} \times \{ 1, \ldots, \ell_{k-1} \} }
\in \Reals^{ \ell_k \times \ell_{k-1} } $, $ k \in \N $, 
and $ \fb^{ k, \theta } = ( \fb^{ k, \theta }_1, \ldots, \fb^{ k, \theta }_{\ell_k} )
\in \Reals^{ \ell_k} $,
$ k \in \N $, 
satisfy for all 
$ k \in \{ 1, \ldots, L\} $, $ i \in \{ 1, \ldots, \ell_k \} $,
$ j \in \{ 1, \ldots, \ell_{k-1} \} $ that
\begin{equation}
\label{wb_sgd}
\fw^{ k, \theta }_{ i,j} =
\theta_{ (i-1)\ell_{k-1} + j+\sum_{h= 1 }^{ k-1} \ell_h( \ell_{h-1} + 1)}
\qquad\text{and} \qquad
\fb^{ k, \theta }_{ i} =
\theta_{\ell_k\ell_{k-1} + i+\sum_{h= 1 }^{ k-1} \ell_h( \ell_{h-1} + 1)} \,,
\end{equation}
for every $ k \in \N $, $ \theta \in \Reals^{ \fd } $
let $ \mA_k^\theta = ( \mA_{k,1}^\theta, \ldots, \mA_{k, \ell_k}^\theta)
\colon \Reals^{ \ell_{k-1} } \to \Reals^{ \ell_k} $
satisfy 
for all $ x \in \Reals^{ \ell_{k-1} } $
that $ \mA_k^\theta(x)=\fb^{ k, \theta } + \fw^{ k, \theta }x$,
let 
$ \R_r \colon \Rr \to \Rr $, 
$ r \in [1, \infty] $,
satisfy for all 
$ r \in [1, \infty) $, 
$ x \in ( - \infty, \A r^{ - 1 } ] $, 
$ 
y \in \Rr
$, 
$
z \in [ \B r^{ - 1 }, \infty ) 
$
that 
\begin{equation}
\label{lim_R_sgd}
\R_r \in C^1( \Reals, \Reals ) ,
\qquad
\R_r(x) = 0,
\qquad 
0 \leq \R_r(y) \leq \R_{ \infty }( y ) 
= 
\max\{ y, 0 \}
,
\qqandqq
\R_r(z) = z
,
\end{equation}
assume 
$
\sup_{ r \in [1, \infty) }
\sup_{ x \in \Reals } 
| ( \R_r )'( x ) | < \infty 
$,
let $ \norm{\cdot} \colon( \cup_{ n \in \enne } \Reals^n ) \to\Rr$,
$ \ip{\cdot}{\cdot} \colon( \cup_{ n \in \enne }( \Reals^n\times \Reals^n )) \to\Reals $, and
$ \mathfrak{M}_r \colon( \cup_{ n \in \enne } \Reals^n ) \to ( \cup_{ n \in \enne } \Reals^n ) $, $ r\in [1 , \infty ] $, 
satisfy for all $ r\in [1 , \infty ] $, $ n \in \N $,
$ x = (x_1, \ldots, x_n ) $, $y= (y_1, \ldots,y_n ) \in \Rr^n $
that 
\begin{equation}\label{norm_scalar_M_sgd}
  \textstyle{  
  \norm{x} = ( \sum_{ i= 1 }^n|x_{ i}|^2)^{ 1/2},
  \qquad \ip{x}{y} =\sum_{ i= 1 }^nx_iy_i, 
  \qquad \text{and}
  \qquad \mathfrak{M}_r(x)= ( \R_r(x_1), \ldots, \R_r(x_n ))},
\end{equation}
for every $ \theta \in \Rr^{ \fd } $
let $ \mN^{ k, \theta }_r = ( \mN^{ k, \theta }_{r,1}, \ldots, \mN^{ k, \theta }_{r, \ell_k} )
\colon \Rr^{ \ell_0 } \to \Rr^{ \ell_k} $,
$ r\in [1 , \infty ] $, $ k \in \N $, 
and 
$ \mX^{ k, \theta }_{ i} \subseteq\Rr^{ \ell_0 } $, $k,i \in \N $,
satisfy 
for all $ r\in [1 , \infty ]  $, $ k \in \N $, $ i \in \{ 1, \ldots, \ell_k \} $,
$ x \in \Reals^{ \ell_0 } $
that
\begin{equation}
\begin{split}
\label{N_sgd}
  &
  \mN^{ 1, \theta }_r(x)=\mA^\theta_1(x), \qquad
  \mN^{ k+1, \theta }_r(x)=\mA_{k+1}^\theta( \mathfrak{M}_{r^{1 / k } }( \mN^{ k, \theta }_r(x))),
\end{split}
\end{equation}
and $ \mX^{ k, \theta }_{ i} = \{ y \in [a,b]^{ \ell_0 } \colon
\mN^{ k, \theta }_{ \infty, i}(y)>0\} $,
let $ ( \Omega, \F, \P) $ be a probability space,
let $ X^{n,m} =$ $(X^{n,m}_1, \ldots,$ $ X^{n,m}_{\ell_0 } )\colon	\Omega\to [a,b]^{ \ell_0 } $,
$ n, m \in \enne_0 $ be i.i.d.\ random variables,
let $ \mL \colon \Reals^{ \fd } \to \Reals $ and
$ V \colon \Reals^{ \fd } \to \Reals $ satisfy 
for all $ \theta \in \Reals^{ \fd } $
that 
$ 
  \mL( \theta ) = \E[ \norm{ \mN_{ \infty }^{ L, \theta }( X^{0,0} ) - \xi }^2 ] 
$
and
\begin{equation}
\label{V_c_sgd}
\textstyle 
  V( \theta) = 
  \br[\big]{ 
    \sum_{ k = 1 }^L 
    \pr[\big]{ 
      k \| \fb^{ k, \theta } \|^2 
      +
      \sum_{ i = 1 }^{ \ell_k } 
      \sum_{ j = 1 }^{ \ell_{k-1} }
      |
        \fw^{ k, \theta }_{ i, j } 
      |^2 
    } 
  }
  - 2 L \langle \xi, \fb^{ L, \theta } \rangle 
  ,
\end{equation}
let $ ( M_n )_{ n \in \enne_0 } \subseteq \enne $,
for every $ n \in \enne_0 $, $ r \in [ 1 , \infty ]  $
let 
$ 
  \fL^n_r \colon \Reals^{ \fd } \times \Omega \to \Reals 
$ 
satisfy for all $ \theta \in \Reals^{ \fd } $, $ \omega \in \Omega $ that
$ 
  \fL^n_r( \theta, \omega ) 
  = \frac{ 1 }{ M_n } 
  \sum_{ m = 1 }^{ M_n } 
  \norm{ \mN^{ L, \theta }_r( X^{ n, m }( \omega ) ) - \xi }^2
$,
for every $ n \in \enne_0 $
let 
$ 
  \fG^n = ( \fG^n_1, \dots, \fG^n_{ \fd } ) 
  \colon \Reals^{ \fd } \times \Omega \to \Reals^{ \fd } 
$ 
satisfy for all 
$ \theta \in \Reals^{ \fd } $, 
$ 
  \omega \in \{ w \in \Omega 
  \colon ( ( \nabla_\theta \fL^n_r)( \theta, w ) )_{ r \in [ 1 , \infty ) } 
  \text{ is } $ $ \text{convergent} \} 
$ 
that
\begin{equation}
\label{eq:generalized_stochastic_gradients}
\textstyle 
  \fG^n( \theta, \omega ) = 
  \lim_{ r \to \infty }( \nabla_{ \theta } \fL^n_r )( \theta, \omega ) ,
\end{equation}
let 
$ 
  \Theta = ( \Theta_n )_{ n \in \enne_0 } \colon \enne_0 \times \Omega \to \Reals^{ \fd } 
$
be a stochastic process,
let $ ( \gamma_n )_{ n \in \enne_0 } \subseteq [0, \infty) $,
assume that 
$ \Theta_0 $ and 
$
  ( X^{ n, m } )_{ (n,m) \in ( \enne_0)^2 } 
$ 
are independent,
and assume for all $ n \in \enne_0 $, $ \omega \in \Omega $
that $ \Theta_{n+1}( \omega)=\Theta_n( \omega)-\gamma_n\fG^n( \Theta_n( \omega), \omega) $.
\end{setting}

\subsection{Explicit representations for the generalized gradients 
of the empirical risk function}
\label{ssec:SGD_representation}

\begin{prop}
\label{prop:L_grad}
  Assume \cref{SGD_setting} and let $ n \in \enne_0 $, $ \omega \in \Omega $.
  Then
  \begin{enumerate}[(i)]
  \item\label{it:L_C1}
  it holds for all $ r \in [ 1 , \infty )  $ that 
  $ 
    ( \Rr^{ \fd } \ni \theta \mapsto \fL^n_r( \theta, \omega ) \in \Rr ) \in C^1( \Reals^{ \fd }, \Reals) 
  $,
\item\label{it:L_grad1}
it holds for all $ r\in [ 1 , \infty )  $, $ k \in \{ 1, \ldots, L\} $, 
$ i \in \{ 1, \ldots, \ell_k \} $,
$ j \in \{ 1, \ldots, \ell_{k-1} \} $, 
$ \theta = ( \theta_1, \ldots, \theta_{ \fd }) \in \Reals^{ \fd } $ 
that
\begin{align}
\label{L_grad1}
& 
   \biggl(  
     \frac{ \partial }{
       \partial\theta_{ (i-1) \ell_{k-1} + j 
       + \sum_{ h = 1 }^{ k - 1 } \ell_h ( \ell_{ h - 1 } + 1 ) } 
     }
     \fL^n_r 
   \biggr)( \theta, \omega)
\\
&  =
  \!\!\!\!\!
   \sum_{\substack{v_k,v_{k+1},  
  \ldots,v_L\in \enne, \\\forall w\in \enne
  \cap[k, L]\colon v_w\leq\ell_w} }
  \!\!\!\!\!
  \textstyle
  \frac{2}{M_n} \sum\limits_{ m = 1 }^{M_n}
  \Big[
  \R_{ r ^{ 1 / ( \max \cu { k - 1 , 1 } ) } } ( \mN^{ \max\{k-1,1\}, \theta }_{r,j}(X^{n,m}( \omega)))
  \indicator{  (1, L]}(k)
  +X^{n,m}_j( \omega) \indicator{ \{1\} }(k)
  \Big]
\nonumber 
\\
\nonumber
&
  \quad\cdot
  \Big[ \indicator{ \{ i \} }(v_k)\Big]
  \Big[ \mN_{r,v_L}^{L, \theta }(X^{n,m}( \omega))-\xi_{v_L} \Big] 
  \Bigl[
    \textstyle{ \prod }_{ q = k + 1 }^L 
    \bigl(
      \fw^{q, \theta }_{v_{q}, v_{q-1} }
      \bigl[
        ( 
          \R_{ r^{ 1 / ( q - 1 ) } } 
        )'( 
          \mN^{ q - 1, \theta }_{ r, v_{ q - 1 } }( X^{ n, m }( \omega ) )
        )
      \bigr]
    \bigr)
  \Bigr] ,
\end{align}
\item
\label{it:L_grad2}
it holds for all 
$ r \in [ 1, \infty ) $, $ k \in \{ 1, \dots, L \} $, 
$ i \in \{ 1, \dots, \ell_k \} $, 
$ \theta = ( \theta_1, \ldots, \theta_{ \fd } ) \in \Reals^{ \fd } $ 
that
\begin{equation}
\label{L_grad2}
\begin{split}
& 
  \biggl( 
    \frac{ \partial }{
      \partial \theta_{ 
        \ell_k \ell_{k-1} + i 
        + \sum_{ h = 1 }^{ k-1} \ell_h ( \ell_{ h - 1 } + 1 ) 
      } 
    }
    \fL^n_r 
  \biggr)( \theta, \omega ) 
\\
&
  =
  \sum_{
    \substack{v_k,v_{k+1},  
    \ldots,v_L\in \enne, \\\forall w\in \enne
    \cap[k, L]\colon v_w\leq\ell_w} 
  }
  \textstyle 
  \frac{2}{M_n} 
  \sum\limits_{ m = 1 }^{ M_n }
  \Big[ \indicator{ \{ i \} }(v_k)\Big]
  \Big[ \mN_{r,v_L}^{L, \theta }(X^{n,m}( \omega))-\xi_{v_L} \Big] 
  \\
  &\quad\cdot
  \Big[
  \textstyle{\prod}_{q={k+1} }^{L} \big(
  \fw^{q, \theta }_{v_{q}, v_{q-1} }
  \big[( \R_ {r^{ 1 / ( q - 1 ) } } )'( \mN^{q-1, \theta }_{r,v_{q-1} }(X^{n,m}( \omega)))\big]
  \big)
  \Big],
\end{split}
\end{equation}
\item
\label{it:grad_L_conv_G}
it holds 
for all $ \theta = ( \theta_1, \ldots, \theta_{ \fd } ) \in \Reals^{ \fd } $ 
that
$
  \limsup\nolimits_{r\to \infty } \norm{ ( \nabla \fL_r^n )( \theta, \omega)
  -\fG^n( \theta, \omega)} =0
$,
\item  
\label{it:G_1}
it holds for all $ k \in \{ 1, \ldots, L\} $,
$ i \in \{ 1, \ldots, \ell_k \} $, 
$ j \in \{ 1, \ldots, \ell_{k-1} \} $, 
$ \theta = ( \theta_1, \ldots, \theta_{ \fd } ) \in \Reals^{ \fd } $ 
that
\begin{equation}
\label{G_1}
\begin{split}
&
  \fG^n_{ (i-1) \ell_{ k - 1 } + j + \sum_{ h = 1 }^{ k - 1 } 
  \ell_h ( \ell_{ h - 1 } + 1 ) }( \theta, \omega ) 
\\
&
  =
  \sum_{
    \substack{ 
      v_k, v_{k+1}, \dots, v_L \in \enne, 
    \\
      \forall w \in \enne \cap [k, L] \colon v_w \leq \ell_w 
    } 
  }
  \textstyle 
  \frac{2}{M_n} 
  \sum\limits_{ m = 1 }^{ M_n }
  \Big[
  \R_{ \infty }( \mN^{ \max\{k-1,1\}, \theta }_{ \infty,j}(X^{n,m}( \omega)))
  \indicator{  (1, L]}(k)
  +X^{n,m}_j( \omega) \indicator{ \{1\} }(k)
  \Big]\\
  &\quad\cdot
  \Big[ \indicator{ \{ i \} }(v_k)\Big]
  \Big[ \mN_{ \infty,v_L}^{L, \theta }(X^{n,m}( \omega))-\xi_{v_L} \Big] 
  \Big[
  \textstyle{\prod}_{q={k+1} }^{L} \big(
  \fw^{q, \theta }_{v_{q}, v_{q-1} }
  \indicator{ \mX^{q-1, \theta }_{v_{q-1} }}(X^{n,m}( \omega))
  \big)
  \Big],
\end{split}
\end{equation}
and
\item  
\label{it:G_2}
it holds for all 
$ k \in \{ 1, \ldots, L \} $,
$ i \in \{ 1, \ldots, \ell_k \} $, 
$ \theta = ( \theta_1, \ldots, \theta_{ \fd } ) \in \Reals^{ \fd } $ 
that
\begin{equation}
\label{G_2}
\begin{split}
&
  \fG^n_{ \ell_k \ell_{k-1} + i + \sum_{ h = 1 }^{ k - 1 } 
  \ell_h ( \ell_{ h - 1 } + 1 ) }( \theta, \omega )
  =
  \sum_{
    \substack{ 
      v_k, v_{k+1}, \dots, v_L \in \enne, 
    \\
      \forall w \in \enne \cap [k, L] \colon v_w \leq \ell_w 
    } 
  }
  \textstyle 
  \frac{ 2 }{ M_n } 
  \sum\limits_{ m = 1 }^{ M_n }
  \Bigl[ 
    \indicator{ \{ i \} }( v_k )
  \Bigr]
   \\
  &\quad\cdot
  \Big[ \mN_{ \infty,v_L}^{L, \theta }(X^{n,m}( \omega))-\xi_{v_L} \Big]
  \Big[
  \textstyle{\prod}_{q={k+1} }^{L} \big(
  \fw^{q, \theta }_{v_{q}, v_{q-1} }
  \indicator{ \mX^{q-1, \theta }_{v_{q-1} }}(X^{n,m}( \omega))
  \big)
  \Big].
\end{split}
\end{equation}
\end{enumerate}
\end{prop}
\begin{proof}[Proof of \cref{prop:L_grad}]
\Nobs that \cref{G1,G2',G2'',G3,G4',G4''} of 
\cref{prop:G} (applied with
$ 
  \mu \curvearrowleft 
  ( 
    \mathcal{B}( [a,b]^{ \ell_0 } ) 
    \ni A \mapsto 
    \frac{ 1 }{ M_n } 
    \sum_{ m = 1 }^{ M_n } 
    \mathbbm{1}_A( X^{ n, m }( \omega ) ) 
    \in [0,1] 
  ) 
$,
$
  f \curvearrowleft ( [a,b]^{ \ell_0 } \ni x
  \mapsto \xi \in \Reals^{ \ell_L} ) 
$ 
in the notation of \cref{prop:G}) prove 
\cref{it:L_C1,it:L_grad1,it:L_grad2,it:grad_L_conv_G,it:G_1,it:G_2}.
The proof of \cref{prop:L_grad} is thus complete.
\end{proof}

\subsection{Properties of expectations of the empirical risk functions}
\label{ssec:SGD_properties1}

\begin{prop}
\label{prop:L_exp}
Assume \cref{SGD_setting}.
Then it holds for all $ n \in \enne_0 $, $ \theta \in \Reals^{ \fd } $ that
\begin{equation}
\label{fL_exp}
  \E[ \fL^n_{ \infty }( \theta ) ] = \mL( \theta ) 
  .
\end{equation} 
\end{prop}
\begin{proof}[Proof of \cref{prop:L_exp}]
\Nobs that the fact that
for all $ n \in \N_0 $,
$ \theta \in \Reals^{ \fd } $, 
$ \omega \in \Omega $ 
it holds that
$ 
  \fL^n_{ \infty }( \theta, \omega ) = 
  \frac{ 1 }{ M_n } 
  \sum_{ m = 1 }^{ M_n } 
  \norm{
    \mN^{ L, \theta }_{ \infty }( X^{ n, m }( \omega ) ) - \xi 
  }^2
$, 
the fact that for all 
$ n \in \N_0 $, $ \theta \in \Rr^{ \fd } $
it holds that 
$
  \mL( \theta ) = 
  \E[ \norm{ \mN_{ \infty }^{ L, \theta }( X^{0,0} ) - \xi }^2 ] 
$, 
and the assumption that
$ 
  X^{ n, m } \colon \Omega \to [a,b]^{ \ell_0 } 
$, 
$ n,m \in \enne_0 
$, 
are i.i.d.\ random variables
imply that for all $ n \in \enne_0 $, $ \theta \in \Reals^{ \fd } $ 
we have that
\begin{equation}
  \E[ \fL^n_{ \infty }( \theta ) ]
  =
  \frac{ 1 }{ M_n } 
  \sum_{ m = 1 }^{ M_n } 
  \E\bigl[ 
    \norm{ \mN^{ L, \theta }_{ \infty }( X^{n,m} ) - \xi }^2 
  \bigr]
  =
  \E\bigl[ 
    \norm{ \mN^{ L, \theta }_{ \infty }( X^{0,0} ) - \xi }^2 
  \bigr]
  =
  \mL( \theta ) 
  .
\end{equation}
The proof of \cref{prop:L_exp} is thus complete.
\end{proof}

\begin{lemma}
\label{lem:G_meas}
  Assume \cref{SGD_setting} and let 
  $ \mathbb{F}_n\subseteq\F$, $ n \in \enne_0 $,
   satisfy 
  for all $ n \in \enne $ that 
  $ \mathbb{F}_0=\sigma( \Theta_0) $ and 
  $ \mathbb{F}_n=\sigma( \Theta_0,(X^{ \mathfrak{n}, \mathfrak{m} } )_{\mathfrak{n}, \mathfrak{m} \in ( \enne\cap[0,n))\times \enne_0 } ) $.
  Then 
  \begin{enumerate}[label=(\roman*)]
  \item
  \label{it:G_meas}
  it holds for all $ n \in \enne_0 $ that
  $ \Reals^{ \fd }\times \Omega\ni( \theta, \omega)\mapsto\fG^n( \theta, \omega) \in \Reals^{ \fd } $
  is $ ( \mathcal{B}( \Reals^{ \fd })\otimes\mathbb{F}_{n+1} )/\mathcal{B}( \Reals^{ \fd }) $-measurable,
  \item
  \label{it:Theta_meas}
  it holds for all $ n \in \enne_0 $ that
  $ \Theta_n$ is $ \mathbb{F}_n/\mathcal{B}( \Reals^{ \fd }) $-measurable,
  and
  \item
  \label{it:ind}
  it holds for all $m,n\in \enne_0 $ that
  $ \sigma(X^{n,m} ) $ and $ \mathbb{F}_n$ are independent. 
  \end{enumerate}   
\end{lemma}
\begin{proof}[Proof of \cref{lem:G_meas}]
\Nobs that \cref{lem:loc_lip}, 
\cref{it:L_grad1,it:L_grad2} of \cref{prop:L_grad},
and the assumption that for all $ r\in [ 1 , \infty ) $ it holds that
$ \R_r \in C^1( \Reals, \Reals ) $
ensure that for all 
$ n \in \enne_0 $, $ r \in [ 1, \infty ) $, $ \omega \in \Omega$
we have that
$ 
  \Reals^{ \fd } \ni \theta \mapsto 
  ( \nabla_\theta \fL^n_r )( \theta, \omega) 
  \in \Reals^{ \fd } 
$
is continuous. 
\Moreover \cref{it:L_grad1,it:L_grad2} of \cref{prop:L_grad}
and the assumption that for all $ m, n \in \enne_0 $ 
it holds that $ X^{n,m} $ is 
$ \mathbb{F}_{n+1} $/$ \mathcal{B}( [a,b]^{ \ell_0 } ) $-measurable
imply that for all $ n \in \enne_0 $, $ r\in [ 1 , \infty ) $,
$ \theta \in \Reals^{ \fd } $ it holds that
$ 
  \Omega \ni \omega \mapsto 
  ( \nabla_\theta \fL^n_r( \theta, \omega)) \in \Reals^{ \fd } 
$
is $ \mathbb{F}_{n+1}/\mathcal{B}( \Reals^{ \fd }) $-measurable.
Combining this and, e.g., 
Beck et al.~\cite[Lemma 2.4]{MR4293960} 
(applied with 
$
  (E, \delta) 
  \curvearrowleft 
  ( \Reals^{ \fd }, \mathcal{B}( \Reals^{ \fd } ) ) 
$, 
$ 
  ( \Omega, \F ) 
  \curvearrowleft 
  ( \Omega, \mathbb{F}_{ n + 1 } ) 
$,
$
  X \curvearrowleft 
  ( 
    \Reals^{ \fd } \times \Omega \ni ( \theta, \omega)
    \mapsto ( \nabla_{ \theta } \fL^n_r )( \theta, \omega ) 
    \in \Reals^{ \fd } 
  ) 
$
in the notation of \cite[Lemma 2.4]{MR4293960})
demonstrates that
for all $ n \in \enne_0 $, $ r \in [ 1 , \infty ) $ 
we have that
$ 
  \Reals^{ \fd } \times \Omega 
  \ni ( \theta, \omega) \mapsto
  ( \nabla_\theta \fL^n_r )( \theta, \omega ) \in \Reals^{ \fd } 
$
is 
$ ( \mathcal{B}( \Reals^{ \fd } ) \otimes \mathbb{F}_{ n + 1 } ) 
$/$ \mathcal{B}( \Reals^{ \fd } ) $-measurable.
\Cref{it:grad_L_conv_G} of \cref{prop:L_grad}
\hence proves that for all $ n \in \enne_0 $ it holds that
$ 
  \Reals^{ \fd } \times \Omega \ni ( \theta, \omega ) \mapsto
  \fG^n( \theta, \omega ) \in \Reals^{ \fd } 
$
is 
$ ( \mathcal{B}( \Reals^{ \fd } ) \otimes \mathbb{F}_{ n + 1 } ) 
$/$ \mathcal{B}( \Reals^{ \fd } ) $-measurable. 
This establishes \cref{it:G_meas}.
  
Next we prove \cref{it:Theta_meas} by induction on $ n \in \enne_0 $.
The assumption that $ \mathbb{F}_0 = \sigma( \Theta_0 ) $ 
implies that $ \Theta_0 $ is 
$ \mathbb{F}_0 $/$ \mathcal{B}( \Reals^{ \fd } ) $-measurable.
This establishes \cref{it:Theta_meas} in the base case $ n = 0 $.
For the induction step let $ n \in \enne_0 $ satisfy that
$ \Theta_n $ is $ \mathbb{F}_n $/$ \mathcal{B}( \Reals^{ \fd } ) $-measurable.
\Nobs that the fact that
$ 
  \mathbb{F}_n = 
  \sigma( 
    \Theta_0, 
    ( X^{ \mathfrak{n}, \mathfrak{m} } )_{ 
      \mathfrak{n}, \mathfrak{m} \in ( \enne\cap[0,n))\times \enne_0 
    } 
  ) 
$  
ensures that $ \mathbb{F}_n \subseteq \mathbb{F}_{n+1} $.
\Hence that
$ \Theta_n $ is $ \mathbb{F}_{n+1} $/$ \mathcal{B}( \Reals^{ \fd } ) $-measurable.
\Moreover the fact that 
$ \mathbb{F}_n \subseteq \mathbb{F}_{n+1} $
and \cref{it:G_meas} imply that
$ \fG_n $ is $ \mathbb{F}_{n+1} $/$ \mathcal{B}( \Reals^{ \fd } ) $-measurable.
Combining this, the fact that  
$ \Theta_n $ is $ \mathbb{F}_{n+1} $/$ \mathcal{B}( \Reals^{ \fd } ) $-measurable,
and the assumption that
$ \Theta_{n+1} = \Theta_n - \gamma_n \fG^n( \Theta_n ) $
proves that $ \Theta_{n+1} $
is $ \mathbb{F}_{n+1} $/$ \mathcal{B}( \Reals^{ \fd } ) $-measurable.
Induction thus establishes \cref{it:Theta_meas}.
  
\Moreover 
the assumption that $ X^{n,m} $, $ n,m \in \enne_0 $,
are independent,
the assumption that 
$ \Theta_0 $ and $ ( X^{n,m} )_{ (n,m) \in ( \enne_0)^2 } $
are independent,  
and the fact that
$ 
  \mathbb{F}_n = 
  \sigma( 
    \Theta_0, 
    ( X^{ \mathfrak{n}, \mathfrak{m} } )_{ 
      \mathfrak{n}, \mathfrak{m} \in ( \enne \cap [0,n) ) \times \enne_0 
    } 
  ) 
$  
establish \cref{it:ind}.
The proof of \cref{lem:G_meas} is thus complete.
\end{proof}

\begin{cor}
\label{cor:L_exp}
Assume \cref{SGD_setting}.
Then it holds for all $ n \in \enne_0 $ that
\begin{equation}
\label{L_exp}
  \E[ \fL^n_{ \infty }( \Theta_n ) ] = \E[ \mL( \Theta_n ) ] .
\end{equation}
\end{cor}
\begin{proof}[Proof of \cref{cor:L_exp}]
  Throughout this proof let $ \mathbb{F}_n\subseteq\F$, $ n \in \enne_0 $,
  satisfy for all $ n \in \enne_0 $ that 
  $ \mathbb{F}_0=\sigma( \Theta_0) $ and
  $ \mathbb{F}_n=\sigma( \Theta_0,(X^{ \mathfrak{n}, \mathfrak{m} } )_{\mathfrak{n}, \mathfrak{m} \in ( \enne\cap[0,n))\times \enne_0 } ) $
  and let $ \boldsymbol{L}^n\colon([a,b]^{ \ell_0 } )^{M_n} \times \Reals^{ \fd }\to\Reals $, 
  $ n \in \enne_0 $,
  satisfy for all $ n \in \enne_0 $, $ \theta \in \Reals^{ \fd } $, $x_1, \ldots, x_{M_n} \in [a,b]^{ \ell_0 } $
   that
  \begin{equation}
  \label{bold_L}
    \boldsymbol{L}^n(x_{1}, \ldots, x_{M_n}, \theta)
    =\frac{1}{M_n} \sum_{m= 1 }^{M_n} \norm{\mN^{L, \theta }_{ \infty }(x_m)-\xi}^2.
  \end{equation}  
  \Nobs that \cref{bold_L}
  and the assumption that
  for all 
  $ n \in \enne_0 $, 
  $ \theta \in \Reals^{ \fd } $, $ \omega\in \Omega$ it holds that
$ \fL^n_{ \infty }( \theta, \omega)=\frac{1}{M_n} \sum_{m= 1 }^{M_n} \norm{\mN^{L, \theta }_{ \infty }(X^{n,m}( \omega))-\xi}^2$
  imply that for all 
  $ n \in \enne_0 $, $ \theta \in \Reals^{ \fd } $, $ \omega\in \Omega$
  we have that
  \begin{equation}
  \label{L=bold_L1}
    \fL^n_{ \infty }( \theta, \omega)
    =\boldsymbol{L}^n(X^{n,1}( \omega), \ldots,X^{n,M_n}( \omega), \theta).
  \end{equation}
  Combining this with \cref{prop:L_exp}
  demonstrates that for all $ n \in \enne_0 $, $ \theta \in \Reals^{ \fd } $
  it holds that
  \begin{equation}
  \label{exp_bold_L1}
    \E[\boldsymbol{L}^n(X^{n,1}, \ldots,X^{n,M_n}, \theta)]
    =\E[\fL^n_{ \infty }( \theta)]
    =\mL( \theta).
  \end{equation}  
  Moreover, \nobs that \cref{L=bold_L1} assures  
  that for all 
  $ n \in \enne_0 $ it holds that
  \begin{equation}
  \label{L=bold_L2}
   \fL^n_{ \infty }( \Theta_n )
   =\boldsymbol{L}^n(X^{n,1}, \ldots,X^{n,M_n}, \Theta_n ).
  \end{equation}
  Combining this, \cref{exp_bold_L1}, 
  \cref{lem:G_meas},
  and, e.g.,
  \cite[Lemma 2.8]{JentzenKuckuckNeufeldVonWurstemberger2021}
  (applied with $ ( \Omega, \F, \P)\curvearrowleft( \Omega, \F, \P) $,
  $ \mG\curvearrowleft\mathbb{F}_n$, 
  $ ( \mathbb{X}, \mathcal{X} )\curvearrowleft(([a,b]^{ \ell_0 } )^{M_n},
  \mathcal{B}(([a,b]^{ \ell_0 } )^{M_n} )) $,
  $ ( \mathbb{Y}, \mathcal{Y} )\curvearrowleft( \Reals^{ \fd }, \mathcal{B}( \Reals^{ \fd })) $,
  $X\curvearrowleft( \Omega\ni\omega\mapsto
  (X^{n,1}( \omega), \ldots,$ $ X^{n,M_n}( \omega)) \in ([a,b]^{ \ell_0 } )^{M_n} ) $,
  $Y\curvearrowleft( \Omega\ni\omega\mapsto\Theta_n( \omega) \in \Reals^{ \fd }) $
  in the notation of
  \cite[Lemma 2.8]{JentzenKuckuckNeufeldVonWurstemberger2021})
  proves that for all $ n \in \enne_0 $
  it holds that
  \begin{equation}
   \E[\boldsymbol{L}^n(X^{n,1}, \ldots,X^{n,M_n}, \Theta_n )]
    =\E[\mL( \Theta_n )].
  \end{equation}
  This and 
  \cref{L=bold_L1}
  establish \cref{L_exp}.
  The proof of \cref{cor:L_exp}
  is thus complete.
\end{proof}

\subsection{Upper estimates for the norm of the generalized gradients 
of the empirical risk function}
\label{ssec:SGD_properties2}

\begin{lemma}
\label{lem:G_upper}
Assume \cref{SGD_setting} and let
$ n \in \enne_0 $, $ \theta \in \Reals^{ \fd } $, $ \omega\in \Omega$.
Then 
\begin{equation}
\label{G_upper}
  \norm{\fG^n( \theta, \omega)}^2
  \leq 4 L \mathbf{a}^2\br[\big]{ \textstyle{\prod}_{p={0} }^{L}( \ell_p+1)}
    ( \norm{\theta }^2 + 1 )^{(L-1)} \fL^n_{ \infty }( \theta, \omega).
\end{equation}
\end{lemma}
\begin{proof}[Proof of \cref{lem:G_upper}]
\Nobs that \cref{it:G_upper_estimate} of \cref{prop:G_upper_estimate} 
(applied with
$ 
  \mu \curvearrowleft 
  ( 
    \mathcal{B}( [a,b]^{ \ell_0 } ) \ni A \mapsto 
    \frac{ 1 }{ M_n } 
    \sum_{ m = 1 }^{ M_n } \mathbbm{1}_A( X^{ n, m }( \omega ) ) \in [0,1] 
  ) 
$,
$
  f \curvearrowleft 
  (
    [a,b]^{ \ell_0 } \ni x \mapsto \xi \in \Reals^{ \ell_L} 
  ) 
$ 
in the notation of \cref{prop:G_upper_estimate})
establishes \cref{G_upper}.
The proof of \cref{lem:G_upper} is thus complete.
\end{proof}

\begin{lemma}
\label{lem:G_cpt_bd}
Assume \cref{SGD_setting} and let $ K \subseteq \Reals^{ \fd } $
be compact. Then
\begin{equation}
\label{G_cpt_bd}
  \sup\nolimits_{ n \in \enne_0 } 
  \sup\nolimits_{ \theta \in K } 
  \sup\nolimits_{ \omega \in \Omega }
  \norm{
    \fG^n( \theta, \omega ) 
  }
  < \infty .
\end{equation}
\end{lemma}
\begin{proof}[Proof of \cref{lem:G_cpt_bd}]
\Nobs that \cref{lem:loc_lip}
(applied with
$ 
  \mu \curvearrowleft 
  ( 
    \mathcal{B}( [a,b]^{ \ell_0 } ) \ni A \mapsto 
    \frac{ 1 }{ M_n } \sum_{ m = 1 }^{ M_n } \mathbbm{1}_A( X^{ n, m }( \omega ) 
$
$
    \in [0,1] 
  ) 
$,
$
  f \curvearrowleft 
  ( [a,b]^{ \ell_0 } \ni x \mapsto \xi \in \Reals^{ \ell_L } ) 
$
in the notation of \cref{lem:loc_lip}) 
ensures that there exists
$ \fC \in \Reals $ 
which satisfies for all $ \theta \in K $ that
$ 
  \sup_{ x \in [a,b]^{ \ell_0 } } 
  \norm{ \mN^{ L, \theta }_{ \infty }( x ) }^2
  \leq \fC 
$.
The fact that for all $ n, m \in \enne_0 $, $ \omega \in \Omega $
it holds that $ X^{n,m}( \omega ) \in [a,b]^{ \ell_0 } $
\hence demonstrates that 
for all $ n \in \enne_0 $, $ \theta \in K $, $ \omega \in \Omega $
we have that
\begin{equation}
\textstyle
  \fL^n_{ \infty }( \theta, \omega )
  = 
  \frac{ 1 }{ M_n } 
  \sum_{ m = 1 }^{ M_n } 
  \norm{ 
    \mN^{ L, \theta }_{ \infty }( X^{ n, m }( \omega ) ) - \xi 
  }^2
  \leq
  \frac{ 2 }{ M_n } 
  \sum_{ m = 1 }^{ M_n }
  \bigl[
    \norm{ \mN^{ L, \theta }_{ \infty }( X^{ n, m }( \omega ) ) }^2 
    + 
    \norm{ \xi }^2 
  \bigr]
  \leq 2 \fC + 2 \norm{ \xi }^2 
  .    
\end{equation}
Combining this with \cref{lem:G_upper} establishes \cref{G_cpt_bd}.
The proof of \cref{lem:G_cpt_bd} is thus complete.
\end{proof}

\subsection{Lyapunov type estimates for the dynamics of SGD processes}
\label{ssec:SGD_lyapunov}

\begin{lemma}
\label{lem:V_diff_sgd}
Assume \cref{SGD_setting} and let
$ n \in \enne_0 $, $ \theta \in \Reals^{ \fd } $, $ \omega \in \Omega$. Then
\begin{equation}
\begin{split}
\label{V_diff_sgd}
&
  V( \theta - \gamma_n \fG^n( \theta, \omega ) ) - V( \theta ) 
\\
&
\textstyle
  =
  ( \gamma_n )^2 
  \norm{ \fG^n( \theta, \omega ) }^2
  +
  ( \gamma_n )^2
  \biggl[
    \sum\limits_{ k = 1 }^L 
    \sum\limits_{ i = 1 }^{ \ell_k }
    ( k - 1 )
    \abs{
      \fG^n_{ 
        \ell_k \ell_{ k - 1 } + i 
        + \sum_{ h = 1 }^{ k - 1 } \ell_h ( \ell_{ h - 1 } + 1 ) 
      }( \theta, \omega ) 
    }^2
  \biggr]
    -4\gamma_n L\fL^n_{ \infty }( \theta, \omega)
\\
&
  \leq 
  ( \gamma_n )^2 L \norm{ \fG^n( \theta, \omega ) }^2
  - 4 \gamma_n L \fL^n_{ \infty }( \theta, \omega ) .
\end{split}
\end{equation}
\end{lemma}
\begin{proof}[Proof of \cref{lem:V_diff_sgd}]
\Nobs[Observe] that \cref{cor:liap3} 
(applied with 
$ 
  \mu \curvearrowleft 
  ( 
    \mathcal{B}( [a,b]^{ \ell_0 } ) \ni A \mapsto
    \frac{ 1 }{ M_n } 
    \sum_{ m = 1 }^{ M_n } 
    \allowbreak
    \mathbbm{1}_A( X^{ n, m }( \omega ) )
    \in [0,1]
  ) 
$,
$ 
  f(0) \curvearrowleft \xi 
$, 
$ \gamma \curvearrowleft \gamma_n $ 
in the notation of \cref{cor:liap3})
establishes \cref{V_diff_sgd}.
The proof of \cref{lem:V_diff_sgd} is thus complete.
\end{proof}

\subsection{Upper estimates for compositions of Lyapunov functions and SGD processes}
\label{ssec:SGD_upper}

\begin{lemma}
\label{lem:Theta_diff_sgd}
Assume \cref{SGD_setting}. 
Then it holds for all $ n \in \enne_0 $ that
\begin{equation}
\label{Theta_diff_sgd}
\begin{split}
  V( \Theta_{ n + 1 } ) - V( \Theta_n )
  \leq 
  4 L 
  \bigl(
    ( \gamma_n )^2 L \mathbf{a}^2
    \bigl[ 
      \textstyle{ \prod }_{ p = 0 }^L 
      ( \ell_p + 1 ) 
    \bigr]
    ( 2 V( \Theta_n ) + 4 L^2 \norm{ \xi }^2 + 1 )^{ (L - 1) } - \gamma_n 
  \bigr)
  \fL^n_{ \infty }( \Theta_n )
  .
\end{split}
\end{equation}
\end{lemma}
\begin{proof}[Proof of \cref{lem:Theta_diff_sgd}]
\Nobs that \cref{lem:G_upper} and \cref{it:ineq} of 
\cref{prop:liap1}
(applied with 
$ 
  \mu \curvearrowleft 
  ( 
    \mathcal{B}( [a,b]^{ \ell_0 } ) 
    \allowbreak
    \ni A \mapsto
    \frac{ 1 }{ M_n } \sum_{ m = 1 }^{ M_n } 
    \mathbbm{1}_A( X^{ n, m }( \omega ) ) \in [0,1]
  ) 
$,
$
  f \curvearrowleft 
  ( [a,b]^{ \ell_0 } \ni x \mapsto \xi \in \Reals^{ \ell_L } ) 
$ 
in the notation of \cref{prop:liap1}) 
demonstrate that
for all $ n \in \enne_0 $ it holds that
\begin{equation}
\begin{split}
&
  \norm{ \fG^n( \Theta_n ) }^2
  \leq 
  4 L \mathbf{a}^2 
  \bigl[
    \textstyle{ \prod }_{ p = 0 }^L 
    ( \ell_p + 1 ) 
  \bigr]
  ( \norm{ \Theta_n }^2 + 1 )^{ (L - 1) } 
  \fL^n_{ \infty }( \Theta_n )
\\
&
  \leq
  4 L \mathbf{a}^2 
  \bigl[ 
    \textstyle{ \prod }_{ p = 0 }^L 
    ( \ell_p + 1 ) 
  \bigr]
  ( 2 V( \Theta_n ) + 4 L^2 \norm{ \xi }^2 + 1 )^{ (L - 1) } 
  \fL^n_{ \infty }( \Theta_n )
  .
\end{split}  
\end{equation}
\cref{lem:V_diff_sgd} \hence implies that
\begin{equation}
\begin{split}
&
  V( \Theta_{ n + 1 } ) - V( \Theta_n )
  \leq 
  ( \gamma_n )^2 L \norm{ \fG^n( \Theta_n ) }^2
  - 4 \gamma_n L \fL^n_{ \infty }( \Theta_n )
\\
&
  \leq
  4 ( \gamma_n )^2 L^2 \mathbf{a}^2
  \br[\big]{ 
    \textstyle{ \prod }_{ p = 0 }^L 
    ( \ell_p + 1 ) 
  }
  ( 2 V( \Theta_n ) + 4 L^2 \norm{ \xi }^2 + 1 )^{ (L - 1) }
  \fL^n_{ \infty }( \Theta_n )
  - 4 \gamma_n L \fL^n_{ \infty }( \Theta_n )
\\
&
  = 
  4 L 
  \bigl( 
    ( \gamma_n )^2 L \mathbf{a}^2
    \bigl[
      \textstyle{ \prod }_{ p = 0 }^L 
      ( \ell_p + 1 ) 
    \bigr]
    ( 2 V( \Theta_n ) + 4 L^2 \norm{ \xi }^2 + 1 )^{ (L - 1) } - \gamma_n 
  \bigr)
  \fL^n_{ \infty }( \Theta_n )
  .
\end{split}
  \end{equation}
  The proof of \cref{lem:Theta_diff_sgd}
  is thus complete.
\end{proof}

\begin{cor}
\label{cor:Theta_diff_sgd2}
Assume \cref{SGD_setting}
and assume
\begin{equation}
\label{gamma_ass3}
  \P\bigl( 
    \sup\nolimits_{ n \in \enne_0 } 
    \gamma_n 
    \leq 
    \br[\big]{ 
      L \mathbf{a}^2
      \br[\big]{ 
        \textstyle{ \prod }_{ p = 0 }^L 
        ( \ell_p + 1 ) 
      }
      ( 2 V( \Theta_0 ) + 4 L^2 \norm{ \xi }^2 + 1 )^{ (L - 1) } 
    }^{ - 1 } 
  \bigr)
  = 1 .
\end{equation}
Then it holds for all $ n \in \enne_0 $ that
\begin{equation}
\label{Theta_diff_sgd2}
\begin{multlined}
  \P\biggl(
    V( \Theta_{ n + 1 } ) - V( \Theta_n )
    \leq - 4 \gamma_n L 
    \Bigl[ 
      1 - [\sup\nolimits_{m\in \enne_0 } \gamma_m]
\\
      \cdot L \mathbf{a}^2 
      \bigl[ 
        \textstyle{ \prod }_{ p = 0 }^L 
        ( \ell_p + 1 ) 
      \bigr]
      ( 2 V( \Theta_0 ) + 4 L^2 \norm{ \xi }^2 + 1 )^{ (L - 1) }  
    \Bigr]
    \fL^n_{ \infty }( \Theta_n )
    \leq 0
  \biggr)
  = 1 .
\end{multlined}
\end{equation}
\end{cor}
\begin{proof}[Proof of \cref{cor:Theta_diff_sgd2}]
  Throughout this proof let $ \mathfrak{g} \in \Reals $
  satisfy that $ \mathfrak{g} =\sup_{ n \in \enne_0 } \gamma_n$.
  We prove \cref{Theta_diff_sgd2} by induction on $ n \in \enne_0 $.
  \Nobs that 
  \cref{lem:Theta_diff_sgd}
  and the assumption that $ \gamma_0\leq \mathfrak{g} $
  ensure that it holds $ \P$-a.s.\ that
  \begin{equation} 
  \begin{split}
    &V( \Theta_1)-V( \Theta_0)\\
    &\leq 
    4L\br*{ -\gamma_0+\gamma_0^2 L \mathbf{a}^2\br[\big]{ \textstyle{\prod}_{p={0} }^{L}( \ell_p+1)}
    (2V( \Theta_0)+4L^2\norm{\xi}^2 + 1 )^{ (L-1)} }
    \fL^0_{ \infty }( \Theta_0)\\
    &\leq 
    4L \pr*{ -\gamma_0+\gamma_0\mathfrak{g} L \mathbf{a}^2\br[\big]{ \textstyle{\prod}_{p={0} }^{L}( \ell_p+1)}
    (2V( \Theta_0)+4L^2\norm{\xi}^2 + 1 )^{ (L-1)} }
    \fL^0_{ \infty }( \Theta_0)\\
    &= 
    -4\gamma_0 L \pr*{ 1-\mathfrak{g} L \mathbf{a}^2\br[\big]{ \textstyle{\prod}_{p={0} }^{L}( \ell_p+1)}
    (2V( \Theta_0)+4L^2\norm{\xi}^2 + 1 )^{ (L-1)}  }
    \fL^0_{ \infty }( \Theta_0).
  \end{split}
  \end{equation}
  This and \cref{gamma_ass3} establish \cref{Theta_diff_sgd2} in the
  base case $ n =0$.
  For the induction step let $ n \in \enne $ satisfy that
  for all $ m \in \{0, 1, \dots,n-1\} $ it holds $ \P$-a.s.\ that
  \begin{equation}
  \begin{split} \label{V_diff6}
    &V( \Theta_{m+1} )-V( \Theta_m)\\
    &\leq -4\gamma_m L
    \pr*{ 1-\mathfrak{g} L \mathbf{a}^2\br[\big]{ \textstyle{\prod}_{p={0} }^{L}( \ell_p+1)}
    (2V ( \Theta_0)+4L^2\norm{\xi}^2 + 1 )^{ (L-1)} }
    \fL^m_{ \infty }( \Theta_m)
    \leq 0.
  \end{split}
  \end{equation}
  \Nobs that \cref{V_diff6} implies that it holds $ \P$-a.s.\ that
  $V( \Theta_n )\leq V( \Theta_{n-1} )
  \leq\ldots \leq V( \Theta_0) $.
  Combining \cref{lem:Theta_diff_sgd} 
  with \cref{gamma_ass3} and the 
  assumption that $ \gamma_n\leq \mathfrak{g} $ 
  \hence 
  shows that it holds $\P$-a.s.\ that
  \begin{equation}
  \begin{split}
    &V( \Theta_{n+1} )-V( \Theta_n )\\
    &\leq 
    4L \pr*{ -\gamma_n+\gamma_n^2 L \mathbf{a}^2\br[\big]{ \textstyle{\prod}_{p={0} }^{L}( \ell_p+1)}
    (2V( \Theta_0)+4L^2\norm{\xi}^2 + 1 )^{ (L-1)} }
    \fL^n_{ \infty }( \Theta_n )\\
    &\leq 
    4 L \pr*{ -\gamma_n+\gamma_n\mathfrak{g} L \mathbf{a}^2\br[\big]{ \textstyle{\prod}_{p={0} }^{L}( \ell_p+1)}
    (2V( \Theta_0)+4L^2\norm{\xi}^2 + 1 )^{ (L-1)} }
    \fL^n_{ \infty }( \Theta_n )\\
    &= 
    -4\gamma_n L \pr*{ 1-\mathfrak{g} L \mathbf{a}^2\br[\big]{ \textstyle{\prod}_{p={0} }^{L}( \ell_p+1)}
    (2V( \Theta_0)+4L^2\norm{\xi}^2 + 1 )^{ (L-1)} }
    \fL^n_{ \infty }( \Theta_n )
    \leq 0.
  \end{split}
  \end{equation}
  Induction thus establishes \cref{Theta_diff_sgd2}.
  The proof of \cref{cor:Theta_diff_sgd2}
  is thus complete.
\end{proof}

\subsection{Convergence analysis for SGD processes}
\label{ssec:SGD_convergence}

\begin{theorem}
\label{thm:sgd_conv}
Assume \cref{SGD_setting},
assume $ \sum_{ n = 0 }^{ \infty } \gamma_n = \infty $, 
and let $ \delta \in (0,1) $ satisfy 
\begin{equation}
\label{ass_delta}
\textstyle 
  \inf_{ n \in \N }
  \P\bigl( 
      \gamma_n
      L \mathbf{a}^2
      [ 
        \textstyle{ \prod }_{ p = 0 }^L 
        ( \ell_p + 1 ) 
      ]
      ( 2 V( \Theta_0 ) + 4 L^2 \norm{ \xi }^2 + 1 )^{ (L - 1) }
    \leq \delta 
  \bigr) 
  = 1 .
\end{equation}
Then 
\begin{enumerate}[label=(\roman*)]
\item
\label{it:sgd_conv1}
there exists $ \fC \in \Reals $ such that
$ \P( \sup_{ n \in \enne_0 } \norm{ \Theta_n } \leq \fC ) = 1 $,
\item
\label{it:sgd_conv2}
it holds that
$ \P( \limsup_{ n \to \infty } \mL( \Theta_n ) = 0 ) = 1 $,
and
\item
\label{it:sgd_conv3}
it holds that
$ \limsup_{ n \to \infty } \E[ \mL( \Theta_n ) ] = 0 $.
\end{enumerate}
\end{theorem}
\begin{proof}[Proof of \cref{thm:sgd_conv}]
Throughout this proof let $ \mathfrak{g} \in [0, \infty]$
satisfy 
$ \mathfrak{g} =\sup_{ n \in \enne_0 } \gamma_n $.
\Nobs that 
\cref{ass_delta}
and 
the assumption that $ \sum_{ n = 0 }^{ \infty } \gamma_n = \infty $
ensure that
$ \mathfrak{g} \in (0, \infty) $.
This and \cref{ass_delta} imply that
there exists $ \fC\in [1, \infty) $ which satisfies 
\begin{equation}
\label{eq:V_theta_0_leq_fC}
  \P( V( \Theta_0 ) \leq \fC ) = 1 .
\end{equation}
\Nobs that \cref{eq:V_theta_0_leq_fC} and \cref{cor:Theta_diff_sgd2} 
demonstrates that
\begin{equation}
\label{V_c_Theta_bd}
  \P\pr[\big]{ 
    \sup\nolimits_{ n \in \enne_0 }V( \Theta_n ) \leq \fC 
  } = 1 .
\end{equation}
\Cref{it:ineq} of \cref{prop:liap1}
(applied with 
$ 
  \mu \curvearrowleft ( \mathcal{B}( [a,b]^{ \ell_0 } ) \ni A \mapsto
  \frac{ 1 }{ M_n } \sum_{ m = 1 }^{ M_n } \mathbbm{1}_{ A }( X^{ n, m }( \omega ) ) \in [0,1] ) 
$,
$ 
  f \curvearrowleft ( [a,b]^{ \ell_0 } \ni x\mapsto \xi \in \Reals^{ \ell_L } ) 
$ 
in the notation of \cref{prop:liap1})
\hence shows that
\begin{equation}
\label{Theta_bd}
  \P\bigl( 
    \sup\nolimits_{ n \in \enne_0 } 
    \norm{ \Theta_n } \leq \fC 
  \bigr) = 1 .
\end{equation}
This establishes \cref{it:sgd_conv1}.
\Nobs that \cref{cor:Theta_diff_sgd2} and \cref{ass_delta} 
ensure that for all $ n \in \enne_0 $
it holds $ \P$-a.s.\ that
\begin{equation}
\begin{split}
&
  - ( V( \Theta_n ) - V( \Theta_{ n + 1 } ) )
\\
&
  \leq - 4 \gamma_n L 
  \bigl(
    1 - \mathfrak{g} L \mathbf{a}^2 
    \br[\big]{ 
      \textstyle{ \prod }_{ p = 0 }^L 
      ( \ell_p + 1 ) 
    }
    ( 2 V( \Theta_0 ) + 4 L^2 \norm{ \xi }^2 + 1 )^{ (L - 1) } 
  \bigr)    
  \fL^n_{ \infty }( \Theta_n )
  .
\end{split} 
\end{equation}  
This and \cref{ass_delta} prove  
 that for all $ n \in \enne_0 $
it holds $ \P$-a.s.\ that
\begin{equation}
\begin{split} 
\label{gamma_n_L_n}
  \gamma_n \fL^n_{ \infty }( \Theta_n )
&
  \leq 
  \frac{ 
    V( \Theta_n ) - V( \Theta_{ n + 1 } ) 
  }{
    4 L 
    [ 
      1 - \mathfrak{g} L \mathbf{a}^2
      \br{ 
        \textstyle{ \prod }_{ p = 0 }^L 
        ( \ell_p + 1 )
      }
      ( 2 V( \Theta_0 ) + 4 L^2 \norm{ \xi }^2 + 1 )^{ (L - 1) } 
    ] 
  }
  \leq 
  \frac{ 
    V( \Theta_n ) - V( \Theta_{ n + 1 } ) 
  }{ 4 L ( 1 - \delta ) } .
\end{split}
\end{equation}
\Moreover \cref{it:ineq} of \cref{prop:liap1}
(applied with 
$ 
  \mu \curvearrowleft 
  ( \mathcal{B}( [a,b]^{ \ell_0 } ) \ni A \mapsto
  \frac{ 1 }{ M_n } \sum_{ m = 1 }^{ M_n } 
  \allowbreak 
  \mathbbm{1}_A( X^{ n, m }( \omega ) ) \in [0,1] ) 
$,
$
  f \curvearrowleft ( [a,b]^{ \ell_0 } \ni x \mapsto \xi \in \Reals^{ \ell_L } ) 
$
in the notation of \cref{prop:liap1})
ensures that
for all $ n \in \enne_0 $
it holds $ \P $-a.s.\ that
$
  V( \Theta_n ) \geq 
  \frac{ 1 }{ 2 } \norm{ \Theta_n }^2
  - 2 L^2 \norm{ f(0) }^2
  \geq - 2 L^2 \norm{ f(0) }^2
$.
Combining this with \cref{V_c_Theta_bd,gamma_n_L_n} 
implies that for all $ N \in \enne_0 $ 
it holds $ \P $-a.s.\ that
\begin{equation}
\label{sum_gamma_L}
\begin{split}
&
  \sum_{ n = 0 }^{ N - 1 } 
  \gamma_n 
  \fL^n_{ \infty }( \Theta_n )
  \leq 
  \frac{ 
    \sum_{ n = 0 }^{ N - 1 } 
    ( V( \Theta_n ) - V( \Theta_{ n + 1 } ) ) 
  }{ 
    4 L ( 1 - \delta ) 
  }
  = 
  \frac{
    V( \Theta_0 ) - V( \Theta_N ) 
  }{
    4 L ( 1 - \delta ) 
  } 
\\
&
  \leq 
  \frac{ 
    V( \Theta_0 )
    + 2 L^2 \norm{ f(0) }^2 
  }{ 4 L ( 1 - \delta ) }
  \leq
  \frac{
    \fC + 2 L^2 \norm{ f(0) }^2 
  }{
    4 L ( 1 - \delta ) 
  }
  < \infty .
\end{split}
\end{equation}
\Hence that
\begin{equation}
\label{sum_gamma_exp_L}
\textstyle 
  \sum_{ n = 0 }^{ \infty }
  \gamma_n 
  \E[ \fL^n_{ \infty }( \Theta_n ) ]
  =
  \lim_{ N \to \infty } 
  \bigl[ 
    \sum_{ n = 0 }^{ N - 1 } 
    \gamma_n \fL^n_{ \infty }( \Theta_n )
  \bigr]
  \leq
  \frac{ \fC + 2 L^2 \norm{ f(0) }^2 }{ 4 L ( 1 - \delta ) }
  < \infty .
\end{equation}
\Moreover \cref{cor:L_exp} proves that for all $ n \in \enne_0 $
it holds that
$
  \E[ \fL^n_{ \infty }( \Theta_n ) ]
  = \E[ \mL( \Theta_n ) ] 
$.
This and \cref{sum_gamma_exp_L} show that
$
  \sum_{ n = 0 }^{ \infty } 
  \gamma_n
  \E[ \mL( \Theta_n ) ]
  < \infty 
$.
The monotone convergence theorem
and the fact that for all $ n \in \enne_0 $, 
$ \omega \in \Omega $
it holds that 
$ \mL( \Theta_n( \omega) ) \geq 0 $ 
\hence ensure that
\begin{equation}
\textstyle 
  \E\br*{ 
    \sum_{ n = 0 }^{ \infty } 
    \gamma_n \mL( \Theta_n ) 
  }
  =
  \sum_{ n = 0 }^{ \infty } 
  \E[ \gamma_n \mL( \Theta_n ) ]
  < \infty .
\end{equation}
\Hence that
\begin{equation}
\label{sum_gamma_L_bd}
\textstyle 
  \P\pr*{ 
    \sum_{ n = 0 }^{ \infty } \gamma_n \mL( \Theta_n ) < \infty 
  }
  = 1 .
\end{equation}
Next let $ A \subseteq \Omega $ satisfy 
\begin{equation}
\label{A_set}
  A = 
  \bigl\{
    \omega \in \Omega \colon 
    \bigl[ 
      \pr[\big]{ 
        \textstyle{ \sum }_{ n = 0 }^{ \infty } 
        \gamma_n \mL( \Theta_n( \omega ) ) < \infty 
      } 
      \wedge
      \pr*{ 
        \sup\nolimits_{ n \in \enne_0 } \norm{ \Theta_n( \omega) } \leq \fC 
      } 
    \bigr]
  \bigr\}
  .
\end{equation}
\Nobs that \cref{Theta_bd,sum_gamma_L_bd,A_set}
ensure that $ A \in \F $ and $ \P( A ) = 1 $.
In the following let $ \omega \in A $ be arbitrary.
\Nobs that \cref{A_set} ensures that  
\begin{equation}
\label{sup_Theta}
  \sup\nolimits_{ n \in \enne_0 } 
  \norm{ \Theta_n( \omega) } 
  \leq \fC 
  .    
\end{equation} 
\Moreover the assumption that
$ 
  \sum_{ n = 0 }^{ \infty } \gamma_n = \infty 
$ 
and
the fact that
$ 
  \sum_{ n = 0 }^{ \infty } 
  \gamma_n 
  \mL( \Theta_n( \omega ) ) < \infty 
$
show that
$ 
  \liminf_{ n \to \infty } \mL( \Theta_n( \omega ) ) = 0 
$.  
In the following we prove \cref{it:sgd_conv2} by contradiction.
For this assume that
\begin{equation}
  \limsup\nolimits_{ n \to \infty } 
  \mL( \Theta_n( \omega ) ) > 0 
  .
\end{equation}
This implies that there exists $ \varepsilon\in (0, \infty) $
which satisfies 
\begin{equation}
\label{epsilon}
  0 = 
  \liminf\nolimits_{ n \to \infty } 
  \mL( \Theta_n( \omega) )
  < \varepsilon
  < 2 \varepsilon
  < \limsup\nolimits_{ n \to \infty } \mL( \Theta_n( \omega ) ) 
  .
\end{equation}
\Nobs that \cref{epsilon} assures that
there exist 
$
  ( m_k, n_k ) \in \enne^2 
$, $ k \in \enne $,
which satisfy that for all $ k \in \enne $ it holds that
$ m_k < n_k < m_{ k + 1 } $, 
$ \mL( \Theta_{ m_k }( \omega ) ) > 2 \varepsilon 
$, 
and
$ 
  \mL( \Theta_{n_k}( \omega ) )
  < \varepsilon < 
  \min_{ j \in \enne \cap [m_k, n_k) } \mL( \Theta_j( \omega ) ) 
$.
Combining this with 
the fact that
$ 
  \sum_{ n = 0 }^{ \infty } 
  \gamma_n \mL( \Theta_n( \omega ) ) < \infty 
$
shows that
\begin{equation}
\label{sum_gamma_sgd}
  \sum_{ k = 1 }^{ \infty }
  \sum_{ j = m_k }^{ n_k - 1 } 
  \gamma_j
  \leq
  \frac{ 1 }{ \varepsilon } 
  \Biggl[
    \sum_{ k = 1 }^{ \infty } 
    \sum_{ j = m_k }^{ n_k - 1 }
    ( \gamma_j \mL( \Theta_j( \omega ) ) )
  \Biggr]
  \leq
  \frac{ 1 }{ \varepsilon } 
  \Bigg[
    \sum_{ k = 0 }^{ \infty }
    ( \gamma_k \mL( \Theta_k( \omega ) ) )
  \Bigg]
  < \infty .
\end{equation}
\Moreover \cref{sup_Theta} and \cref{lem:G_cpt_bd}
ensure that there exists $ \mathfrak{D} \in \Reals $
which satisfies for all $ n \in \enne_0 $ that
$ 
  \norm{
    \fG^n( \Theta_n( \omega ), \omega ) 
  }
  \leq \mathfrak{D} 
$.
This, \cref{sum_gamma_sgd}, and 
the fact that for all $ n \in \enne_0 $,
$ \omega \in \Omega $ it holds that 
$ 
  \Theta_{ n + 1 }( \omega ) - \Theta_n( \omega )
  = - \gamma_n \fG^n( \Theta_n( \omega ), \omega ) 
$
demonstrate that
\begin{equation}
\label{sum_theta}
\begin{split}
  \sum_{ k = 1 }^{ \infty }
  \norm{ 
    \Theta_{ n_k }( \omega ) - \Theta_{ m_k }( \omega ) 
  }
&
  \leq \sum_{ k = 1 }^{ \infty }
  \sum_{ j = m_k }^{ n_k - 1 } 
  \norm{ \Theta_{ j + 1 }( \omega ) - \Theta_j( \omega ) }
  =
  \sum_{ k = 1 }^{ \infty }
  \sum_{ j = m_k }^{ n_k - 1 } 
  \gamma_j
  \norm{ \fG^j( \Theta_j( \omega ), \omega ) } 
\\
&
  \leq
  \mathfrak{D} 
  \Biggl[
    \sum_{ k = 1 }^{ \infty }
    \sum_{ j = m_k }^{ n_k - 1 } 
    \gamma_j
  \Biggr]
  < \infty .
\end{split}
\end{equation}
\Moreover \cref{lem:loc_lip}
(applied with 
$ 
  \mu \curvearrowleft 
  ( \mathcal{B}( [a,b]^{ \ell_0 } ) \ni A \mapsto
  \frac{ 1 }{ M_n } 
  \sum_{ m = 1 }^{ M_n } 
  \mathbbm{1}_A( X^{ n, m }( \omega ) ) \in [0,1] ) 
$,
$
  f \curvearrowleft 
  ( [a,b]^{ \ell_0 } \ni x \mapsto \xi \in \Reals^{ \ell_L} ) 
$
in the notation of \cref{lem:loc_lip})
and \cref{sup_Theta}
prove that
there exists $ \mathscr{L} \in \Reals $ which satisfies
for all $ m, n \in \enne_0 $ that
$
  | \mL( \Theta_m( \omega ) ) - \mL( \Theta_n( \omega ) ) |
  \leq \mathscr{L} \norm{ \Theta_m( \omega ) - \Theta_n( \omega ) } 
$.
Combining this with \cref{sum_theta} proves that
\begin{equation}
  \limsup\nolimits_{ k \to \infty } 
  \abs{
    \mL( \Theta_{ n_k }( \omega ) ) - 
    \mL( \Theta_{ m_k }( \omega ) ) 
  }
  \leq 
  \limsup\nolimits_{ k \to \infty } 
  \bigl(
    \mathscr{L} 
    \norm{
      \Theta_{ n_k }( \omega ) - 
      \Theta_{ m_k }( \omega ) 
    }
  \bigr)
  = 0 .
\end{equation}
This and 
the fact that for all $ k \in \enne_0 $
it holds that 
$ 
  \mL( \Theta_{ n_k }( \omega ) ) 
  < \varepsilon < 2 \varepsilon
  \leq \mL( \Theta_{ m_k } ) 
$
demonstrate that
\begin{equation}
  0 < \varepsilon
  \leq
  \inf\nolimits_{ k \in \enne } 
  \abs{
    \mL( \Theta_{ n_k }( \omega ) ) - 
    \mL( \Theta_{ m_k }( \omega ) ) 
  }
  \leq
  \limsup\nolimits_{ k \to \infty } 
  \abs{
    \mL( \Theta_{ n_k }( \omega ) ) - 
    \mL( \Theta_{ m_k }( \omega ) ) 
  }
  = 0 .
\end{equation}
This contradiction proves that 
$ 
  \limsup_{ n \to \infty } 
  \mL( \Theta_n( \omega ) ) = 0
$.
Combining this with the fact that $ \P(A) = 1 $
establishes \cref{it:sgd_conv2}. 
\Nobs that \cref{it:sgd_conv1} and 
the fact that $ \mL $ is continuous 
show that there exists $ \fC \in \Reals $
which satisfies that
\begin{equation}
\label{eq:sup_probability_equal_to_1}
\textstyle 
  \P( \sup_{ n \in \enne_0 } \mL( \Theta_n ) \leq \fC ) = 1 
  .
\end{equation}
\Nobs that 
\cref{it:sgd_conv2}, 
\cref{eq:sup_probability_equal_to_1}, 
and the dominated convergence theorem
establish \cref{it:sgd_conv3}. 
The proof of \cref{thm:sgd_conv} is thus complete.  
\end{proof}

\begin{cor}
\label{cor:sgd_conv}
Assume \cref{SGD_setting},
assume $ \sum_{ n = 0 }^{ \infty } \gamma_n = \infty $,
and assume for all $ n \in \enne_0 $ that
\begin{equation}
\label{ass_Theta}
  \P\bigl( 
    ( 4 L \fd \max\{ \mathbf{a}, \norm{\xi} \} )^{ 2 L } 
    \gamma_n 
    \leq 
    ( \norm{ \Theta_0 } + 1 )^{ - 2 L } 
  \bigr) = 1 .
\end{equation}
Then
\begin{enumerate}[label=(\roman*)]
\item
\label{it:sgd_conv1_cor}
there exists $ \fC\in \Reals $ such that
$ 
  \P( \sup_{ n \in \enne_0 } \norm{ \Theta_n } \leq \fC ) = 1
$,
\item
\label{it:sgd_conv2_cor}
it holds that
$ 
  \P( \limsup_{ n \to \infty } \mL( \Theta_n ) = 0 ) = 1
$,
and
\item
\label{it:sgd_conv3_cor}
it holds that
$ \limsup_{ n \to \infty } \E[\mL( \Theta_n )]=0$.
\end{enumerate}
\end{cor}
\begin{proof}[Proof of \cref{cor:sgd_conv}]
Throughout this proof let 
$
  \mathbf{B} \in \Rr
$
satisfy
$
  \mathbf{B} = \max\{\mathbf{a}, \norm{\xi} \}
$. 
\Nobs that \cref{it:ineq} of \cref{prop:liap1}
(applied with 
$ 
  \mu \curvearrowleft 
  ( \mathcal{B}( [a,b]^{ \ell_0 } ) \ni A \mapsto
  \frac{ 1 }{ M_n } 
  \sum_{ m = 1 }^{ M_n } 
  \mathbbm{1}_A( X^{ n, m }( \omega ) ) \in [0,1] ) 
$,
$
  f \curvearrowleft 
  ( 
    [a,b]^{ \ell_0 } \ni x \mapsto \xi \in \Reals^{ \ell_L } 
  ) 
$
in the notation of \cref{prop:liap1}) 
and the fact that for all 
$ x, y \in \Reals $, $ M \in \enne $
it holds that
$
  ( x + y )^M \leq 
  ( 2^{ M + 1 } - 1 )( x^M + y^M ) 
$
ensure that it holds $ \P$-a.s.\ that 
\begin{equation}
\begin{split}
&
  ( 2 V( \Theta_0 ) + 4 L^2 \norm{ \xi }^2 + 1 )^{ (L - 1) }
  \leq 
  \bigl( 
    2 
    ( 
      2 L \norm{ \Theta_0 }^2
      + L \norm{ \xi }^2 
    )
    + 4 L^2 \norm{ \xi }^2 + 1 
  \bigr)^{ (L-1) } 
\\
&
  \leq 
  ( 
    4 L \norm{ \Theta_0 }^2 + 6 L^2 \norm{ \xi }^2 + 1 
  )^{ (L - 1) }
  \leq
  ( 2^L - 1 ) 
  \bigl(
    ( 4 L \norm{ \Theta_0 }^2 )^{ (L - 1) } 
    + 
    ( 6 L^2 \norm{ \xi }^2 + 1 )^{ (L - 1) } 
  \bigr) .
\end{split}  
\end{equation}
\Hence that it holds $ \P $-a.s.\ that
  \begin{equation} 
  \begin{split}
    &(2V( \Theta_0)+4 L^2 \norm{\xi}^2 + 1 )^{ (L-1)}
    \leq (2^{L}-1)\Big[
    (4 L )^{ (L-1)} \norm{\Theta_0 }^{2(L-1)} + (7 \mathbf{B}^2 L ^2 )^{(L-1)} \Big]\\
    &\leq (2^{L}-1) (7 L^2 )^{ (L-1)} \mathbf{B}^{2(L-1)}(
    \norm{\Theta_0 }^{2(L-1)} + 1)
    \leq 14^{L} L^{2 ( L - 1 ) } \mathbf{B}^{2(L-1)}(
    \norm{\Theta_0 } + 1)^{2L}.
  \end{split}
  \end{equation}
This and \cref{ass_Theta} show 
that for all $ n \in \enne_0 $
it holds $ \P$-a.s.\ that
  \begin{equation}
  \begin{split}
    &\gamma_n
     L \mathbf{a}^2\br[\big]{ \textstyle{\prod}_{p={0} }^{L}( \ell_p+1)}
    (2V( \Theta_0)+4L^2\norm{\xi}^2 + 1 )^{ (L-1)} \\
    &\leq
    \gamma_n L^{2 L - 1 }  \mathbf{B}^{2 L } 14^L
    \br[\big]{ \textstyle{\prod}_{p={0} }^{L}( \ell_p+1)} ( \norm{\Theta_0 } + 1) ^{ 2 L }
    \le
    \gamma_n 14^L( L  \mathbf{B} )^{2 L }
   \br[\big]{ \textstyle{\prod}_{p={0} }^{L}( \ell_p+1)} ( \norm{\Theta_0 } + 1) ^{ 2 L }\\
    &\leq
    \gamma_n 14^L ( L  \mathbf{B} )^{2 L }
    \mathfrak{d}^{L + 1 } ( \norm{\Theta_0 } + 1) ^{ 2 L }
    \leq
    \gamma_n 14^L ( L  \mathbf{B} \mathfrak{d} )^{2 L }
     ( \norm{\Theta_0 } + 1) ^{ 2 L }
    \leq 
    14^L 4^{ - 2 L }
    = 7^L 8^{ - L }
    < 1
    .
\end{split}
\end{equation}
Combining this with \cref{thm:sgd_conv} establishes
\cref{it:sgd_conv1_cor,it:sgd_conv2_cor,it:sgd_conv3_cor}.
The proof of \cref{cor:sgd_conv} is thus complete.
\end{proof}

\section{Numerical simulations} 
\label{sec:numerics}

In this section we supplement 
the theoretical findings of 
\cref{thm:main_thm}
and 
\cref{thm:sgd_conv}, 
respectively, 
by means of 
two numerical examples.

\subsection{Numerical simulation of an SGD processes for certain ANNs with one hidden layer}
\label{subsec:numerical_example_1}

\begin{figure}
\centering
\begin{adjustbox}{width=\textwidth}
\def\layersep{5cm}
\begin{tikzpicture}[shorten >=1pt, ->, draw=black!50, node distance=\layersep]
  \tikzstyle{neuron} = [circle, draw=black!80, minimum size=17pt, line width=0.5mm, inner sep=0pt]
  \tikzstyle{input neuron} = [neuron, fill=red!50]
  \tikzstyle{output neuron} = [neuron, fill=violet!50]
  \tikzstyle{hidden neuron} = [neuron, fill=blue!50]
  \tikzstyle{annot} = [text width=9em, text centered]
  \tikzstyle{annot2} = [text width=4em, text centered]

  \foreach \y in {1,...,1}
    \node[input neuron] (I-\y) at (0,0) {};
  \foreach \y in {1,...,1}    
    \node[annot,left of=I-\y, node distance=0.7cm, align=center] () {$ x $};
  \foreach \y in {1,...,7}
    \node[hidden neuron] (H-\y) at (\layersep, 4-\y) {};      
  \foreach \y in {1,...,1}
    \node[output neuron](O-\y) at (2*\layersep, 0) {};   
  \foreach \y in {1,...,1}
    \node[annot,right of=O-\y, node distance=1.1cm, align=center] () {$ \mN^{ 2, \theta }_{ \infty }(x) $};
            
  \foreach \source in {1,...,1}
    \foreach \dest in {1,...,7}
      \path[>=stealth] (I-\source) edge (H-\dest);     
  \foreach \source in {1,...,7}
    \foreach \dest in {1,...,1}
      \path[>=stealth] (H-\source) edge (O-\dest);     

  \node[annot,above of=H-1, node distance=1.0cm, align=center] (hl) {Hidden layer\\(2nd layer)};
  \node[annot,left of=hl, align=center] {Input layer\\ (1st layer)};
  \node[annot,right of=hl, align=center] {Output layer\\(3rd layer)};    
  \node[annot2,below of=H-7, node distance=1.0cm, align=center] (sl) {$ \ell_1 = 7 $};
  \node[annot2,right of=sl, align=center] (sl2) {$ \ell_2 = 1$};
  \node[annot2,left of=sl, align=center] {$ \ell_0 = 1 $};
\end{tikzpicture}
\end{adjustbox}
\caption{Graphical illustration of the ANN architecture used in \cref{subsec:numerical_example_1} 
($ \ell_0 = 1 $ neuron on the input layer, $ \ell_1 = 7 $ neurons on the hidden layer, and $ \ell_2 = 1 $ neuron on the output layer). 
In this situation we have for every ANN parameter vector $ \theta \in \Reals^{ \fd } = \Reals^{ 22 } $ 
that the realization function 
$ \Reals \ni x \mapsto \mN^{ 2, \theta }_{ \infty }(x) \in \Reals $ 
of the considered ANN maps the 1-dimensional input $ x \in [0,1] $ to the 1-dimensional 
output $ \mN^{ 2, \theta }_{ \infty }( x ) \in \Reals $.}
\label{figure_ANN_architecture_sim1}
\end{figure}

In this subsection we present numerical simulations for a certain SGD process 
in the training of some shallow ANNs with $ 1 $ neuron on the input layer, with $ 7 $ neurons on the hidden layer, 
and with $ 1 $ neuron on the output layer (see \cref{figure_ANN_architecture_sim1}, Table~\ref{table:ANN_num_sim1}, 
and Figure~\ref{fig:ANN_num_sim1}).
\begin{figure}
\begin{center}
\includegraphics[scale=0.75]{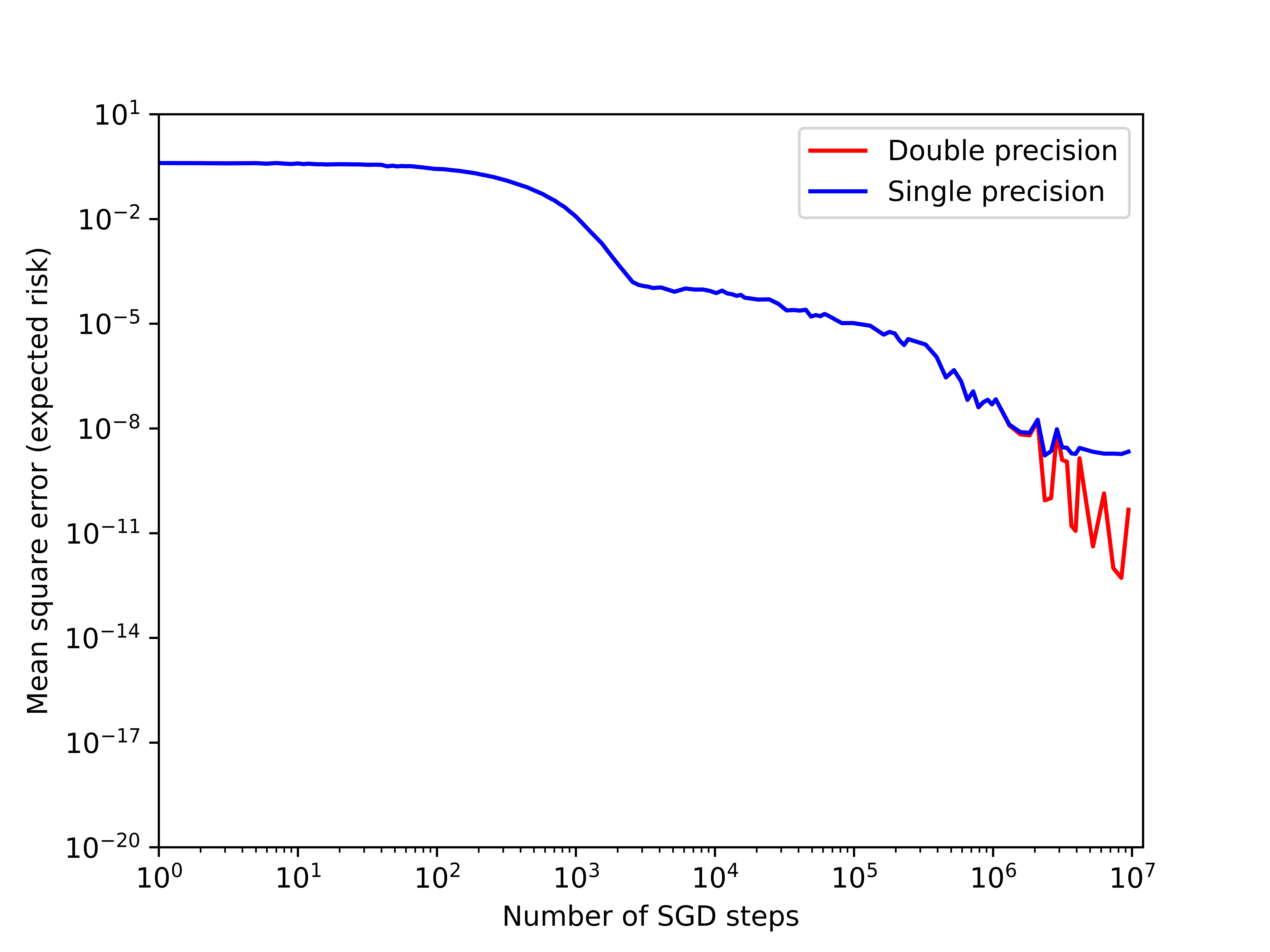}
\caption{Plot of the estimated mean square error (expected risk) against the number of SGD steps 
for \cref{subsec:numerical_example_1}}
\label{fig:ANN_num_sim1}
\end{center}
\end{figure}
More formally, assume \cref{SGD_setting}, 
let $ e_1, e_2, \dots, e_{ \fd } \in \Reals^{ \fd } $ 
satisfy 
$
  e_1 = ( 1, 0, \dots, 0 )
$, 
$
  e_2 = ( 0, 1, \dots, 0 )
$,
$ \dots $,
$
  e_{ \fd } = ( 0, \dots, 0, 1 )
$,
assume 
\begin{equation}
\label{eq:num_example_1_assumption1}
  L = 2 ,
  \qquad 
  \ell_0 = 1 ,
  \qquad 
  \ell_1 = 7 ,
  \qquad 
  \ell_2 = 1 ,
  \qquad 
  \xi = 1 ,
  \qquad 
  a = 0 ,
  \qquad 
  \text{and}
  \qquad 
  b = 1
  ,
\end{equation}
assume for all 
$ k \in \{ 1, 2, \dots, \fd \} $, $ x \in \Reals $ that 
\begin{equation}
\label{eq:num_example_1_assumption2}
\textstyle 
  \P(
    X^{ 0, 0 } < x
  )
  =
  \max\{ \min\{ x, 1 \}, 0 \}
\quad 
\text{and}
\quad 
  \P(
    \langle e_k, (\ell_1 )^{ 1 / 2 } \Theta_0 \rangle 
    < x
  )
  =
  \int_{ - \infty }^x
  ( 2 \pi )^{ - 1 / 2 }
  \exp\bigl( - \frac{ y^2 }{ 2 } \bigr)
  \, \d y
  ,
\end{equation}
and assume for all $ n \in \N_0 $ that $ M_n = 100 $ and $ \gamma_n = \frac{ 1 }{ 2000 } $. 
\Nobs that \cref{eq:num_example_1_assumption1} ensures that 
the number of ANN parameters $ \fd $ satisfies 
\begin{equation}
\textstyle 
  \fd = 
  \sum_{ k = 1 }^L
  \ell_k ( \ell_{ k - 1 } + 1 ) 
  =
  2 \ell_1 + \ell_1 + 1
  =
  3 \ell_1 + 1
  =
  22 
  .
\end{equation}
\Moreover \cref{eq:num_example_1_assumption1} assures that the considered 
ANNs consist of $ 14 $ weight parameters and $ 8 $ bias parameters. 
We also refer to \cref{figure_ANN_architecture_sim1} for a graphical 
illustration of the ANN architecture described in \cref{eq:num_example_1_assumption1}. 
\Moreover \cref{eq:num_example_1_assumption2} ensures that for all $ n, m \in \N_0 $ 
we have that 
$
  X^{ n, m }
$
is continuous uniformly distributed on $ (0,1) $. 
\Moreover \cref{eq:num_example_1_assumption2} assures that 
$
  ( \ell_1 )^{ 1 / 2 } \Theta_0
$
is a $ 22 $-dimensional standard normal random vector. 
In \cref{table:ANN_num_sim1} we approximately specify 
the mean square error (expected risk)
\begin{equation}
\label{eq:expected_risk_ex1}
  \E[ \mL( \Theta_n ) ] 
\end{equation}
against the number $ n $ of SGD steps 
for several $ n \in \{ 1, 2, \dots, 10^6 \} $ 
and in \cref{fig:ANN_num_sim1} we approximately plot the values of \cref{table:ANN_num_sim1}. 
In \cref{table:ANN_num_sim1} and \cref{fig:ANN_num_sim1} we approximate 
the expectations in \cref{eq:expected_risk_ex1} by means of Monte Carlo 
averages with $ 10^6 $ samples. 
The {\sc Python} source code used to create 
\cref{table:ANN_num_sim1} and \cref{fig:ANN_num_sim1} can be found in \cref{lst:sgd1}. 
Moreover, in \cref{lst:sgd1_short} we present a simplified variant of the 
{\sc Python} code in \cref{lst:sgd1}. 
\newpage
\begin{longtable}{|c|S[table-auto-round,
                 table-format=1.15, table-column-width=10em]|
                 S[table-auto-round,
                 table-format=1.15, table-column-width=10em]|}
\caption{Estimated mean square error (expected risk)\\ against the number of SGD steps for \cref{subsec:numerical_example_1}}
\label{table:ANN_num_sim1}\\
  \hline
  \multicolumn{1}{|c|}{Number of SGD steps}
   & \multicolumn{2}{|c|}{Estimated mean square error (expected risk)} \\
  \cline{2-3}
  \multicolumn{1}{|c|}{}
  & \multicolumn{1}{|c|}{Single precision}
  & \multicolumn{1}{|c|}{Double precision}\\
  \hline
\endhead
\hline
\endfoot
  \begin{filecontents*}{mean_square_error_example1.csv}
1, 0.4001478254795074462890625000000000000000000000000000000000000000, 0.4001477573344729421478405129164457321166992187500000000000000000
2, 0.3968733847141265869140625000000000000000000000000000000000000000, 0.3968733121084598858274716803862247616052627563476562500000000000
3, 0.3922770619392395019531250000000000000000000000000000000000000000, 0.3922769657247877628236665259464643895626068115234375000000000000
4, 0.3945565819740295410156250000000000000000000000000000000000000000, 0.3945565050694311293177918287256034091114997863769531250000000000
5, 0.3974660933017730712890625000000000000000000000000000000000000000, 0.3974659971168091332494043399492511525750160217285156250000000000
6, 0.3829766213893890380859375000000000000000000000000000000000000000, 0.3829765688772381881932460601092316210269927978515625000000000000
7, 0.4006451666355133056640625000000000000000000000000000000000000000, 0.4006451505429836634775142556463833898305892944335937500000000000
8, 0.3842973709106445312500000000000000000000000000000000000000000000, 0.3842973141095110678477908550121355801820755004882812500000000000
9, 0.3755445182323455810546875000000000000000000000000000000000000000, 0.3755444916253696807650896971608744934201240539550781250000000000
10, 0.3891813158988952636718750000000000000000000000000000000000000000, 0.3891813064716923009633831043174723163247108459472656250000000000
11, 0.3732098639011383056640625000000000000000000000000000000000000000, 0.3732098803307113410454576296615414321422576904296875000000000000
12, 0.3829722702503204345703125000000000000000000000000000000000000000, 0.3829723221316445047612830876460066065192222595214843750000000000
13, 0.3737317621707916259765625000000000000000000000000000000000000000, 0.3737317631638207848787658349465345963835716247558593750000000000
14, 0.3681764602661132812500000000000000000000000000000000000000000000, 0.3681765162760781695538980784476734697818756103515625000000000000
15, 0.3688359856605529785156250000000000000000000000000000000000000000, 0.3688360114724684724052394813043065369129180908203125000000000000
16, 0.3633897006511688232421875000000000000000000000000000000000000000, 0.3633897453309009595301404260681010782718658447265625000000000000
20, 0.3703134953975677490234375000000000000000000000000000000000000000, 0.3703135235001356329043176174309337511658668518066406250000000000
24, 0.3672173619270324707031250000000000000000000000000000000000000000, 0.3672174343425499420945357087475713342428207397460937500000000000
28, 0.3649030625820159912109375000000000000000000000000000000000000000, 0.3649032106188291879611540480254916474223136901855468750000000000
32, 0.3559529483318328857421875000000000000000000000000000000000000000, 0.3559530702497348819690614618593826889991760253906250000000000000
36, 0.3588987290859222412109375000000000000000000000000000000000000000, 0.3588988924762241405908014257875038310885429382324218750000000000
40, 0.3551684319972991943359375000000000000000000000000000000000000000, 0.3551685735333238169708636178256710991263389587402343750000000000
44, 0.3215409517288208007812500000000000000000000000000000000000000000, 0.3215410347878009944544430709356674924492835998535156250000000000
48, 0.3356299698352813720703125000000000000000000000000000000000000000, 0.3356300577063904277963501954218372702598571777343750000000000000
52, 0.3198622465133666992187500000000000000000000000000000000000000000, 0.3198623505596965843622569991566706448793411254882812500000000000
56, 0.3288863003253936767578125000000000000000000000000000000000000000, 0.3288863730836352750053208637837087735533714294433593750000000000
60, 0.3239665925502777099609375000000000000000000000000000000000000000, 0.3239666701510269253283524903963552787899971008300781250000000000
64, 0.3267349600791931152343750000000000000000000000000000000000000000, 0.3267349733764033925531578006484778597950935363769531250000000000
80, 0.2987520396709442138671875000000000000000000000000000000000000000, 0.2987521529297333389152413474221248179674148559570312500000000000
96, 0.2722776830196380615234375000000000000000000000000000000000000000, 0.2722777336258312974592854516231454908847808837890625000000000000
112, 0.2667976021766662597656250000000000000000000000000000000000000000, 0.2667977748517285085938510746927931904792785644531250000000000000
128, 0.2503516674041748046875000000000000000000000000000000000000000000, 0.2503518942092312626890304727567126974463462829589843750000000000
144, 0.2395697385072708129882812500000000000000000000000000000000000000, 0.2395700308162438951775641271524364128708839416503906250000000000
160, 0.2249167561531066894531250000000000000000000000000000000000000000, 0.2249171500233367870791312270739581435918807983398437500000000000
176, 0.2129349112510681152343750000000000000000000000000000000000000000, 0.2129354290762933699099335171922575682401657104492187500000000000
192, 0.2014226913452148437500000000000000000000000000000000000000000000, 0.2014232894597591250196444434550357982516288757324218750000000000
208, 0.1875210255384445190429687500000000000000000000000000000000000000, 0.1875215786101888659764824751619016751646995544433593750000000000
224, 0.1775955259799957275390625000000000000000000000000000000000000000, 0.1775961291562305754965223059116397053003311157226562500000000000
240, 0.1675380915403366088867187500000000000000000000000000000000000000, 0.1675386617455505711138386004677158780395984649658203125000000000
256, 0.1586101204156875610351562500000000000000000000000000000000000000, 0.1586107179767252128854693182802293449640274047851562500000000000
320, 0.1250997781753540039062500000000000000000000000000000000000000000, 0.1251002025668623840459758866927586495876312255859375000000000000
384, 0.0982034802436828613281250000000000000000000000000000000000000000, 0.0982037211556151257152436073738499544560909271240234375000000000
448, 0.0804683119058609008789062500000000000000000000000000000000000000, 0.0804685443083404533526348245686676818877458572387695312500000000
512, 0.0635731071233749389648437500000000000000000000000000000000000000, 0.0635732136677536890934092639326991047710180282592773437500000000
576, 0.0519851669669151306152343750000000000000000000000000000000000000, 0.0519853084150529581863864336810365784913301467895507812500000000
640, 0.0409081168472766876220703125000000000000000000000000000000000000, 0.0409082880567182283204630266482126899063587188720703125000000000
704, 0.0337469540536403656005859375000000000000000000000000000000000000, 0.0337470458702493897318319682199216913431882858276367187500000000
768, 0.0264944341033697128295898437500000000000000000000000000000000000, 0.0264945026619202335382841084765459527261555194854736328125000000
832, 0.0218510478734970092773437500000000000000000000000000000000000000, 0.0218510906446209678122283293077998678199946880340576171875000000
896, 0.0168275330215692520141601562500000000000000000000000000000000000, 0.0168276231753345002772892513576152850873768329620361328125000000
960, 0.0136076025664806365966796875000000000000000000000000000000000000, 0.0136076155426913779356024036815142608247697353363037109375000000
1024, 0.0107735013589262962341308593750000000000000000000000000000000000, 0.0107735051547879264605134252974494302179664373397827148437500000
1280, 0.0042227236554026603698730468750000000000000000000000000000000000, 0.0042227096414185630660065307040440529817715287208557128906250000
1536, 0.0020111112389713525772094726562500000000000000000000000000000000, 0.0020110895061265358398083780144816046231426298618316650390625000
1792, 0.0009041370940394699573516845703125000000000000000000000000000000, 0.0009041255454691193979219199228225534170633181929588317871093750
2048, 0.0004614095378201454877853393554687500000000000000000000000000000, 0.0004614093497469542439967660207855715270852670073509216308593750
2304, 0.0002609413350000977516174316406250000000000000000000000000000000, 0.0002609471463812920504973802415804584597935900092124938964843750
2560, 0.0001561905082780867815017700195312500000000000000000000000000000, 0.0001561964262453840882396982614466196537250652909278869628906250
2816, 0.0001299716386711224913597106933593750000000000000000000000000000, 0.0001299747611102498433960361623462631541769951581954956054687500
3072, 0.0001197227029479108750820159912109375000000000000000000000000000, 0.0001197243641067268052390407118146242737566353753209114074707031
3328, 0.0001144237248809076845645904541015625000000000000000000000000000, 0.0001144236808854348652359639304521010672033298760652542114257812
3584, 0.0001049559941748157143592834472656250000000000000000000000000000, 0.0001049551141166729260542633817010482744080945849418640136718750
3840, 0.0001074681495083495974540710449218750000000000000000000000000000, 0.0001074670342148931320627042174109533334558364003896713256835938
4096, 0.0001091197700588963925838470458984375000000000000000000000000000, 0.0001091174312480602440851226719864541792048839852213859558105469
5120, 0.0000822892980067990720272064208984375000000000000000000000000000, 0.0000822887634450031882826434581623686881357571110129356384277344
6144, 0.0001018746916088275611400604248046875000000000000000000000000000, 0.0001018740568148458720343710326439179425506154075264930725097656
7168, 0.0000948197048273868858814239501953125000000000000000000000000000, 0.0000948194607263291459669715677271994991315295919775962829589844
8192, 0.0000952790651354007422924041748046875000000000000000000000000000, 0.0000952798670336079349786964631796593039325671270489692687988281
9216, 0.0000866958653205074369907379150390625000000000000000000000000000, 0.0000866971099609489836458620093395666117430664598941802978515625
10240, 0.0000752704872866161167621612548828125000000000000000000000000000, 0.0000752724540788973062567859817484361428796546533703804016113281
11264, 0.0000885072149685584008693695068359375000000000000000000000000000, 0.0000885086770750781465103354528345391827315324917435646057128906
12288, 0.0000735499706934206187725067138671875000000000000000000000000000, 0.0000735511159507406685111735589188697304052766412496566772460938
13312, 0.0000696362913004122674465179443359375000000000000000000000000000, 0.0000696365945168458098674807432715283539437223225831985473632812
14336, 0.0000626636756351217627525329589843750000000000000000000000000000, 0.0000626652472343203073915085532519242406124249100685119628906250
15360, 0.0000668750435579568147659301757812500000000000000000000000000000, 0.0000668749710555043003342920449050268416613107547163963317871094
16384, 0.0000554985235794447362422943115234375000000000000000000000000000, 0.0000554978342588263487650084015978535489921341650187969207763672
20480, 0.0000490611564600840210914611816406250000000000000000000000000000, 0.0000490583565281399274198709248295813267759513109922409057617188
24576, 0.0000498052540933713316917419433593750000000000000000000000000000, 0.0000498025071586330982304803571825146946139284409582614898681641
28672, 0.0000369010413123760372400283813476562500000000000000000000000000, 0.0000368995202773496724143326153111388521210756152868270874023438
32768, 0.0000240253557421965524554252624511718750000000000000000000000000, 0.0000240248042511352281728771035362868246920697856694459915161133
36864, 0.0000245392348006134852766990661621093750000000000000000000000000, 0.0000245369568920889966794681574802083900976867880672216415405273
40960, 0.0000236803007283015176653861999511718750000000000000000000000000, 0.0000236800598766063623073704902388669779611518606543540954589844
45056, 0.0000250168559432495385408401489257812500000000000000000000000000, 0.0000250134187107764484044123559636219056301342789083719253540039
49152, 0.0000160359868459636345505714416503906250000000000000000000000000, 0.0000160337510020704820815420910129489584505790844559669494628906
53248, 0.0000176497705979272723197937011718750000000000000000000000000000, 0.0000176483622931551025272978172608873137505725026130676269531250
57344, 0.0000163802160386694595217704772949218750000000000000000000000000, 0.0000163786994753712785636949600398537540968391112983226776123047
61440, 0.0000189813872566446661949157714843750000000000000000000000000000, 0.0000189785399766058045527197112312123294941557105630636215209961
65536, 0.0000168542792380321770906448364257812500000000000000000000000000, 0.0000168530331615719138451184311300323770410614088177680969238281
81920, 0.0000103681786640663631260395050048828125000000000000000000000000, 0.0000103697347755917503958579192757660791812668321654200553894043
98304, 0.0000104306618595728650689125061035156250000000000000000000000000, 0.0000104326322128717275885970497406596280143276089802384376525879
114688, 0.0000094905381047283299267292022705078125000000000000000000000000, 0.0000094937836916515147442320252113567846663499949499964714050293
131072, 0.0000087307371359202079474925994873046875000000000000000000000000, 0.0000087329250024793424260179455775876533607515739277005195617676
147456, 0.0000064006781030911952257156372070312500000000000000000000000000, 0.0000063981052921249660766968986314040535035019274801015853881836
163840, 0.0000048723991312726866453886032104492187500000000000000000000000, 0.0000048671797194825745531271410371232377656269818544387817382812
180224, 0.0000058050004554388578981161117553710937500000000000000000000000, 0.0000058011889949629266737592819813684741347969975322484970092773
196608, 0.0000052389764277904760092496871948242187500000000000000000000000, 0.0000052281089592904196093419398605472991903297952376306056976318
212992, 0.0000033079088552767643705010414123535156250000000000000000000000, 0.0000032956262988330833157665008398273798206901119556277990341187
229376, 0.0000024621940610813908278942108154296875000000000000000000000000, 0.0000024406808726215399139717384530134935971545928623527288436890
245760, 0.0000036117087347520282492041587829589843750000000000000000000000, 0.0000035936323562051334290414622563147872824629303067922592163086
262144, 0.0000033362530302838422358036041259765625000000000000000000000000, 0.0000033179158658770516717795903000043367114813008811324834823608
327680, 0.0000025368544811499305069446563720703125000000000000000000000000, 0.0000025203105287323361134385352883002795465472445357590913772583
393216, 0.0000011195180604772758670151233673095703125000000000000000000000, 0.0000011089446362758204086391110310128027549581020139157772064209
458752, 0.0000002906567999616527231410145759582519531250000000000000000000, 0.0000002861799208172746283703944108339767282700449868571013212204
524288, 0.0000004659355568037426564842462539672851562500000000000000000000, 0.0000004624865793254293028185619284647911442220902245026081800461
589824, 0.0000002262784448703314410522580146789550781250000000000000000000, 0.0000002240319312883808820071132334589281143166772380936890840530
655360, 0.0000000665867503357731038704514503479003906250000000000000000000, 0.0000000651288835981632980778668353359639464628116911626420915127
720896, 0.0000001152603630316662020049989223480224609375000000000000000000, 0.0000001139098213740171831816588008702562095919574858271516859531
786432, 0.0000000409246787569372827420011162757873535156250000000000000000, 0.0000000399774280513396780964354305237806741502026852685958147049
851968, 0.0000000565102133975869946880266070365905761718750000000000000000, 0.0000000556721290943565208450560539282792049675663292873650789261
917504, 0.0000000661160157733320374973118305206298828125000000000000000000, 0.0000000656336610561822446047815140385273124934428778942674398422
983040, 0.0000000489024287730899231974035501480102539062500000000000000000, 0.0000000487642194071860087145432803394651255501912601175718009472
1048576, 0.0000000679597604857917758636176586151123046875000000000000000000, 0.0000000679383560297964124319127269223395071406912393285892903805
1310720, 0.0000000129076660471127979690209031105041503906250000000000000000, 0.0000000122312493417484777132182952345471060873194346640957519412
1572864, 0.0000000078401436454100803530309349298477172851562500000000000000, 0.0000000067643266581812481652681734607939628700634671076841186732
1835008, 0.0000000074473693878474023222224786877632141113281250000000000000, 0.0000000063110929834338728173842572783704130134907472893246449530
2097152, 0.0000000177477978979823092231526970863342285156250000000000000000, 0.0000000165406800060360026927435035670743701530227554030716419220
2359296, 0.0000000016811421144069527144893072545528411865234375000000000000, 0.0000000000867565361993408824949851040782267379974923571239742159
2621440, 0.0000000022346369199510718317469581961631774902343750000000000000, 0.0000000001013817203012429463804467550853169281743770824277817155
2883584, 0.0000000094449408294394743279553949832916259765625000000000000000, 0.0000000079225659219114021137969595346556284098937794624362140894
3145728, 0.0000000028580435795788616815116256475448608398437500000000000000, 0.0000000012482819388645445929673371567546202864917859187698923051
3407872, 0.0000000027845137307025424888706766068935394287109375000000000000, 0.0000000011077292181515010350272164413288310258476343506117700599
3670016, 0.0000000019247776705810792918782681226730346679687500000000000000, 0.0000000000161035551024396158318396452482240167880522818677491159
3932160, 0.0000000018538267587686618753650691360235214233398437500000000000, 0.0000000000115638049488775006531713655317161401625547512139746686
4194304, 0.0000000027562061521990699475281871855258941650390625000000000000, 0.0000000013965544925804094062392595593236066520148597192019224167
5242880, 0.0000000021403712135992236653692089021205902099609375000000000000, 0.0000000000041586382860724255665783731760401514690808633112339976
6291456, 0.0000000018919785738091832172358408570289611816406250000000000000, 0.0000000001371763188637059281778555369204542803418522822767045000
7340032, 0.0000000018966221926319803969818167388439178466796875000000000000, 0.0000000000009922498177999927945037022359876574388072578702946203
8388608, 0.0000000018445186489302045629301574081182479858398437500000000000, 0.0000000000005217066693139013030133459880237448033591871054959199
9437184, 0.0000000021552388762557939116959460079669952392578125000000000000, 0.0000000000467010895164354290956987528779189742730859613573102251
\end{filecontents*}
\csvreader[head=false, late after line=\\\hline, late after last line=\ ]{mean_square_error_example1.csv}{}
  {\csvcoli & \csvcolii & \csvcoliii}
\end{longtable}

\clearpage
\begin{center}
\begin{lstlisting}
import numpy as np
import torch
import copy
import matplotlib.pyplot as plt

# use GPU if available
dev = torch.device("cuda") if torch.cuda.is_available() else torch.device("cpu")

# set precision to double
torch.set_default_dtype(torch.float64)

# configure the training parameters
steps = 10000000                # n
batch_size = 100                # M_n, number of sample points
gamma = 1/2000                  # step size (learning rate)

# get sample data
torch.manual_seed(0)
sample_data_double = (torch.rand(batch_size*(steps+1), 1)).to(dev)
sample_data_single = sample_data_double.type(torch.float32)

# set parameters
l_0, l_1, l_2, = 1, 7, 1            # dimensions of the layers
a, b = 0, 1                         # the domain will be [a,b]
xi = 1                              # constant target function


# define loss function
def loss(N, x):
    y = N(x)
    return (y-xi).square().mean()


# function setting the starting model parameters as normally distributed
def init_weights(n):
    if isinstance(n, torch.nn.Linear):
        torch.nn.init.normal_(n.weight, 0, 1/np.sqrt(l_1))
        torch.nn.init.normal_(n.bias, 0, 1/np.sqrt(l_1))


# set up NN model
N_double = torch.nn.Sequential(
    torch.nn.Linear(l_0, l_1),
    torch.nn.ReLU(),
    torch.nn.Linear(l_1, l_2)
).to(dev)

# set up normally distributed starting model parameters
torch.manual_seed(0)
N_double.apply(init_weights)
N_single = copy.deepcopy(N_double)
N_single.type(torch.float32)

# define observation points
obs = np.empty(shape=129)
index = 0
for i in range(0, 23, 2):
    for k in range(2**i, min(2**(i+2), steps)):
        if k % max(1, 2**(i-2)) == False:
            obs[index] = k
            index = index + 1

# set up table
losses_table = np.empty((len(obs), 3), dtype=object)

# approximation in double precision
# optimizer function
optimizer_double = torch.optim.SGD(N_double.parameters(), lr=gamma)

# train the network
for step in range(steps+1):
    # generate uniformly distributed samples from [0,1]
    x = sample_data_double[step*batch_size:(step+1)*batch_size]
    # set the gradients back to zero
    optimizer_double.zero_grad()
    # compute the loss
    L = loss(N_double, x)
    # compute the gradients of L w.r.t. the model parameters
    L.backward()
    # update the model parameters
    optimizer_double.step()
    # save results in the table
    if step in obs:
        losses_table[np.where(obs == step)[0], [0, 2]] = step, L.data

# approximation in single precision
# optimizer function
optimizer_single = torch.optim.SGD(N_single.parameters(), lr=gamma)

# train the network
for step in range(steps+1):
    # generate uniformly distributed samples from [0,1]
    x = sample_data_single[step*batch_size:(step+1)*batch_size]
    # set the gradients back to zero
    optimizer_single.zero_grad()
    # compute the loss
    L = loss(N_single, x)
    # compute the gradients of L w.r.t. the model parameters
    L.backward()
    # update the model parameters
    optimizer_single.step()
    # print and save results in the table
    if step in obs:
        losses_table[np.where(obs == step)[0], [0, 1]] = step, L.data

# save table
np.savetxt('mean_square_error_example1.csv', losses_table, delimiter=',',
           fmt='%s, %.64f, %.64f')

# plot mean square error
plt.loglog(losses_table[:, 0], losses_table[:, 2], "r", label='Double precision')
plt.loglog(losses_table[:, 0], losses_table[:, 1], "b", label='Single precision')
plt.ylim([0.00000000000000000001, 10])
plt.xlim([1, 12000000])
plt.xlabel('Number of SGD steps')
plt.ylabel('Mean square error (expected risk)')
plt.legend()
plt.savefig('mean_square_error_example1_loglog_graph', dpi=1200)
\end{lstlisting}
    \captionof{listing}{\label{lst:sgd1} {\sc Python} code used to create Table~\ref{table:ANN_num_sim1} 
    and Figure~\ref{fig:ANN_num_sim1} in \cref{subsec:numerical_example_1} (filename: \texttt{sgd1.py}).}
    \bigskip
\end{center}

\begin{center}
\begin{lstlisting}
import numpy as np
import torch
import matplotlib.pyplot as plt

# use GPU if available
dev = torch.device("cuda") if torch.cuda.is_available() else torch.device("cpu")

# set precision to double
torch.set_default_dtype(torch.float64)
torch.manual_seed(0)

# configure the training parameters
steps = 10000000                # n
batch_size = 100                # M_n, number of sample points
gamma = 1/2000                  # step size (learning rate)

# set parameters
l_0, l_1, l_2, = 1, 7, 1        # dimensions of the layers
a, b = 0, 1                     # the domain will be [a,b]
xi = 1                          # constant target function


# define loss function
def loss(N, x):
    y = N(x)
    return (y-xi).square().mean()


# function setting the starting model parameters as normally distributed
def init_weights(n):
    if isinstance(n, torch.nn.Linear):
        torch.nn.init.normal_(n.weight, 0, 1/np.sqrt(l_1))
        torch.nn.init.normal_(n.bias, 0, 1/np.sqrt(l_1))


# set up NN model
N = torch.nn.Sequential(
    torch.nn.Linear(l_0, l_1),
    torch.nn.ReLU(),
    torch.nn.Linear(l_1, l_2)
).to(dev)

# set the starting model parameters as normally distributed
N.apply(init_weights)

# optimizer function
optimizer = torch.optim.SGD(N.parameters(), lr=gamma)

# define observation points
obs = np.empty(shape=129)
index = 0
for i in range(0, 23, 2):
    for k in range(2**i, min(2**(i+2), steps)):
        if k % max(1, 2**(i-2)) == False:
            obs[index] = k
            index = index + 1

# train the network
loss_table = np.empty((len(obs), 2), dtype=object)
for step in range(steps+1):
    # generate uniformly distributed samples from [0,1]
    x = (torch.rand(batch_size, 1)).to(dev)
    # set the gradients back to zero
    optimizer.zero_grad()
    # compute the loss
    L = loss(N, x)
    # compute the gradients of L w.r.t. the model parameters
    L.backward()
    # update the model parameters
    optimizer.step()
    # save results in the table
    if step in obs:
        loss_table[np.where(obs == step)[0], :] = step, L.data

# plot mean square error
plt.loglog(loss_table[:, 0], loss_table[:, 1], color= "r")
plt.ylim([0.00000000000000000001, 10])
plt.xlim([1, 12000000])
plt.xlabel('Number of SGD steps')
plt.ylabel('Mean square error (expected risk)')
plt.show()
\end{lstlisting}
    \captionof{listing}{\label{lst:sgd1_short} Simplified variant of the {\sc Python}script in \cref{lst:sgd1} above 
    (filename: \texttt{sgd1\_short.py}).}
    \bigskip
\end{center}

\FloatBarrier

\subsection{Numerical simulation of an SGD process for certain ANNs with two hidden layers}
\label{subsec:numerical_example_2}

In this subsection we present numerical simulations for a certain SGD process 
in the training of some deep ANNs with two hidden layers with $ 1 $ neuron on the input layer, with $ 3 $ neurons on the first hidden layer, 
with $ 7 $ neurons on the second hidden layer,
and with $ 1 $ neuron on the output layer (see \cref{figure_ANN_architecture_sim2},
\begin{figure}
\centering
\begin{adjustbox}{width=\textwidth}
\begin{tikzpicture}[shorten >=1pt,->,draw=black!50, node distance=\layersep]
    \tikzstyle{neuron} =[circle, draw=black!80, 
    minimum size=17pt,line width=0.5mm, inner sep=0pt]
    \tikzstyle{input neuron} =[neuron, fill=red!50];
    \tikzstyle{output neuron} =[neuron,
    fill=violet!50];
    \tikzstyle{hidden neuron} =[neuron,
    fill=blue!50];
    \tikzstyle{annot} = [text width=9em, text centered]
    \tikzstyle{annot2} = [text width=4em, text centered]

        \node[input neuron] (I-1) at (0,0) {};  
        \node[annot,left of=I-1, node distance=0.7cm, align=center] () {$x$};
    \foreach \name / \y in {1,...,3}
        \path
            node[hidden neuron] (H-\name) at ( \layersep,2-\y ) {};      
    \foreach \name / \y in {1,...,7}
        \path
            node[hidden neuron] (H2-\name) at (2*\layersep,4-\y ) {};          
        \path
            node[output neuron](O-1) at (3*\layersep,0) {};      
            \node[annot,right of=O-1, node distance=1.1cm, align=center] () {$ \mN^{3, \theta }_{ \infty}(x) $};
            
        \foreach \dest in {1,...,3}
            \path[>=stealth] (I-1) edge (H-\dest);     
    \foreach \source in {1,...,3}
        \foreach \dest in {1,...,7}
            \path[>=stealth] (H-\source) edge (H2-\dest);     
    \foreach \source in {1,...,7}
            \path[>=stealth] (H2-\source) edge (O-1);

    \node[annot,above of=H-1, node distance=3cm, align=center] (hl) {1st hidden layer\\(2nd layer)};
    \node[annot,above of=H2-1, node distance=1cm, align=center] (hl2) {2nd hidden layer\\(3rd layer)};
    \node[annot,left of=hl, align=center] {Input layer\\ (1st layer)};
    \node[annot,right of=hl2, align=center] {Output layer\\(4th layer)};
    
    \node[annot2,below of=H-3, node distance=3cm, align=center] (sl) {$ \ell_1=3$};
    \node[annot2,below of=H2-7, node distance=1cm, align=center] (sl2) {$ \ell_2=7$};
    \node[annot2,left of=sl, align=center] {$ \ell_0=1$};
    \node[annot2,right of=sl2, align=center] {$ \ell_3=1$};
\end{tikzpicture}
\end{adjustbox}
\caption{Graphical illustration of the ANN architecture used in \cref{subsec:numerical_example_2} 
($ \ell_0 = 1 $ neuron on the input layer, $ \ell_1 = 3 $ neurons on the first hidden layer,
$ \ell_2 = 7 $ neurons on the second hidden layer, and $ \ell_3 = 1 $ neuron on the output layer). 
In this situation we have for every ANN parameter vector $ \theta \in \Reals^{ \fd } = \Reals^{ 42 } $ 
that the realization function 
$ \Reals \ni x \mapsto \mN^{ 3, \theta }_{ \infty }(x) \in \Reals $ 
of the considered ANN maps the 1-dimensional input $ x \in [0,1] $ to the 1-dimensional 
output $ \mN^{ 3, \theta }_{ \infty }( x ) \in \Reals $.}
\label{figure_ANN_architecture_sim2}
\end{figure}
Table~\ref{table:ANN_num_sim2}, 
and Figure~\ref{fig:ANN_num_sim2}).
\begin{figure}
\begin{center}
\includegraphics[scale=0.75]{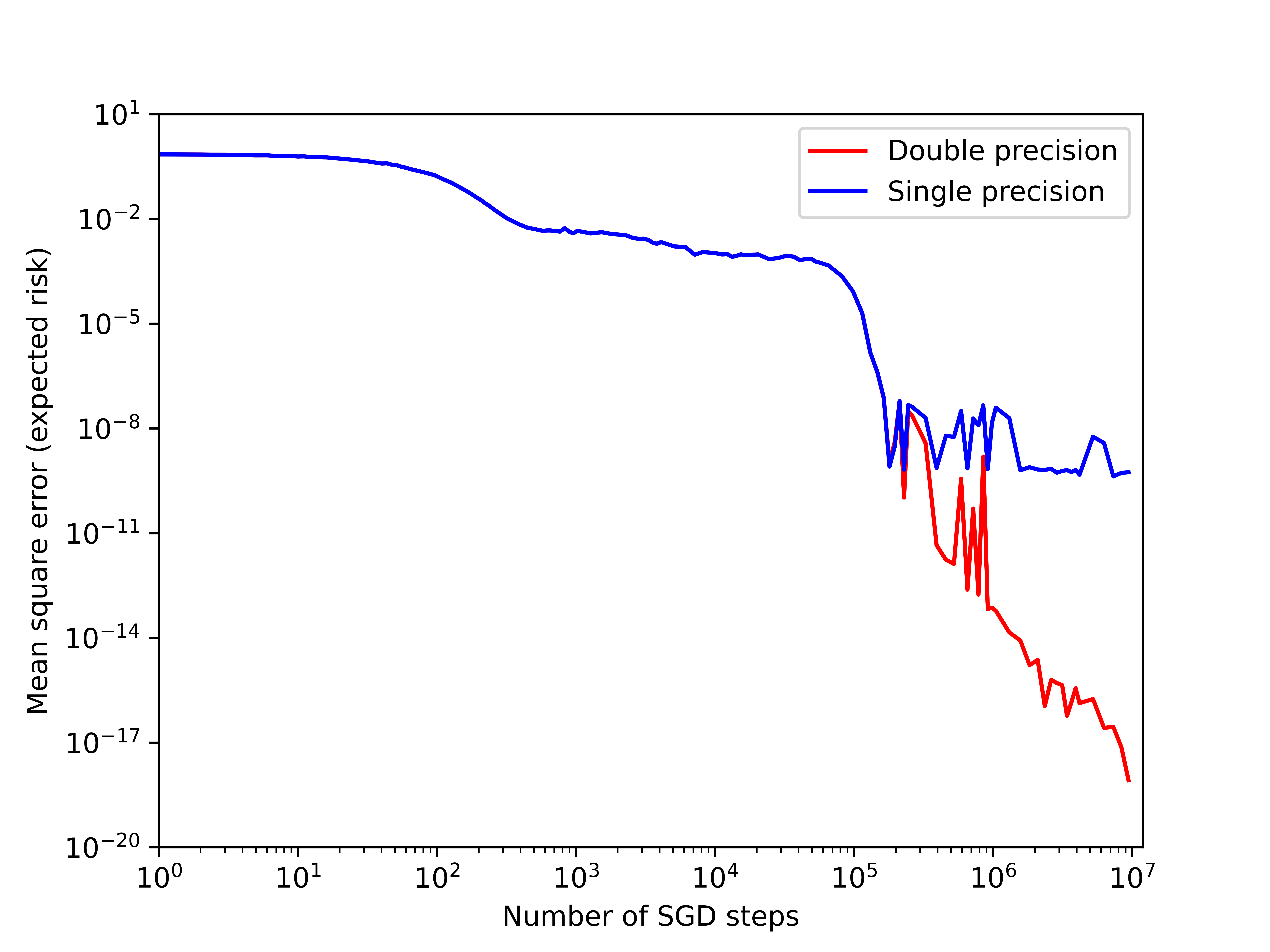}
\caption{Plot of the estimated mean square error (expected risk) against 
the number of SGD steps for \cref{subsec:numerical_example_2}}
\label{fig:ANN_num_sim2}
\end{center}
\end{figure}
Assume \cref{SGD_setting}, 
let $ e_1, e_2, \dots, e_{ \fd } \in \Reals^{ \fd } $ 
satisfy 
$
  e_1 = ( 1, 0, \dots, 0 )
$, 
$
  e_2 = ( 0, 1, \dots, 0 )
$,
$ \dots $,
$
  e_{ \fd } = ( 0, \dots, 0, 1 )
$,
assume 
\begin{equation}
\label{eq:num_example_2_assumption1}
  L = 3 ,
  \qquad 
  \ell_0 = 1 ,
  \qquad 
  \ell_1 = 3 ,
  \qquad 
  \ell_2 = 7 ,
  \qquad 
  \ell_3 = 1 ,
  \qquad 
  \xi = 1 ,
  \qquad 
  a = 0 ,
  \qquad 
  \text{and}
  \qquad 
  b = 1
  ,
\end{equation}
assume for all 
$ k \in \{ 1, 2, \dots, \fd \} $, $ x \in \Reals $ that 
\begin{equation}
\label{eq:num_example_2_assumption2}
\textstyle 
  \P(
    X^{ 0, 0 } < x
  )
  =
  \max\{ \min\{ x, 1 \}, 0 \}
\quad 
\text{and}
\quad 
  \P(
    \langle e_k, (\ell_1 )^{ 1 / 2 } \Theta_0 \rangle 
    < x
  )
  =
  \int_{ - \infty }^x
  ( 2 \pi )^{ - 1 / 2 }
  \exp\bigl( - \frac{ y^2 }{ 2 } \bigr)
  \, \d y
  ,
\end{equation}
and assume for all $ n \in \N_0 $ that $ M_n = 100 $ and $ \gamma_n = \frac{ 1 }{ 2000 } $. 
\Nobs that \cref{eq:num_example_2_assumption1} ensures that 
the number of deep ANN parameters $ \fd $ satisfies 
\begin{equation}
\textstyle 
  \fd = 
  \sum_{ k = 1 }^L
  \ell_k ( \ell_{ k - 1 } + 1 ) 
  =
  2\ell_1 + \ell_2(\ell_1 +1 )
  + \ell_2 +1 
  =
  42 
  .
\end{equation}
\Moreover \cref{eq:num_example_2_assumption1} assures that the considered deep
ANNs consist of $ 31 $ weight parameters and $ 11 $ bias parameters. 
We also refer to \cref{figure_ANN_architecture_sim2} for a graphical 
illustration of the deep ANN architecture described in \cref{eq:num_example_2_assumption1}. 
\Moreover \cref{eq:num_example_2_assumption2} ensures that for all $ n, m \in \N_0 $ 
we have that 
$
  X^{ n, m }
$
is continuous uniformly distributed on $ (0,1) $. 
\Moreover \cref{eq:num_example_2_assumption2} assures that
$
  ( \ell_1 )^{ 1 / 2 } \Theta_0
$
is a $ 42 $-dimensional standard normal random vector. 
In \cref{table:ANN_num_sim2} we approximately specify the mean square error (expected risk) 
\begin{equation}
\label{eq:expected_risk_ex2}
  \E[ \mL( \Theta_n ) ] 
\end{equation}
against the number $ n $ of SGD steps 
for several $ n \in \{ 1,2, \dots, 10^6 \} $ 
and in \cref{fig:ANN_num_sim2} we approximately plot the values of \cref{table:ANN_num_sim2}. 
In \cref{table:ANN_num_sim2} and \cref{fig:ANN_num_sim2} we approximate the expectations in 
\cref{eq:expected_risk_ex2} by means of Monte Carlo 
averages with $ 10^6 $ samples. 
The {\sc Python} source code used to create \cref{table:ANN_num_sim2} and \cref{fig:ANN_num_sim2} can be found in \cref{lst:sgd2}. 
In addition, in \cref{lst:sgd2_short} we present a simplified variant of the {\sc Python} code in \cref{lst:sgd2}.

\clearpage
\begin{longtable}{|l|S[table-auto-round,
                 table-format=1.18]|
                 S[table-auto-round,
                 table-format=1.18]|}
\caption{Estimated mean square error (expected risk)\\ against the number of SGD steps for \cref{subsec:numerical_example_2}}
\label{table:ANN_num_sim2}\\
  \hline
  \multicolumn{1}{|c|}{Number of SGD steps}
   & \multicolumn{2}{|c|}{Estimated mean square error (expected risk)} \\
  \cline{2-3}
  \multicolumn{1}{|c|}{}
  & \multicolumn{1}{|c|}{Single precision}
  & \multicolumn{1}{|c|}{Double precision}\\
  \hline
\endhead
\hline
\endfoot
  \begin{filecontents*}{mean_square_error_example2.csv}
1, 0.7069329023361206054687500000000000000000000000000000000000000000, 0.7069329671382384994160474889213219285011291503906250000000000000
2, 0.7001609206199645996093750000000000000000000000000000000000000000, 0.7001610019236186310820357903139665722846984863281250000000000000
3, 0.6922094821929931640625000000000000000000000000000000000000000000, 0.6922096487134300479837634156865533441305160522460937500000000000
4, 0.6743393540382385253906250000000000000000000000000000000000000000, 0.6743395128711694441747681594279129058122634887695312500000000000
5, 0.6638675928115844726562500000000000000000000000000000000000000000, 0.6638677135423962649696250082342885434627532958984375000000000000
6, 0.6653053760528564453125000000000000000000000000000000000000000000, 0.6653054916432431520334489505330566316843032836914062500000000000
7, 0.6372143030166625976562500000000000000000000000000000000000000000, 0.6372144774377712650448302156291902065277099609375000000000000000
8, 0.6442492604255676269531250000000000000000000000000000000000000000, 0.6442494897592423441778919368516653776168823242187500000000000000
9, 0.6401695013046264648437500000000000000000000000000000000000000000, 0.6401697657063549762312959501286968588829040527343750000000000000
10, 0.6145961880683898925781250000000000000000000000000000000000000000, 0.6145964372334414838050520302203949540853500366210937500000000000
11, 0.6229035854339599609375000000000000000000000000000000000000000000, 0.6229037903968747125205140946491155773401260375976562500000000000
12, 0.5998342633247375488281250000000000000000000000000000000000000000, 0.5998344909712534089862856490071862936019897460937500000000000000
13, 0.5995413064956665039062500000000000000000000000000000000000000000, 0.5995413853849850260857579087314661592245101928710937500000000000
14, 0.5943095684051513671875000000000000000000000000000000000000000000, 0.5943097295219484932360387574590276926755905151367187500000000000
15, 0.5845280289649963378906250000000000000000000000000000000000000000, 0.5845282357949009854891642135044094175100326538085937500000000000
16, 0.5814957618713378906250000000000000000000000000000000000000000000, 0.5814959607835503696549039887031540274620056152343750000000000000
20, 0.5365075469017028808593750000000000000000000000000000000000000000, 0.5365076262729686629171510503510944545269012451171875000000000000
24, 0.5026288628578186035156250000000000000000000000000000000000000000, 0.5026289071515084305730169944581575691699981689453125000000000000
28, 0.4709900915622711181640625000000000000000000000000000000000000000, 0.4709900827277134216508613917540060356259346008300781250000000000
32, 0.4462942481040954589843750000000000000000000000000000000000000000, 0.4462942832813477367359666914126137271523475646972656250000000000
36, 0.4137372672557830810546875000000000000000000000000000000000000000, 0.4137371931013161807655365009850356727838516235351562500000000000
40, 0.3879753947257995605468750000000000000000000000000000000000000000, 0.3879753685238884752806143296766094863414764404296875000000000000
44, 0.3916522264480590820312500000000000000000000000000000000000000000, 0.3916522430894159856329395097418455407023429870605468750000000000
48, 0.3528670370578765869140625000000000000000000000000000000000000000, 0.3528671307223729125546185514394892379641532897949218750000000000
52, 0.3438043296337127685546875000000000000000000000000000000000000000, 0.3438044206673474545432611648720921948552131652832031250000000000
56, 0.3102094233036041259765625000000000000000000000000000000000000000, 0.3102095222953108466512617269472684711217880249023437500000000000
60, 0.2933767735958099365234375000000000000000000000000000000000000000, 0.2933768031318593760836677120096283033490180969238281250000000000
64, 0.2698673009872436523437500000000000000000000000000000000000000000, 0.2698673847510548728934054452111013233661651611328125000000000000
80, 0.2190151214599609375000000000000000000000000000000000000000000000, 0.2190152619395580313454985343923908658325672149658203125000000000
96, 0.1805158853530883789062500000000000000000000000000000000000000000, 0.1805159890314164405911157018636004067957401275634765625000000000
112, 0.1355295628309249877929687500000000000000000000000000000000000000, 0.1355298724064906978892253164303838275372982025146484375000000000
128, 0.1080729886889457702636718750000000000000000000000000000000000000, 0.1080732337867228859407120467039931099861860275268554687500000000
144, 0.0834694355726242065429687500000000000000000000000000000000000000, 0.0834696144616510204983228504715953022241592407226562500000000000
160, 0.0661936923861503601074218750000000000000000000000000000000000000, 0.0661937937649610952384193751640850678086280822753906250000000000
176, 0.0529168732464313507080078125000000000000000000000000000000000000, 0.0529169438586739523944224572460370836779475212097167968750000000
192, 0.0419619642198085784912109375000000000000000000000000000000000000, 0.0419620600468874199728830376443511340767145156860351562500000000
208, 0.0346146300435066223144531250000000000000000000000000000000000000, 0.0346146654735697972826713453287084121257066726684570312500000000
224, 0.0277203097939491271972656250000000000000000000000000000000000000, 0.0277203082283540951091804771522220107726752758026123046875000000
240, 0.0233861580491065979003906250000000000000000000000000000000000000, 0.0233861736345386123092637831177853513509035110473632812500000000
256, 0.0189496297389268875122070312500000000000000000000000000000000000, 0.0189496352410304727542644798177207121625542640686035156250000000
320, 0.0103879943490028381347656250000000000000000000000000000000000000, 0.0103879793662392787972903462900831073056906461715698242187500000
384, 0.0072335959412157535552978515625000000000000000000000000000000000, 0.0072335912526142976883281043853912706254050135612487792968750000
448, 0.0056288405321538448333740234375000000000000000000000000000000000, 0.0056288518767815456239622662337751535233110189437866210937500000
512, 0.0050786347128450870513916015625000000000000000000000000000000000, 0.0050786315353771897471157892312021431280300021171569824218750000
576, 0.0045868414454162120819091796875000000000000000000000000000000000, 0.0045868450283297803465876185669003461953252553939819335937500000
640, 0.0047037755139172077178955078125000000000000000000000000000000000, 0.0047037793469077351232732198127450828906148672103881835937500000
704, 0.0045782662928104400634765625000000000000000000000000000000000000, 0.0045782804075707389870886210303524421760812401771545410156250000
768, 0.0043446025811135768890380859375000000000000000000000000000000000, 0.0043446158387439107184668785066605778411030769348144531250000000
832, 0.0054462002590298652648925781250000000000000000000000000000000000, 0.0054462221672845678374641487096141645452007651329040527343750000
896, 0.0043114623986184597015380859375000000000000000000000000000000000, 0.0043114830747704490157556911356095952214673161506652832031250000
960, 0.0038789457175880670547485351562500000000000000000000000000000000, 0.0038789586470748638018235521229826190392486751079559326171875000
1024, 0.0045674135908484458923339843750000000000000000000000000000000000, 0.0045674338206859290203998291701736889081075787544250488281250000
1280, 0.0038584077265113592147827148437500000000000000000000000000000000, 0.0038584178827186068130861151814769982593134045600891113281250000
1536, 0.0041750576347112655639648437500000000000000000000000000000000000, 0.0041750693955850846569854972756274946732446551322937011718750000
1792, 0.0037320354022085666656494140625000000000000000000000000000000000, 0.0037320307664668282698894330451366840861737728118896484375000000
2048, 0.0035602771677076816558837890625000000000000000000000000000000000, 0.0035602839403052260677795537446854723384603857994079589843750000
2304, 0.0033904714509844779968261718750000000000000000000000000000000000, 0.0033904754928652108825204436470812652260065078735351562500000000
2560, 0.0028824335895478725433349609375000000000000000000000000000000000, 0.0028824264197497093814570590097901003900915384292602539062500000
2816, 0.0026991129852831363677978515625000000000000000000000000000000000, 0.0026991101364093887833572882328780906391330063343048095703125000
3072, 0.0027230074629187583923339843750000000000000000000000000000000000, 0.0027229850906086289219487639456929173320531845092773437500000000
3328, 0.0025008853990584611892700195312500000000000000000000000000000000, 0.0025008651521389157340591324896195146720856428146362304687500000
3584, 0.0020696325227618217468261718750000000000000000000000000000000000, 0.0020696051854332624640731008724969797185622155666351318359375000
3840, 0.0019491473212838172912597656250000000000000000000000000000000000, 0.0019491285221250179756496967442558343464042991399765014648437500
4096, 0.0021795057691633701324462890625000000000000000000000000000000000, 0.0021795036939933659821122535760196115006692707538604736328125000
5120, 0.0016343385213986039161682128906250000000000000000000000000000000, 0.0016343328219309866556480104193838087667245417833328247070312500
6144, 0.0015643645310774445533752441406250000000000000000000000000000000, 0.0015643588532299290520810597371337280492298305034637451171875000
7168, 0.0009464485920034348964691162109375000000000000000000000000000000, 0.0009464487574457178167136595092756579106207937002182006835937500
8192, 0.0011291434057056903839111328125000000000000000000000000000000000, 0.0011291494840130352449170114681464838213287293910980224609375000
9216, 0.0010833853157237172126770019531250000000000000000000000000000000, 0.0010833919040555651733259701074985059676691889762878417968750000
10240, 0.0010376540012657642364501953125000000000000000000000000000000000, 0.0010376570684155065390358307908513779693748801946640014648437500
11264, 0.0009606786770746111869812011718750000000000000000000000000000000, 0.0009606776195133252468916218447247956646606326103210449218750000
12288, 0.0009814897784963250160217285156250000000000000000000000000000000, 0.0009814882302154701466251340846724815492052584886550903320312500
13312, 0.0008244680357165634632110595703125000000000000000000000000000000, 0.0008244657397394463760678662289649309968808665871620178222656250
14336, 0.0008789316634647548198699951171875000000000000000000000000000000, 0.0008789320360124225567713396323199503967771306633949279785156250
15360, 0.0009697995847091078758239746093750000000000000000000000000000000, 0.0009697926606264674039351025847111031907843425869941711425781250
16384, 0.0009234633180312812328338623046875000000000000000000000000000000, 0.0009234582494821422411768185867231295560486614704132080078125000
20480, 0.0009614962618798017501831054687500000000000000000000000000000000, 0.0009614921951956705969272176659501383255701512098312377929687500
24576, 0.0007043821970000863075256347656250000000000000000000000000000000, 0.0007043902264682336484061253578659034246811643242835998535156250
28672, 0.0007626851438544690608978271484375000000000000000000000000000000, 0.0007626889852584323213971484456408234109403565526008605957031250
32768, 0.0008867711294442415237426757812500000000000000000000000000000000, 0.0008867741411388151574063387627688825887162238359451293945312500
36864, 0.0008288892568089067935943603515625000000000000000000000000000000, 0.0008288943632051085086409614177682669833302497863769531250000000
40960, 0.0006597200990654528141021728515625000000000000000000000000000000, 0.0006597273555841508090036295897107265773229300975799560546875000
45056, 0.0007129385485313832759857177734375000000000000000000000000000000, 0.0007129428538756365845915907719643200834980234503746032714843750
49152, 0.0007257650140672922134399414062500000000000000000000000000000000, 0.0007257554269058215856355587014547836588462814688682556152343750
53248, 0.0005988157354295253753662109375000000000000000000000000000000000, 0.0005987953921537627526364144792125898675294592976570129394531250
57344, 0.0005565075553022325038909912109375000000000000000000000000000000, 0.0005564808009888868017139840205231848813127726316452026367187500
61440, 0.0005056711961515247821807861328125000000000000000000000000000000, 0.0005056357024279323816168263405756988504435867071151733398437500
65536, 0.0004691379144787788391113281250000000000000000000000000000000000, 0.0004691146855322134287188950985125757142668589949607849121093750
81920, 0.0002294167497893795371055603027343750000000000000000000000000000, 0.0002294056463268532468965238413360907543392386287450790405273438
98304, 0.0000849957796162925660610198974609375000000000000000000000000000, 0.0000849932918885528017587799309850993267900776118040084838867188
114688, 0.0000201562488655326887965202331542968750000000000000000000000000, 0.0000201533757210469412408417044790809313781210221350193023681641
131072, 0.0000015017377563708578236401081085205078125000000000000000000000, 0.0000015013061349875780999191862002950514920485147740691900253296
147456, 0.0000004079188613559381337836384773254394531250000000000000000000, 0.0000004046628735236710963254862499088870819718977145384997129440
163840, 0.0000000761403029514440277125686407089233398437500000000000000000, 0.0000000736435357994146755576580424922095158279944371315650641918
180224, 0.0000000008037926435555675652722129598259925842285156250000000000, 0.0000000010186682457981251675010427504199789217054217260738369077
196608, 0.0000000030453888300030484970193356275558471679687500000000000000, 0.0000000040779643654116341836569316814300328877251899939437862486
212992, 0.0000000602781042857714055571705102920532226562500000000000000000, 0.0000000587522467539942275527320176414791674801563203800469636917
229376, 0.0000000006603254054660112615238176658749580383300781250000000000, 0.0000000001046088254811220696185292968866901042873651661579970096
245760, 0.0000000468557352917287062155082821846008300781250000000000000000, 0.0000000311517137615620929012916261267168716031505937280599027872
262144, 0.0000000413597547321842284873127937316894531250000000000000000000, 0.0000000235523725434106132704500832482710781512480480159865692258
327680, 0.0000000200147027840102964546531438827514648437500000000000000000, 0.0000000037677495691540699833554334632380877767587890048162080348
393216, 0.0000000007401233514947591629606904461979866027832031250000000000, 0.0000000000044970867768420673397248397015245818331488636410142590
458752, 0.0000000061885185864696268254192546010017395019531250000000000000, 0.0000000000017289445509528816756881095907327785762309935169156461
524288, 0.0000000056778248769262518180767074227333068847656250000000000000, 0.0000000000013065969127908108753281962087202722937531729208160414
589824, 0.0000000317320605347504169913008809089660644531250000000000000000, 0.0000000003623789835998009480672622276110298594753444945126830135
655360, 0.0000000007161105597397465771791758015751838684082031250000000000, 0.0000000000002398908410883154844985463811281760357417384976752572
720896, 0.0000000192585591918259524391032755374908447265625000000000000000, 0.0000000000505180816733494149183772660012775417015529200170931290
786432, 0.0000000122505401378703027148731052875518798828125000000000000000, 0.0000000000001732745179143837872058518415909252983889751037471783
851968, 0.0000000457700579659103823360055685043334960937500000000000000000, 0.0000000015530458278342690640784083590947295405726436001714318991
917504, 0.0000000006770494720420572321017971262335777282714843750000000000, 0.0000000000000661033705115921148234572776270164384256966885744511
983040, 0.0000000144261491641373140737414360046386718750000000000000000000, 0.0000000000000733392251870272964248451575979370158753873829016179
1048576, 0.0000000393853873958960321033373475074768066406250000000000000000, 0.0000000000000598710342110839609076946482252348362982155571820897
1310720, 0.0000000197379890209958830382674932479858398437500000000000000000, 0.0000000000000143157866562351586308506142867048789883363535402339
1572864, 0.0000000006303325639223089638107921928167343139648437500000000000, 0.0000000000000084204935712320680414073626685771180508513272452986
1835008, 0.0000000007719090366897773947130190208554267883300781250000000000, 0.0000000000000016483391614708660964732132078045347101798813243753
2097152, 0.0000000006639413463460641651181504130363464355468750000000000000, 0.0000000000000023284342393928731995385267800028650719657106072519
2359296, 0.0000000006494750848240471441386034712195396423339843750000000000, 0.0000000000000001113556465060151866331729772395904728081931611260
2621440, 0.0000000006917678652129666261316742748022079467773437500000000000, 0.0000000000000006259878932860678705718583773259602401875843031773
2883584, 0.0000000005381075030008730664121685549616813659667968750000000000, 0.0000000000000005062925182181731792536584066049575652562080886609
3145728, 0.0000000006025511756213575154106365516781806945800781250000000000, 0.0000000000000004435454332944356995710135310208180560875184435948
3407872, 0.0000000006396082552484472216747235506772994995117187500000000000, 0.0000000000000000584823602822889541997287148871404196752354708944
3670016, 0.0000000005598355112823583112913183867931365966796875000000000000, 0.0000000000000001444977206367163321276733603895278932787755191104
3932160, 0.0000000006433220067769696015602676197886466979980468750000000000, 0.0000000000000003598252071969682011099972852573101253560071887891
4194304, 0.0000000004703323841503959101828513666987419128417968750000000000, 0.0000000000000001343941796054269976474692318086673499793332053898
5242880, 0.0000000057872360237354314449476078152656555175781250000000000000, 0.0000000000000001767829687594857250848239985506401484489613472198
6291456, 0.0000000038296299464946059742942452430725097656250000000000000000, 0.0000000000000000264044814755142880526484090218735637497412574416
7340032, 0.0000000004214325277640540434731519781053066253662109375000000000, 0.0000000000000000280390905923354137453959688551088643947523874455
8388608, 0.0000000005308937178760686492751119658350944519042968750000000000, 0.0000000000000000073651440021857528199758768774466404307061849038
9437184, 0.0000000005555002013934995375166181474924087524414062500000000000, 0.0000000000000000008372654284266455969386346528936153250156395562
\end{filecontents*}
\csvreader[head=false, late after line=\\\hline, late after last line=\ ]{mean_square_error_example2.csv}{}
  {\csvcoli & \csvcolii & \csvcoliii}
\end{longtable}

\newpage
\begin{center}
\begin{lstlisting}
import numpy as np
import torch
import copy
import matplotlib.pyplot as plt

# use GPU if available
dev = torch.device("cuda") if torch.cuda.is_available() else torch.device("cpu")

# set precision to double
torch.set_default_dtype(torch.float64)

# configure the training parameters
steps = 10000000                # n
batch_size = 100                # M_n, number of sample points
gamma = 1/2000                  # step size (learning rate)

# get sample data
torch.manual_seed(0)
sample_data_double = (torch.rand(batch_size*(steps+1), 1)).to(dev)
sample_data_single = sample_data_double.type(torch.float32)

# set parameters
l_0, l_1, l_2, l_3 = 1, 3, 7, 1     # dimensions of the layers
a, b = 0, 1                         # the domain will be [a,b]
xi = 1                              # constant target function


# define loss function
def loss(N, x):
    y = N(x)
    return (y-xi).square().mean()


# function setting the starting model parameters as normally distributed
def init_weights(n):
    if isinstance(n, torch.nn.Linear):
        torch.nn.init.normal_(n.weight, 0, 1/np.sqrt(l_1))
        torch.nn.init.normal_(n.bias, 0, 1/np.sqrt(l_1))


# set up NN model
N_double = torch.nn.Sequential(
    torch.nn.Linear(l_0, l_1),
    torch.nn.ReLU(),
    torch.nn.Linear(l_1, l_2),
    torch.nn.ReLU(),
    torch.nn.Linear(l_2, l_3)
).to(dev)

# set up normally distributed starting model parameters
torch.manual_seed(0)
N_double.apply(init_weights)
N_single = copy.deepcopy(N_double)
N_single.type(torch.float32)

# define observation points
obs = np.empty(shape=129)
index = 0
for i in range(0, 23, 2):
    for k in range(2**i, min(2**(i+2), steps)):
        if k % max(1, 2**(i-2)) == False:
            obs[index] = k
            index = index + 1

# set up table
losses_table = np.empty((len(obs), 3), dtype=object)

# approximation in double precision
# optimizer function
optimizer_double = torch.optim.SGD(N_double.parameters(), lr=gamma)

# train the network
for step in range(steps+1):
    # generate uniformly distributed samples from [0,1]
    x = sample_data_double[step*batch_size:(step+1)*batch_size]
    # set the gradients back to zero
    optimizer_double.zero_grad()
    # compute the loss
    L = loss(N_double, x)
    # compute the gradients of L w.r.t. the model parameters
    L.backward()
    # update the model parameters
    optimizer_double.step()
    # save results in the table
    if step in obs:
        losses_table[np.where(obs == step)[0], [0, 2]] = step, L.data

# approximation in single precision
# optimizer function
optimizer_single = torch.optim.SGD(N_single.parameters(), lr=gamma)

# train the network
for step in range(steps+1):
    # generate uniformly distributed samples from [0,1]
    x = sample_data_single[step*batch_size:(step+1)*batch_size]
    # set the gradients back to zero
    optimizer_single.zero_grad()
    # compute the loss
    L = loss(N_single, x)
    # compute the gradients of L w.r.t. the model parameters
    L.backward()
    # update the model parameters
    optimizer_single.step()
    # save results in the table
    if step in obs:
        losses_table[np.where(obs == step)[0], [0, 1]] = step, L.data

# save table
np.savetxt('mean_square_error_example2.csv', losses_table, delimiter=',',
           fmt='%s, %.64f, %.64f')

# plot mean square error
plt.loglog(losses_table[:, 0], losses_table[:, 2], "r", label='Double precision')
plt.loglog(losses_table[:, 0], losses_table[:, 1], "b", label='Single precision')
plt.ylim([0.00000000000000000001, 10])
plt.xlim([1, 12000000])
plt.xlabel('Number of SGD steps')
plt.ylabel('Mean square error (expected risk)')
plt.legend()
plt.savefig('mean_square_error_example2_loglog_graph', dpi=1200)
\end{lstlisting}
    \captionof{listing}{\label{lst:sgd2}
    {\sc Python} code used to create Table~\ref{table:ANN_num_sim2} 
        and \cref{fig:ANN_num_sim2} in \cref{subsec:numerical_example_2} 
        (filename: \texttt{sgd2.py}).}
    \bigskip
\end{center}

\begin{center}
\begin{lstlisting}
import numpy as np
import torch
import matplotlib.pyplot as plt

# use GPU if available
dev = torch.device("cuda") if torch.cuda.is_available() else torch.device("cpu")

# set precision to double
torch.set_default_dtype(torch.float64)
torch.manual_seed(0)

# configure the training parameters
steps = 10000000                    # n
batch_size = 100                    # M_n, number of sample points
gamma = 1/2000                      # step size (learning rate)

# set parameters
l_0, l_1, l_2, l_3 = 1, 3, 7, 1     # dimensions of the layers
a, b = 0, 1                         # the domain will be [a,b]
xi = 1                              # constant target function


# define loss function
def loss(N, x):
    y = N(x)
    return (y-xi).square().mean()


# function setting the starting model parameters as normally distributed
def init_weights(n):
    if isinstance(n, torch.nn.Linear):
        torch.nn.init.normal_(n.weight, 0, 1/np.sqrt(l_1))
        torch.nn.init.normal_(n.bias, 0, 1/np.sqrt(l_1))


# set up NN model
N = torch.nn.Sequential(
    torch.nn.Linear(l_0, l_1),
    torch.nn.ReLU(),
    torch.nn.Linear(l_1, l_2),
    torch.nn.ReLU(),
    torch.nn.Linear(l_2, l_3)
).to(dev)

# set the starting model parameters as normally distributed
N.apply(init_weights)

# optimizer function
optimizer = torch.optim.SGD(N.parameters(), lr=gamma)

# define observation points
obs = np.empty(shape=129)
index = 0
for i in range(0, 23, 2):
    for k in range(2**i, min(2**(i+2), steps)):
        if k % max(1, 2**(i-2)) == False:
            obs[index] = k
            index = index + 1

# train the network
loss_table = np.empty((len(obs), 2), dtype=object)
for step in range(steps+1):
    # generate uniformly distributed samples from [0,1]
    x = (torch.rand(batch_size, 1)).to(dev)
    # set the gradients back to zero
    optimizer.zero_grad()
    # compute the loss
    L = loss(N, x)
    # compute the gradients of L w.r.t. the model parameters
    L.backward()
    # update the model parameters
    optimizer.step()
    # save results in the table
    if step in obs:
        loss_table[np.where(obs == step)[0], :] = step, L.data

# plot mean square error
plt.loglog(loss_table[:, 0], loss_table[:, 1], color= "r")
plt.ylim([0.00000000000000000001, 10])
plt.xlim([1, 12000000])
plt.xlabel('Number of SGD steps')
plt.ylabel('Mean square error (expected risk)')
plt.show()
\end{lstlisting}
    \captionof{listing}{\label{lst:sgd2_short}
    Simplified variant of the {\sc Python}script in \cref{lst:sgd2} above 
    (filename: \texttt{sgd2\_short.py}).}
    \bigskip
\end{center}

\subsection*{Acknowledgments}
Benno Kuckuck is gratefully acknowledged for several helpful suggestions. 
This work has been funded by the Deutsche Forschungsgemeinschaft 
(DFG, German Research Foundation) 
under Germany's Excellence Strategy EXC 2044-390685587, Mathematics M\"{u}nster: 
Dynamics-Geometry-Structure and under the research project HU1889/7-1.

\bibliographystyle{acm}
\bibliography{ref2}

\end{document}